\documentclass{article}

\usepackage{microtype}
\usepackage{graphicx}
\usepackage{subcaption}
\usepackage{booktabs} 

\usepackage{hyperref}

\usepackage[preprint]{icml2026}

\usepackage{amsmath}
\usepackage{amssymb}
\usepackage{mathtools}
\usepackage{amsthm}
\usepackage{comment}
\usepackage{enumitem}
\usepackage{bm}
\usepackage{tikz}
\usetikzlibrary{arrows.meta, automata, positioning, quotes}
\usepackage{cancel}
\usepackage{multirow}
\usepackage{tabularx}
\usepackage{amssymb}
\allowdisplaybreaks
\makeatletter
\newtheorem{repeatthm@}{Proposition}
\newenvironment{repeatthm}[1]{%
    \def\therepeatthm@{\ref{#1}}
    \repeatthm@
}
{\endrepeatthm@}
\makeatother
\usepackage[capitalize,noabbrev]{cleveref}
\theoremstyle{plain}
\newtheorem{theorem}{Theorem}[section]
\newtheorem{prop}[theorem]{Proposition}
\newtheorem{lemma}[theorem]{Lemma}
\newtheorem{corollary}[theorem]{Corollary}
\theoremstyle{definition}

\theoremstyle{remark}
\newtheorem{remark}[theorem]{Remark}

\usepackage[textsize=tiny]{todonotes}

\icmltitlerunning{Inverse problems with diffusion models: MAP estimation via mode-seeking loss}

\def\rvf{{\mathbf{f}}}

\def\rvt{{\mathbf{t}}}

\def\rvw{{\mathbf{w}}}
\def\rvx{{\mathbf{x}}}
\def\rvy{{\mathbf{y}}}
\def\rvz{{\mathbf{z}}}

\DeclareMathOperator*{\argmax}{arg\,max}

\newcommand{\KL}{D_{\mathrm{KL}}}
\newcommand{\Cov}{\mathrm{Cov}}

\begin{document}

\twocolumn[
  \icmltitle{Inverse problems with diffusion models: MAP estimation via mode-seeking loss}



  \icmlsetsymbol{equal}{*}

  \begin{icmlauthorlist}
    \icmlauthor{Sai Bharath Chandra Gutha}{rpl}
    \icmlauthor{Ricardo Vinuesa}{umich}
    \icmlauthor{Hossein Azizpour}{rpl,scilife}
  \end{icmlauthorlist}


  \icmlaffiliation{rpl}{KTH Royal Institute of Technology, Sweden}
  \icmlaffiliation{umich}{University of Michigan, USA}
  \icmlaffiliation{scilife}{Science for Life Laboratory, Sweden}

  \icmlcorrespondingauthor{Sai Bharath Chandra Gutha}{sbcgutha@kth.se}

  \icmlkeywords{Machine Learning, ICML}

  \vskip 0.3in
]



\printAffiliationsAndNotice{}  

\begin{abstract}
A pre-trained unconditional diffusion model, combined with posterior sampling or maximum a posteriori (MAP) estimation techniques, can solve arbitrary inverse problems without task-specific training or fine-tuning. However, existing posterior sampling and MAP estimation methods often rely on modeling approximations and can also be computationally demanding. In this work, we propose a new MAP estimation strategy for solving inverse problems with a pre-trained unconditional diffusion model. Specifically, we introduce the variational mode-seeking loss (VML) and show that its minimization at each reverse diffusion step guides the generated sample towards the MAP estimate (modes in practice). VML arises from a novel perspective of minimizing the Kullback-Leibler (KL) divergence between the diffusion posterior $p(\rvx_0|\rvx_t)$ and the measurement posterior $p(\rvx_0|\rvy)$, where $\rvy$ denotes the measurement. Importantly, for linear inverse problems, VML can be analytically derived without any modeling approximations. Based on further theoretical insights, we propose VML-MAP, an empirically effective algorithm for solving inverse problems via VML minimization, and validate its efficacy in both performance and computational time through extensive experiments on diverse image-restoration tasks across multiple datasets.
\end{abstract}

\section{Introduction}\label{sec:introduction}

Solving an inverse problem essentially involves estimation of the original data sample $\rvx$ based on a given partially degraded measurement $\rvy$. Formally, Equation~(\ref{eqn:invprob}) relates the degraded measurement with the original data sample, where $\mathcal{A}$ denotes the known degradation operator and $\eta$ is a random variable denoting measurement noise, which is typically assumed to be Gaussian distributed with known standard deviation $\sigma_\rvy$, i.e., $\mathbf{\eta} \sim \mathcal{N}(0,\sigma^2_\rvy \mathbf{I})$. This implies that $p(\rvy|\rvx) = \mathcal{N}(\mathcal{A}(\rvx),\sigma^2_{\rvy}\mathbf{I})$. For linear inverse problems, $\mathcal{A}$ is linear and can be denoted with a matrix instead, i.e., $\mathcal{A}(\rvx)=\mathrm{H}\rvx$, where we use the matrix $\mathrm{H}$ to denote a linear degradation operator throughout the paper. 
\begin{equation}\label{eqn:invprob}
\rvy = \mathcal{A}(\rvx) + \mathbf{\eta}    
\end{equation}
Inverse problems are typically ill-posed, where many plausible data samples $\rvx$ could correspond to a given degraded measurement $\rvy$, rendering a probabilistic approach essential. In a Bayesian framework, solving an inverse problem amounts to finding a sample of the posterior distribution $p(\rvx|\rvy)$. 
As diffusion models~\cite{ddpm,songsde,ldm} gain prominence in generative modeling, leveraging a pre-trained unconditional diffusion model (which samples $\rvx \sim p(\rvx)$) to solve arbitrary inverse problems in a zero-shot (plug-and-play) fashion is becoming increasingly attractive~\cite{repaint,ddrm,dps,pgdm,diffpir,rout2023solving,resample,reddiff,dcps,mapga,daps,vmps,zilberstein2025repulsive}. 

In a diffusion process, noise is added progressively to the training data samples $\rvx_0 \sim p_{data}(\rvx_0)$ to convert these into pure noise gradually over a time horizon $t \in \left(0, T \right]$. This process can be reversed stochastically with a Stochastic Differential Equation (SDE) or deterministically with the Probability Flow Ordinary Differential Equation (PF ODE), both of which require the score function of the marginal distribution at time $t$, i.e., $\nabla_{\rvx_t}\log p(\rvx_t)$, for all $t \in \left( 0,T \right]$~\cite{songsde}. The true score functions are typically intractable, so a diffusion model $s_{\theta}(\rvx_t,t)$ is trained via score-matching loss~\cite{dsm,ssm} to approximate these. The trained diffusion model can then generate samples $\rvx_0 \sim p_{data}(\rvx_0)$, starting from random noise.

\textbf{Related works.} As mentioned earlier, solving an inverse problem amounts to finding a sample of the posterior. Approaches for solving inverse problems with a pre-trained unconditional diffusion model can be roughly categorized into ({\romannumeral 1}) Posterior sampling and ({\romannumeral 2}) MAP estimation methods. Posterior sampling methods, by design, aim to draw samples proportional to the posterior probability. To achieve that, it suffices to replace the unconditional score function $\nabla_{\rvx_t} \log p(\rvx_t)$ with the conditional score $\nabla_{\rvx_t} \log p(\rvx_t|\rvy)$ in the reverse diffusion process. Since $\nabla_{\rvx_t} \log p(\rvx_t|\rvy) = \nabla_{\rvx_t} \log p(\rvx_t) + \nabla_{\rvx_t} \log p(\rvy|\rvx_t)$, the former term can be replaced with the unconditional diffusion model $s_{\theta}(\rvx_t,t)$ but the latter term remains intractable due to the intractability of $p(\rvy|\rvx_t)$~\cite{dps}. Works like DPS~\cite{dps}, $\Pi$GDM~\cite{pgdm}, TMPD~\cite{boys}, ~\cite{peng} use Gaussian approximations, while some others~\cite{trippe2022diffusion,cardoso2023monte,achituve2025inverse} also rely on sequential Monte-Carlo sample approximations. Some posterior sampling methods, such as DAPS~\cite{daps}, do not require estimating this conditional score, but still involve other modeling approximations. However, note that the main objective of all posterior sampling methods remains the same, which is to draw samples proportional to the posterior probability.

MAP estimation approaches aim to find the MAP estimate, i.e., the sample with the highest posterior probability, and fundamentally differ in theory and method compared to posterior sampling schemes. As the MAP estimate is a point estimate, MAP solvers, unlike posterior sampling methods, do not try to estimate the conditional score but rather find different approaches to optimize the MAP objective. Earlier works on MAP estimation, like DiffPIR~\cite{diffpir}, adapt variable splitting optimization algorithms such as HQS or ADMM (which relaxes the original MAP objective into several proximal subproblems) to the diffusion framework. Later, more efficient works, such as DMPlug~\cite{dmplug} and MAP-GA~\cite{mapga}, considered the direct mapping from noise to the data given by the PF ODE or the consistency model~\cite{consistencymodels} and solved for the optimal noise using gradient-based optimizers on the reparameterized MAP objective. These works use heuristic approximations and also require a consistency model, which, however, can be empirically approximated by multiple diffusion denoiser steps in practice, but results in expensive gradient computations.

\textbf{Motivation.} It is more accurate to evaluate posterior sampling schemes based on their capacity to cover the true posterior proportionally. But this has not been the focus in the literature and understandably so, due to the limited availability of ground-truth posterior samples. Many image restoration inverse tasks often require generating a plausible image consistent with the measurement, which also fuels this evaluation criterion towards plausibility of estimates rather than covering the posterior proportionally. Such a requirement is inconsistent with the principle of posterior sampling and aligns more closely with MAP estimation, as the MAP estimate is the most plausible (probable) posterior sample. This remains the main motivation for MAP estimation methods in general, including ours, as we propose a novel MAP estimation strategy for solving inverse problems with a pre-trained unconditional diffusion model.

\textbf{Novelty.} We introduce the variational mode-seeking loss (VML) and show that its minimization at each diffusion time $t$ through the reverse diffusion process will steer the generated sample towards the MAP estimate (modes in practice) as $t \to 0$. Unlike existing MAP solvers, VML can be derived in a closed form for linear inverse problems without any modeling approximations. Under mild assumptions, we further show that VML converges to the negative log posterior, i.e., the MAP objective as $t \to 0$. Not only does our approach differ from posterior sampling schemes (since our theoretical objective is different), but it also significantly differs from existing MAP estimation methods. To the best of our knowledge, VML is the first approximation-free modeling of MAP estimation for the linear operator setting with formal analysis, in the context of solving inverse problems leveraging only a pre-trained unconditional diffusion model.

{\textbf{Contributions.}} Our main contributions in this work are summarized below. 
\begin{itemize}[left=1em, label=\textbullet]
    \item We introduce the variational mode-seeking loss (VML), which, when minimized at each reverse diffusion time step $t$, steers the generated sample towards the MAP estimate (modes in practice) as $t \to 0$. 
    \item For linear inverse problems, we derive a closed-form expression for VML without any modeling approximations and provide a formal theoretical analysis of its convergence under mild assumptions. 
    \item Based on further theoretical insights, we propose an empirically effective algorithm (VML-MAP) for solving inverse problems, and also a preconditioner for ill-conditioned linear degradation operators.
    \item We demonstrate VML-MAP's effectiveness over other approaches through extensive experiments on diverse image-restoration tasks across multiple datasets.
\end{itemize}

\begin{figure*}[t!]
    \begin{center}
    \begin{tikzpicture}
          $\draw[{Straight Barb[length=2mm, width=2.5mm, color=darkgray]}-,line width=1pt, color=darkgray] (0,0)--(4.9,0) node[right] () {\textit{\small{\color{darkgray}{\textbf{reverse diffusion}}}}}; \draw[line width=1pt, color=darkgray] (7.3,0)--(12.4,0);$
    \end{tikzpicture}\hfill%
    \\[1pt]
    \begin{subfigure}{.148\textwidth}
        \centering
        \includegraphics[width=\textwidth]{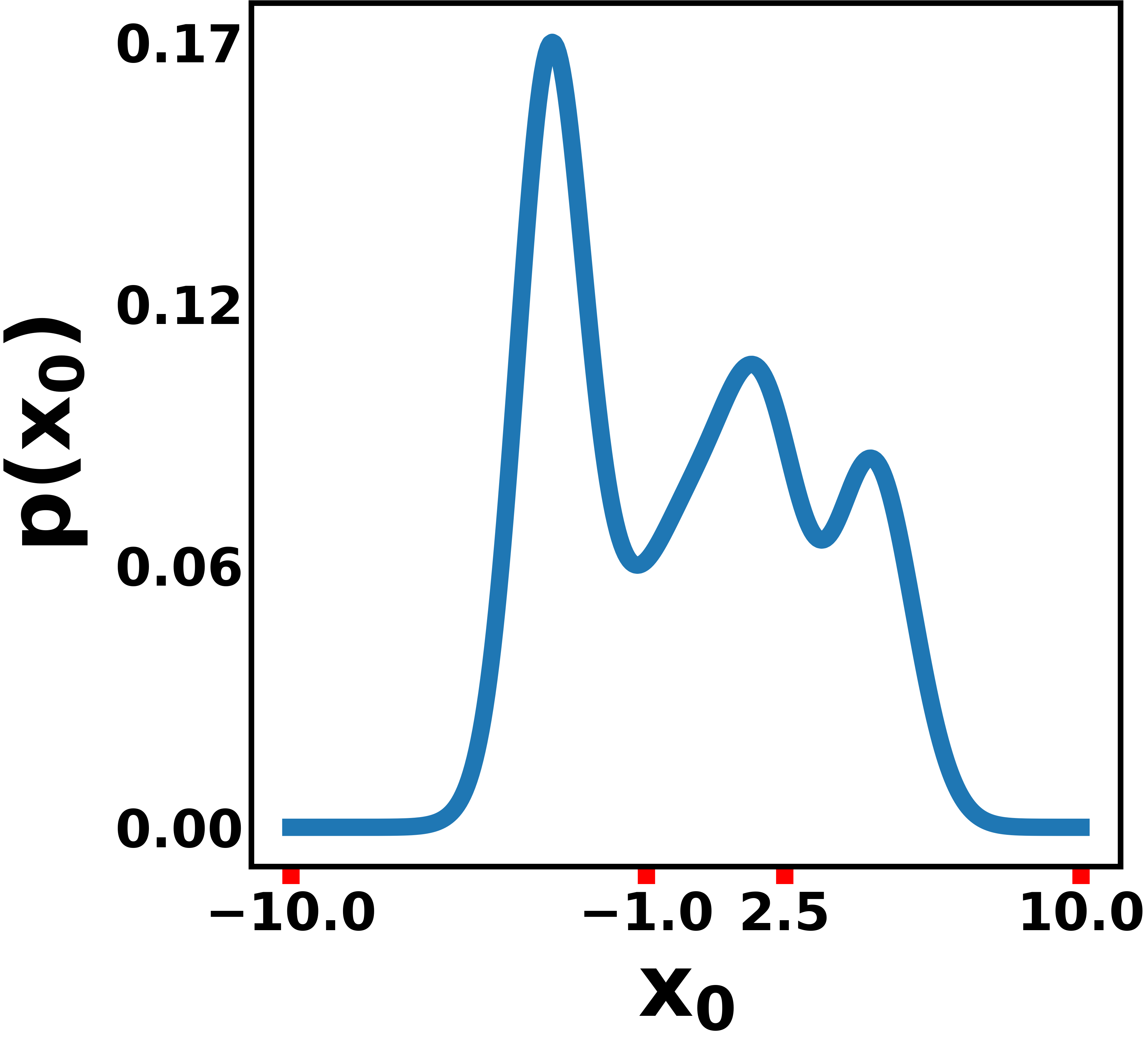}
    \end{subfigure}\hfill%
    \begin{subfigure}{.148\textwidth}
        \centering
        \includegraphics[width=\textwidth]{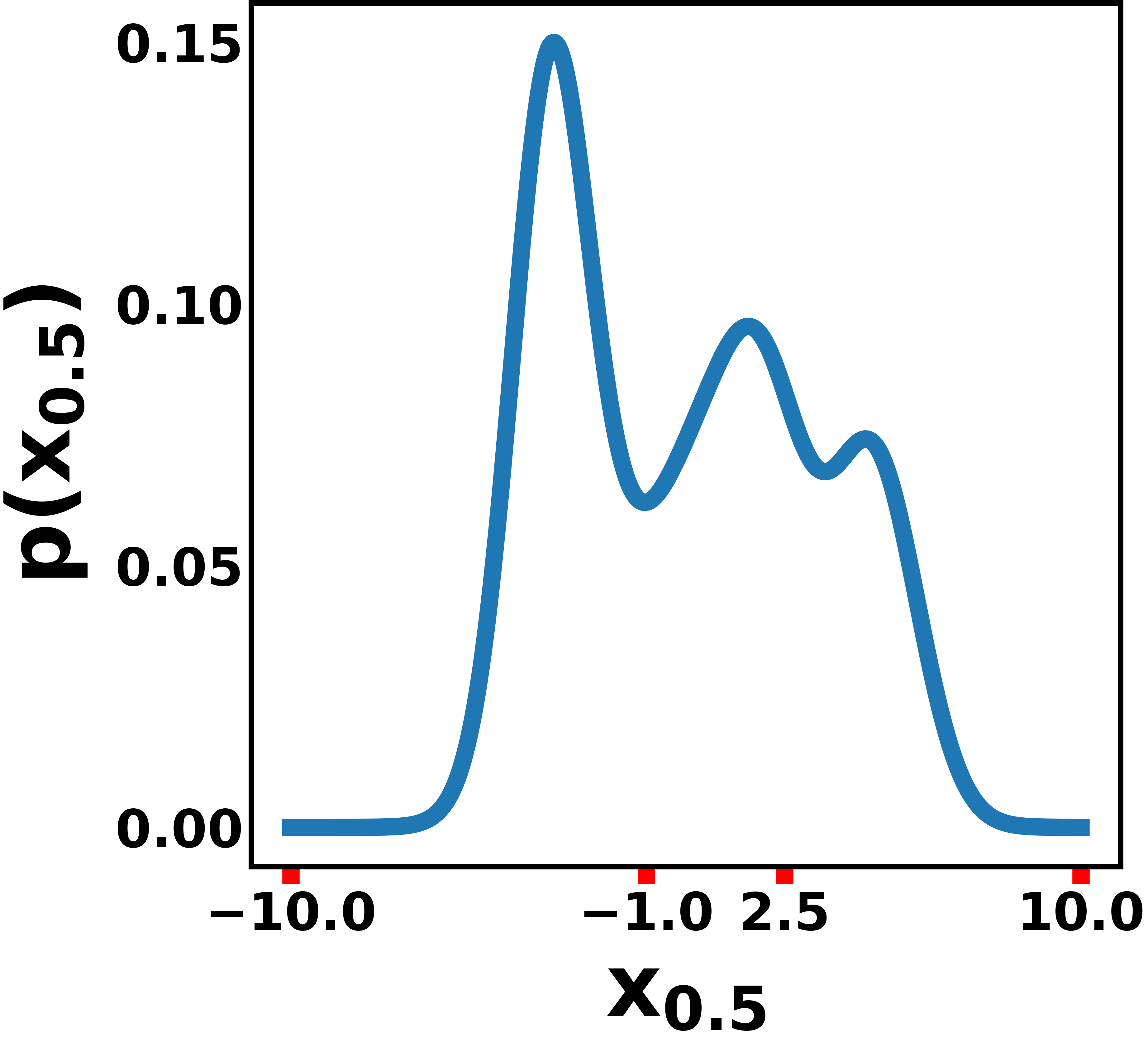}
    \end{subfigure}\hfill%
    \begin{subfigure}{.148\textwidth}
        \centering
        \includegraphics[width=\textwidth]{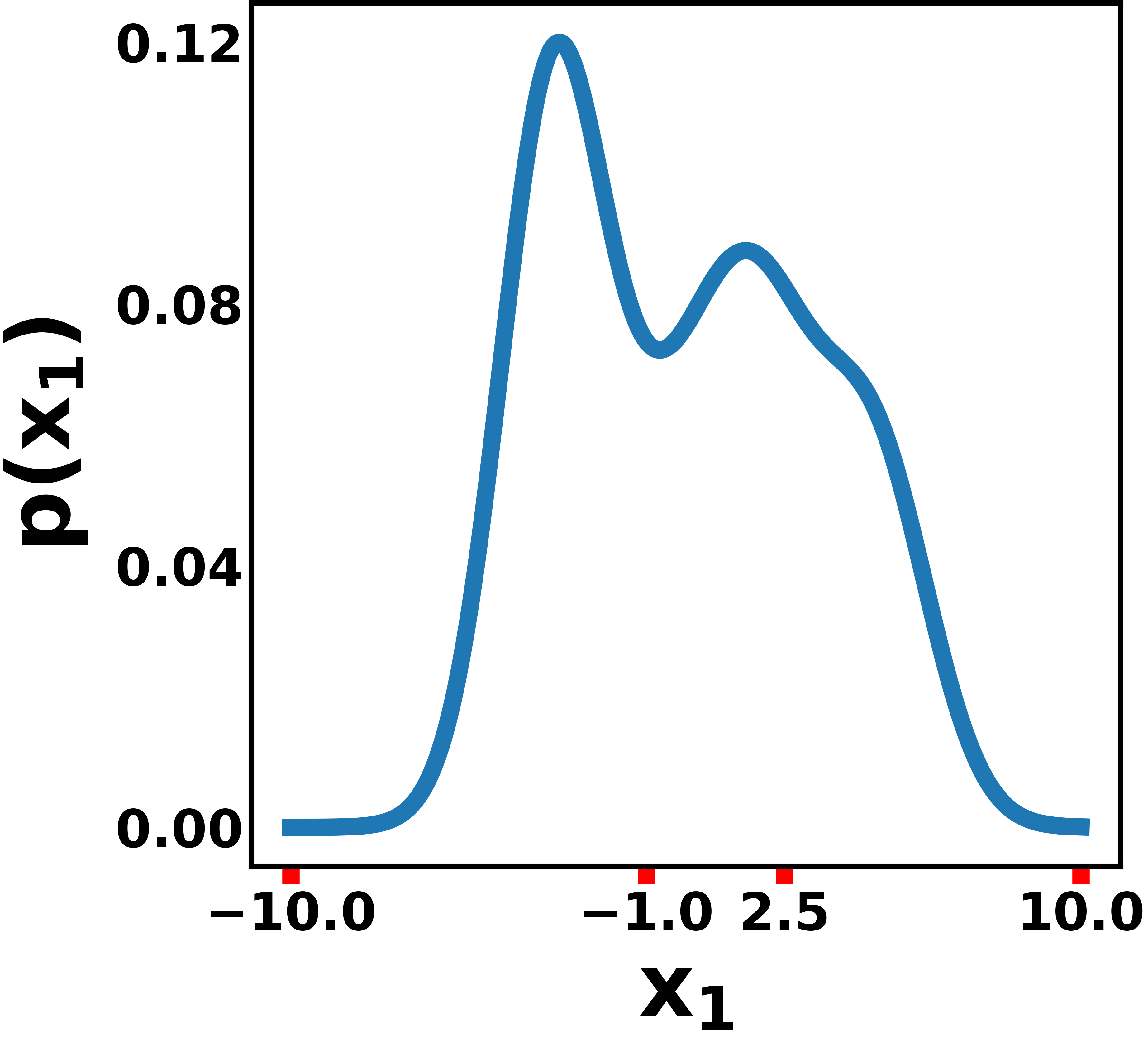}
    \end{subfigure}\hfill%
    \begin{subfigure}{.148\textwidth}
        \centering
        \includegraphics[width=\textwidth]{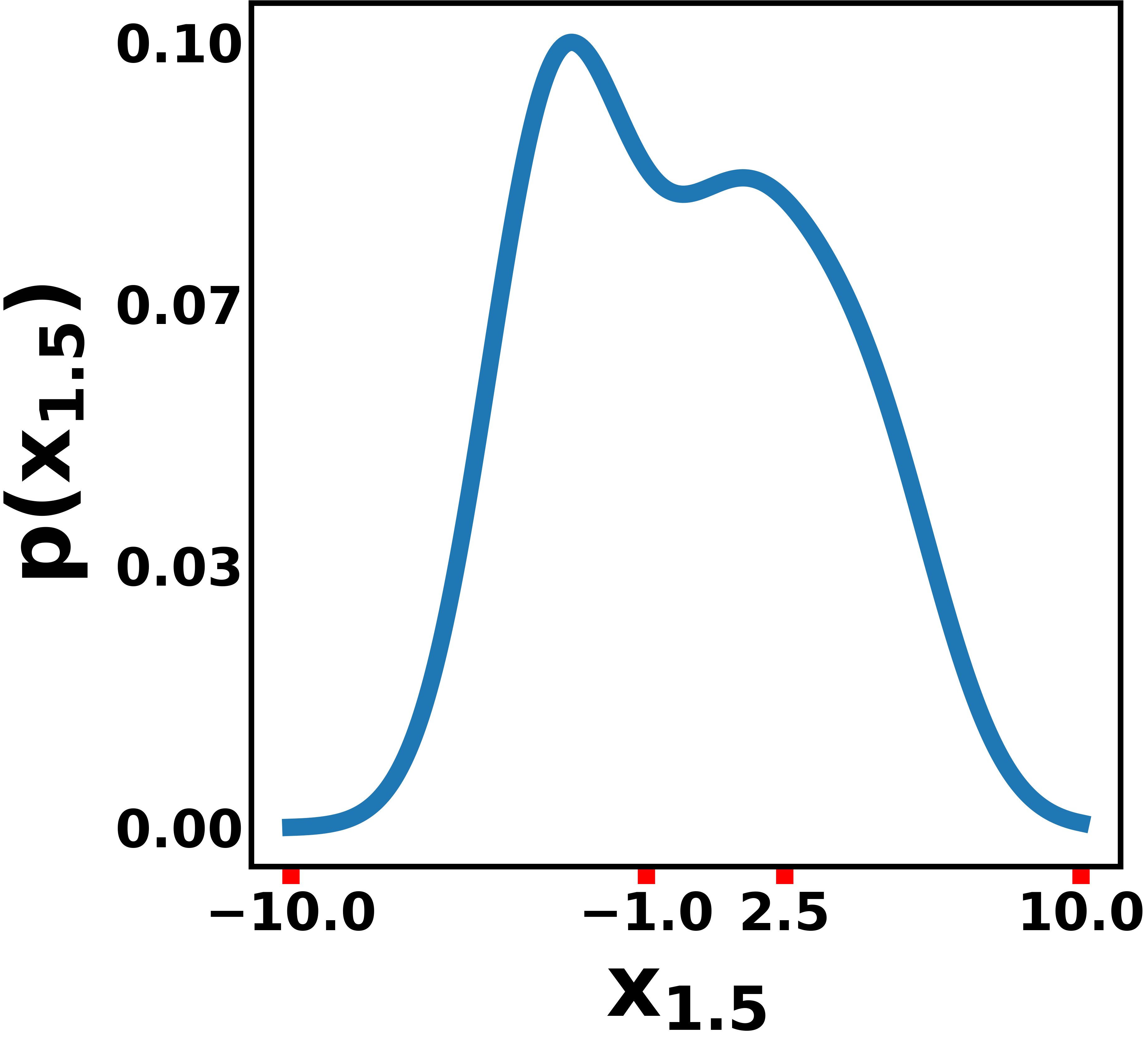}
    \end{subfigure}\hfill%
    \begin{subfigure}{.148\textwidth}
        \centering
        \includegraphics[width=\textwidth]{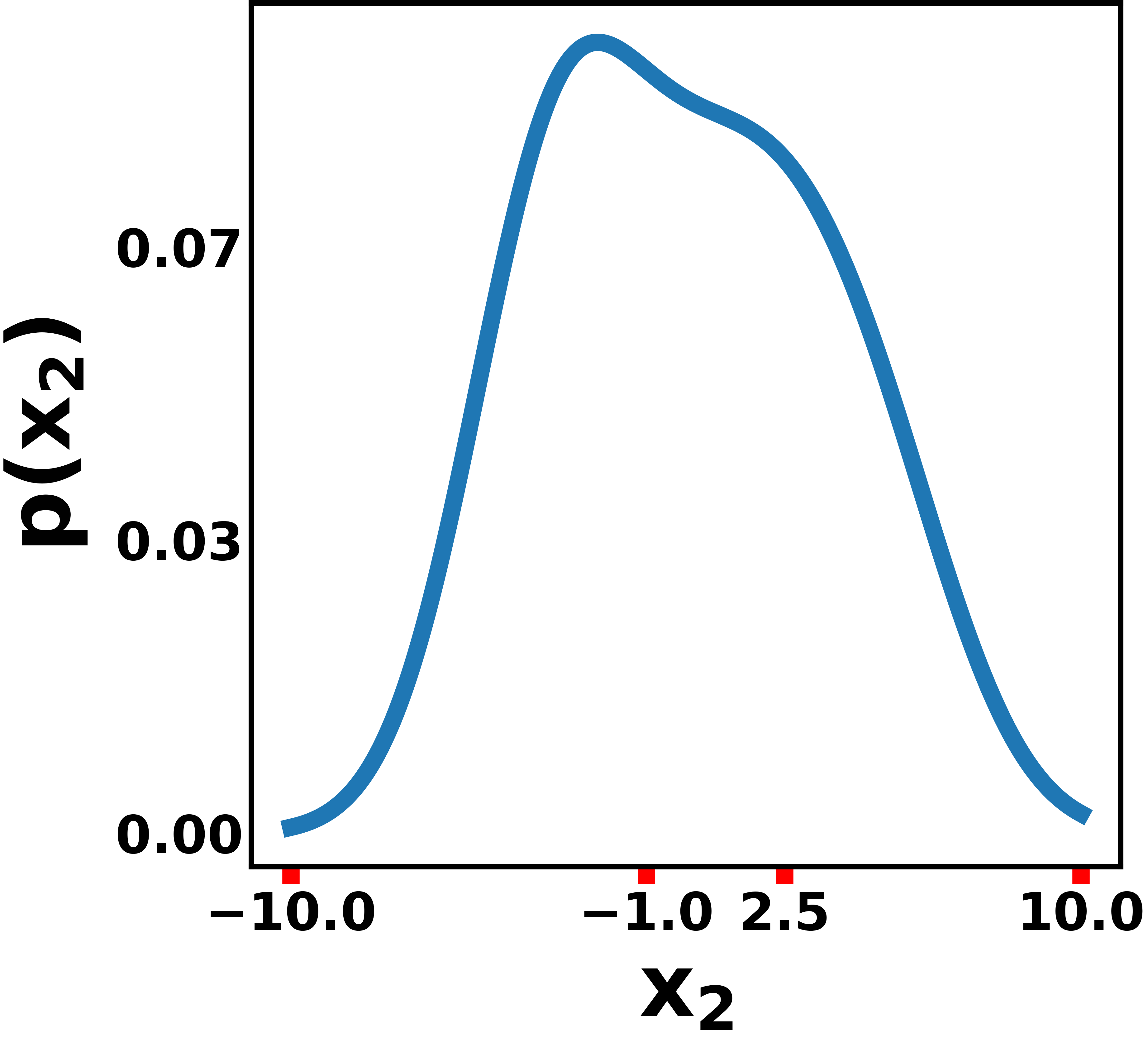}
    \end{subfigure}\hfill%
    \\[1pt]
    \begin{tikzpicture}
    \node[draw, fill=orange!20, text width=.98\textwidth, align=center] at (0,0) {\scriptsize{${p(\rvx_t) \text{ along different time steps of a diffusion process. }p(\rvx_0)\text{ is a Gaussian mixture model, and }\sigma(t)=t.}$}};
    \end{tikzpicture}
    \\[2pt]
    \begin{subfigure}{.148\textwidth}
        \centering
        \includegraphics[width=\textwidth]{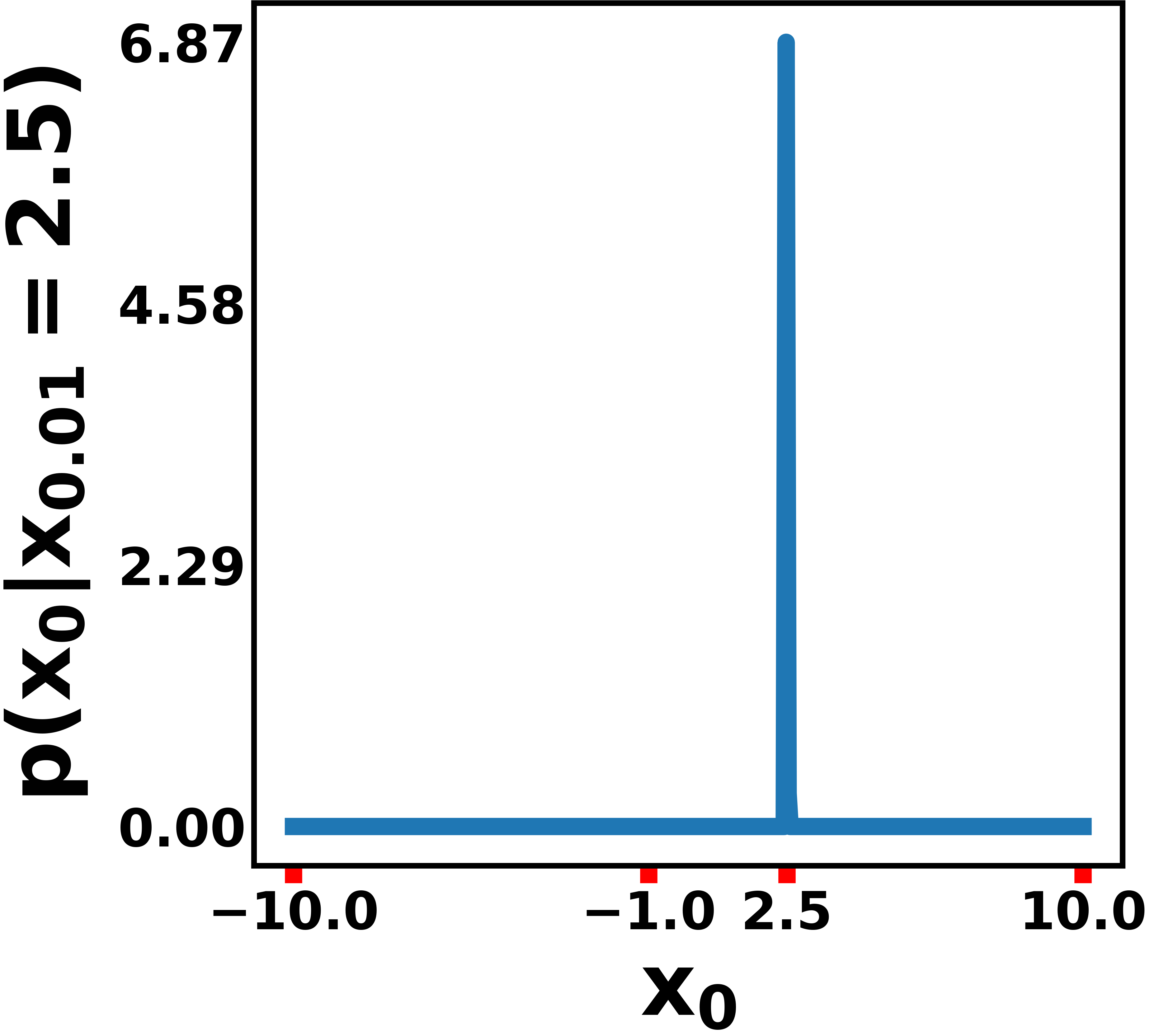}
    \end{subfigure}\hfill%
    \begin{subfigure}{.148\textwidth}
        \centering
        \includegraphics[width=\textwidth]{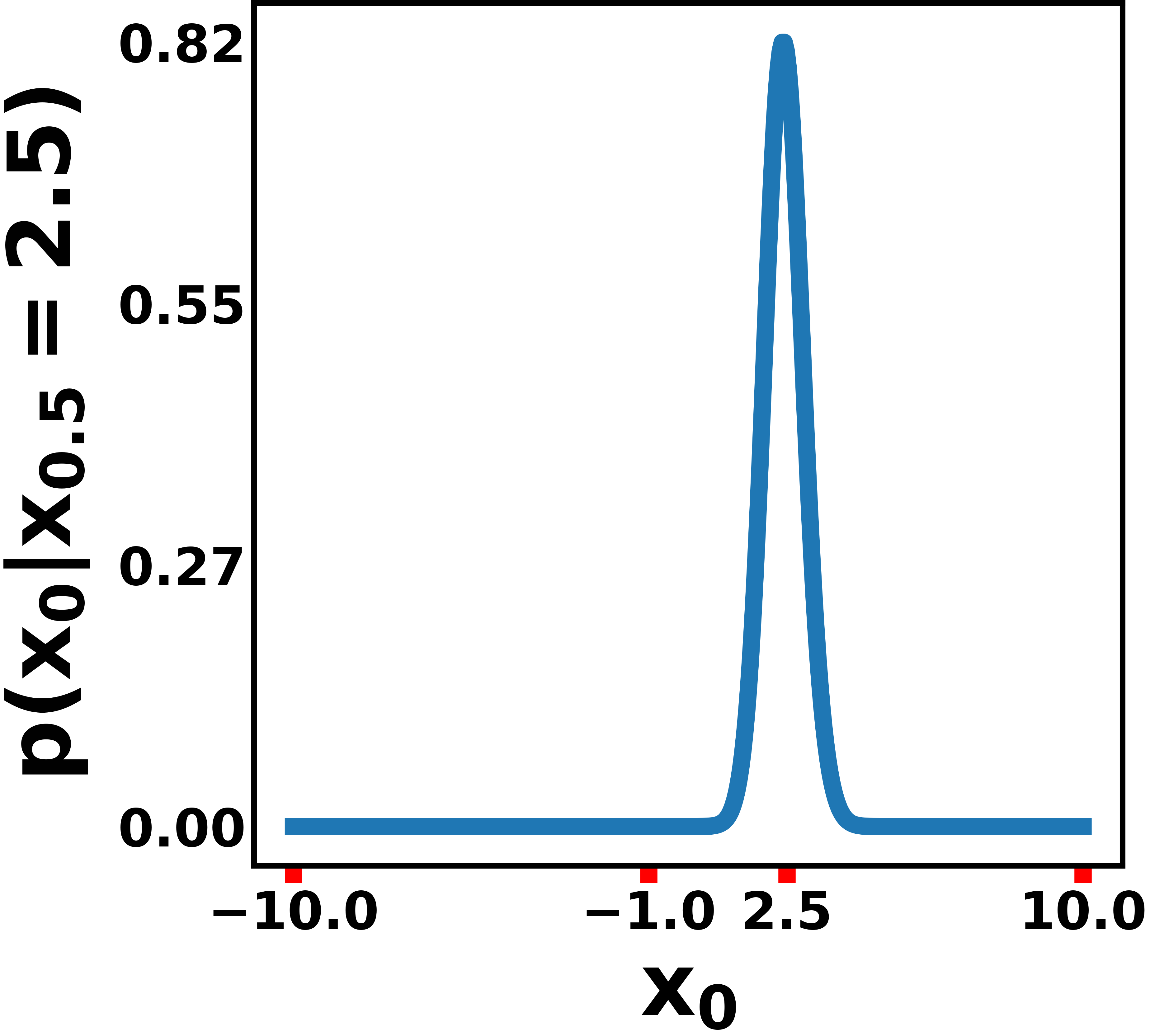}
    \end{subfigure}\hfill%
    \begin{subfigure}{.148\textwidth}
        \centering
        \includegraphics[width=\textwidth]{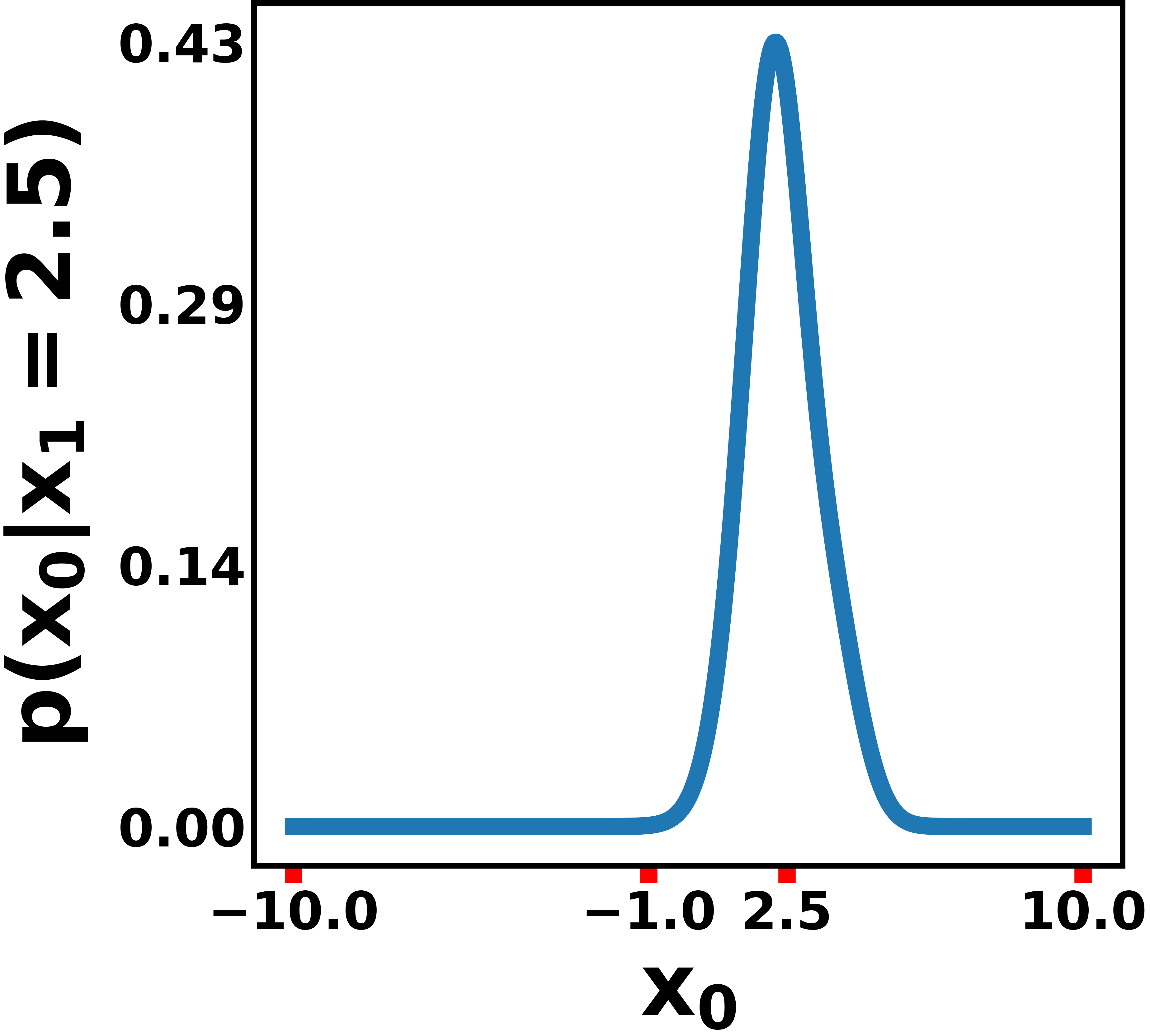}
    \end{subfigure}\hfill%
    \begin{subfigure}{.148\textwidth}
        \centering
        \includegraphics[width=\textwidth]{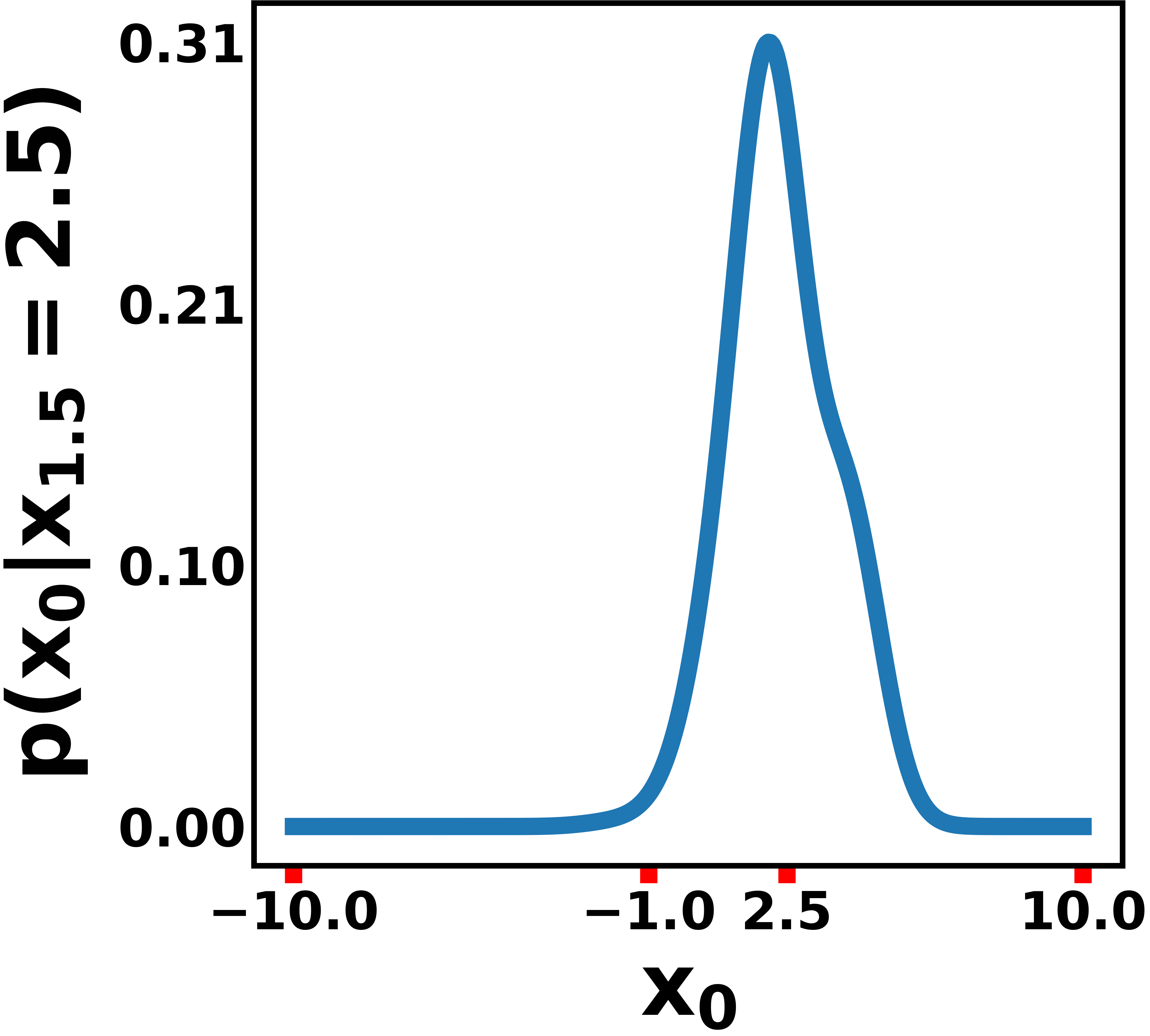}
    \end{subfigure}\hfill%
    \begin{subfigure}{.148\textwidth}
        \centering
        \includegraphics[width=\textwidth]{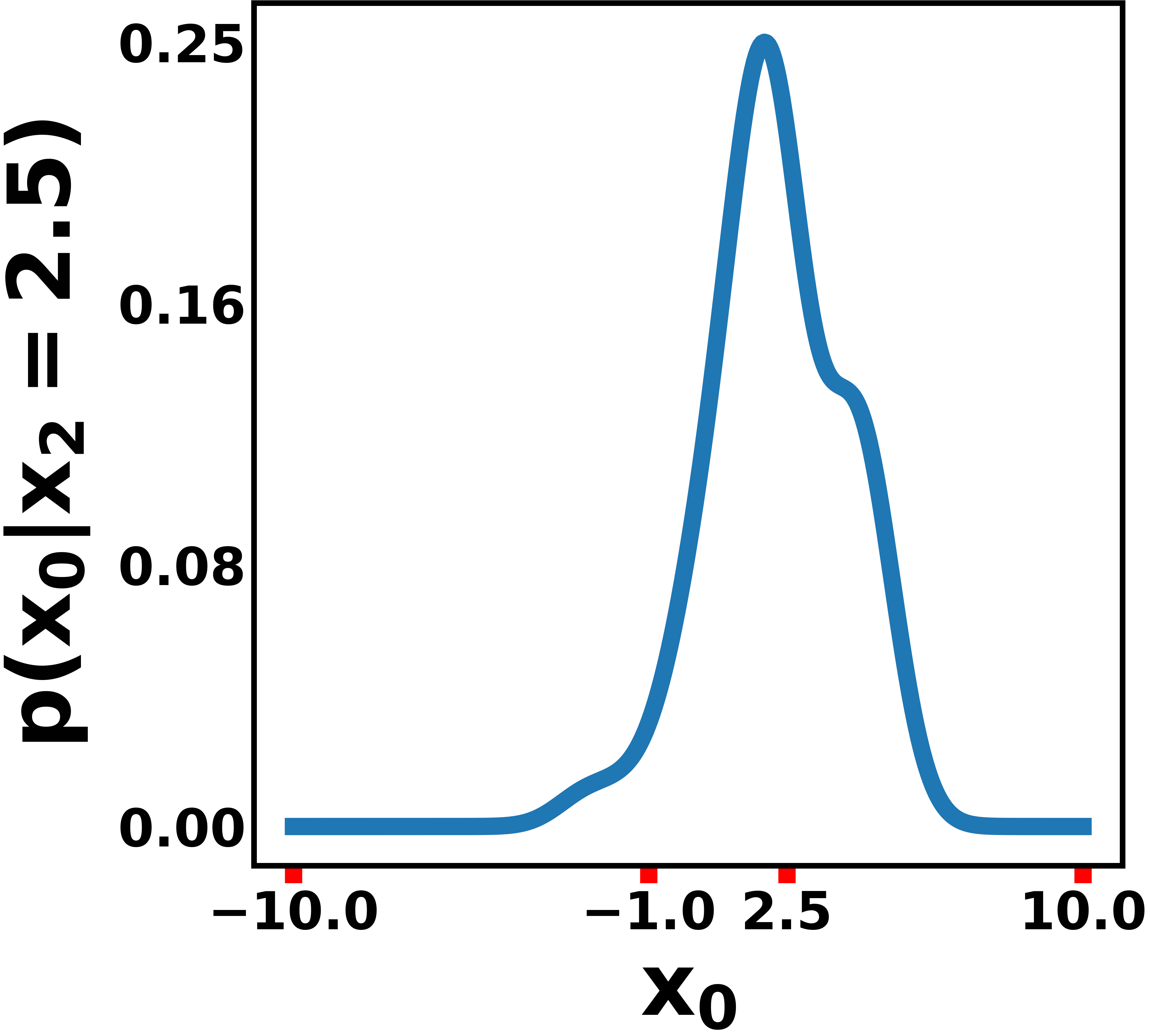}
    \end{subfigure}\hfill%
    \\[1pt]
    \begin{tikzpicture}
    \node[draw, fill=orange!20, text width=.98\textwidth, align=center] at (0,0) {\scriptsize{$p(\rvx_0|\rvx_t=2.5)\text{ along different time steps of a diffusion process. It converges to }\delta(\rvx_0 - 2.5)\text{ as } t \to 0 $}};
    \end{tikzpicture}
    \\[2pt]
    \begin{subfigure}{.148\textwidth}
        \centering
        \includegraphics[width=\textwidth]{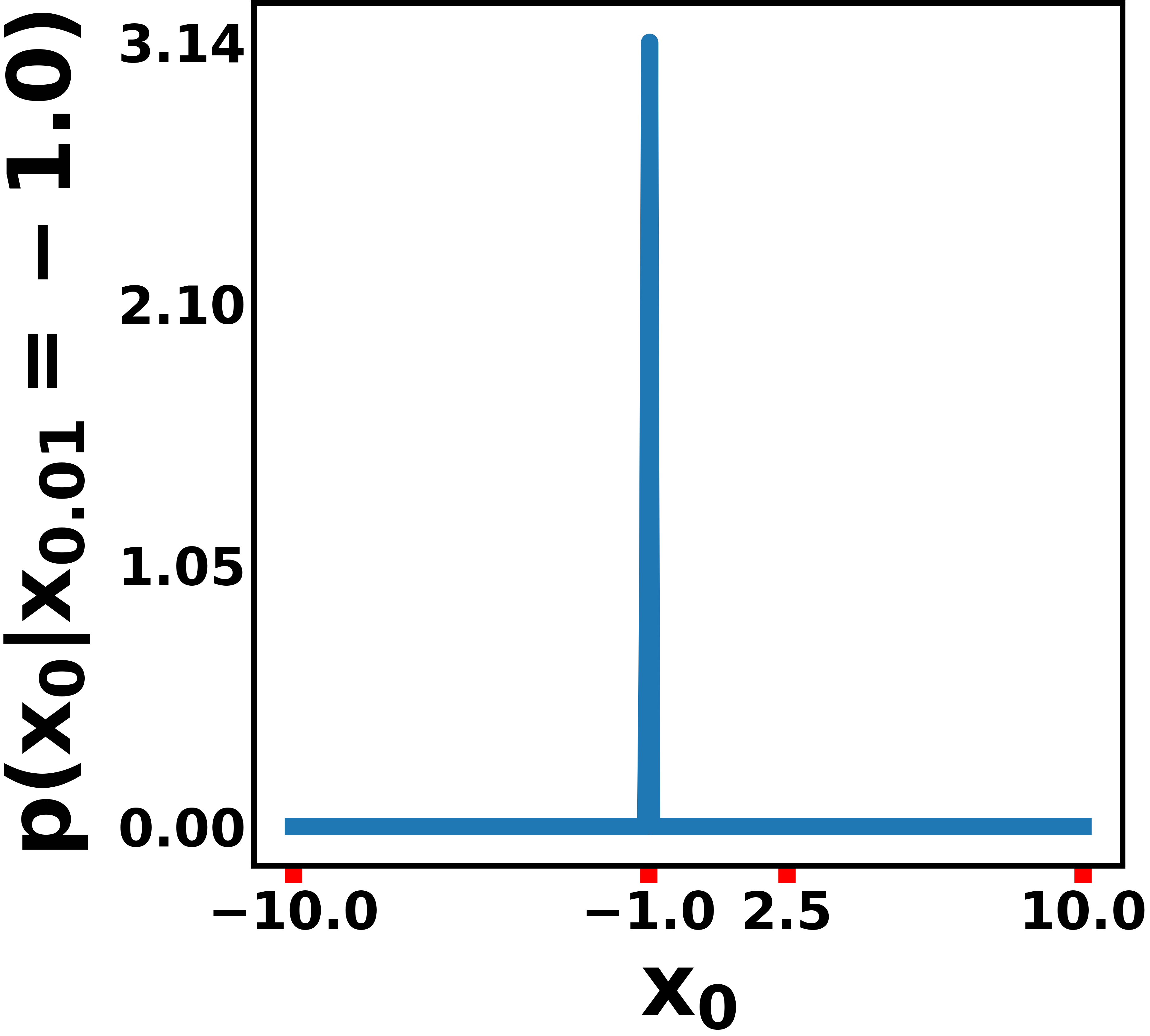}
    \end{subfigure}\hfill%
    \begin{subfigure}{.148\textwidth}
        \centering
        \includegraphics[width=\textwidth]{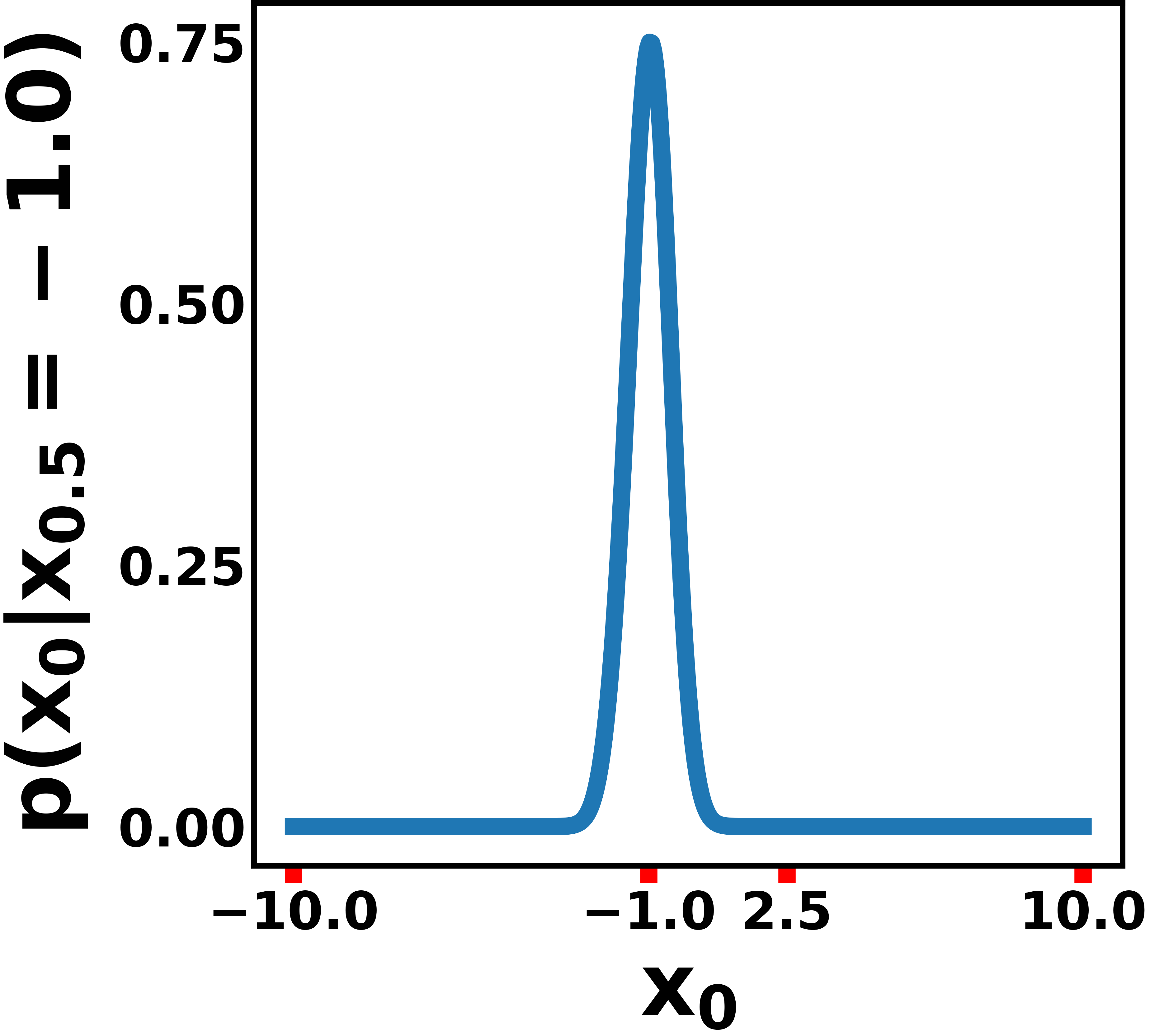}
    \end{subfigure}\hfill%
    \begin{subfigure}{.148\textwidth}
        \centering
        \includegraphics[width=\textwidth]{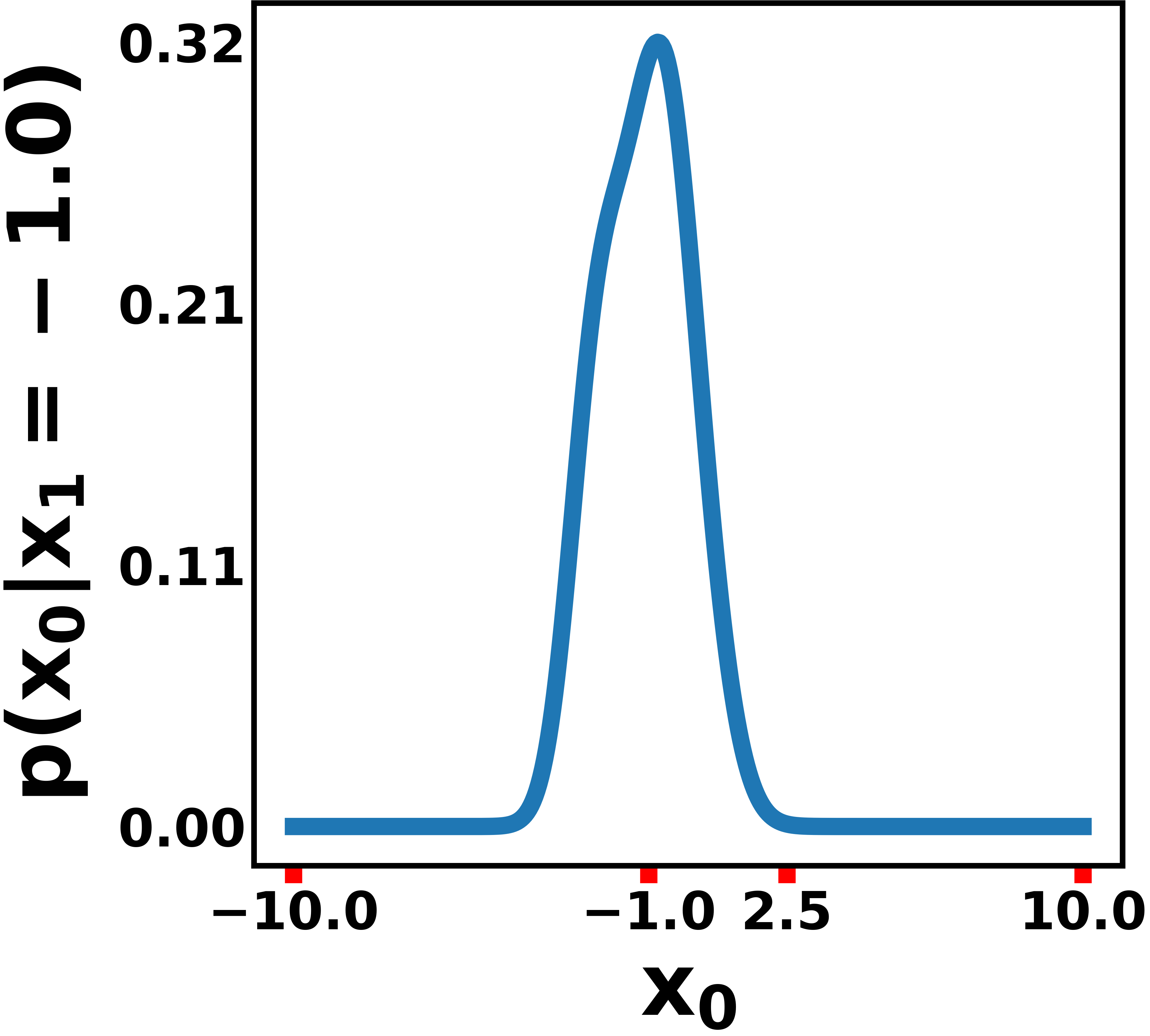}
    \end{subfigure}\hfill%
    \begin{subfigure}{.148\textwidth}
        \centering
        \includegraphics[width=\textwidth]{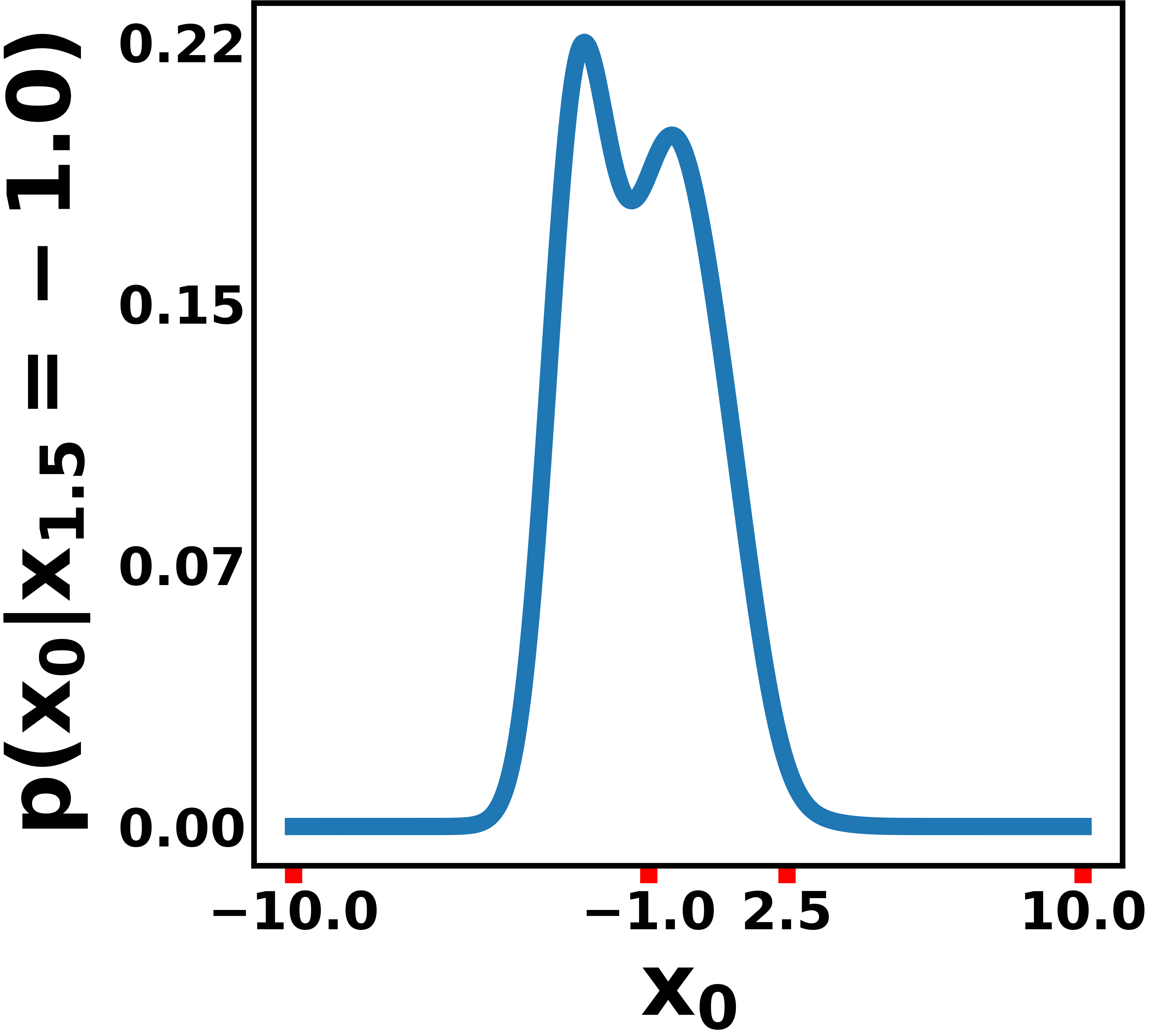}
    \end{subfigure}\hfill%
    \begin{subfigure}{.148\textwidth}
        \centering
        \includegraphics[width=\textwidth]{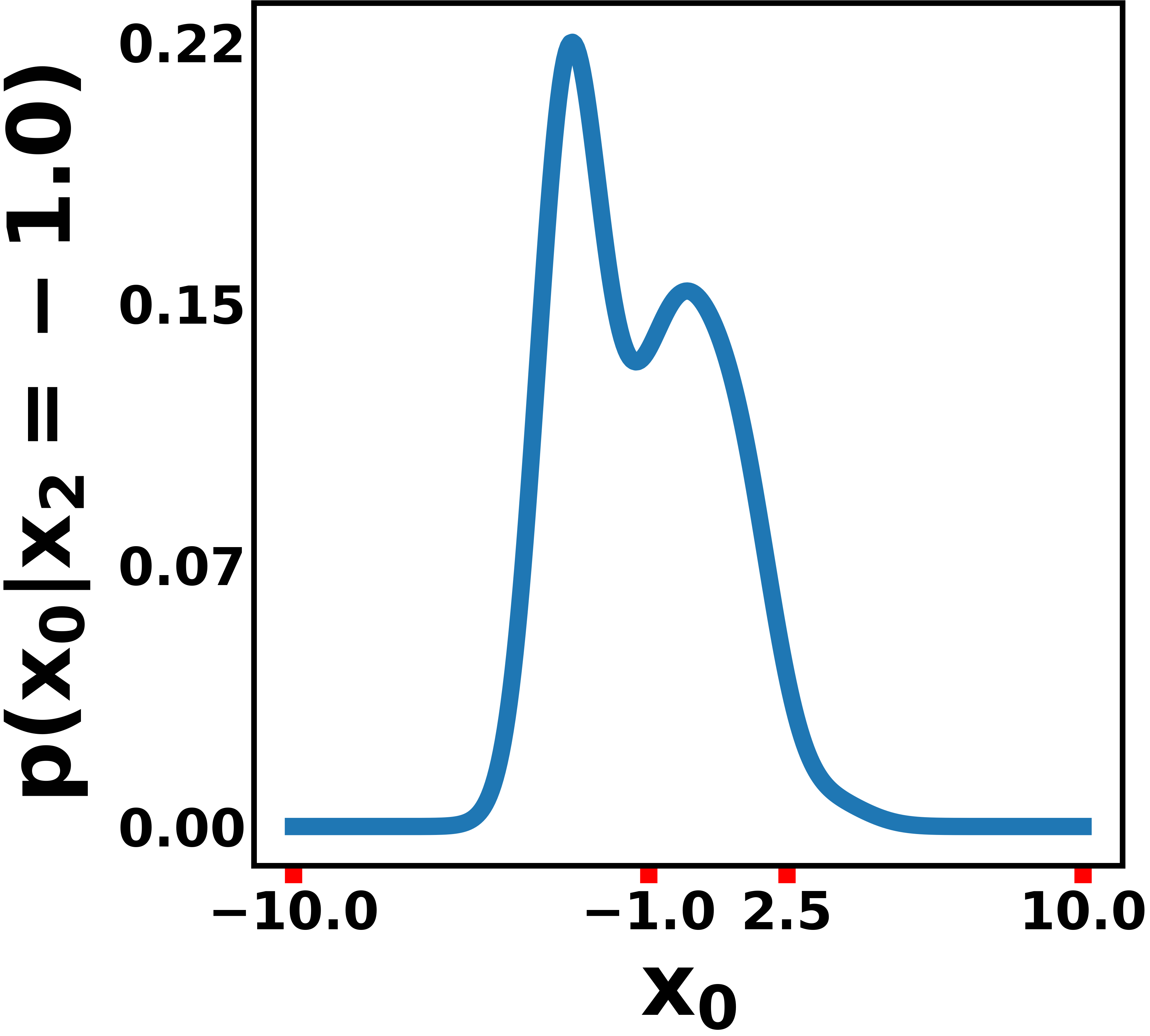}
    \end{subfigure}\hfill%
    \\[1pt]
    \begin{tikzpicture}
    \node[draw, fill=orange!20, text width=.98\textwidth, align=center] at (0,0) {\scriptsize{$p(\rvx_0|\rvx_t=-1.0)\text{ along different time steps of a diffusion process. It converges to }\delta\left(\rvx_0 - (-1.0)\right)\text{ as } t\to 0 $}};
    \end{tikzpicture}
    \end{center}
    \caption{ 
    The figure depicts how the functional form of $p(\rvx_0|\rvx_t)$ gets peaky around $\rvx_t$ as $t \to 0$.}
    \label{fig:1}
\end{figure*}
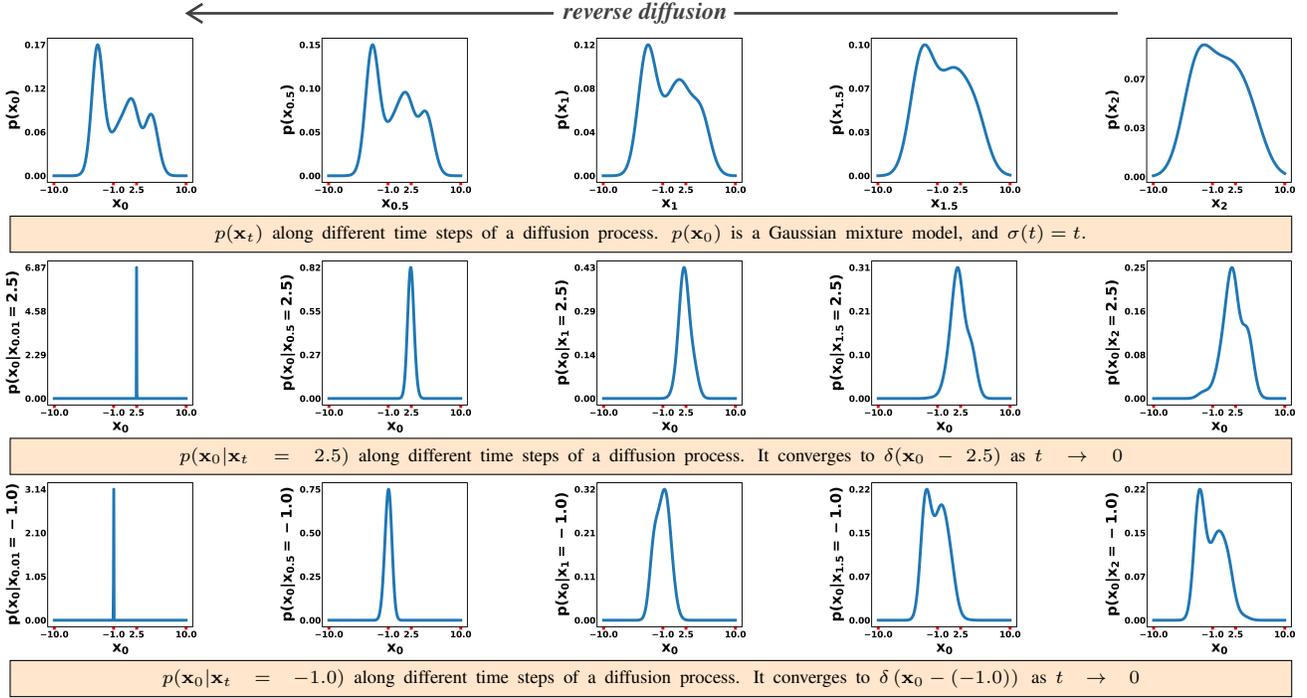

\section{Background on diffusion models}\label{sec:background}
The forward process in a diffusion model corrupts a clean sample of the input data distribution i.e., $\rvx_0 \sim P(\rvx_0)$ into intermediate noisy sample $\rvx_t$, $t \in \left( 0,T \right]$, modeled by the forward SDE given by Equation~(\ref{eqn:forwardsde}), where $\rvf(\cdot,t):\mathbb{R}^n\to \mathbb{R}^n$, and $g:\mathbb{R}\to \mathbb{R}$ are the drift and diffusion coefficients, respectively, and $\rvw_t$ denotes a standard Wiener process~\cite{songsde}. The drift, diffusion coefficients, and $T$ are chosen such that the distribution of $\rvx_T$ is tractable to sample from and is typically independent of the input data distribution.
\begin{equation}\label{eqn:forwardsde}
\mathrm{d} \rvx_t = \rvf(\rvx_t, t) \mathrm{d} t + g(t) \mathrm{d} \rvw_t
\end{equation}
The reverse SDE in Equation~(\ref{eqn:reversesde}) converts a noisy sample $\rvx_t$ into a clean data sample $\rvx_0$, where $\bar{\rvw}_t$ denotes a standard Wiener process in reverse time. $p(\rvx_t)$ denotes the marginal distribution at time $t$. Its score function, i.e., $\nabla_{\rvx_t} \log p(\rvx_t)$, is usually intractable, so a neural network $s_{\theta}(\rvx_t,t)$ is trained using score-matching loss~\cite{dsm,ssm} to approximate this for all $t$. 
\begin{equation}\label{eqn:reversesde}
\mathrm{d} \rvx_t = \{ \rvf(\rvx_t, t) - g^2(t) \nabla_{\rvx_t} \log p(\rvx_t) \}\mathrm{d} t + g(t) \mathrm{d} \bar{\rvw}_t 
\end{equation}
For a given choice of $\rvf, g$, the PF ODE in Equation~(\ref{eqn:pfode}) describes a deterministic process where, an intermediate sample $\rvx_t$ generated by the ODE share the same marginal probability $p(\rvx_t)$ as that simulated by the forward SDE for all $t \in \left( 0,T \right]$.
\begin{equation}\label{eqn:pfode}
\mathrm{d} \rvx_t = \{ \rvf(\rvx_t, t) - \frac{1}{2}g^2(t) \nabla_{\rvx_t} \log p(\rvx_t) \}\mathrm{d} t
\end{equation}
To sample $\rvx_0 \sim p(\rvx_0)$, we first sample $\rvx_T \sim p(\rvx_T)$ and solve either the reverse SDE with SDE solvers or the PF ODE with ODE solvers using the learned score function $s_{\theta}(\rvx_t,t)$. Corresponding to the Variance Exploding (VE) SDE formulation from~\cite{songsde}, throughout this paper, we fix $\rvf(\rvx_t, t) = \mathbf{0}$, $g(t) = \sqrt{\mathrm{d} \sigma^2(t)/\mathrm{d} \rvt}$, where $\sigma(t)$ denotes the noise schedule for $t \in \left[ 0,T \right]$. We use $\sigma(t)$ and $\sigma_t$ interchangeably to denote the noise level at time $t$. With the above choice of $\rvf$, and $g$, the corresponding perturbation kernel is given by $p(\rvx_t|\rvx_0) = \mathcal{N}(\rvx_0, \sigma^2_t\mathbf{I})$, and $p(\rvx_T)\approx \mathcal{N}(0,\sigma^2_T\mathbf{I})$. Note that setting $\sigma(t) = t$ recovers the case of EDM preconditioning~\cite{edm}.

\section{Variational mode-seeking loss}\label{sec:vml}

{\color{red}{\textbf{Motivation.}}} Given a measurement $\rvy$, and an unconditional diffusion model that generates samples $\rvx_0 \sim p(\rvx_0)$, we are primarily interested in finding the MAP estimate, i.e.,  $\argmax_{\rvx_0} \log p(\rvx_0|\rvy)$. The main motivation for VML stems from the following observation. 
Starting the reverse diffusion process from a fixed noisy sample $\rvx_t$ at time $t$ results in a distribution over $\rvx_0$, i.e., $p(\rvx_0|\rvx_t)$. If we find an optimal $\rvx^{*}_t$ such that $p(\rvx_0|\rvx^{*}_t)$ shares modes, i.e., high-density regions, with the posterior $p(\rvx_0|\rvy)$, then, by starting the reverse diffusion process from $\rvx^{*}_t$ at time $t$, one may expect to generate a probable sample of the posterior. If we repeat the task of finding such $\rvx^{*}_t$ at each diffusion time step $t$, then, as $t \to 0$, $\rvx^{*}_t$ converges\footnote{Under mild assumptions on marginal pdfs, i.e., $p(\rvx_t)$, and their convergence to $p(\rvx_0)$. Note that $p(\rvx_t|\rvx_0)=\mathcal{N}(0,\sigma^2_t\mathbf{I})$} to the MAP estimate as explained in the rest of this section.

For a fixed value of $\rvx_t$, say $\rvx_t = \bm{\gamma}$, the behavior of $p(\rvx_0|\rvx_t=\bm\gamma)$ along various time steps of a diffusion process is shown in Figure~\ref{fig:1}. As $p(\rvx_0|\rvx_t=\bm\gamma) \propto p(\rvx_t=\bm\gamma|\rvx_0)p(\rvx_0)$, and with $p(\rvx_t|\rvx_0) = \mathcal{N}(\rvx_0, \sigma^2_t \mathbf{I})$, the distribution $p(\rvx_0|\rvx_t=\bm\gamma)$ is essentially proportional to the product of $p(\rvx_0)$ and a Gaussian kernel with variance $\sigma^2_t$, centered at $\bm\gamma$ (due to the symmetrical form of $p(\rvx_t|\rvx_0)$). 
Since $\sigma_t \to 0$ as $t \to 0$, the dependence of $p(\rvx_0|\rvx_t=\bm\gamma)$ on $\bm\gamma$ grows stronger as $t$ decreases, with $p(\rvx_0|\rvx_t=\bm\gamma)$ converging to the Dirac delta function $\delta(\rvx_0-\bm\gamma)$ as $t \to 0$, for any $\bm\gamma$.

Suppose, at each time step $t$, we find an optimal $\bm\gamma^{*}_t$ such that $p(\rvx_0|\rvx_t=\bm\gamma^{*}_t)$ shares modes i.e., high-density regions with $p(\rvx_0|\rvy)$. Note that for $t$ arbitrarily close to $0$, $p(\rvx_0|\rvx_t=\bm\gamma^{*}_t)$ is an extremely peaky distribution around $\bm\gamma^{*}_t$. For this distribution to share modes with the posterior $p(\rvx_0|\rvy)$, it is ideal for $\bm\gamma^{*}_t$ to be closer, and converging to the MAP estimate as $t \to 0$, since the MAP estimate is the highest posterior mode, i.e., the sample with the highest posterior probability density.

{\color{red}{\textbf{Formalism.}}} From a variational perspective, at a diffusion time step $t$, $p(\rvx_0|\rvx_t)$ is a parameterized distribution of $\rvx_t$. 
During each reverse diffusion time step $t$, we aim to find a specific distribution $p(\rvx_0|\rvx^{*}_t)$ from the class of parameterized distributions $\{p(\rvx_0|\rvx_t)\}_{\rvx_t}$ such that $p(\rvx_0|\rvx^{*}_t)$ shares modes i.e., high-density regions with the posterior. The functional form of $p(\rvx_0|\rvx_t)$ getting arbitrarily peaky as $t \to 0$ implies that the optimal $\rvx^{*}_t$ ideally converges to the MAP estimate as $t \to 0$. The reverse KL divergence is known to promote this mode-matching behavior of distributions, so we choose $\KL(p(\rvx_0|\rvx_t)||p(\rvx_0|\rvy))$ as the minimization objective at each time step, which we refer to as the variational mode-seeking loss (VML). In practice, however, finding the exact MAP estimate is extremely challenging, as the VML can be highly non-convex, rendering optimization approaches ineffective. Therefore, we settle instead for the posterior modes found by the VML optimizer in practice. 
\vspace{2pt}
\begin{prop}\label{prop:1}
The variational mode-seeking loss (VML) at diffusion time $t$, for a 
given degradation operator $\mathcal{A}$, measurement $\rvy$, and measurement noise variance $\sigma^2_{\rvy}$ is 
{\small{\begin{align*}
& \mathrm{VML}(\rvx_t,t) = \KL( p(\rvx_0|\rvx_t)||p(\rvx_0|\rvy)) = - \log p(\rvx_t) \\
&- \frac{\|\mathrm{D}(\rvx_t,t)-\rvx_t\|^2}{2\sigma^2_t} - \frac{1}{2\sigma^2_t} \mathrm{Tr}\left\{ \Cov[\rvx_0|\rvx_t] \right\} + \frac{1}{2\sigma^2_\rvy} \biggl(  -2 \rvy^{\top} \\
& \int_{\rvx_0} \mathcal{A}(\rvx_0) p(\rvx_0|\rvx_t) \mathrm{d} \rvx_0   + \int_{\rvx_0} \|\mathcal{A}(\rvx_0)\|^2 p(\rvx_0|\rvx_t) \mathrm{d} \rvx_0  \biggr) + \mathrm{C}
\end{align*}}} where $\mathrm{C}$ is a constant, independent of $\rvx_t$. $\mathrm{Tr}$ denotes the matrix trace, $\Cov$ denotes the covariance matrix, and $\mathrm{D}(\cdot,\cdot)$ denotes the true denoiser (see Appendix~\ref{ssec:tweediesformula}).
\end{prop}
\vspace{2pt}
\begin{prop}\label{prop:2}
The variational mode-seeking loss (VML) at diffusion time $t$, for a given \textbf{linear} degradation matrix $\mathrm{H}$, measurement $\rvy$, and measurement noise variance $\sigma^2_{\rvy}$ is 
{\small{
\begin{align*}
& \mathrm{VML}(\rvx_t,t) = \KL(p(\rvx_0|\rvx_t)||p(\rvx_0|\rvy)) = - \log p(\rvx_t) \\
& - \frac{\|\mathrm{D}(\rvx_t,t)-\rvx_t\|^2}{2\sigma^2_t} - \frac{1}{2\sigma^2_t} \mathrm{Tr}\left\{ \Cov[\rvx_0|\rvx_t] \right\} \\ 
&+  \underbrace{\frac{\|\rvy-\mathrm{H}\mathrm{D}(\rvx_t,t)\|^2}{2\sigma^2_\rvy}}_{\text{measurement consistency}} + \frac{1}{2\sigma^2_\rvy} \mathrm{Tr}\left\{ \mathrm{H}\Cov[\rvx_0|\rvx_t]\mathrm{H}^{\top} \right\} + \mathrm{C}
\end{align*}}} where $\mathrm{C}$ is a constant, independent of $\rvx_t$. $\mathrm{Tr}$ denotes the matrix trace, $\Cov$ denotes the covariance matrix, and $\mathrm{D}(\cdot,\cdot)$ denotes the true denoiser (see Appendix~\ref{ssec:tweediesformula}).
\end{prop}

Proofs are provided in Appendix~\ref{ssec:proofs}. In the case of a linear degradation operator, VML can be derived in a closed form as shown above, which consists of a measurement consistency term and several prior terms. Proposition~\ref{prop:4} shows that under mild assumptions, the sequence of functions given by $\mathrm{VML}(\cdot,t)$ converges to the MAP objective, i.e., the negative log posterior $-\log p(\rvx_0|\rvy)$ as $t \to 0$. 

\begin{figure*}[t!]
    \centering
    \begin{subfigure}{.22\textwidth}
        \includegraphics[width=\textwidth]{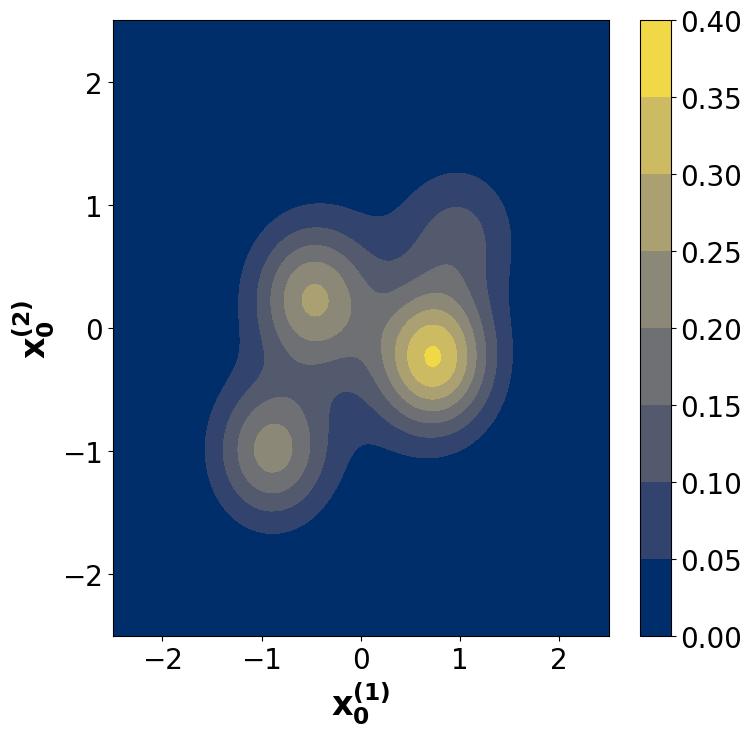}
    \end{subfigure}\hfill%
    \begin{subfigure}{.22\textwidth}
        \includegraphics[width=\textwidth]{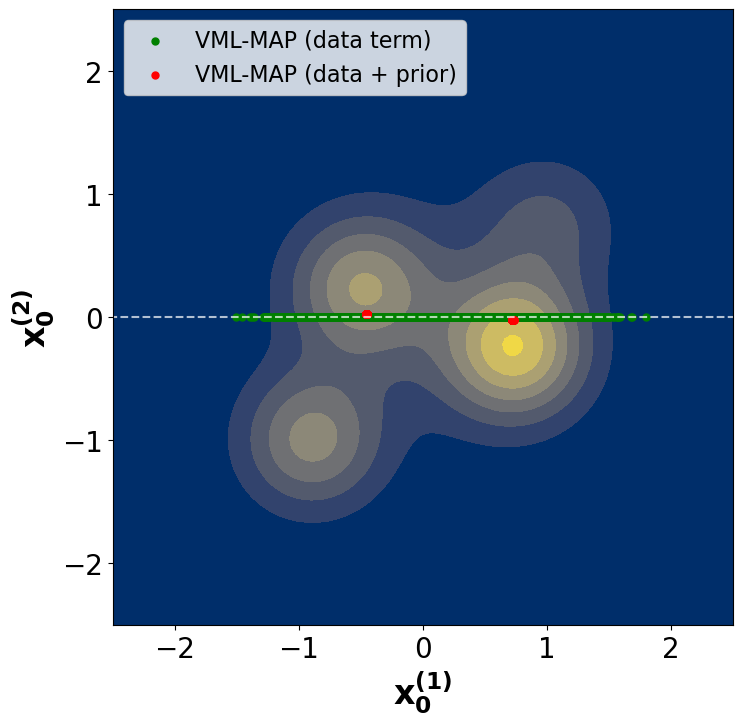}
    \end{subfigure}\hfill%
    \begin{subfigure}{.22\textwidth}
        \includegraphics[width=\textwidth]{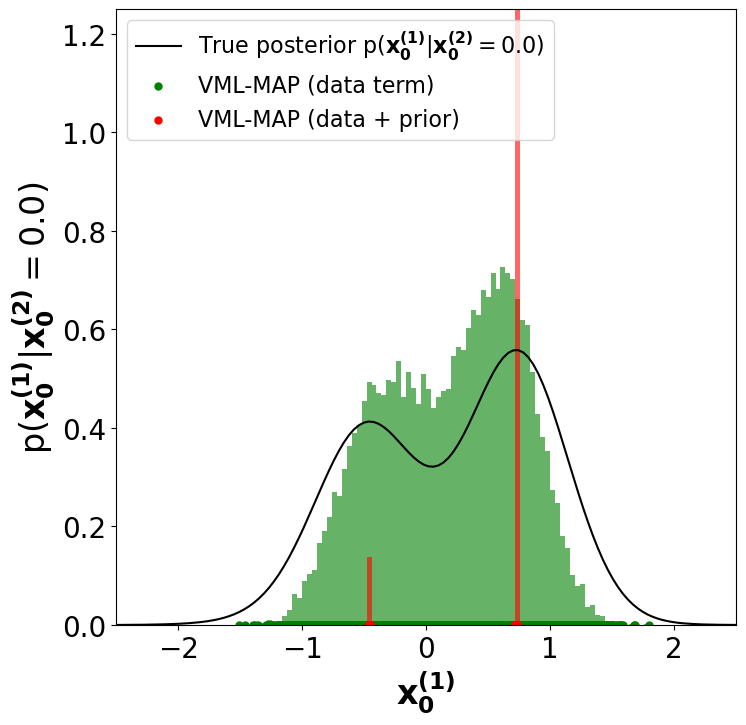}
    \end{subfigure}\hfill%
    \begin{subfigure}{.22\textwidth}
        \includegraphics[width=\textwidth]{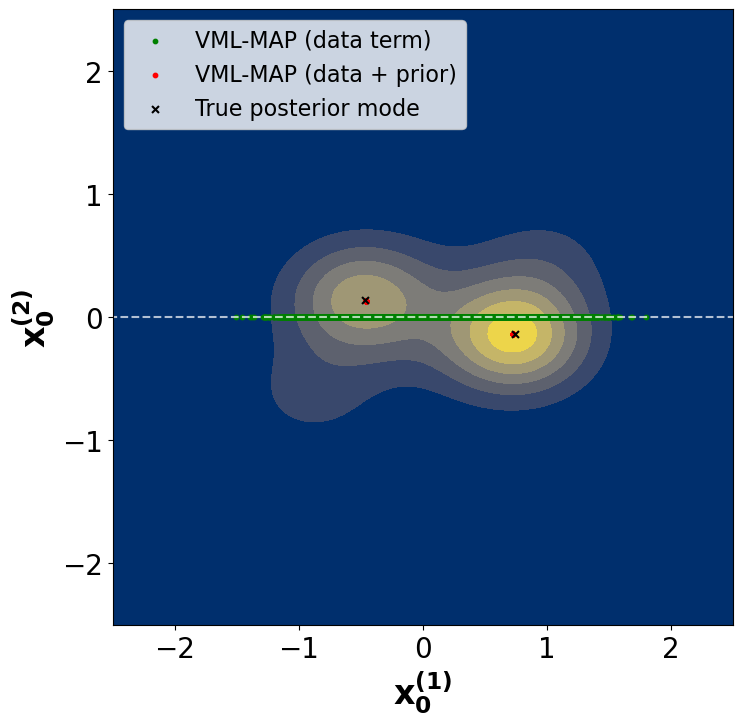}
    \end{subfigure}\\
    \begin{subfigure}{.22\textwidth}
        \includegraphics[width=\textwidth]{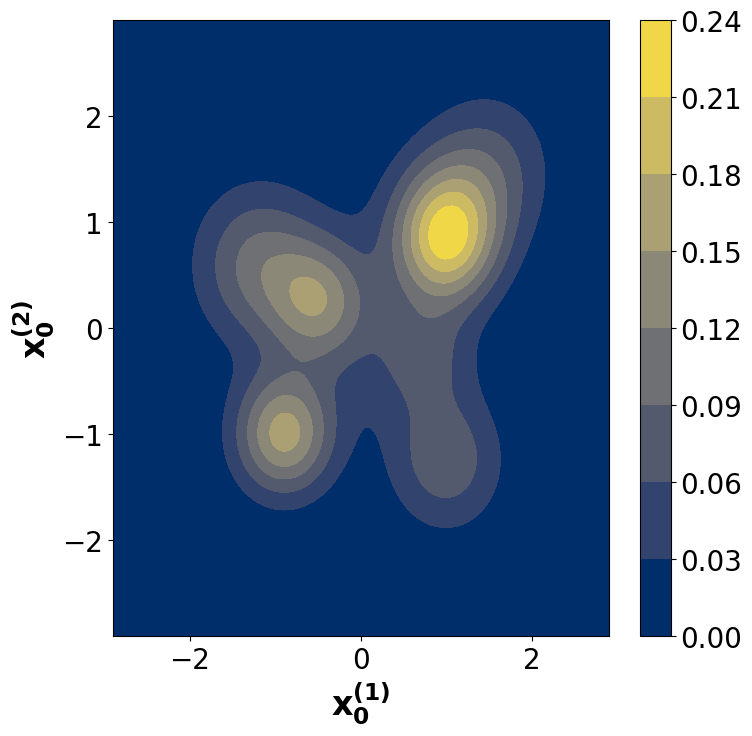}
        \caption{A 2D Gaussian mixture model based input data distribution i.e., $p(\rvx_0)$.}
        \label{fig:gmminput}
    \end{subfigure}\hfill%
    \begin{subfigure}{.22\textwidth}
        \includegraphics[width=\textwidth]{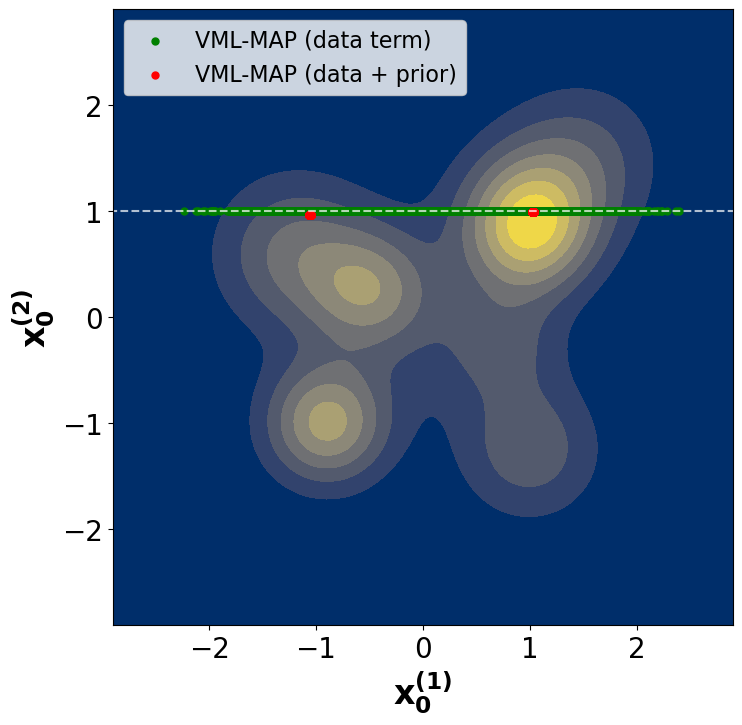}
        \caption{VML-MAP samples for noiseless inpainting i.e., $\rvy = \rvx^{(2)}_0$. Dashed line shows $\rvy$.}
        \label{fig:noiselessinpainting}
    \end{subfigure}\hfill%
    \begin{subfigure}{.22\textwidth}
        \includegraphics[width=\textwidth]{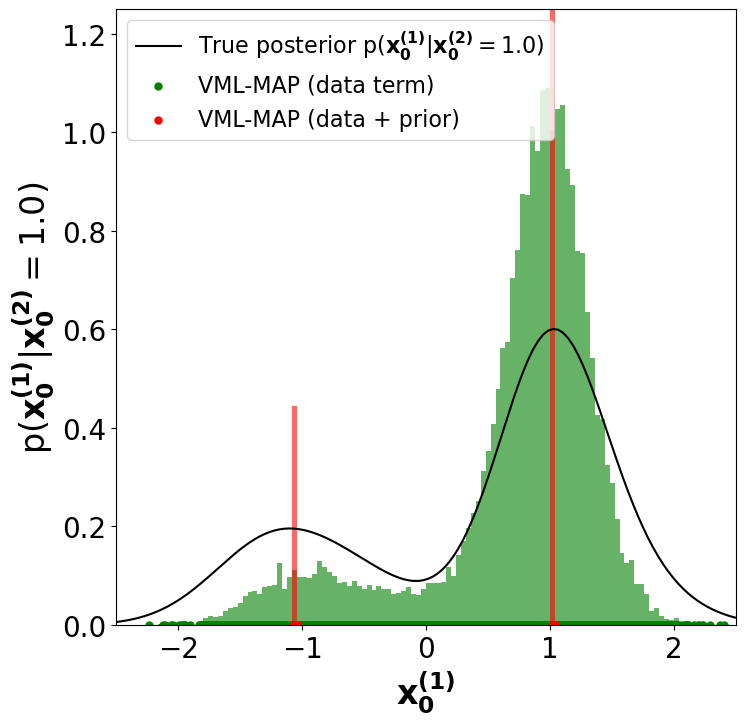}
        \caption{VML-MAP generated sample distributions and the true posterior from Figure~\ref{fig:noiselessinpainting}.}
        \label{fig:noiselessinpaintingdistribution}
    \end{subfigure}\hfill%
    \begin{subfigure}{.22\textwidth}
        \includegraphics[width=\textwidth]{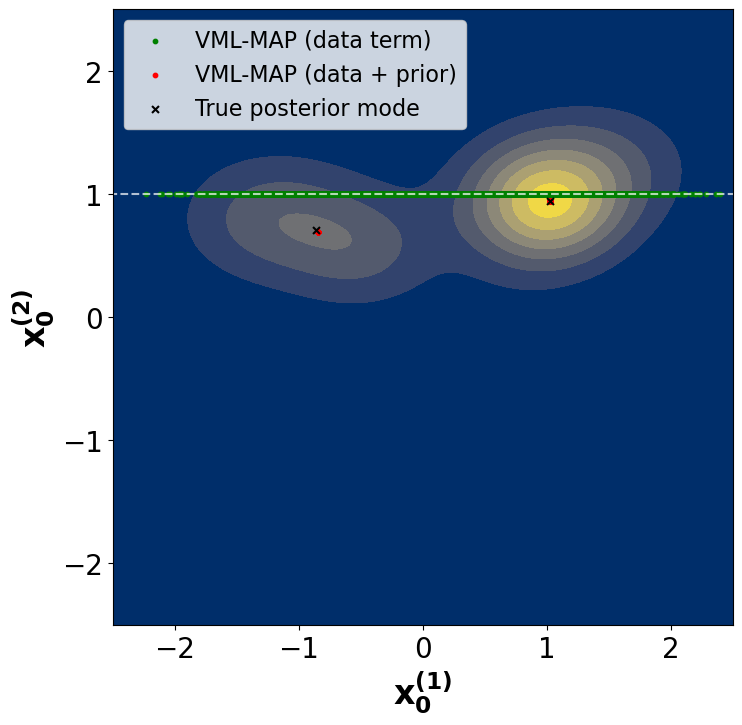}
        \caption{VML-MAP samples for noisy inpainting. The heatmap shows the true posterior.}
        \label{fig:noisyinpainting}
    \end{subfigure}
        \caption{Each row depicts a different example. In each example, the input data distribution $p(\rvx_0)$ follows a 2D Gaussian mixture model, and we evaluate VML-MAP using Algorithm~\ref{algo:vmlmap} on noiseless and noisy inpainting tasks. In the figures above, we generated 20000 samples each for VML-MAP (data term) and VML-MAP (data + prior).}
    \label{fig:toyexample}
\end{figure*}

Note that the higher-order terms of VML involving $\Cov[\rvx_0|\rvx_t]$ are computationally demanding, especially when VML has to be differentiable, for instance, to use gradient-based optimization for minimization. However, Proposition~\ref{prop:3} shows that these higher-order terms of $\mathrm{VML}(\cdot,t)$ converge uniformly to the zero function as $t \to 0$, which implies that the simplified-VML objective $\mathrm{VML}_{\mathrm{S}}(\cdot,t)$, that drops these higher-order terms as shown in Equation~\ref{eqn:svml}, also converges to the MAP objective as $t \to 0$ (see Corollary~\ref{corollary:1}). Given the above and the fact that no practical optimizer can guarantee finding the global minimizer of VML at each diffusion time step, we remark that these higher-order terms can be safely ignored in practice. See Remarks~\ref{remark:1},~\ref{remark:2}, and~\ref{remark:3} for additional details.
\begin{align}\label{eqn:svml}
\resizebox{0.9\columnwidth}{!}{%
 $\mathrm{VML}_{\mathrm{S}}(\rvx_t,t) = \frac{\| \rvy - \mathrm{H}\mathrm{D}(\rvx_t,t) \|^2}{2\sigma^2_\rvy} -\log p(\rvx_t) - \frac{\| \mathrm{D}(\rvx_t,t) - \rvx_t\|^2 }{2\sigma^2_t} + \mathrm{C}$}
\end{align}
\begin{equation}
\resizebox{.7\columnwidth}{!}{%
$\begin{aligned}\label{eqn:svmlrho}
\mathrm{VML}^{\rho}_{\mathrm{S}}(\rvx_t,t) =  & \frac{\| \rvy - \mathrm{H}\mathrm{D}(\rvx_t,t) \|^2}{2\sigma^2_\rvy} -\rho_t \, \log p(\rvx_t) \\ 
& \,\,\,\,\,\,\,\,\,\,\,\,\,\,\,\, - \rho_t \, \frac{\| \mathrm{D}(\rvx_t,t) - \rvx_t\|^2 }{2\sigma^2_t} + \mathrm{C}
\end{aligned}$ }
\end{equation}
\begin{equation}
\resizebox{.8\columnwidth}{!}{%
$\begin{aligned}\label{eqn:svmlrhograd}
\nabla_{\rvx_t} \mathrm{VML}^{\rho}_{\mathrm{S}}(\rvx_t,t) = &   \underbrace{ - \frac{\partial \mathrm{D}^{\top} (\rvx_t,t)}{\partial \rvx_t} \frac{\mathrm{H}^{\top}(\rvy-\mathrm{H}\mathrm{D}(\rvx_t,t))}{\sigma^2_\rvy} }_{\textit{measurement consistency gradient}} \\ 
 & \,\,\,\,\,\,\,\,\,\,\,\,\,\,\,\,\,\, - \rho_t \, \underbrace{\frac{\partial \mathrm{D}^{\top}(\rvx_t,t)}{\partial \rvx_t} \frac{(\mathrm{D}(\rvx_t,t)-\rvx_t)}{\sigma^2_t} }_{\textit{prior gradient}}
\end{aligned}$ }
\end{equation}
Furthermore, from Corollary~\ref{corollary:2}, reweighting the prior terms of $\mathrm{VML}_{\mathrm{S}}(\cdot,\cdot)$ and $\mathrm{VML}(\cdot,\cdot)$ with a prior weight schedule $\rho(\cdot)$ also leads to both these objectives converging to the MAP objective, if $\rho(t) \to 1$ as $t \to 0$. Note that we use the notations $\rho(t)$ and $\rho_t$ interchangeably. In this regard, Equation~\ref{eqn:svmlrho} shows the reweighted simplified-VML objective $\mathrm{VML}^{\rho}_{\mathrm{S}}(\cdot, \cdot)$, and Equation~\ref{eqn:svmlrhograd} shows its gradient. Note that this gradient computation only requires a single vector-Jacobian product, which modern Automatic differentiation frameworks handle efficiently in practice. Hence, we only consider the reweighted simplified-VML objective, i.e., $\mathrm{VML}^{\rho}_{\mathrm{S}}(\cdot,\cdot)$ (where $\rho_t \to 1$ as $t \to 0$) in the remainder of this paper and in our practical implementation.

{\color{red}{\textbf{VML-MAP.}}} We propose a practical implementation of the reweighted simplified-VML minimization using a gradient-descent optimizer as shown in Algorithm~\ref{algo:vmlmap}, which we refer to as VML-MAP. The inputs to Algorithm~\ref{algo:vmlmap} consists of the learned diffusion denoiser $\mathrm{D}_{\theta}(\cdot,\cdot)$, the linear degradation matrix $\mathrm{H}$, the measurement $\rvy$ with noise variance $\sigma^2_{\rvy}$, the diffusion noise schedule $\sigma(\cdot)$, the total number of reverse diffusion steps $N$ with the discretized time step schedule specified by $t_{i\in\{0,\dots N\}}$, where $t_0 = 0$, the gradient descent iterations per step given by $K$, the learning rate $\gamma$, and the prior weight schedule $\rho(\cdot)$. Note that we use the notations $\sigma(t)$ and $\sigma_t$ interchangeably.
\begin{algorithm}[h!]
\caption{\textbf{VML-MAP} / $\textbf{VML-MAP}_{pre}$}
\label{algo:vmlmap}
\begin{algorithmic}[1]
\STATE {\bfseries Input:} $\mathrm{D}_{\theta}(\cdot,\cdot)$, $\mathrm{H}$, $\mathbf{y}$, $\sigma_{\mathbf{y}}$, $\sigma(\cdot)$, $\{t_i\}_{i=0}^N$, $K$, $\gamma$, $\rho(\cdot)$
\STATE {\bfseries Output:} $\mathbf{x}_{t_0}$ 
\STATE {\bfseries Initialize} $\mathbf{x}_{t_N} \sim \mathcal{N}(\mathbf{0}, \sigma_{t_N}^2 \mathbf{I})$
\FOR{$i = N$ \textbf{down to} $1$}
    \FOR{$j = 1$ \textbf{to} $K$}
        \IF{\textbf{VML-MAP}}
            \STATE $\mathbf{x}_{t_i} \gets \mathbf{x}_{t_i}
            - \gamma \, \nabla_{\mathbf{x}_{t_i}} \mathrm{VML}^{\rho}_{\mathrm{S}}(\rvx_t,t)$ \COMMENTNOBRACES{Eq.~(\ref{eqn:svmlrhograd})}
        \ELSIF{\textbf{VML-MAP}$_{pre}$}
            \STATE $\mathbf{x}_{t_i} \gets \mathbf{x}_{t_i}
            - \gamma \, \nabla_{\mathbf{x}_{t_i}} \mathrm{VML}^{\rho}_{\mathrm{S}_{\text{pre}}}(\rvx_t,t)$ \COMMENTNOBRACES{Eq.~(\ref{eqn:preconditionedrhogradient})}
        \ENDIF
    \ENDFOR
    \STATE $\mathbf{x}_{t_{i-1}} \sim
    \mathcal{N}\left(\mathrm{D}_{\theta}(\mathbf{x}_{t_i}, t_i),
    \sigma_{t_{i-1}}^2 \mathbf{I}\right)$
\ENDFOR
\STATE {\bfseries Return} $\mathbf{x}_{t_0}$
\end{algorithmic}
\end{algorithm}

{\color{red}{\textbf{Validation check.}}} To empirically validate the reweighted simplified-VML minimization leading to the MAP estimate (modes in practice), we apply VML-MAP (Algorithm~\ref{algo:vmlmap}) to noiseless and noisy inpainting tasks, where the input data distribution follows a 2D (i.e., $\rvx_0 \in \mathbb{R}^2$) Gaussian mixture model (GMM). Two examples, each following a different 2D GMM input distribution, are shown in Figure~\ref{fig:gmminput}. 

In Figure~\ref{fig:noiselessinpainting}, we consider noiseless inpainting, where the observation corresponds to the second coordinate, i.e., $\rvy = \rvx^{(2)}_0$, and the task is to infer the full sample $\rvx_0$. For each of the two corresponding examples, Figure~\ref{fig:noiselessinpainting} shows a white dashed line, denoting the observed value of $\rvy$ (or equivalently $\rvx^{(2)}_0$ in this specific case), and also samples generated by VML-MAP using ({\romannumeral 1}) only the data/measurement consistency term ({\romannumeral 2}) both the data and prior terms, i.e., the full $\mathrm{VML}^{\rho}_{\mathrm{S}}(\cdot,\cdot)$. The empirical distributions of these generated samples are depicted and compared against the true posterior $p(\rvx_0|\rvy)$ (or equivalently $p(\rvx^{(1)}_0|\rvx^{(2)}_0)$ in this specific case) in Figure~\ref{fig:noiselessinpaintingdistribution}, which shows the samples of VML-MAP (data + prior) concentrating around the posterior modes. 

In Figure~\ref{fig:noisyinpainting}, we consider the case of noisy inpainting, where the observation $\rvy = \rvx^{(2)}_0 + \eta$, and $\eta \sim \mathcal{N}(0,\sigma^2_{\rvy} \mathbf{I})$, with $\sigma_{\rvy} = 0.5$. Figure~\ref{fig:noisyinpainting} shows a white dashed line denoting the observed value of $\rvy$, the heatmap of the true posterior $p(\rvx_0|\rvy)$, and the samples generated by VML-MAP (data term) and VML-MAP (data + prior). The figure again reveals that the samples of VML-MAP (data + prior) concentrate around the posterior modes, which is the intended behavior, since no optimizer in practice can always guarantee finding the global minimizer of VML at each diffusion time step. Further details regarding the experimental setup and implementation are provided in Appendix~\ref{ssec:toysetup}.


{\color{red}{\textbf{Preconditioner.}}} Note that VML-MAP uses a simple gradient-descent optimizer for computational efficiency, which can struggle especially when dealing with ill-conditioned loss objectives. In this regard, a preconditioner can help accelerate convergence and thereby improve the effectiveness of the optimizer.
{{\begin{equation}\label{eqn:preconditionerM}
\resizebox{.45\columnwidth}{!}{%
$\mathrm{M} = (\mathbf{I} - \Sigma^{+}\Sigma) + \mathrm{H}^{\top}\mathrm{H}$ }
\end{equation}}}
\vspace{-.3cm}
{{\begin{equation}\label{eqn:preconditioner}
\resizebox{.68\columnwidth}{!}{%
$\mathrm{P} = \left( \frac{\partial \mathrm{D}^{\top}(\rvx_t,t)}{\partial \rvx_t} \right) \mathrm{M^{-1}} \left( \frac{\partial \mathrm{D}^{\top}(\rvx_t,t)}{\partial \rvx_t} \right)^{\mathrm{-1}}$ }
\end{equation}}}
\begin{equation}
\resizebox{.85\columnwidth}{!}{%
${{\begin{aligned}\label{eqn:preconditionedrhogradient}
\nabla_{\rvx_t}\mathrm{VML}^{\rho}_{\mathrm{S}_{pre}}(\rvx_t,t) = & -\frac{\partial \mathrm{D}^{\top}(\rvx_t,t)}{\partial \rvx_t} \mathrm{M^{-1}} \frac{\mathrm{H}^{\top}(\rvy-\mathrm{H}\mathrm{D}(\rvx_t,t))}{\sigma^2_\rvy} \\ & \,\,\,\,\,\,\,\,\,\,\,\,\,\,\,\, - \rho_t \, \frac{\partial \mathrm{D}^{\top}(\rvx_t,t)}{\partial \rvx_t} \mathrm{M^{-1}}\frac{(\mathrm{D}(\rvx_t,t)-\rvx_t)}{\sigma^2_t}
\end{aligned}}}$ }
\end{equation}
Assuming a linear degradation matrix $\mathrm{H}$ with the singular value decomposition (SVD) given by $\mathrm{H}=\mathrm{U}\Sigma \mathrm{V}^{\top}$, where $\mathrm{U}$,$\mathrm{V}$ denote the left and right singular orthogonal matrices respectively, with $\Sigma$ denoting the singular values matrix, we use the preconditioner $\mathrm{P}$ in Equation~(\ref{eqn:preconditioner}) to essentially replace the gradient $\nabla_{\rvx_t}\mathrm{VML}^{\rho}_{\mathrm{S}}(\cdot,\cdot)$ in VML-MAP with the preconditioned gradient $\mathrm{P}\nabla_{\rvx_t}\mathrm{VML}^{\rho}_{\mathrm{S}}(\cdot,\cdot)$. We refer to this method as $\text{VML-MAP}_{pre}$, also presented in Algorithm~\ref{algo:vmlmap}. Note that $\Sigma^{+}$ denotes the pseudoinverse of $\Sigma$. 

Equation~(\ref{eqn:preconditionedrhogradient}) further expands this preconditioned gradient $\mathrm{P}\nabla_{\rvx_t}\mathrm{VML}^{\rho}_{\mathrm{S}}(\cdot,\cdot)$ that we denote with $\nabla_{\rvx_t}\mathrm{VML}^{\rho}_{\mathrm{S}_{pre}}(\cdot, \cdot)$. Note that {\small{$\frac{\partial \mathrm{D}^{\top}(\rvx_t,t)}{\partial \rvx_t}$}} {\small{$= \sigma^2_t\Cov[\rvx_0|\rvx_t]$}} (see Appendix~\ref{ssec:covarianceformula}). Positive-definiteness of {\small{$\Cov[\rvx_0|\rvx_t] $}}, implies the positive definiteness and hence the invertibility of {\small{$\frac{\partial \mathrm{D}^{\top}(\rvx_t,t)}{\partial \rvx_t}$}}, which further implies the invertibility of $\mathrm{P}$. 

Note that the non-preconditioned variant (VML-MAP) does not require any SVD and is a viable option in most cases. Our proposed preconditioner applies to linear operators with affordable SVD and may be infeasible in other cases where SVD computation remains a bottleneck. However, note that preconditioners are not unique, and that different designs based on the operator and the availability of SVD are possible. Our proposed preconditioner is only one such design.

\section{VML-MAP for Image restoration}\label{sec:image-restoration}

\begin{table*}[ht!]
\caption{Noiseless image restoration: Experiments on Half-mask inpainting, $4\times$ Super-resolution, and Uniform deblurring on $100$ validation images of ImageNet$256$ and $100$ images of FFHQ$256$ are repeated across 4 different seeds. The mean and standard deviation across these runs are reported. Best values in \textbf{bold}, second best values \underline{underlined}.}
\label{tab:exp1}
\begin{center}
\resizebox{0.7\textwidth}{!}{
\begin{tabular}{ l l c c c c c c } 
\toprule
\multicolumn{1}{c}{} & \multicolumn{1}{c}{} & \multicolumn{6}{c}{} \\[-7pt]
 \multicolumn{1}{c}{{Dataset}}  & \multicolumn{1}{l}{{Method}} & \multicolumn{2}{c}{{Inpainting}} & \multicolumn{2}{c}{{$4\times${ Super-res}}}  & \multicolumn{2}{c}{{Deblurring}} \\[3pt]
 \multicolumn{1}{c}{} & \multicolumn{1}{c}{} & \multicolumn{1}{c}{\scriptsize{LPIPS$\downarrow$}} & \multicolumn{1}{c}{\scriptsize{FID$\downarrow$}} & \multicolumn{1}{c}{\scriptsize{LPIPS$\downarrow$}} & \multicolumn{1}{c}{\scriptsize{FID$\downarrow$}} & \multicolumn{1}{c}{\scriptsize{LPIPS$\downarrow$}} & \multicolumn{1}{c}{\scriptsize{FID$\downarrow$}} \\[3pt]
\toprule 
 \multicolumn{1}{c}{} & \multicolumn{1}{c}{} & \multicolumn{6}{c}{}\\[-7pt]
 \multicolumn{1}{c}{} & \multicolumn{1}{l}{{\scriptsize{$\text{DDRM}$}}}  & 
 \multicolumn{1}{c}{\scriptsize{$0.393$}\tiny{\color{lightgray}{$\pm0.001$}}} & \multicolumn{1}{c}{\scriptsize{$103.5$}\tiny{\color{lightgray}{$\pm0.657$}}}  & \multicolumn{1}{c}{\scriptsize{$0.289$}\tiny{\color{lightgray}{$\pm0.000$}}} & \multicolumn{1}{c}{\scriptsize{$89.40$}\tiny{\color{lightgray}{$\pm0.962$}}} & \multicolumn{1}{c}{\scriptsize{$0.618$}\tiny{\color{lightgray}{$\pm0.000$}}} & \multicolumn{1}{c}{\scriptsize{$233.4$}\tiny{\color{lightgray}{$\pm2.033$}}}\\
 \multicolumn{1}{c}{} & \multicolumn{1}{l}{{\scriptsize{$\Pi\text{GDM}$}}}  & \multicolumn{1}{c}{\scriptsize{$0.373$}\tiny{\color{lightgray}{$\pm0.000$}}} & \multicolumn{1}{c}{\scriptsize{$103.7$}\tiny{\color{lightgray}{$\pm0.694$}}}  & \multicolumn{1}{c}{\scriptsize{$0.292$}\tiny{\color{lightgray}{$\pm0.001$}}} & \multicolumn{1}{c}{\scriptsize{$83.88$}\tiny{\color{lightgray}{$\pm0.226$}}} & \multicolumn{1}{c}{\scriptsize{$0.562$}\tiny{\color{lightgray}{$\pm0.000$}}} & \multicolumn{1}{c}{\scriptsize{$231.8$}\tiny{\color{lightgray}{$\pm0.631$}}}\\ 
 \multicolumn{1}{c}{} & \multicolumn{1}{l}{{\scriptsize{$\text{MAPGA}$}}}  & \multicolumn{1}{c}{\scriptsize{$\underline{0.290}$}\tiny{\color{lightgray}{$\pm0.001$}}} & \multicolumn{1}{c}{\scriptsize{$\underline{80.76}$}\tiny{\color{lightgray}{$\pm2.013$}}}  & \multicolumn{1}{c}{\scriptsize{$0.273$}\tiny{\color{lightgray}{$\pm0.001$}}} & \multicolumn{1}{c}{\scriptsize{$76.12$}\tiny{\color{lightgray}{$\pm0.587$}}} & \multicolumn{1}{c}{\scriptsize{$\underline{0.459}$}\tiny{\color{lightgray}{$\pm0.002$}}} & \multicolumn{1}{c}{\scriptsize{$\underline{194.9}$}\tiny{\color{lightgray}{$\pm1.548$}}}\\
 \multicolumn{1}{c}{\scriptsize{ImageNet}} & \multicolumn{1}{l}{{\scriptsize{$\text{DAPS-1K}$}}}  & \multicolumn{1}{c}{\scriptsize{$0.384$}\tiny{\color{lightgray}{$\pm0.001$}}} & \multicolumn{1}{c}{\scriptsize{$98.45$}\tiny{\color{lightgray}{$\pm1.879$}}}  & \multicolumn{1}{c}{\scriptsize{$0.254$}\tiny{\color{lightgray}{$\pm0.001$}}} & \multicolumn{1}{c}{\scriptsize{$71.60$}\tiny{\color{lightgray}{$\pm1.010$}}} & \multicolumn{1}{c}{\scriptsize{$0.605$}\tiny{\color{lightgray}{$\pm0.001$}}} & \multicolumn{1}{c}{\scriptsize{$217.8$}\tiny{\color{lightgray}{$\pm3.042$}}}\\
  \multicolumn{1}{c}{} & \multicolumn{1}{l}{{\scriptsize{$\text{DAPS-4K}$}}}  & \multicolumn{1}{c}{\scriptsize{$0.365$}\tiny{\color{lightgray}{$\pm0.004$}}} & \multicolumn{1}{c}{\scriptsize{$93.62$}\tiny{\color{lightgray}{$\pm0.856$}}}  & \multicolumn{1}{c}{\scriptsize{$0.244$}\tiny{\color{lightgray}{$\pm0.000$}}} & \multicolumn{1}{c}{\scriptsize{$69.71$}\tiny{\color{lightgray}{$\pm1.554$}}} & \multicolumn{1}{c}{\scriptsize{$0.593$}\tiny{\color{lightgray}{$\pm0.001$}}} & \multicolumn{1}{c}{\scriptsize{$227.9$}\tiny{\color{lightgray}{$\pm2.991$}}}\\[2pt]
 \cline{2-8}
 \multicolumn{1}{c}{} & \multicolumn{1}{c}{} & \multicolumn{6}{c}{}\\[-7pt]
 \multicolumn{1}{l}{} & \multicolumn{1}{l}{\bf{\scriptsize{$\text{VML-MAP}$}}}  & \multicolumn{1}{c}{\scriptsize{$\bf{0.262}$}\tiny{\color{lightgray}{$\pm0.002$}}} & \multicolumn{1}{c}{\scriptsize{$\bf{74.21}$}\tiny{\color{lightgray}{$\pm3.209$}}} & \multicolumn{1}{c}{\scriptsize{$\bf{0.194}$}\tiny{\color{lightgray}{$\pm0.001$}}} & \multicolumn{1}{c}{\scriptsize{$\underline{60.08}$}\tiny{\color{lightgray}{$\pm0.687$}}} & \multicolumn{1}{c}{\scriptsize{$0.509$}\tiny{\color{lightgray}{$\pm0.003$}}} & \multicolumn{1}{c}{\scriptsize{$200.4$}\tiny{\color{lightgray}{$\pm2.812$}}}\\
 \multicolumn{1}{l}{} & \multicolumn{1}{l}{\bf{\scriptsize{$\text{VML-MAP}_{pre}$}}}  & \multicolumn{1}{c}{\scriptsize{$\bf{0.262}$}\tiny{\color{lightgray}{$\pm0.002$}}} & \multicolumn{1}{c}{\scriptsize{$\bf{74.21}$}\tiny{\color{lightgray}{$\pm3.209$}}} & \multicolumn{1}{c}{\scriptsize{$\underline{0.196}$}\tiny{\color{lightgray}{$\pm0.003$}}} & \multicolumn{1}{c}{\scriptsize{$\bf{58.60}$}\tiny{\color{lightgray}{$\pm2.076$}}} & \multicolumn{1}{c}{\scriptsize{$\bf{0.367}$}\tiny{\color{lightgray}{$\pm0.002$}}} & \multicolumn{1}{c}{\scriptsize{$\bf{165.2}$}\tiny{\color{lightgray}{$\pm2.037$}}}\\[2pt]
 \midrule
 \multicolumn{1}{c}{} & \multicolumn{1}{c}{} & \multicolumn{6}{c}{}\\[-7pt]
 \multicolumn{1}{c}{} & \multicolumn{1}{l}{{\scriptsize{$\text{DDRM}$}}}  & 
 \multicolumn{1}{c}{\scriptsize{$0.246$}\tiny{\color{lightgray}{$\pm0.001$}}} & \multicolumn{1}{c}{\scriptsize{$71.23$}\tiny{\color{lightgray}{$\pm0.521$}}}  & \multicolumn{1}{c}{\scriptsize{$0.154$}\tiny{\color{lightgray}{$\pm0.000$}}} & \multicolumn{1}{c}{\scriptsize{$70.07$}\tiny{\color{lightgray}{$\pm0.495$}}} & \multicolumn{1}{c}{\scriptsize{$0.307$}\tiny{\color{lightgray}{$\pm0.000$}}} & \multicolumn{1}{c}{\scriptsize{$117.9$}\tiny{\color{lightgray}{$\pm0.427$}}}\\
 \multicolumn{1}{c}{} & \multicolumn{1}{l}{{\scriptsize{$\Pi\text{GDM}$}}}  & \multicolumn{1}{c}{\scriptsize{$0.237$}\tiny{\color{lightgray}{$\pm0.001$}}} & \multicolumn{1}{c}{\scriptsize{$70.40$}\tiny{\color{lightgray}{$\pm0.494$}}}  & \multicolumn{1}{c}{\scriptsize{$0.147$}\tiny{\color{lightgray}{$\pm0.000$}}} & \multicolumn{1}{c}{\scriptsize{$68.38$}\tiny{\color{lightgray}{$\pm0.605$}}} & \multicolumn{1}{c}{\scriptsize{$0.293$}\tiny{\color{lightgray}{$\pm0.000$}}} & \multicolumn{1}{c}{\scriptsize{$114.0$}\tiny{\color{lightgray}{$\pm0.302$}}}\\ 
 \multicolumn{1}{c}{} & \multicolumn{1}{l}{{\scriptsize{$\text{MAPGA}$}}}  & \multicolumn{1}{c}{\scriptsize{$\underline{0.206}$}\tiny{\color{lightgray}{$\pm0.000$}}} & \multicolumn{1}{c}{\scriptsize{$64.15$}\tiny{\color{lightgray}{$\pm0.660$}}}  & \multicolumn{1}{c}{\scriptsize{$0.132$}\tiny{\color{lightgray}{$\pm0.000$}}} & \multicolumn{1}{c}{\scriptsize{$63.82$}\tiny{\color{lightgray}{$\pm0.455$}}} & \multicolumn{1}{c}{\scriptsize{$0.235$}\tiny{\color{lightgray}{$\pm0.000$}}} & \multicolumn{1}{c}{\scriptsize{$119.2$}\tiny{\color{lightgray}{$\pm1.045$}}}\\ 
 \multicolumn{1}{c}{\scriptsize{FFHQ}} & \multicolumn{1}{l}{{\scriptsize{$\text{DAPS-1K}$}}}  & \multicolumn{1}{c}{\scriptsize{$0.233$}\tiny{\color{lightgray}{$\pm0.001$}}} & \multicolumn{1}{c}{\scriptsize{$61.22$}\tiny{\color{lightgray}{$\pm0.759$}}}  & \multicolumn{1}{c}{\scriptsize{$0.113$}\tiny{\color{lightgray}{$\pm0.000$}}} & \multicolumn{1}{c}{\scriptsize{$60.89$}\tiny{\color{lightgray}{$\pm1.028$}}} & \multicolumn{1}{c}{\scriptsize{$0.260$}\tiny{\color{lightgray}{$\pm0.000$}}} & \multicolumn{1}{c}{\scriptsize{$100.8$}\tiny{\color{lightgray}{$\pm0.951$}}}\\
  \multicolumn{1}{c}{} & \multicolumn{1}{l}{{\scriptsize{$\text{DAPS-4K}$}}}  & \multicolumn{1}{c}{\scriptsize{$0.224$}\tiny{\color{lightgray}{$\pm0.000$}}} & \multicolumn{1}{c}{\scriptsize{$\underline{60.22}$}\tiny{\color{lightgray}{$\pm0.738$}}}  & \multicolumn{1}{c}{\scriptsize{$\underline{0.100}$}\tiny{\color{lightgray}{$\pm0.000$}}} & \multicolumn{1}{c}{\scriptsize{$58.29$}\tiny{\color{lightgray}{$\pm0.398$}}} & \multicolumn{1}{c}{\scriptsize{$\underline{0.230}$}\tiny{\color{lightgray}{$\pm0.001$}}} & \multicolumn{1}{c}{\scriptsize{$\underline{93.71}$}\tiny{\color{lightgray}{$\pm1.217$}}}\\[2pt]
 \cline{2-8}
 \multicolumn{1}{c}{} & \multicolumn{1}{c}{} & \multicolumn{6}{c}{}\\[-7pt]
 \multicolumn{1}{l}{} & \multicolumn{1}{l}{\bf{\scriptsize{$\text{VML-MAP}$}}}  & \multicolumn{1}{c}{\scriptsize{$\bf{0.180}$}\tiny{\color{lightgray}{$\pm0.001$}}} & \multicolumn{1}{c}{\scriptsize{$\bf{52.76}$}\tiny{\color{lightgray}{$\pm0.526$}}} & \multicolumn{1}{c}{\scriptsize{$\bf{0.100}$}\tiny{\color{lightgray}{$\pm0.000$}}} & \multicolumn{1}{c}{\scriptsize{$\underline{57.55}$}\tiny{\color{lightgray}{$\pm1.414$}}} & \multicolumn{1}{c}{\scriptsize{$0.247$}\tiny{\color{lightgray}{$\pm0.000$}}} & \multicolumn{1}{c}{\scriptsize{$99.40$}\tiny{\color{lightgray}{$\pm1.799$}}}\\
 \multicolumn{1}{l}{} & \multicolumn{1}{l}{{\bf{\scriptsize{$\text{VML-MAP}_{pre}$}}}}  & \multicolumn{1}{c}{\scriptsize{$\bf{0.180}$}\tiny{\color{lightgray}{$\pm0.001$}}} & \multicolumn{1}{c}{\scriptsize{$\bf{52.76}$}\tiny{\color{lightgray}{$\pm0.526$}}} & \multicolumn{1}{c}{\scriptsize{$\bf{0.100}$}\tiny{\color{lightgray}{$\pm0.000$}}} & \multicolumn{1}{c}{\scriptsize{$\bf{52.20}$}\tiny{\color{lightgray}{$\pm0.814$}}} & \multicolumn{1}{c}{\scriptsize{$\bf{0.182}$}\tiny{\color{lightgray}{$\pm0.000$}}} & \multicolumn{1}{c}{\scriptsize{$\bf{93.48}$}\tiny{\color{lightgray}{$\pm0.681$}}}\\[2pt]
 \bottomrule
 \end{tabular}}
\end{center}
\end{table*}

Several image restoration tasks in computer vision, such as inpainting, super-resolution, deblurring, etc., can be modeled as linear inverse problems. In this section, we apply VML-MAP to the aforementioned image restoration tasks to understand its effectiveness in practice. In all our experiments, we evaluate on $100$ validation images of ImageNet~\cite{imagenet} with a resolution of $256\times 256$, using the unconditional ImageNet$256$ pre-trained diffusion model from~\citet{ddpm}, and on $100$ images of FFHQ~\cite{ffhq}, with a resolution of $256\times 256$, using the FFHQ$256$ pre-trained diffusion model from~\citet{dps}. 

\textbf{Metrics.} Note that the goal of image restoration is to reconstruct a perceptually good image consistent with the given measurement (i.e., the degraded image). While PSNR and SSIM serve as conventional reconstruction metrics, in the context of image restoration inverse problems, they do not always capture perceptual quality. As explained in~\citet{blau2018perception}, for pixel-to-pixel based mean squared error reconstruction metrics (e.g. PSNR), there exist a unique reconstructed image, given by the minimum mean squared error estimator, i.e., the posterior mean, which achieves a high PSNR, but is often perceptually worse, since the posterior mean does not necessarily have a high probability, as it may not even lie on the input data manifold. Such a metric also disregards the fact that there can be several equally plausible, i.e., perceptually good reconstructions consistent with the given degraded image (see Figure~\ref{fig:diverse_seeds}). Hence, in the context of image restoration, achieving low distortion metrics (i.e., a high PSNR or high SSIM) does not necessarily imply better perceptual sample quality, while achieving high perceptual metrics (i.e., a low LPIPS and low FID) naturally correlate more strongly with a sample's perceptual quality and also align closely with the goal of image restoration. The perception-distortion tradeoff~\cite{blau2018perception} further demonstrates, both formally and empirically, the tradeoff between these distortion metrics (e.g., PSNR, SSIM) and the perceptual metrics (e.g., LPIPS, FID), which implies that no image-restoration algorithm can reconstruct an image with the maximal perceptual quality without sacrificing performance on the distortion metrics. 

Therefore, in our experiments, we focus on perceptual metrics such as LPIPS and FID, which capture measurement consistency and perceptual quality of the restored images. Using these, we report and compare the performance of several existing baselines against VML-MAP and its preconditioned variant in the latter sections. For completeness, though not to be emphasized, we report PSNR and SSIM metrics in Appendix~\ref{ssec:implementationdetails1} for our experiments in Table~\ref{tab:exp1}.

\subsection{Noiseless image restoration}\label{ssec:noiseless-image-restoration}

In Table~\ref{tab:exp1}, we consider the challenging tasks of image inpainting with a half-mask (where the right half of the image is masked), $4\times$ super-resolution, and uniform deblurring with a $16\times 16$ kernel. We make the deblurring task even more challenging by setting the singular values below a high threshold to zero. Note that only a handful of prior works considered such challenging inverse tasks in the literature. In all our experiments, we fix a budget of approximately 1000 neural function evaluations of the diffusion model for both VML-MAP and $\text{VML-MAP}_{pre}$. For DAPS~\cite{daps}, we consider two configurations, with 1000 and 4000 neural function evaluations, denoted DAPS-1K and DAPS-4K, respectively. See Appendix~\ref{ssec:implementationdetails1} for further details regarding the experimental setup and hyperparameters.

\begin{figure}[h!]
    \centering
    \begin{subfigure}[h]{0.025\textwidth}
    \centering    
    \subfloat{\bf{\color{gray}{ }}}
    \end{subfigure}
    \begin{subfigure}[h]{0.09\textwidth}
    \centering
    \subfloat{\scriptsize{\bf{{\color{gray}{\textit{Original}}}}}}    
    \end{subfigure}
    \begin{subfigure}[h]{0.09\textwidth}
    \centering
    \subfloat{\scriptsize{\bf{{\color{gray}{\textit{Measurement}}}}}}    
    \end{subfigure}
    \begin{subfigure}[h]{0.09\textwidth}
    \centering
    \subfloat{\scriptsize{\bf{{\color{gray}{\textit{DAPS-1K}}}}}}   
    \end{subfigure}
    \begin{subfigure}[h]{0.09\textwidth}
    \centering
    \subfloat{\scriptsize{{\bf{\color{gray}{\textit{DAPS-4K}}}}}}    
    \end{subfigure}
    \begin{subfigure}[h]{0.09\textwidth}
    \centering
    \subfloat{\scriptsize{\bf{\color{gray}{\textit{VML-MAP}}}}}
    \end{subfigure}
    \\[2pt]
    \begin{subfigure}[h]{0.025\textwidth}
    \centering
    \subfloat{{\rotatebox[origin=c]{90}{\bf{\color{gray}{\textit{Inpaint}}}}}}    
    \end{subfigure}
    \begin{subfigure}[h]{0.09\textwidth}
    \centering
    \includegraphics[width=\textwidth]{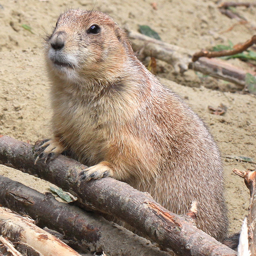}    
    \end{subfigure}
    \begin{subfigure}[h]{0.09\textwidth}
    \centering
    \includegraphics[width=\textwidth]{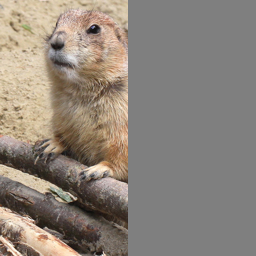}    
    \end{subfigure}
    \begin{subfigure}[h]{0.09\textwidth}
    \centering
    \includegraphics[width=\textwidth]{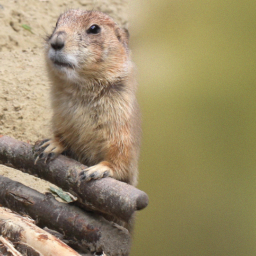}  
    \end{subfigure}
    \begin{subfigure}[h]{0.09\textwidth}
    \centering
    \includegraphics[width=\textwidth]{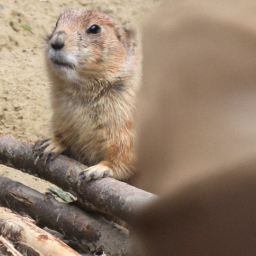}    
    \end{subfigure}
    \begin{subfigure}[h]{0.09\textwidth}
    \centering
    \includegraphics[width=\textwidth]{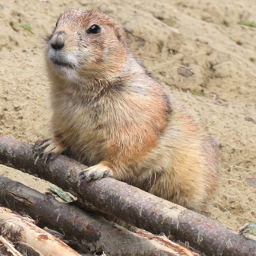}    
    \end{subfigure}
    \\
    \begin{subfigure}[h]{0.025\textwidth}
    \centering
    \subfloat{{\rotatebox[origin=c]{90}{\color{gray}{$\mathbf{4\times}$\bf{\textit{SR}}}}}}    
    \end{subfigure}
    \begin{subfigure}[h]{0.09\textwidth}
    \centering
    \includegraphics[width=\textwidth]{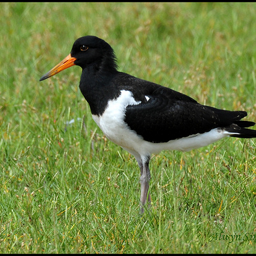}    
    \end{subfigure}
    \begin{subfigure}[h]{0.09\textwidth}
    \centering
    \includegraphics[width=\textwidth]{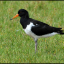}    
    \end{subfigure}
    \begin{subfigure}[h]{0.09\textwidth}
    \centering
    \includegraphics[width=\textwidth]{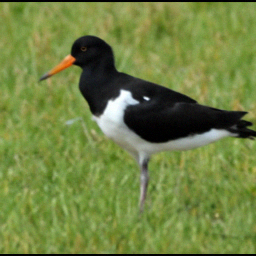}    
    \end{subfigure}
    \begin{subfigure}[h]{0.09\textwidth}
    \centering
    \includegraphics[width=\textwidth]{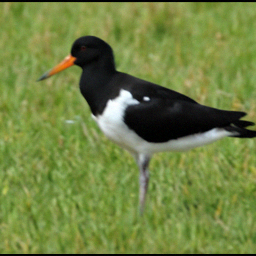}    
    \end{subfigure}
    \begin{subfigure}[h]{0.09\textwidth}
    \centering
    \includegraphics[width=\textwidth]{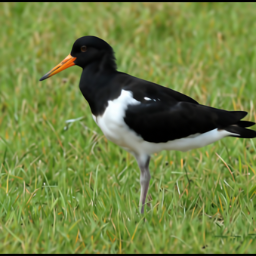}    
    \end{subfigure}
    \\
    \begin{subfigure}[h]{0.025\textwidth}
    \centering
    \subfloat{{\rotatebox[origin=c]{90}{\bf{\color{gray}{\textit{Deblur}}}}}}    
    \end{subfigure}
    \begin{subfigure}[h]{0.09\textwidth}
    \centering
    \includegraphics[width=\textwidth]{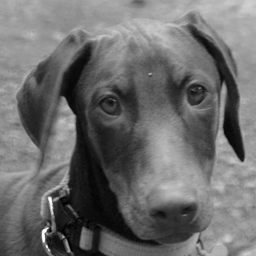}    
    \end{subfigure}
    \begin{subfigure}[h]{0.09\textwidth}
    \centering
    \includegraphics[width=\textwidth]{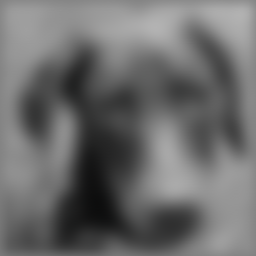}    
    \end{subfigure}
    \begin{subfigure}[h]{0.09\textwidth}
    \centering
    \includegraphics[width=\textwidth]{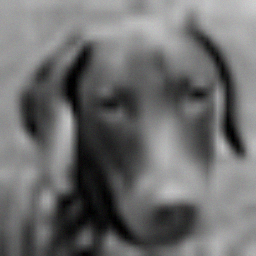} 
    \end{subfigure}
    \begin{subfigure}[h]{0.09\textwidth}
    \centering
    \includegraphics[width=\textwidth]{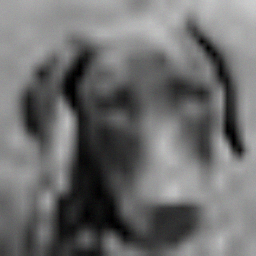} 
    \end{subfigure}
    \begin{subfigure}[h]{0.09\textwidth}
    \centering
    \includegraphics[width=\textwidth]{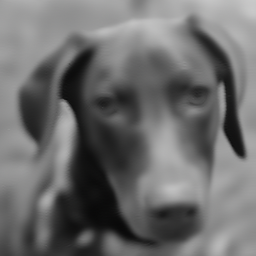} 
    \end{subfigure}
    \caption{Half-mask inpainting, $4\times$ super-resolution, and deblurring tasks from Table~\ref{tab:exp1}. Note that DAPS-1K, DAPS-4K, and VML-MAP only require the forward operation of the linear degradation operator and not its SVD. Zoom in for the best view.}\label{fig:exp1}
\end{figure}%
\begin{figure}[h!]
    \centering
    \begin{subfigure}[h]{0.025\textwidth}
    \centering    
    \subfloat{\bf{\color{gray}{ }}}
    \end{subfigure}
    \begin{subfigure}[h]{0.075\textwidth}
    \centering
    \subfloat{\tiny{\bf{{\color{gray}{$\textit{Original}_{\ }$}}}}}    
    \end{subfigure}
    \begin{subfigure}[h]{0.075\textwidth}
    \centering
    \subfloat{\tiny{\bf{{\color{gray}{$\textit{Measurement}_{\ }$}}}}}    
    \end{subfigure}
    \begin{subfigure}[h]{0.075\textwidth}
    \centering
    \subfloat{\tiny{{\bf{\color{gray}{$\textit{DDRM}_{\ }$}}}}}    
    \end{subfigure}
    \begin{subfigure}[h]{0.075\textwidth}
    \centering
    \subfloat{\tiny{{\bf{\color{gray}{$\textit{$\Pi$GDM}_{\ }$}}}}}    
    \end{subfigure}
    \begin{subfigure}[h]{0.075\textwidth}
    \centering
    \subfloat{\tiny{\bf{\color{gray}{$\textit{MAPGA}_{\ }$}}}}
    \end{subfigure}
    \begin{subfigure}[h]{0.075\textwidth}
    \centering
    \subfloat{\tiny{\bf{{\color{gray}{$\textit{VML-MAP}_{pre}$}}}}}   
    \end{subfigure}
    \\[2pt]
    \begin{subfigure}[h]{0.025\textwidth}
    \centering
    \subfloat{{\rotatebox[origin=c]{90}{\color{gray}{$\mathbf{4\times}$\bf{\textit{SR}}}}}}    
    \end{subfigure}
    \begin{subfigure}[h]{0.075\textwidth}
    \centering
    \includegraphics[width=\textwidth]{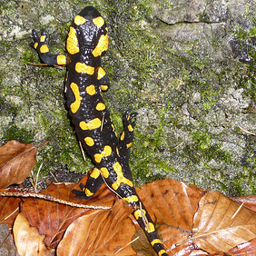}    
    \end{subfigure}
    \begin{subfigure}[h]{0.075\textwidth}
    \centering
    \includegraphics[width=\textwidth]{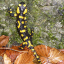}    
    \end{subfigure}
    \begin{subfigure}[h]{0.075\textwidth}
    \centering
    \includegraphics[width=\textwidth]{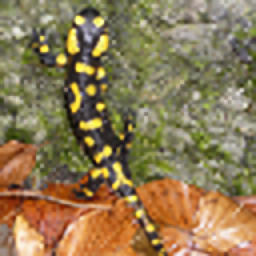}  
    \end{subfigure}
    \begin{subfigure}[h]{0.075\textwidth}
    \centering
    \includegraphics[width=\textwidth]{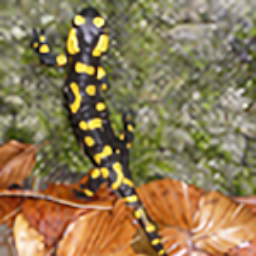}    
    \end{subfigure}
    \begin{subfigure}[h]{0.075\textwidth}
    \centering
    \includegraphics[width=\textwidth]{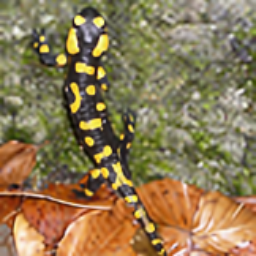}
    \end{subfigure}
    \begin{subfigure}[h]{0.075\textwidth}
    \centering
    \includegraphics[width=\textwidth]{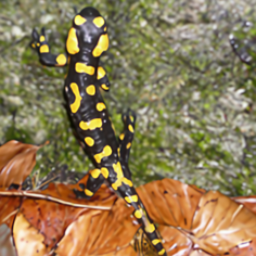}    
    \end{subfigure}
    \\
    \begin{subfigure}[h]{0.025\textwidth}
    \centering
    \subfloat{{\rotatebox[origin=c]{90}{\bf{\color{gray}{\textit{Deblur}}}}}}    
    \end{subfigure}
    \begin{subfigure}[h]{0.075\textwidth}
    \centering
    \includegraphics[width=\textwidth]{icml2026/figures/imagenet256new/deblur16/im49_original.png}    
    \end{subfigure}
    \begin{subfigure}[h]{0.075\textwidth}
    \centering
    \includegraphics[width=\textwidth]{icml2026/figures/imagenet256new/deblur16/im49_degraded.png}    
    \end{subfigure}
    \begin{subfigure}[h]{0.075\textwidth}
    \centering
    \includegraphics[width=\textwidth]{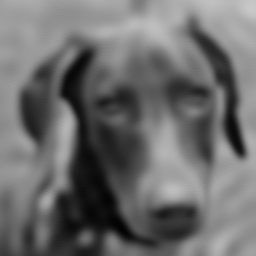} 
    \end{subfigure}
    \begin{subfigure}[h]{0.075\textwidth}
    \centering
    \includegraphics[width=\textwidth]{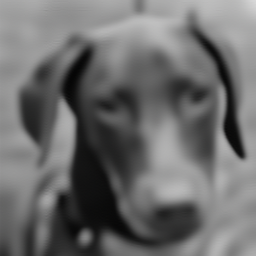} 
    \end{subfigure}
    \begin{subfigure}[h]{0.075\textwidth}
    \centering
    \includegraphics[width=\textwidth]{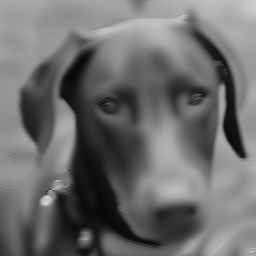}
    \end{subfigure}
    \begin{subfigure}[h]{0.075\textwidth}
    \centering
    \includegraphics[width=\textwidth]{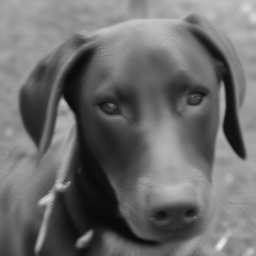}
    \end{subfigure}
    \caption{$4\times$ Super-resolution, and Deblurring tasks from Table~\ref{tab:exp1}. Note that DDRM, $\Pi$GDM, MAPGA, and $\text{VML-MAP}_{pre}$ require SVD of the linear degradation operator. Zoom in for the best view.}\label{fig:exp2}
\end{figure}

The quantitative results from Table~\ref{tab:exp1} and the corresponding qualitative comparisons from Figures~\ref{fig:exp1} and~\ref{fig:exp2} highlight the effectiveness of VML-MAP and $\text{VML-MAP}_{pre}$ over existing methods. The results also indicate the effectiveness of the preconditioner on $4\times$ super-resolution and deblurring tasks as $\text{VML-MAP}_{pre}$ shows significant improvements in LPIPS and FID over VML-MAP and other baselines, denoting higher perceptual quality of the restored images. For inpainting, VML-MAP and $\text{VML-MAP}_{pre}$ are essentially equivalent since $\mathrm{M} = \mathrm{P} = \mathbf{I}$.

Note that DDRM, $\Pi$GDM, MAPGA, and $\text{VML-MAP}_{pre}$ require SVD of $\mathrm{H}$, while DAPS and VML-MAP only require the forward operation of $\mathrm{H}$. To ensure a fair comparison in this regard, in Appendix~\ref{ssec:inpexp}, we evaluate all the methods on the image inpainting task with diverse masks, where the SVD of $\mathrm{H}$ is trivial. The quantitative results from Table~\ref{tab:inpexp} and the corresponding qualitative visualizations from Figures~\ref{fig:qv4},~\ref{fig:qv5},~\ref{fig:qv6}, and~\ref{fig:qv7} reveal the superior performance of VML-MAP in all the inpainting tasks. We refer to Appendix~\ref{ssec:inpexp} for further details regarding the experimental setup and hyperparameters.

\begin{figure}[h!]
    \centering
    \begin{subfigure}[h]{0.49\columnwidth}
    \centering
    \includegraphics[width=\columnwidth]{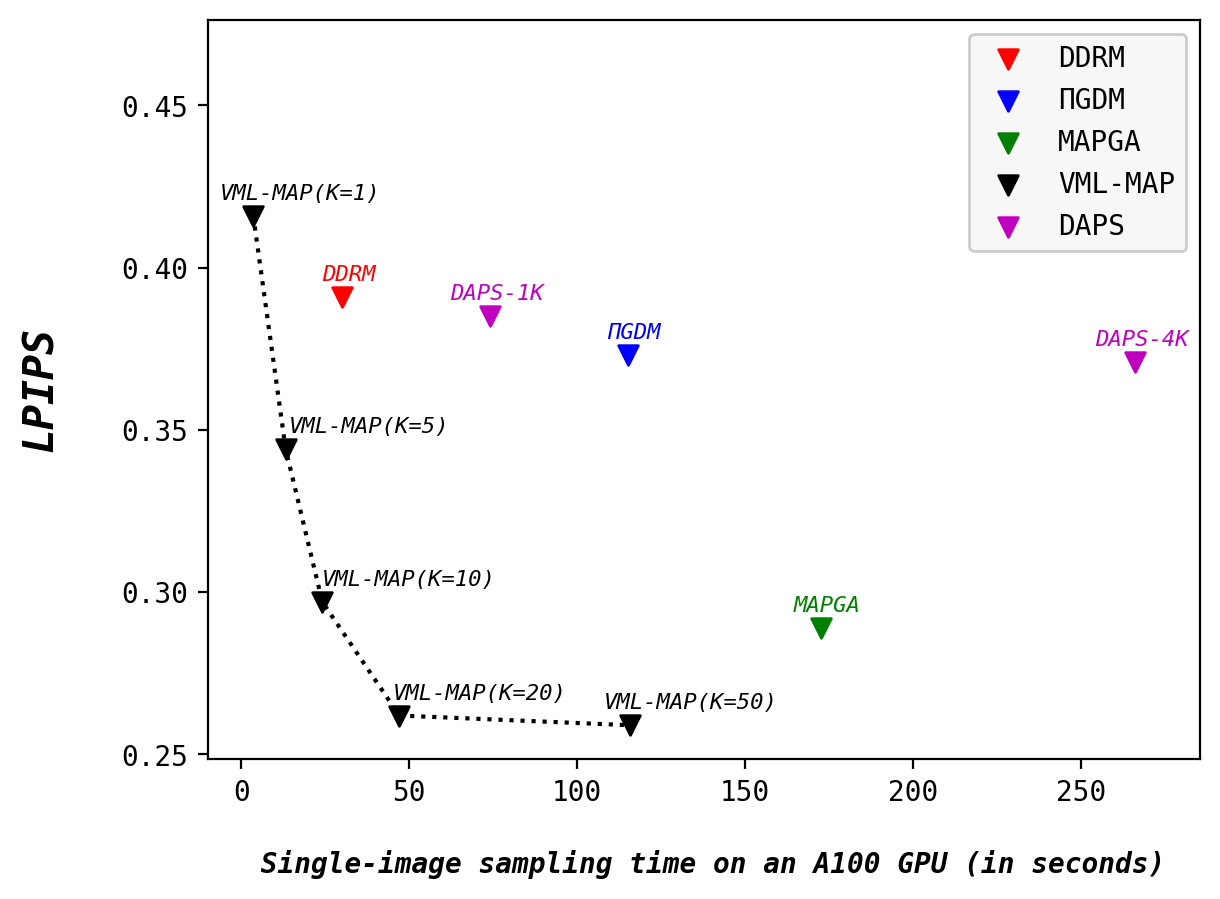}
    \end{subfigure}
    \begin{subfigure}[h]{0.49\columnwidth}
    \centering
    \includegraphics[width=\columnwidth]{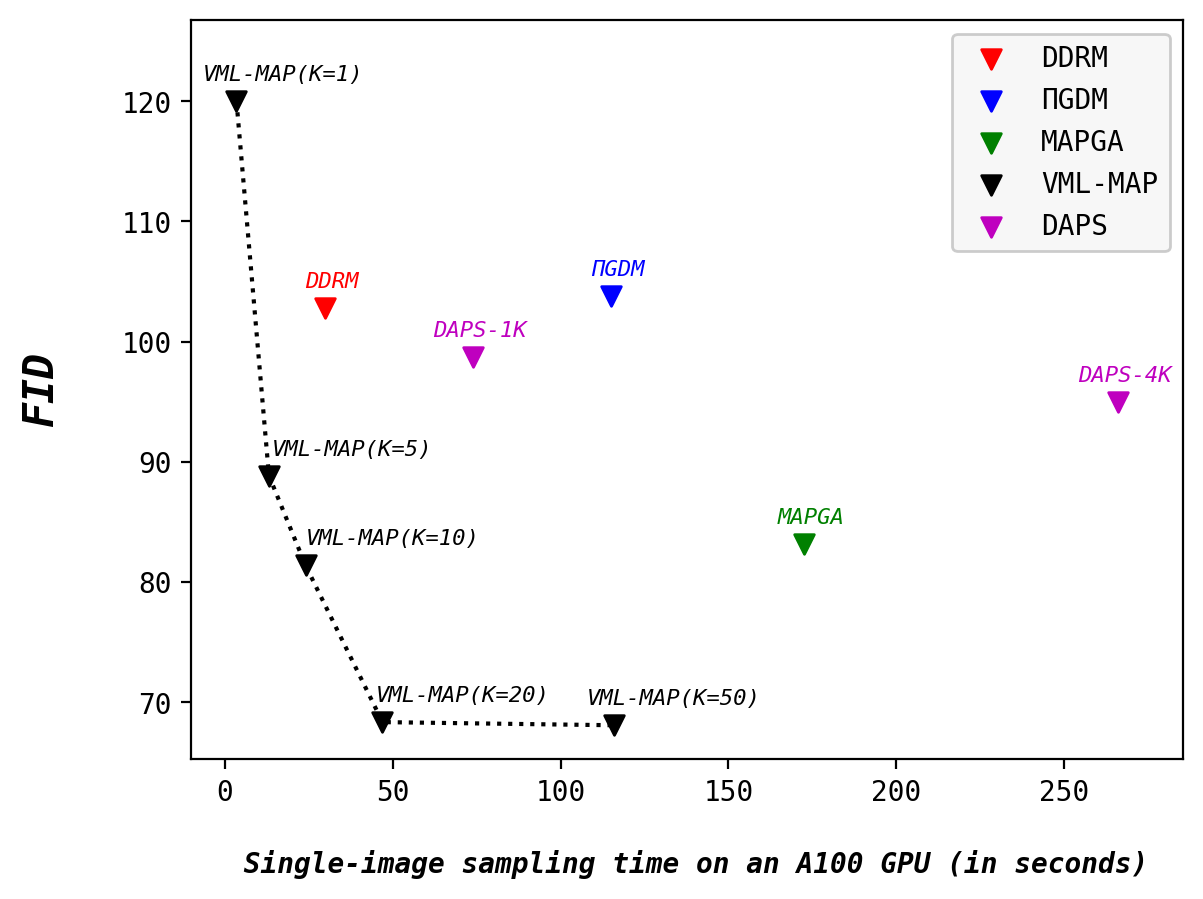}
    \end{subfigure}
    \caption{Runtime vs Perceptual quality for the half-mask inpainting experiment in Table~\ref{tab:exp1}. DDRM and $\Pi$GDM use $500$ and $1000$ reverse diffusion steps, respectively, to achieve their best results. For VML-MAP, we fix the reverse diffusion steps ($N$) to $20$, and vary the number of gradient-descent iterations per step ($K$) across $\{1,5,10,20,50\}$. See Appendix~\ref{ssec:implementationdetails1} for more details. VML-MAP achieves better perceptual quality with lower compute.}
    \label{fig:rtvsq}
\end{figure}

\begin{table}[t!]
\caption{Noisy image restoration ($\sigma_\rvy = 0.05$): Experiments on Half-mask inpainting, $4\times$ Super-resolution, and Uniform deblurring on $100$ validation images of ImageNet$256$ and $100$ images of FFHQ$256$ are repeated across 4 different seeds. The mean and standard deviation across these runs are reported. Best values in \textbf{bold}, second best values \underline{underlined}.}
\label{tab:exp2}
\begin{center}
\resizebox{\columnwidth}{!}{
\begin{tabular}{ l l c c c c c c } 
\toprule
\multicolumn{1}{c}{} & \multicolumn{1}{c}{} & \multicolumn{6}{c}{} \\[-7pt]
 \multicolumn{1}{c}{{Dataset}}  & \multicolumn{1}{l}{{Method}} & \multicolumn{2}{c}{{Inpainting}} & \multicolumn{2}{c}{{$4\times${ Super-res}}}  & \multicolumn{2}{c}{{Deblurring}} \\[3pt]
 \multicolumn{1}{c}{} & \multicolumn{1}{c}{} & \multicolumn{1}{c}{\scriptsize{LPIPS$\downarrow$}} & \multicolumn{1}{c}{\scriptsize{FID$\downarrow$}} & \multicolumn{1}{c}{\scriptsize{LPIPS$\downarrow$}} & \multicolumn{1}{c}{\scriptsize{FID$\downarrow$}} & \multicolumn{1}{c}{\scriptsize{LPIPS$\downarrow$}} & \multicolumn{1}{c}{\scriptsize{FID$\downarrow$}} \\[3pt]
\toprule 
 \multicolumn{1}{c}{} & \multicolumn{1}{c}{} & \multicolumn{6}{c}{}\\[-7pt]
 \multicolumn{1}{c}{} & \multicolumn{1}{l}{{\scriptsize{$\Pi\text{GDM}$}}}  & \multicolumn{1}{c}{\scriptsize{$0.396$}\tiny{\color{lightgray}{$\pm0.000$}}} & \multicolumn{1}{c}{\scriptsize{$126.6$}\tiny{\color{lightgray}{$\pm0.932$}}}  & \multicolumn{1}{c}{\scriptsize{$0.369$}\tiny{\color{lightgray}{$\pm0.002$}}} & \multicolumn{1}{c}{\scriptsize{$149.3$}\tiny{\color{lightgray}{$\pm2.100$}}} & \multicolumn{1}{c}{\scriptsize{$0.656$}\tiny{\color{lightgray}{$\pm0.001$}}} & \multicolumn{1}{c}{\scriptsize{$255.3$}\tiny{\color{lightgray}{$\pm0.840$}}}\\ 
 \multicolumn{1}{c}{} & \multicolumn{1}{l}{{\scriptsize{$\text{TMPD}$}}} & \multicolumn{1}{c}{\scriptsize{$0.407$}\tiny{\color{lightgray}{$\pm0.004$}}} & \multicolumn{1}{c}{\scriptsize{$155.3$}\tiny{\color{lightgray}{$\pm4.009$}}}  & \multicolumn{1}{c}{\scriptsize{$0.345$}\tiny{\color{lightgray}{$\pm0.002$}}} & \multicolumn{1}{c}{\scriptsize{$108.9$}\tiny{\color{lightgray}{$\pm4.358$}}} & \multicolumn{1}{c}{\scriptsize{$0.548$}\tiny{\color{lightgray}{$\pm0.002$}}} & \multicolumn{1}{c}{\scriptsize{$246.7$}\tiny{\color{lightgray}{$\pm3.764$}}}\\ 
  \multicolumn{1}{c}{\scriptsize{ImageNet}} & \multicolumn{1}{l}{{\scriptsize{$\text{DAPS-4K}$}}}  & \multicolumn{1}{c}{\scriptsize{$\underline{0.392}$}\tiny{\color{lightgray}{$\pm0.003$}}} & \multicolumn{1}{c}{\scriptsize{$\underline{107.7}$}\tiny{\color{lightgray}{$\pm1.151$}}}  & \multicolumn{1}{c}{\scriptsize{$\underline{0.343}$}\tiny{\color{lightgray}{$\pm0.001$}}} & \multicolumn{1}{c}{\scriptsize{$103.3$}\tiny{\color{lightgray}{$\pm3.612$}}} & \multicolumn{1}{c}{\scriptsize{$0.600$}\tiny{\color{lightgray}{$\pm0.002$}}} & \multicolumn{1}{c}{\scriptsize{$233.7$}\tiny{\color{lightgray}{$\pm4.641$}}}\\[2pt]
 \cline{2-8}
 \multicolumn{1}{c}{} & \multicolumn{1}{c}{} & \multicolumn{6}{c}{}\\[-7pt]
 \multicolumn{1}{l}{} & \multicolumn{1}{l}{\bf{\scriptsize{$\text{VML-MAP}^{\bm \tau}$}}}  & \multicolumn{1}{c}{\scriptsize{$\bf{0.287}$}\tiny{\color{lightgray}{$\pm0.002$}}} & \multicolumn{1}{c}{\scriptsize{$\bf{86.97}$}\tiny{\color{lightgray}{$\pm2.317$}}} & \multicolumn{1}{c}{\scriptsize{$\bf{0.210}$}\tiny{\color{lightgray}{$\pm0.001$}}} & \multicolumn{1}{c}{\scriptsize{$\underline{69.19}$}\tiny{\color{lightgray}{$\pm1.125$}}} & \multicolumn{1}{c}{\scriptsize{$\underline{0.547}$}\tiny{\color{lightgray}{$\pm0.003$}}} & \multicolumn{1}{c}{\scriptsize{$\underline{186.7}$}\tiny{\color{lightgray}{$\pm2.646$}}}\\
 \multicolumn{1}{l}{} & \multicolumn{1}{l}{\bf{\scriptsize{$\text{VML-MAP}^{\bm \tau}_{pre}$}}}  & \multicolumn{1}{c}{\scriptsize{$\bf{0.287}$}\tiny{\color{lightgray}{$\pm0.002$}}} & \multicolumn{1}{c}{\scriptsize{$\bf{86.97}$}\tiny{\color{lightgray}{$\pm2.317$}}} & \multicolumn{1}{c}{\scriptsize{$\bf{0.210}$}\tiny{\color{lightgray}{$\pm0.000$}}} & \multicolumn{1}{c}{\scriptsize{$\bf{67.40}$}\tiny{\color{lightgray}{$\pm0.945$}}} & \multicolumn{1}{c}{\scriptsize{$\bf{0.400}$}\tiny{\color{lightgray}{$\pm0.002$}}} & \multicolumn{1}{c}{\scriptsize{$\bf{151.4}$}\tiny{\color{lightgray}{$\pm4.084$}}}\\[2pt]
 \midrule
 \multicolumn{1}{c}{} & \multicolumn{1}{c}{} & \multicolumn{6}{c}{}\\[-7pt]
 \multicolumn{1}{c}{} & \multicolumn{1}{l}{{\scriptsize{$\Pi\text{GDM}$}}}  & \multicolumn{1}{c}{\scriptsize{$0.256$}\tiny{\color{lightgray}{$\pm0.000$}}} & \multicolumn{1}{c}{\scriptsize{$91.63$}\tiny{\color{lightgray}{$\pm0.535$}}}  & \multicolumn{1}{c}{\scriptsize{$0.210$}\tiny{\color{lightgray}{$\pm0.000$}}} & \multicolumn{1}{c}{\scriptsize{$120.1$}\tiny{\color{lightgray}{$\pm0.812$}}} & \multicolumn{1}{c}{\scriptsize{$0.330$}\tiny{\color{lightgray}{$\pm0.000$}}} & \multicolumn{1}{c}{\scriptsize{$139.4$}\tiny{\color{lightgray}{$\pm1.140$}}}\\ 
 \multicolumn{1}{c}{} & \multicolumn{1}{l}{{\scriptsize{$\text{TMPD}$}}}  & \multicolumn{1}{c}{\scriptsize{$0.271$}\tiny{\color{lightgray}{$\pm0.002$}}} & \multicolumn{1}{c}{\scriptsize{$84.37$}\tiny{\color{lightgray}{$\pm1.767$}}}  & \multicolumn{1}{c}{\scriptsize{$0.123$}\tiny{\color{lightgray}{$\pm0.001$}}} & \multicolumn{1}{c}{\scriptsize{$\underline{62.13}$}\tiny{\color{lightgray}{$\pm1.170$}}} & \multicolumn{1}{c}{\scriptsize{$0.273$}\tiny{\color{lightgray}{$\pm0.001$}}} & \multicolumn{1}{c}{\scriptsize{$102.0$}\tiny{\color{lightgray}{$\pm1.272$}}}\\
 \multicolumn{1}{c}{\scriptsize{FFHQ}} & \multicolumn{1}{l}{{\scriptsize{$\text{DAPS-4K}$}}}  & \multicolumn{1}{c}{\scriptsize{$\underline{0.255}$}\tiny{\color{lightgray}{$\pm0.000$}}} & \multicolumn{1}{c}{\scriptsize{$\underline{76.63}$}\tiny{\color{lightgray}{$\pm0.527$}}}  & \multicolumn{1}{c}{\scriptsize{$0.131$}\tiny{\color{lightgray}{$\pm0.000$}}} & \multicolumn{1}{c}{\scriptsize{$66.24$}\tiny{\color{lightgray}{$\pm0.832$}}} & \multicolumn{1}{c}{\scriptsize{$0.236$}\tiny{\color{lightgray}{$\pm0.002$}}} & \multicolumn{1}{c}{\scriptsize{$91.99$}\tiny{\color{lightgray}{$\pm1.038$}}}\\[2pt]
 \cline{2-8}
 \multicolumn{1}{c}{} & \multicolumn{1}{c}{} & \multicolumn{6}{c}{}\\[-7pt]
 \multicolumn{1}{l}{} & \multicolumn{1}{l}{\bf{\scriptsize{$\text{VML-MAP}^{\bm \tau}$}}}  & \multicolumn{1}{c}{\scriptsize{$\bf{0.202}$}\tiny{\color{lightgray}{$\pm0.001$}}} & \multicolumn{1}{c}{\scriptsize{$\bf{62.81}$}\tiny{\color{lightgray}{$\pm0.706$}}} & \multicolumn{1}{c}{\scriptsize{$\underline{0.120}$}\tiny{\color{lightgray}{$\pm0.000$}}} & \multicolumn{1}{c}{\scriptsize{$64.85$}\tiny{\color{lightgray}{$\pm1.590$}}} & \multicolumn{1}{c}{\scriptsize{$\underline{0.214}$}\tiny{\color{lightgray}{$\pm0.001$}}} & \multicolumn{1}{c}{\scriptsize{$\bf{84.88}$}\tiny{\color{lightgray}{$\pm0.929$}}}\\
 \multicolumn{1}{l}{} & \multicolumn{1}{l}{{\bf{\scriptsize{$\text{VML-MAP}^{\bm \tau}_{pre}$}}}}  & \multicolumn{1}{c}{\scriptsize{$\bf{0.202}$}\tiny{\color{lightgray}{$\pm0.001$}}} & \multicolumn{1}{c}{\scriptsize{$\bf{62.81}$}\tiny{\color{lightgray}{$\pm0.706$}}} & \multicolumn{1}{c}{\scriptsize{$\bf{0.112}$}\tiny{\color{lightgray}{$\pm0.000$}}} & \multicolumn{1}{c}{\scriptsize{$\bf{61.54}$}\tiny{\color{lightgray}{$\pm1.037$}}} & \multicolumn{1}{c}{\scriptsize{$\bf{0.202}$}\tiny{\color{lightgray}{$\pm0.000$}}} & \multicolumn{1}{c}{\scriptsize{$\underline{88.38}$}\tiny{\color{lightgray}{$\pm0.761$}}}\\[2pt]
 \bottomrule
 \end{tabular}}
\end{center}
\end{table}

Also from Figure~\ref{fig:rtvsq}, which shows the tradeoff between runtime and perceptual quality of reconstructed images for several baselines, VML-MAP achieves better perceptual quality with lower runtime than other methods, highlighting its computational efficiency. Note that $\text{VML-MAP}_{pre}$ has an almost similar runtime as VML-MAP, as we compute $\mathrm{M}^{-1}$ in Equation~(\ref{eqn:preconditionedrhogradient}) with negligible overhead using SVD. 

\subsection{Noisy image restoration}\label{ssec:noisy-image-restoration}

Note that VML-MAP (Algorithm~\ref{algo:vmlmap}) is designed to handle the general case of noisy inverse problems, as evident from its application to the noisy inpainting example in Figure~\ref{fig:toyexample}. However, we observe that VML-MAP for noisy image restoration, using pre-trained image diffusion models~\cite{ddpm,dps}, involves optimization challenges at lower diffusion time steps, where the resulting final reconstructed image consists of sharp artifacts. 

 \begin{algorithm}[H]
\caption{$\textbf{VML-MAP}^{\bm \tau}$ / $\textbf{VML-MAP}^{\bm \tau}_{pre}$}
\label{algo:vmlmap2}
\begin{algorithmic}[1]
\STATE {\bfseries Input:} $\mathrm{D}_{\theta}(\cdot,\cdot)$, $\mathrm{H}$, $\mathbf{y}$, $\sigma_{\mathbf{y}}$, $\sigma(\cdot)$, $\{t_i\}_{i=0}^N$, $K$, $\gamma$, $\rho(\cdot)$, $\bm \tau$, $\{\hat{t}_j\}_{j=0}^S$ \ \ \ \ \ \ \ \COMMENTNOBRACES{Note that $t_{0} = \hat{t}_S = \bm \tau$, and $\hat{t}_0 = 0$}
\STATE {\bfseries Output:} $\mathbf{x}_{\hat{t}_0}$ 
\STATE Compute $\rvx_{t_0}$ with {\small{$\textbf{VML-MAP}$ / $\textbf{VML-MAP}_{pre}$}}
\STATE {\bfseries Initialize} $\mathbf{x}_{\hat{t}_S} = \rvx_{t_0}$
\FOR{$j = S$ \textbf{down to} $1$}
    \STATE $\mathbf{x}_{\hat{t}_{j-1}}$ $\ \gets\ $ {\small{$\textbf{DDIM\_UPDATE} (\rvx_{\hat{t}_j}, \hat{t}_j)$}}
\ENDFOR
\STATE {\bfseries Return} $\mathbf{x}_{\hat{t}_0}$
\end{algorithmic}
\end{algorithm}

While higher-order optimizers could mitigate such issues, another potential cause could be the pre-trained diffusion model not approximating the true score sufficiently well enough for the diffusion-modeled posterior modes to reflect the perceptual quality of images. Identifying and resolving these potential causes involves a substantial study of its own, and may also require developing efficient yet practically feasible higher-order optimizers. Therefore, we leave addressing it to future work. Meanwhile, in this work, we propose a simple and effective strategy for solving noisy image restoration. Specifically, we limit VML minimization (using VML-MAP) until a diffusion time threshold $\bm \tau$, and then follow the standard reverse diffusion sampling process with DDIM~\cite{ddim} update steps to generate a highly-probable image from the posterior. The full procedure $\text{VML-MAP}^{\bm \tau}$ is detailed in Algorithm~\ref{algo:vmlmap2}, including its preconditioned variant denoted as $\text{VML-MAP}^{\bm \tau}_{pre}$.

With the same experimental setup as in Section~\ref{ssec:noiseless-image-restoration}, we consider the case of noisy image restoration ($\sigma_\rvy = 0.05$). The quantitative results from Table~\ref{tab:exp2}, and the qualitative visualizations from Figures~\ref{fig:noisyexp1}, and~\ref{fig:noisyexp2} indicate the effectiveness of $\text{VML-MAP}^{\bm \tau}$ and $\text{VML-MAP}^{\bm \tau}_{pre}$ again, validating our simplified-VML minimization principle. See Appendix~\ref{ssec:implementationdetails2} for further implementation details. Also, see our extension of VML-MAP to non-linear operators via Latent Diffusion Models (LDMs)~\cite{ldm} in Appendix~\ref{ssec:ldmexperiments}.

\section{Conclusion}\label{sec:conclusion}
In this work, we proposed a novel inference-time MAP estimation strategy for solving inverse problems with a pre-trained unconditional diffusion model. The core of our approach relies on minimizing the proposed variational mode-seeking loss (VML) at each reverse diffusion time step, which steers the generated sample towards the MAP estimate (modes in practice). We derived VML in a closed form for linear inverse problems without any modeling approximations and provided a formal theoretical analysis of its convergence under mild assumptions. Based on VML minimization, we proposed empirically effective algorithms (VML-MAP, $\text{VML-MAP}^{\bm \tau}$, LatentVML-MAP) and demonstrated the effectiveness of our approach through extensive experiments on image restoration across multiple datasets.

\section{Limitations}\label{sec:limitations}
Although this paper primarily focuses on developing a principled framework and establishing the theoretical foundations of VML, the availability of a practically effective optimizer for minimizing the VML objective is equally critical. While ill-conditioned and/or non-linear degradation operators exacerbate these optimization challenges, higher-order methods and advanced optimization strategies could mitigate such issues and improve performance, as illustrated by our proposed preconditioner. However, it is of utmost importance that this additional performance gain must not come at the expense of prohibitive computational costs. Therefore, designing optimizers that are both efficient and practically feasible remains an essential direction for future work. 

\section*{Impact Statement}

This paper presents work whose goal is to advance the field of Machine Learning. There are many potential societal consequences of our work, none of which we feel must be specifically highlighted here.

\section*{Acknowledgements}

This work is part of the Marie Skłodowska-Curie Actions project \textit{MODELAIR}, funded by the European Commission under the Horizon Europe program through grant agreement no. 101072559. The computations and the data handling were enabled by resources provided by the National Academic Infrastructure for Supercomputing in Sweden (NAISS), partially funded by the Swedish Research Council through grant agreement no. 2022-06725.

\bibliography{example_paper}
\bibliographystyle{icml2026}
\nocite{lpips}
\nocite{fid}
\newpage
\appendix
\onecolumn
\section{Appendix}\label{sec:appendix}
\subsection{Tweedie's formula}\label{ssec:tweediesformula}
\begin{equation*}
\mathrm{D}(\rvx_t,t) = \mathbb{E}[\rvx_0|\rvx_t] = \int_{\rvx_0} \rvx_0 p(\rvx_0|\rvx_t) \mathrm{d} \rvx_0 = \rvx_t + \sigma^2_t \nabla_{\rvx_t} \log p(\rvx_t) \approx \rvx_t + \sigma^2_t S_{\theta}(\rvx_t,t)
\end{equation*}
where, $\mathrm{D}(\rvx_t,t)$ denotes the true denoiser, and $S_{\theta}(\cdot,\cdot)$ denotes the learned score function. 
\begin{proof}
{\small{
\begin{align*}
& \nabla_{\rvx_t} p(\rvx_t) = \nabla_{\rvx_t} \int_{\rvx_0} p(\rvx_t|\rvx_0) p(\rvx_0) \mathrm{d}\rvx_0 = \int_{\rvx_0} p(\rvx_0) \nabla_{\rvx_t} p(\rvx_t|\rvx_0) \mathrm{d}\rvx_0 \\
& \nabla_{\rvx_t} p(\rvx_t) = \int_{\rvx_0} p(\rvx_0) p(\rvx_t|\rvx_0) \frac{\rvx_0 - \rvx_t}{\sigma^2_t} \mathrm{d}\rvx_0\ \ \ \ \big\{\text{Note that}\ p(\rvx_t|\rvx_0) = \mathcal{N}(\rvx_0, \sigma^2_t \mathbf{I})\big\} \\
& \sigma^2_t\nabla_{\rvx_t} p(\rvx_t) = \left( \int_{\rvx_0} \rvx_0 p(\rvx_0) p(\rvx_t|\rvx_0) \mathrm{d}\rvx_0 \right)  -  \rvx_t p(\rvx_t)\\
& \rvx_t p(\rvx_t) + \sigma^2_t\nabla_{\rvx_t} p(\rvx_t) = \int_{\rvx_0} \rvx_0 p(\rvx_0) p(\rvx_t|\rvx_0) \mathrm{d}\rvx_0 = p(\rvx_t) \mathbb{E}[\rvx_0|\rvx_t] \\
& \rvx_t + \sigma^2_t\nabla_{\rvx_t} \log p(\rvx_t) = \mathbb{E}[\rvx_0|\rvx_t] \approx \rvx_t + \sigma^2_t S_{\theta}(\rvx_t,t)
\end{align*}}}
\end{proof}
\subsection{Covariance formula}\label{ssec:covarianceformula}
\begin{equation*}
    \Cov[\rvx_0|\rvx_t] = \int_{\rvx_0} (\rvx_0-\mathbb{E}[\rvx_0|\rvx_t])(\rvx_0-\mathbb{E}[\rvx_0|\rvx_t])^{\top} p(\rvx_0|\rvx_t) \mathrm{d}\rvx_0 = \sigma^2_t \frac{\partial \mathrm{D}(\rvx_t,t)}{\partial \rvx_t} \approx \sigma^2_t \frac{\partial \mathrm{D}_{\theta}(\rvx_t,t)}{\partial \rvx_t}
\end{equation*}
where, $\mathrm{D}(\rvx_t,t) = \mathbb{E}[\rvx_0|\rvx_t]$ is the true denoiser, and $\mathrm{D}_{\theta}(\rvx_t,t)$ is the learned denoiser.
\begin{proof}
{\small{
\begin{align*}
& \mathrm{D}(\rvx_t, t) = \mathbb{E}[\rvx_0|\rvx_t] = \int_{\rvx_0} p(\rvx_0|\rvx_t) \mathrm{d}\rvx_0 \\ 
& \frac{\partial \mathrm{D}(\rvx_t,t)}{\partial \rvx_t} = \frac{\partial }{\partial \rvx_t} \int_{\rvx_0} \rvx_0 p(\rvx_0|\rvx_t) =  \int_{\rvx_0} \rvx_0 p(\rvx_0) \frac{\partial }{\partial \rvx_t} \left( \frac{p(\rvx_t|\rvx_0)}{p(\rvx_t)} \right) \mathrm{d}\rvx_0 \\[4pt]
& \big\{\text{Note that}\ p(\rvx_t|\rvx_0) = \mathcal{N}(\rvx_0, \sigma^2_t \mathbf{I})\big\} \\[4pt]
& \frac{\partial \mathrm{D}(\rvx_t,t)}{\partial \rvx_t} = \int_{\rvx_0} \rvx_0 p(\rvx_0)   \left( \frac{p(\rvx_t)p(\rvx_t|\rvx_0) \frac{(\rvx_0-\rvx_t)^{\top}}{\sigma^2_t} - p(\rvx_t|\rvx_0) \frac{\partial p(\rvx_t)}{\partial \rvx_t} }{ p(\rvx_t)^2}  \right) \mathrm{d}\rvx_0 \\
& \frac{\partial \mathrm{D}(\rvx_t,t)}{\partial \rvx_t} = \left( \int_{\rvx_0} \rvx_0 p(\rvx_0)   \frac{p(\rvx_t|\rvx_0) \frac{(\rvx_0-\rvx_t)^{\top}}{\sigma^2_t} }{ p(\rvx_t)}  \mathrm{d}\rvx_0 \right) - \left( \int_{\rvx_0} \rvx_0 p(\rvx_0) \frac{p(\rvx_t|\rvx_0) \frac{\partial p(\rvx_t)}{\partial \rvx_t} }{p(\rvx_t)^2} \mathrm{d}\rvx_0 \right) \\
& \frac{\partial \mathrm{D}(\rvx_t,t)}{\partial \rvx_t} = \left( \int_{\rvx_0} \rvx_0 \frac{(\rvx_0-\rvx_t)^{\top}}{\sigma^2_t} p(\rvx_0|\rvx_t) \mathrm{d}\rvx_0 \right) - \left( \int_{\rvx_0} \rvx_0 \frac{\partial \log p(\rvx_t)}{\partial \rvx_t} p(\rvx_0|\rvx_t) \mathrm{d}\rvx_0 \right) \\
& \frac{\partial \mathrm{D}(\rvx_t,t)}{\partial \rvx_t} = \int_{\rvx_0}   \rvx_0 \left( \frac{(\rvx_0-\rvx_t)^{\top}}{\sigma^2_t} - \frac{\partial \log p(\rvx_t)}{\partial \rvx_t} \right) p(\rvx_0|\rvx_t) \mathrm{d}\rvx_0 \\
& \frac{\partial \mathrm{D}(\rvx_t,t)}{\partial \rvx_t} = \int_{\rvx_0}   \rvx_0  \frac{(\rvx_0-\mathrm{D}(\rvx_t,t))^{\top}}{\sigma^2_t}  p(\rvx_0|\rvx_t) \mathrm{d}\rvx_0 \\
& \sigma^2_t\frac{\partial \mathrm{D}(\rvx_t,t)}{\partial \rvx_t} = \int_{\rvx_0}   \rvx_0  (\rvx_0-\mathrm{D}(\rvx_t,t))^{\top}  p(\rvx_0|\rvx_t) \mathrm{d}\rvx_0 \\
& \sigma^2_t\frac{\partial \mathrm{D}(\rvx_t,t)}{\partial \rvx_t} =  \int_{\rvx_0}   (\rvx_0-\mathrm{D}(\rvx_t,t)) (\rvx_0-\mathrm{D}(\rvx_t,t))^{\top}  p(\rvx_0|\rvx_t) \mathrm{d}\rvx_0 \\
& \sigma^2_t\frac{\partial \mathrm{D}(\rvx_t,t)}{\partial \rvx_t} = \Cov[\rvx_0|\rvx_t] \approx \sigma^2_t\frac{\partial \mathrm{D}_{\theta}(\rvx_t,t)}{\partial \rvx_t}
\end{align*}}}
\end{proof}
\subsection{Lemmas}\label{ssec:lemmas}
\begin{lemma}\label{lemma:1}
$\int_{\rvx_0} \|\rvx_0\|^2 p(\rvx_0|\rvx_t) \mathrm{d}\rvx_0 = \mathrm{Tr}\left\{ \Cov[\rvx_0|\rvx_t] \right\} + \|\mathrm{D}(\rvx_t,t)\|^2$
\end{lemma}
\begin{proof}
{\small{
\begin{align*}
& \int_{\rvx_0} \|\rvx_0\|^2 p(\rvx_0|\rvx_t) \mathrm{d}\rvx_0 = \int_{\rvx_0} \mathrm{Tr}\left\{ \rvx_0 \rvx_0^{\top} \right\} p(\rvx_0|\rvx_t) \mathrm{d}\rvx_0 =  \mathrm{Tr}\left\{ \int_{\rvx_0}  \rvx_0 \rvx_0^{\top} p(\rvx_0|\rvx_t) \mathrm{d}\rvx_0 \right\}\\
& \int_{\rvx_0} \|\rvx_0\|^2 p(\rvx_0|\rvx_t) \mathrm{d}\rvx_0 = \mathrm{Tr}\left\{ \Cov[\rvx_0|\rvx_t] + \mathrm{D}(\rvx_t,t)\mathrm{D}(\rvx_t,t)^{\top} \right\}\\
& \int_{\rvx_0} \|\rvx_0\|^2 p(\rvx_0|\rvx_t) \mathrm{d}\rvx_0 = \mathrm{Tr}\left\{ \Cov[\rvx_0|\rvx_t] \right\} + \|\mathrm{D}(\rvx_t,t)\|^2
\end{align*}}}
\end{proof}
\begin{lemma}\label{lemma:2}
$\int_{\rvx_0} \|\mathrm{H}\rvx_0\|^2 p(\rvx_0|\rvx_t) \mathrm{d}\rvx_0 = \mathrm{Tr}\left\{ \mathrm{H}\Cov[\rvx_0|\rvx_t]\mathrm{H}^{\top} \right\} + \|\mathrm{H}\mathrm{D}(\rvx_t,t)\|^2$
\end{lemma}
\begin{proof}
{\small{
\begin{align*}
& \int_{\rvx_0} \|\mathrm{H}\rvx_0\|^2 p(\rvx_0|\rvx_t) \mathrm{d}\rvx_0 = \int_{\rvx_0} \mathrm{Tr}\left\{ \mathrm{H}\rvx_0 \rvx_0^{\top}\mathrm{H}^{\top} \right\} p(\rvx_0|\rvx_t) \mathrm{d}\rvx_0 \\
& \int_{\rvx_0} \|\mathrm{H}\rvx_0\|^2 p(\rvx_0|\rvx_t) \mathrm{d}\rvx_0 =  \mathrm{Tr}\left\{ \mathrm{H} \left( \int_{\rvx_0}  \rvx_0 \rvx_0^{\top} p(\rvx_0|\rvx_t) \mathrm{d}\rvx_0 \right) \mathrm{H}^{\top} \right\}\\
& \int_{\rvx_0} \|\mathrm{H}\rvx_0\|^2 p(\rvx_0|\rvx_t) \mathrm{d}\rvx_0 = \mathrm{Tr}\left\{  \mathrm{H}\Cov[\rvx_0|\rvx_t]\mathrm{H}^{\top} + \mathrm{H}\mathrm{D}(\rvx_t,t)\mathrm{D}(\rvx_t,t)^{\top}\mathrm{H}^{\top}  \right\}\\
& \int_{\rvx_0} \|\mathrm{H}\rvx_0\|^2 p(\rvx_0|\rvx_t) \mathrm{d}\rvx_0 = \mathrm{Tr}\left\{ \mathrm{H}\Cov[\rvx_0|\rvx_t]\mathrm{H}^{\top} \right\} + \|\mathrm{H}\mathrm{D}(\rvx_t,t)\|^2
\end{align*}}}    
\end{proof}
\subsection{Proofs}\label{ssec:proofs}
\begin{repeatthm}{prop:1}
The variational mode-seeking-loss (VML) at diffusion time $t$, for a non-linear degradation operator $\mathcal{A}$, measurement $\rvy$, and measurement noise variance $\sigma^2_{\rvy}$ is given by
{\small{\begin{align*}
\mathrm{VML}(\rvx_t,t) = \KL(& p(\rvx_0|\rvx_t)||p(\rvx_0|\rvy)) = - \log p(\rvx_t) - \frac{\|\mathrm{D}(\rvx_t,t)-\rvx_t\|^2}{2\sigma^2_t}  - \frac{1}{2\sigma^2_t} \mathrm{Tr}\left\{ \Cov[\rvx_0|\rvx_t] \right\} \\ 
&+ \frac{1}{2\sigma^2_\rvy} \left(  -2 \rvy^{\top} \int_{\rvx_0} \mathcal{A}(\rvx_0) p(\rvx_0|\rvx_t) \mathrm{d} \rvx_0   + \int_{\rvx_0} \|\mathcal{A}(\rvx_0)\|^2 p(\rvx_0|\rvx_t) \mathrm{d} \rvx_0  \right) + \mathrm{C}
\end{align*}}}
where $\mathrm{C}$ is a constant, independent of $\rvx_t$. $\mathrm{Tr}$ denotes the matrix trace, $\Cov$ denotes the covariance matrix, and $\mathrm{D}(\cdot,\cdot)$ denotes the denoiser.
\end{repeatthm}
\begin{proof}
{\small{
\begin{align*}
\KL(p(\rvx_0|\rvx_t) || p(\rvx_0|\rvy)) &= \int_{\rvx_0} p(\rvx_0|\rvx_t) \log \frac{p(\rvx_0|\rvx_t)}{p(\rvx_0|\rvy)} \mathrm{d}\rvx_0 \\
\KL(p(\rvx_0|\rvx_t) || p(\rvx_0|\rvy)) &= \int_{\rvx_0} p(\rvx_0|\rvx_t) \log \frac{p(\rvx_t|\rvx_0)\cancel{p(\rvx_0)}p(\rvy)}{p(\rvx_t)p(\rvy|\rvx_0)\cancel{p(\rvx_0)} } \mathrm{d}\rvx_0 \\
\KL(p(\rvx_0|\rvx_t) || p(\rvx_0|\rvy)) &= \log p(\rvy) - \log p(\rvx_t) + \int_{\rvx_0} p(\rvx_0|\rvx_t) \log \frac{p(\rvx_t|\rvx_0)}{p(\rvy|\rvx_0)} \mathrm{d}\rvx_0 \\
\KL(p(\rvx_0|\rvx_t) || p(\rvx_0|\rvy)) &= \log p(\rvy) - \log p(\rvx_t) + \left( \int_{\rvx_0} p(\rvx_0|\rvx_t) \log p(\rvx_t|\rvx_0) \mathrm{d}\rvx_0 \right) \\ &-\left(  \int_{\rvx_0} p(\rvx_0|\rvx_t) \log p(\rvy|\rvx_0) \mathrm{d}\rvx_0 \right) \\
\big\{\text{Note that }p(\rvx_t|\rvx_0) = \mathcal{N}&(\rvx_0, \sigma^2_t \mathbf{I})\text{ and }p(\rvy|\rvx_0) = \mathcal{N}(\mathcal{A}(\rvx_0), \sigma^2_{\rvy}\mathbf{I}).\text{ Also, let }\rvx_0 \in \mathbb{R}^n\text{ and }\rvy \in \mathbb{R}^{m} \big\} \\
\KL(p(\rvx_0|\rvx_t) || p(\rvx_0|\rvy)) &= - \log p(\rvx_t) - \frac{1}{2} \left( \int_{\rvx_0} p(\rvx_0|\rvx_t) \frac{\|\rvx_t-\rvx_0\|^2}{\sigma^2_t} \mathrm{d}\rvx_0 \right) \\ 
+ \frac{1}{2} \Biggl(  \int_{\rvx_0} &p(\rvx_0|\rvx_t) \frac{\| \rvy - \mathcal{A}(\rvx_0) \|^2}{\sigma^2_\rvy} \mathrm{d}\rvx_0 \Biggr) + \underbrace{\log p(\rvy) - \log \frac{\sigma^{n}_t}{\sigma^{m}_{\rvy}} - \frac{n-m}{2}\log 2\pi}_{\mathrm{C}} \\
\KL(p(\rvx_0|\rvx_t) || p(\rvx_0|\rvy)) &= - \log p(\rvx_t) - \frac{1}{2\sigma^2_t} \left( \|\rvx_t\|^2 - 2 \rvx_t^{\top} \mathrm{D}(\rvx_t,t) + \int_{\rvx_0} \|\rvx_0\|^2 p(\rvx_0|\rvx_t) \mathrm{d}\rvx_0 \right) \\ 
&+ \frac{1}{2\sigma^2_\rvy} \left(  -2 \rvy^{\top} \int_{\rvx_0} \mathcal{A}(\rvx_0) p(\rvx_0|\rvx_t) \mathrm{d}\rvx_0   + \int_{\rvx_0} \|\mathcal{A}(\rvx_0)\|^2 p(\rvx_0|\rvx_t) \mathrm{d}\rvx_0  \right) + \mathrm{C}\\
\big\{\text{By Lemma}~\ref{lemma:1} \big\}&\\
\KL(p(\rvx_0|\rvx_t) || p(\rvx_0|\rvy)) &= - \log p(\rvx_t) - \frac{1}{2\sigma^2_t} \Biggl( \|\rvx_t\|^2 - 2 \rvx_t^{\top} \mathrm{D}(\rvx_t,t) + \mathrm{Tr}\left\{ \Cov[\rvx_0|\rvx_t] \right\} \\
+ \|\mathrm{D}(\rvx_t,t)\|^2  \Biggr) &+ \frac{1}{2\sigma^2_\rvy} \left(  -2 \rvy^{\top} \int_{\rvx_0} \mathcal{A}(\rvx_0) p(\rvx_0|\rvx_t) \mathrm{d}\rvx_0   + \int_{\rvx_0} \|\mathcal{A}(\rvx_0)\|^2 p(\rvx_0|\rvx_t) \mathrm{d}\rvx_0  \right) + \mathrm{C}\\
\KL(p(\rvx_0|\rvx_t) || p(\rvx_0|\rvy)) &= - \log p(\rvx_t) - \frac{\|\mathrm{D}(\rvx_t,t)-\rvx_t\|^2}{2\sigma^2_t}  - \frac{1}{2\sigma^2_t} \mathrm{Tr}\left\{ \Cov[\rvx_0|\rvx_t] \right\} \\ 
&+ \frac{1}{2\sigma^2_\rvy} \left(  -2 \rvy^{\top} \int_{\rvx_0} \mathcal{A}(\rvx_0) p(\rvx_0|\rvx_t) \mathrm{d}\rvx_0   + \int_{\rvx_0} \|\mathcal{A}(\rvx_0)\|^2 p(\rvx_0|\rvx_t) \mathrm{d}\rvx_0  \right) + \mathrm{C} 
\end{align*}}}
\end{proof}
\begin{repeatthm}{prop:2}
The variational mode-seeking-loss (VML) at diffusion time $t$, for a linear degradation matrix $\mathrm{H}$, measurement $\rvy$, and measurement noise variance $\sigma^2_{\rvy}$ is given by
{\small{
\begin{align*}
\mathrm{VML}(\rvx_t,t) = \KL(p(\rvx_0|\rvx_t)||&p(\rvx_0|\rvy)) = - \log p(\rvx_t) - \frac{\|\mathrm{D}(\rvx_t,t)-\rvx_t\|^2}{2\sigma^2_t}  - \frac{1}{2\sigma^2_t} \mathrm{Tr}\left\{ \Cov[\rvx_0|\rvx_t] \right\} \\ 
&+  \underbrace{\frac{\|\rvy-\mathrm{H}\mathrm{D}(\rvx_t,t)\|^2}{2\sigma^2_\rvy}}_{\text{measurement consistency}} + \frac{1}{2\sigma^2_\rvy} \mathrm{Tr}\left\{ \mathrm{H}\Cov[\rvx_0|\rvx_t]\mathrm{H}^{\top} \right\} + \mathrm{C}
\end{align*}}}
where $\mathrm{C}$ is a constant, independent of $\rvx_t$. $\mathrm{Tr}$ denotes the matrix trace, $\Cov$ denotes the covariance matrix, and $\mathrm{D}(\cdot,\cdot)$ denotes the denoiser.
\end{repeatthm}
\begin{proof}
{\small{\begin{align*}
\text{Substituting }\mathcal{A}\text{ with }\mathrm{H}\text{ in }&\text{Proposition}~\ref{prop:1} \\
\KL(p(\rvx_0|\rvx_t) || p(\rvx_0|\rvy)) &= - \log p(\rvx_t) - \frac{\|\mathrm{D}(\rvx_t,t)-\rvx_t\|^2}{2\sigma^2_t}  - \frac{1}{2\sigma^2_t} \mathrm{Tr}\left\{ \Cov[\rvx_0|\rvx_t] \right\} \\ 
&+ \frac{1}{2\sigma^2_\rvy} \left(  -2 \rvy^{\top} \int_{\rvx_0} \mathrm{H}\rvx_0 p(\rvx_0|\rvx_t) \mathrm{d}\rvx_0   + \int_{\rvx_0} \|\mathrm{H}\rvx_0\|^2 p(\rvx_0|\rvx_t) \mathrm{d}\rvx_0  \right) + \mathrm{C}\\
\big\{\text{By Lemma}~\ref{lemma:2}\big\} &\\
\KL(p(\rvx_0|\rvx_t) || p(\rvx_0|\rvy)) &= - \log p(\rvx_t) - \frac{\|\mathrm{D}(\rvx_t,t)-\rvx_t\|^2}{2\sigma^2_t}  - \frac{1}{2\sigma^2_t} \mathrm{Tr}\left\{ \Cov[\rvx_0|\rvx_t] \right\} \\ 
&+ \frac{1}{2\sigma^2_\rvy} \left(  -2 \rvy^{\top} \mathrm{H}\mathrm{D}(\rvx_t,t)   + \mathrm{Tr}\left\{ \mathrm{H}\Cov[\rvx_0|\rvx_t]\mathrm{H}^{\top} \right\} + \|\mathrm{H}\mathrm{D}(\rvx_t,t)\|^2  \right) + \mathrm{C}\\
\KL(p(\rvx_0|\rvx_t) || p(\rvx_0|\rvy)) &= - \log p(\rvx_t) - \frac{\|\mathrm{D}(\rvx_t,t)-\rvx_t\|^2}{2\sigma^2_t}  - \frac{1}{2\sigma^2_t} \mathrm{Tr}\left\{ \Cov[\rvx_0|\rvx_t] \right\} \\ 
&+ \frac{\|\rvy-\mathrm{H}\mathrm{D}(\rvx_t,t)\|^2}{2\sigma^2_\rvy} + \frac{1}{2\sigma^2_\rvy} \mathrm{Tr}\left\{ \mathrm{H}\Cov[\rvx_0|\rvx_t]\mathrm{H}^{\top} \right\} + \mathrm{C}\\
\big\{\text{Note that }\mathrm{C} = \log p(\rvy) & - \log \frac{\sigma^{n}_t}{\sigma^{m}_{\rvy}} - \frac{n-m}{2}\log 2\pi,\ \text{see the proof of Proposition}~\ref{prop:1} \big\}
\end{align*}}}
\end{proof}
\textbf{Simplified-VML gradient.} The gradient of the simplified-VML, i.e., $\mathrm{VML}_\mathrm{S}(\cdot,\cdot)$ for a linear degradation matrix $\mathrm{H}$ is 
{\small{
\begin{align*}
\mathrm{VML}_{\mathrm{S}}(\rvx_t,t) = -\log p(\rvx_t) - \frac{\| \mathrm{D}(\rvx_t,t) - \rvx_t\|^2 }{2\sigma^2_t} + \frac{\| \rvy - \mathrm{H}\mathrm{D}(\rvx_t,t) \|^2}{2\sigma^2_\rvy}
\end{align*}}}
{\small{
\begin{align*}
\nabla_{\rvx_t} \mathrm{VML}_{\mathrm{S}}(\rvx_t,t) = \underbrace{ - \frac{\partial \mathrm{D}^{\top}(\rvx_t,t)}{\partial \rvx_t} \frac{\mathrm{H}^{\top}(\rvy-\mathrm{H}\mathrm{D}(\rvx_t,t))}{\sigma^2_\rvy} }_{\textit{measurement consistency gradient}} \underbrace{ - \frac{\partial \mathrm{D}^{\top}(\rvx_t,t)}{\partial \rvx_t} \frac{(\mathrm{D}(\rvx_t,t)-\rvx_t)}{\sigma^2_t} }_{\textit{prior gradient}}
\end{align*}}}
\begin{proof}
{\small{
\begin{align*}
\nabla_{\rvx_t} \mathrm{VML}_\mathrm{S}(\rvx_t,t) &= \left\{ -\nabla_{\rvx_t} \log p(\rvx_t) \right\} - \left\{ \nabla_{\rvx_t} \frac{\|\mathrm{D}(\rvx_t,t)-\rvx_t\|^2}{2\sigma^2_t} \right\} + \left\{ \nabla_{\rvx_t} \frac{\|\rvy-\mathrm{H}\mathrm{D}(\rvx_t,t)\|^2}{2\sigma^2_\rvy} \right\} \\
\nabla_{\rvx_t} \mathrm{VML}_\mathrm{S}(\rvx_t,t)  &= \left\{ - \frac{\mathrm{D}(\rvx_t,t)-\rvx_t}{\sigma^2_t} \right\} - \left\{ \left( \frac{\partial \mathrm{D}^{\top}(\rvx_t,t)}{\partial \rvx_t} \frac{(\mathrm{D}(\rvx_t,t)-\rvx_t)}{\sigma^2_t} \right) - \frac{\mathrm{D}(\rvx_t,t)-\rvx_t}{\sigma^2_t} \right\} \\
&+ \left\{ - \frac{\partial \mathrm{D}^{\top}(\rvx_t,t)}{\partial \rvx_t} \frac{\mathrm{H}^{\top}(\rvy-\mathrm{H}\mathrm{D}(\rvx_t,t))}{\sigma^2_\rvy} \right\} \\
\nabla_{\rvx_t} \mathrm{VML}_{\mathrm{S}}(\rvx_t,t) &= - \frac{\partial \mathrm{D}^{\top}(\rvx_t,t)}{\partial \rvx_t} \frac{\mathrm{H}^{\top}(\rvy-\mathrm{H}\mathrm{D}(\rvx_t,t))}{\sigma^2_\rvy}  - \frac{\partial \mathrm{D}^{\top}(\rvx_t,t)}{\partial \rvx_t} \frac{(\mathrm{D}(\rvx_t,t)-\rvx_t)}{\sigma^2_t}
\end{align*}}}
\end{proof}

\section{Exclusion of higher-order terms and the convergence of VML}\label{sec:ignoringcovvml}

The true denoiser $\mathrm{D}(\rvx_t,t) = \mathbb{E}[\rvx_0|\rvx_t]$ is related to the true score function $\nabla_{\rvx_t}\log p(\rvx_t)$ by Tweedie's formula $\mathrm{D}(\rvx_t,t) = \rvx_t + \sigma^2_t \nabla_{\rvx_t} \log p(\rvx_t)$ (see Appendix~\ref{ssec:tweediesformula}). Applying the derivative to this equation gives $\frac{\partial \mathrm{D}(\rvx_t,t)}{\partial \rvx_t} = \mathbf{I} + \sigma^2_t \nabla^2_{\rvx_t}\log p(\rvx_t) = \frac{1}{\sigma^2_t} \Cov[\rvx_0|\rvx_t]$ (see Appendix~\ref{ssec:covarianceformula}). With these reformulations, the higher-order terms of VML can be equivalently expressed in terms of $\nabla_{\rvx_t}\log p(\rvx_t)$ and $\nabla^2_{\rvx_t}\log p(\rvx_t)$ as follows.\\

\textbf{Reformulating the higher-order terms of VML:} From Proposition~\ref{prop:2}, the higher-order terms of the VML (involving {\small{$\Cov[\rvx_0|\rvx_t]$}}) denoted by $\mathrm{VML}_{\mathrm{High}}(\cdot,\cdot)$, for a linear degradation matrix $\mathrm{H}$ is 
{\small{\begin{align*}
\mathrm{VML}_{\mathrm{High}}(\rvx_t,t) = - \frac{1}{2\sigma^2_t} \mathrm{Tr}\left\{ \Cov[\rvx_0|\rvx_t] \right\} + \frac{1}{2\sigma^2_\rvy} \mathrm{Tr}\left\{ \mathrm{H}\Cov[\rvx_0|\rvx_t]\mathrm{H}^{\top} \right\}
\end{align*}}}
where, $\rvx_t \in \mathbb{R}^n\ \forall\ t \geq 0$, $\rvy \in \mathbb{R}^m$ and $\mathrm{H} \in \mathbb{R}^{m \times n}$. Reformulating $\mathrm{VML}_{\mathrm{High}}(\cdot,\cdot)$ in terms of $\nabla_{\rvx_t}\log p(\rvx_t)$, and $\nabla^2_{\rvx_t}\log p(\rvx_t)$ results in the following
{\small{\begin{align*}
\mathrm{VM}&\mathrm{L}_{\mathrm{High}}(\rvx_t,t) = - \frac{\sigma^2_t}{2} \mathrm{Tr}\left\{ \nabla^2_{\rvx_t}\log p(\rvx_t) \right\} + \frac{\sigma^4_t}{2\sigma^2_\rvy} \mathrm{Tr}\left\{ \mathrm{H} \nabla^2_{\rvx_t} \log p(\rvx_t) \mathrm{H}^{\top} \right\} + \frac{\sigma^2_t}{2\sigma^2_\rvy} \mathrm{Tr}\left\{ \mathrm{H} \mathrm{H}^{\top} \right\} + \mathrm{C}_{\mathrm{VML}_{\mathrm{High}}}
\end{align*}}}
where, $\mathrm{C}_{\mathrm{VML}_{\mathrm{High}}} = - \frac{n}{2}$.\\

\begin{prop}\label{prop:3}
Let $p_{0}(\cdot)$ denote the input data distribution and $p_t(\cdot)$ denote the intermediate marginal distributions of a diffusion process for $t > 0$. Let $\exists\ \tau > 0, d > 0$ such that $p_t \in C^2$ (twice continuously differentiable) and $ \| \rvx \| \leq d$ $\forall$ $t < \tau$ (i.e., $\forall\ t < \tau$, $\rvx$ lies in a compact ball, and $p_t \in C^2$). The function $\mathrm{VML}_{\mathrm{High}_t}(\rvx)$ denoting the higher-order terms (involving $\Cov[\rvx_0|\rvx_t]$) of VML, for a linear degradation operator matrix $\mathrm{H}$, measurement $\rvy$, and measurement noise variance $\sigma^2_{\rvy}$ converges uniformly to $\mathrm{C}_{\mathrm{VML}_{\mathrm{High}}}$ in the limit as $t \to 0$. (Note that $\rvx_t \in \mathbb{R}^n\ \forall\, t \geq 0$, $\rvy \in \mathbb{R}^m$ and $\mathrm{C}_{\mathrm{VML}_{\mathrm{High}}} = -\frac{n}{2}$ as previously mentioned)
\begin{small}
\begin{align*}
\mathrm{VM}\mathrm{L}_{\mathrm{High}_t}(\rvx) &= - \frac{\sigma^2_t}{2} \mathrm{Tr}\left\{ \nabla^2_{\rvx}\log p_t(\rvx) \right\} + \frac{\sigma^4_t}{2\sigma^2_\rvy} \mathrm{Tr}\left\{ \mathrm{H} \nabla^2_{\rvx} \log p_t(\rvx) \mathrm{H}^{\top} \right\} \\
& + \frac{\sigma^2_t}{2\sigma^2_\rvy} \mathrm{Tr}\left\{ \mathrm{H} \mathrm{H}^{\top} \right\} + \mathrm{C}_{\mathrm{VML}_{\mathrm{High}}}\\\\
\text{and }\mathrm{unif} \lim_{t\to 0}&\ \mathrm{VML}_{\mathrm{High}_t}(\rvx) = \mathrm{C}_{\mathrm{VML}_{\mathrm{High}}}\ \ \forall\ \ \rvx\ \ s.t.\ \ \| \rvx \| \leq d
\end{align*}
\end{small}
\end{prop}
\begin{proof}[\textbf{proof sketch}]
For $t < \tau$, $p_t \in C^2$ implies $\nabla^2_{\rvx}\log p_t(\rvx)$ is continuous, which further implies continuity of its component functions, i.e., $\frac{\partial^2 \log p_t(\rvx)}{\partial \rvx^{(i)} \partial \rvx^{(j)}}$ $\forall\ i,j\ \in\ [1,2,\dots,n]$. Since $\rvx$ lies in a compact ball (i.e., $\|\rvx \| \leq d$), it implies that the component functions are bounded $\forall$ $t < \tau$, which further implies boundedness of $\mathrm{Tr}\left\{ \nabla^2_{\rvx}\log p_t(\rvx) \right\}$, and $\mathrm{Tr}\left\{ \mathrm{H} \nabla^2_{\rvx} \log p_t(\rvx) \mathrm{H}^{\top} \right\}$. Note that $\mathrm{Tr}\left\{ \mathrm{H} \mathrm{H}^{\top} \right\}$ is also bounded. With $\sigma_t \to 0$, as $t \to 0$, it is apparent that $\mathrm{VML}_{\mathrm{High}_t}$ converges uniformly (since the bounds are global and hold for all $\rvx$ s.t $\|\rvx \| \leq d$) to $\mathrm{C}_{\mathrm{VML}_{\mathrm{High}}}$.
\end{proof}
\textbf{Reformulating the VML:} From Proposition~\ref{prop:2}, the {\small{$\mathrm{VML}(\cdot,\cdot)$}} for a linear degradation matrix $\mathrm{H}$ is 
{\small{\begin{align*}
\mathrm{VML}(\rvx_t,t) &= \KL(p(\rvx_0|\rvx_t)||p(\rvx_0|\rvy)) = - \log p(\rvx_t) - \frac{\|\mathrm{D}(\rvx_t,t)-\rvx_t\|^2}{2\sigma^2_t}  - \frac{1}{2\sigma^2_t} \mathrm{Tr}\left\{ \Cov[\rvx_0|\rvx_t] \right\} \\ 
&+  \underbrace{\frac{\|\rvy-\mathrm{H}\mathrm{D}(\rvx_t,t)\|^2}{2\sigma^2_\rvy}}_{\text{measurement consistency}} + \frac{1}{2\sigma^2_\rvy} \mathrm{Tr}\left\{ \mathrm{H}\Cov[\rvx_0|\rvx_t]\mathrm{H}^{\top} \right\} + \mathrm{C}
\end{align*}}}
where, $\rvx_t \in \mathbb{R}^n\ \forall\ t \geq 0$, $\rvy \in \mathbb{R}^m$ and $\mathrm{C} = \log p(\rvy) - \log \frac{\sigma^{n}_t}{\sigma^{m}_{\rvy}} - \frac{n-m}{2}\log 2\pi$ (see the proof of Proposition~\ref{prop:2}). Reformulating $\mathrm{VML}(\cdot,\cdot)$ in terms of $\nabla_{\rvx_t}\log p(\rvx_t)$, and $\nabla^2_{\rvx_t}\log p(\rvx_t)$ gives
{\small{\begin{align*}
\mathrm{VML}(\rvx_t,t) & = \KL(p(\rvx_0|\rvx_t)||p(\rvx_0|\rvy)) = - \log p(\rvx_t) - \frac{\sigma^2_t}{2} \|\nabla_{\rvx_t}\log p(\rvx_t)\|^2  - \frac{\sigma^2_t}{2} \mathrm{Tr}\left\{ \nabla^2_{\rvx_t}\log p(\rvx_t) \right\}\\ 
&+ \underbrace{\frac{\|\rvy-\mathrm{H}\mathrm{D}(\rvx_t,t)\|^2}{2\sigma^2_\rvy}}_{\text{measurement consistency}} + \frac{\sigma^4_t}{2\sigma^2_\rvy} \mathrm{Tr}\left\{ \mathrm{H} \nabla^2_{\rvx_t} \log p(\rvx_t) \mathrm{H}^{\top} \right\} + \frac{\sigma^2_t}{2\sigma^2_\rvy} \mathrm{Tr}\left\{ \mathrm{H} \mathrm{H}^{\top} \right\} + \log p(\rvy) - n \log \sigma_t + \mathrm{C}_{\mathrm{VML}}
\end{align*}}}
where, $\mathrm{C}_{\mathrm{VML}} = -\frac{n}{2} + m \log \sigma_{\rvy} - \frac{n-m}{2}\log 2\pi$.
\\

\begin{prop}\label{prop:4}
Let $p_{0}(\cdot)$ denote the input data distribution and $p_t(\cdot)$ denote the intermediate marginal distributions of a diffusion process for $t > 0$. Let $\exists\ \tau > 0, d > 0$, such that $p_t \in C^2$ (twice continuously differentiable) and $ \| \rvx \| \leq d$ $\forall$ $t < \tau$ (In other words, $\exists$ $\tau > 0$, $d > 0$, such that $p_t \in C^2$, and $ \| \rvx_t \| \leq d$  $\forall t < \tau$ where $\rvx_t \in \mathcal{M}_t$ i.e., the intermediate diffusion manifold at time $t$). Assuming sufficient conditions for $\lim_{t\to 0}\ \log p_t(\rvx) = \log p_0(\rvx)$, the function $\mathrm{VML}_t(\rvx) + n \log \sigma_t$, for a linear degradation operator matrix $\mathrm{H}$, measurement $\rvy$, and measurement noise variance $\sigma^2_{\rvy}$ converges pointwise to $-\log p_{0}(\rvx|\rvy) + \hat{\mathrm{C}}_{\mathrm{VML}}$ in the limit as $t \to 0$. (Note that $\rvx_t \in \mathbb{R}^n\ \forall\, t \geq 0$, $\rvy \in \mathbb{R}^m$ and $\hat{\mathrm{C}}_{\mathrm{VML}} = -\frac{n}{2} - \frac{n}{2}\log 2\pi $)
{\small{
\begin{align*}
\mathrm{VML}_t(\rvx) &= - \log p_t(\rvx) - \frac{\sigma^2_t}{2} \|\nabla_{\rvx} \log p_t(\rvx)\|^2 - \frac{\sigma^2_t}{2} \mathrm{Tr}\left\{ \nabla^2_{\rvx} \log p_t(\rvx) \right\} + \frac{\|\rvy-\mathrm{H}\mathrm{D}(\rvx,t)\|^2}{2\sigma^2_\rvy} \\
&+ \frac{\sigma^4_t}{2\sigma^2_\rvy} \mathrm{Tr}\left\{ \mathrm{H} \nabla^2_{\rvx} \log p_t(\rvx) \mathrm{H}^{\top} \right\} + \frac{\sigma^2_t}{2\sigma^2_\rvy} \mathrm{Tr}\left\{ \mathrm{H} \mathrm{H}^{\top} \right\} + \log p(\rvy) - n \log \sigma_t + \mathrm{C}_{\mathrm{VML}}\\\\
\text{and }\lim_{t\to 0}&\ \mathrm{VML}_t(\rvx) + n \log \sigma_t = -\log p_{0}(\rvx|\rvy) + \hat{\mathrm{C}}_{\mathrm{VML}}\ \ \forall\ \ \rvx\ \ s.t.\ \ \| \rvx \| \leq d
\end{align*}}}
\end{prop}
\begin{proof}[\textbf{proof sketch}]
It suffices to show that {\small{$\lim_{t\to 0}\ \{\mathrm{VML}_t(\rvx) + n \log \sigma_t + \log p_{0}(\rvx|\rvy)-\hat{\mathrm{C}}_{\mathrm{VML}} \} = 0\ \ \forall\ \ \rvx\ \ s.t.\ \ \| \rvx \| \leq d$}}. 
{\small{\begin{align*}
    &\mathrm{VML}_t(\rvx) + n \log \sigma_t + \log p_{0}(\rvx|\rvy) - \hat{\mathrm{C}}_{\mathrm{VML}}\\ 
    &=  \mathrm{VML}_t(\rvx) + n \log \sigma_t + \log p_{0}(\rvy|\rvx) + \log p_{0}(\rvx) - \log p(\rvy) - \mathrm{C}_{\mathrm{VML}} + \left\{ m \log \sigma_{\rvy} + \frac{m}{2}\log 2\pi \right\} \\
    &= \underbrace{\left\{  - \frac{\sigma^2_t}{2} \|\nabla_{\rvx} \log p_t(\rvx)\|^2
    - \frac{\sigma^2_t}{2} \mathrm{Tr}\left\{ \nabla^2_{\rvx} \log p_t(\rvx) \right\} + \frac{\sigma^4_t}{2\sigma^2_\rvy} \mathrm{Tr}\left\{ \mathrm{H} \nabla^2_{\rvx} \log p_t(\rvx) \mathrm{H}^{\top} \right\} + \frac{\sigma^2_t}{2\sigma^2_\rvy} \mathrm{Tr}\left\{ \mathrm{H} \mathrm{H}^{\top} \right\} \right\}}_{T_A}\\
    &+ \underbrace{\left\{ -\log p_t(\rvx) + \log p_{0}(\rvx) \right\}}_{T_B} + \underbrace{\left\{ m \log \sigma_y + \frac{m}{2}\log 2\pi + \frac{\|\rvy-\mathrm{H}\mathrm{D}(\rvx,t)\|^2}{2\sigma^2_\rvy} + \log p_{0}(\rvy|\rvx) \right\}}_{T_C}
\end{align*}}}
To show that {\small{$\lim_{t\to 0} \{\mathrm{VML}_t(\rvx) + n \log \sigma_t + \log p_{0}(\rvx|\rvy) - \hat{\mathrm{C}}_{\mathrm{VML}}\} = 0\ \ \forall\ \ \rvx\ \ s.t.\ \ \| \rvx \| \leq d$}}, we need to show that {\small{$\lim_{t \to 0}T_A = \lim_{t \to 0}T_B = \lim_{t \to 0}T_C = 0$}}. Under sufficient conditions assumed for $\lim_{t \to 0}\ \log p_t = \log p_0$, it implies that {\small{$ \lim_{t \to 0} T_B = 0 $}}. 
Considering {\small{$T_A$}}: as $p_t(\rvx) \in C^2$ and $\rvx$ lies in a compact set, it implies that {\small{$\nabla_{\rvx} \log p_t(\rvx)$}} and {\small{$\nabla^2_{\rvx} \log p_t(\rvx)$}} are bounded for all $t < \tau$. With {\small{$\sigma_t \to 0$}} as {\small{$t \to 0$}}, {\small{$\lim_{t \to 0}T_A = 0$}}. Considering {\small{$T_C$}}: Note that {\small{$p_0(\rvy|\rvx) = \mathcal{N}(\mathrm{H}\rvx, \sigma^2_{\rvy}\mathbf{I})$}} (see Equation~\ref{eqn:invprob}) and {\scriptsize{$\log p_0(\rvy|\rvx) = -m \log \sigma_y - \frac{m}{2} \log 2\pi -\frac{\|\rvy-\mathrm{H}\rvx\|^2}{2\sigma^2_{\rvy}}$}}. It can be seen that, as {\small{$t \to 0$}}, {\small{$\mathrm{D}(\rvx,t) \to \rvx$}}, since {\small{$\mathrm{D}(\rvx,t) = \rvx + \sigma^2_t \nabla_{\rvx}\log p_t(\rvx)$}} (Appendix~\ref{ssec:tweediesformula}). This further implies {\small{$\lim_{t \to 0} T_c = 0$}}.
\end{proof}
\textbf{Reformulating the simplified-VML:} From Equation~(\ref{eqn:svml}), $\mathrm{VML}_{\mathrm{S}}(\cdot,\cdot)$ for a linear degradation matrix $\mathrm{H}$ is 
{\small{\begin{align*}
\mathrm{VML}_{\mathrm{S}}(\rvx_t,t) = - \log p(\rvx_t) - \frac{\|\mathrm{D}(\rvx_t,t)-\rvx_t\|^2}{2\sigma^2_t}  +  \underbrace{\frac{\|\rvy-\mathrm{H}\mathrm{D}(\rvx_t,t)\|^2}{2\sigma^2_\rvy}}_{\text{measurement consistency}} + \mathrm{C}
\end{align*}}}
where, $\rvx_t \in \mathbb{R}^n\ \forall\ t \geq 0$, $\rvy \in \mathbb{R}^m$ and $\mathrm{C} = \log p(\rvy) - \log \frac{\sigma^{n}_t}{\sigma^{m}_{\rvy}} - \frac{n-m}{2}\log 2\pi$. Reformulating $\mathrm{VML}_{\mathrm{S}}(\cdot,\cdot)$ in terms of $\nabla_{\rvx_t}\log p(\rvx_t)$, and $\nabla^2_{\rvx_t}\log p(\rvx_t)$ gives
{\small{\begin{align*}
\mathrm{VM}&\mathrm{L}_{\mathrm{S}}(\rvx_t,t) = - \log p(\rvx_t) - \frac{\sigma^2_t}{2} \|\nabla_{\rvx_t}\log p(\rvx_t)\|^2 + \underbrace{\frac{\|\rvy-\mathrm{H}\mathrm{D}(\rvx_t,t)\|^2}{2\sigma^2_\rvy}}_{\text{measurement consistency}} + \log p(\rvy) - n \log \sigma_t + \mathrm{C}_{\mathrm{VML}_{\mathrm{S}}}
\end{align*}}} 
where, $\mathrm{C}_{\mathrm{VML}_{\mathrm{S}}} = m \log \sigma_{\rvy} - \frac{n-m}{2}\log 2\pi$.
\\

\begin{corollary}\label{corollary:1}
Let $p_{0}(\cdot)$ denote the input data distribution and $p_t(\cdot)$ denote the intermediate marginal distributions of a diffusion process for $t > 0$. Let $\exists\ \tau > 0, d > 0$ such that $p_t \in C^1$ (once continuously differentiable) and $ \| \rvx \| \leq d$ $\forall$ $t < \tau$ (i.e., $\forall\ t < \tau$, $\rvx$ lies in a compact ball, and $p_t \in C^1$). Assuming sufficient conditions for $\lim_{t\to 0}\ \log p_t(\rvx) = \log p_0(\rvx)\ \forall\ \rvx$, the function given by $\mathrm{VML}_{\mathrm{S}_t}(\rvx)+ n \log \sigma_t$, (where $\mathrm{VML}_{\mathrm{S}_t}(\rvx)$ denotes the Simplified-VML) for a linear degradation operator matrix $\mathrm{H}$, measurement $\rvy$, and measurement noise variance $\sigma^2_{\rvy}$ converges pointwise to $-\log p_{0}(\rvx|\rvy) + \hat{\mathrm{C}}_{\mathrm{VML}_{\mathrm{S}}}$ in the limit as $t \to 0$. (Note that $\rvx_t \in \mathbb{R}^n\ \forall\, t \geq 0$, $\rvy \in \mathbb{R}^m$ and $\hat{\mathrm{C}}_{\mathrm{VML}_{\mathrm{S}}} = - \frac{n}{2}\log 2\pi $)
\begin{small}
\begin{align*}
\mathrm{VML}_{\mathrm{S}_t}(\rvx) &= - \log p_t(\rvx) - \frac{\sigma^2_t}{2} \|\nabla_{\rvx}\log p_t(\rvx)\|^2 + \frac{\|\rvy-\mathrm{H}\mathrm{D}(\rvx,t)\|^2}{2\sigma^2_\rvy} + \log p(\rvy) - n \log \sigma_t + \mathrm{C}_{\mathrm{VML}_{\mathrm{S}}}\\\\
\text{and }\lim_{t\to 0}&\ \mathrm{VML}_{\mathrm{S}_t}(\rvx) + n \log \sigma_t = -\log p_{0}(\rvx|\rvy) + \hat{\mathrm{C}}_{\mathrm{VML}_{\mathrm{S}}}\ \ \forall\ \ \rvx\ \ s.t.\ \ \| \rvx \| \leq d
\end{align*}
\end{small}
\end{corollary}
\begin{proof}[\textbf{proof sketch}]
    By arguments similar to those in the proof of Proposition~\ref{prop:4}
\end{proof}
\begin{corollary}\label{corollary:2}
Let $p_{0}(\cdot)$ denote the input data distribution and $p_t(\cdot)$ denote the intermediate marginal distributions of a diffusion process for $t > 0$. Let $\exists\ \tau > 0, d > 0$ such that $p_t \in C^1$ (once continuously differentiable) and $ \| \rvx \| \leq d$ $\forall$ $t < \tau$ (i.e., $\forall\ t < \tau$, $\rvx$ lies in a compact ball, and $p_t \in C^1$). Assuming sufficient conditions for $\lim_{t\to 0}\ \log p_t(\rvx) = \log p_0(\rvx)\ \forall\ \rvx$, the function given by $\mathrm{VML}^{\rho}_{\mathrm{S}_t}(\rvx)+ n \log \sigma_t$, (where $\mathrm{VML}^{\rho}_{\mathrm{S}_t}(\rvx)$ denotes the reweighted simplified-VML with prior weight schedule $\rho_t$, such that $\lim_{t \to 0}\rho_t = 1$) for a linear degradation operator matrix $\mathrm{H}$, measurement $\rvy$, and measurement noise variance $\sigma^2_{\rvy}$ converges pointwise to $-\log p_{0}(\rvx|\rvy) + \hat{\mathrm{C}}_{\mathrm{VML}_{\mathrm{S}}}$ in the limit as $t \to 0$. (Note that $\rvx_t \in \mathbb{R}^n\ \forall\, t \geq 0$, $\rvy \in \mathbb{R}^m$ and $\hat{\mathrm{C}}_{\mathrm{VML}_{\mathrm{S}}} = - \frac{n}{2}\log 2\pi $)
\begin{small}
\begin{align*}
\mathrm{VML}^{\rho}_{\mathrm{S}_t}(\rvx) &= - \rho_t \, \log p_t(\rvx) - \rho_t \, \frac{\sigma^2_t}{2} \|\nabla_{\rvx}\log p_t(\rvx)\|^2 + \frac{\|\rvy-\mathrm{H}\mathrm{D}(\rvx,t)\|^2}{2\sigma^2_\rvy} + \log p(\rvy) - n \log \sigma_t + \mathrm{C}_{\mathrm{VML}_{\mathrm{S}}}\\\\
\text{and }\lim_{t\to 0}&\ \mathrm{VML}^{\rho}_{\mathrm{S}_t}(\rvx) + n \log \sigma_t = -\log p_{0}(\rvx|\rvy) + \hat{\mathrm{C}}_{\mathrm{VML}_{\mathrm{S}}}\ \ \forall\ \ \rvx\ \ s.t.\ \ \| \rvx \| \leq d
\end{align*}
\end{small}
\end{corollary}
\begin{proof}[\textbf{proof sketch}]
    By arguments similar to those in the proof of Proposition~\ref{prop:4}, and using the facts that the limit of the products is equal to the product of the limits, and that $\lim_{t \to 0}\rho_t = 1$. 
\end{proof}
\begin{remark}\label{remark:1}
Note that the limit of $\mathrm{VML}_t(\rvx)$ as $t \to 0$ doesn't exist. However, for a given arbitrary $t$, a global minimizer of $\mathrm{VML}_t(\rvx)$ is also a global minimizer of $\mathrm{VML}_t(\rvx) + n \log \sigma_t$ (for $n \log \sigma_t$ is a constant given $t$) and vice-versa. From Proposition~\ref{prop:4}, $\mathrm{VML}_t(\rvx) + n \log \sigma_t$ converges pointwise to $-\log p_0(\rvx|\rvy) + \hat{\mathrm{C}}_{\mathrm{VML}}$ in the limit as $t \to 0$. \hfill \break
\end{remark}
\begin{remark}\label{remark:2}
Note that the limit of $\mathrm{VML}_{\mathrm{S}_t}(\rvx)$ as $t \to 0$ doesn't exist. However, for a given arbitrary $t$, a global minimizer of $\mathrm{VML}_{\mathrm{S}_t}(\rvx)$ is also a global minimizer of $\mathrm{VML}_{\mathrm{S}_t}(\rvx) + n \log \sigma_t + \mathrm{C}_{\mathrm{VML}_{\mathrm{High}}}$ (for $n \log \sigma_t + \mathrm{C}_{\mathrm{VML}_{\mathrm{High}}}$ is a constant given $t$) and vice-versa. From Corollary~\ref{corollary:1}, $\mathrm{VML}_{\mathrm{S}_t}(\rvx) + n \log \sigma_t + \mathrm{C}_{\mathrm{VML}_{\mathrm{High}}}$ converges pointwise to $-\log p_0(\rvx|\rvy) + \hat{\mathrm{C}}_{\mathrm{VML}_{\mathrm{S}}} + \mathrm{C}_{\mathrm{VML}_{\mathrm{High}}} = -\log p_0(\rvx|\rvy) + \hat{\mathrm{C}}_{\mathrm{VML}}$ in the limit as $t \to 0$. \hfill \break
\end{remark}
\begin{remark}\label{remark:3}
From Proposition~\ref{prop:3}, the function $(\mathrm{VML}_{\mathrm{High}_t}(\rvx) - \mathrm{C}_{\mathrm{VML}_{\mathrm{High}}})$ converges uniformly to the zero function in the limit as $t \to 0$. Note that $\mathrm{VML}_{\mathrm{High}_t}(\rvx) - \mathrm{C}_{\mathrm{VML}_{\mathrm{High}}} = (\mathrm{VML}_t(\rvx) + n \log \sigma_t) - (\mathrm{VML}_{\mathrm{S}_t}(\rvx) + n \log \sigma_t + \mathrm{C}_{\mathrm{VML}_{\mathrm{High}}})$, i.e., the difference of essentially equivalent (in terms of global minimizers) functions of VML and Simplified-VML respectively (see Remarks~\ref{remark:1} and~\ref{remark:2}). It implies that the difference of these functions becomes arbitrarily small as $t \to 0$. In practice, this approximation of VML with the simplified-VML (or equivalently with the reweighted simplified-VML from Corollary~\ref{corollary:2}) may not be critical, as the errors arising due to the imperfect optimizer and numerical errors from discretizing the reverse SDE or PF ODE typically dominate early in the reverse diffusion process.
\end{remark}

\section{Experimental setup, implementation details, Qualitative visualizations, and more}\label{sec:implementationdetails}
\subsection{Experiments in Table~\ref{tab:exp1}}\label{ssec:implementationdetails1}
For our experiments in Table~\ref{tab:exp1}, we considered image restoration inverse problems with severe enough degradations to make it more challenging. However, we do not resort to extreme degradations, as the corresponding measurements typically do not provide strong guidance for recovering the ground truth image, since extreme degradations make the posterior highly multimodal to an extent that the restored image is perceptually dissimilar to the ground truth image, which makes it challenging to assess the performance using the usual LPIPS/FID metrics. Our experiments included half-mask inpainting, $4\times$ super-resolution, and uniform deblurring with a $16\times 16$ kernel. We utilize the SVD-based super-resolution and uniform deblurring operators from~\citet{ddrm} to ensure that the preconditioner can be computed efficiently. For uniform deblurring, we observed that the degradations are not severe enough, as the pseudoinverse solution already gives an almost perfect reconstruction. To make it more challenging, we zero out the singular values below a high enough threshold (0.2) as opposed to the default threshold (0.03) used in~\citet{ddrm}.

\begin{table}[h!]
\caption{Best learning rate configuration for each task and dataset in Table~\ref{tab:exp1}, Table~\ref{tab:exp2} and Table~\ref{tab:ir1}. Note that the learning rate is $\gamma_0 \, \sigma^2_\rvy$, with $\gamma_0$ as reported in the table.}
\label{tab:hp1}
\begin{center}
\begin{tabular}{ l l l  l l l  c c c  c c c  c c c }
\toprule
\multicolumn{3}{l}{} & \multicolumn{3}{c}{} & \multicolumn{3}{c}{} & \multicolumn{3}{c}{} & \multicolumn{3}{c}{} \\[-7pt]
\multicolumn{3}{l}{{Dataset}}  & \multicolumn{3}{l}{{Method}} & \multicolumn{3}{c}{{Inpainting}} & \multicolumn{3}{c}{{$4\times$ Super-res}}  & \multicolumn{3}{c}{{Deblurring}} \\[4pt]
\toprule
\multicolumn{3}{l}{} & \multicolumn{3}{c}{} & \multicolumn{3}{c}{} & \multicolumn{3}{c}{} & \multicolumn{3}{c}{}\\[-7pt]
\multicolumn{3}{l}{} & \multicolumn{3}{l}{{\scriptsize{$\text{VML-MAP}$}}}  & 
\multicolumn{3}{c}{\scriptsize{$\gamma_0 = 1.5$}} & \multicolumn{3}{c}{\scriptsize{$\gamma_0 = 30.0$}}  & \multicolumn{3}{c}{\scriptsize{$\gamma_0 = 2.25$}} \\[-5pt]
\multicolumn{3}{l}{\scriptsize{$\text{ImageNet}$64}} & \multicolumn{3}{l}{}  & 
\multicolumn{3}{c}{\scriptsize{$$}} & \multicolumn{3}{c}{\scriptsize{$$}}  & \multicolumn{3}{c}{\scriptsize{$$}} \\[-5pt]
\multicolumn{3}{l}{} & \multicolumn{3}{l}{{\scriptsize{$\text{VML-MAP}_{pre}$}}}  & 
\multicolumn{3}{c}{\scriptsize{$\gamma_0 = 1.5$}} & \multicolumn{3}{c}{\scriptsize{$\gamma_0 = 1.75$}}  & \multicolumn{3}{c}{\scriptsize{$\gamma_0 = 2.0$}} \\[2pt]
\midrule
\multicolumn{3}{l}{} & \multicolumn{3}{c}{} & \multicolumn{3}{c}{} & \multicolumn{3}{c}{} & \multicolumn{3}{c}{}\\[-7pt]
\multicolumn{3}{l}{} & \multicolumn{3}{l}{{\scriptsize{$\text{VML-MAP}$}}}  & 
\multicolumn{3}{c}{\scriptsize{$\gamma_0 = 1.25$}} & \multicolumn{3}{c}{\scriptsize{$\gamma_0 = 25.0$}}  & \multicolumn{3}{c}{\scriptsize{$\gamma_0 = 2.0$}} \\[-5pt]
\multicolumn{3}{l}{\scriptsize{$\text{ImageNet}256$}} & \multicolumn{3}{l}{}  & 
\multicolumn{3}{c}{\scriptsize{$$}} & \multicolumn{3}{c}{\scriptsize{$$}}  & \multicolumn{3}{c}{\scriptsize{$$}} \\[-5pt]
\multicolumn{3}{l}{} & \multicolumn{3}{l}{{\scriptsize{$\text{VML-MAP}_{pre}$}}}  & 
\multicolumn{3}{c}{\scriptsize{$\gamma_0 = 1.25$}} & \multicolumn{3}{c}{\scriptsize{$\gamma_0 = 1.5$}}  & \multicolumn{3}{c}{\scriptsize{$\gamma_0 = 1.5$}} \\[2pt]
\midrule
\multicolumn{3}{l}{} & \multicolumn{3}{c}{} & \multicolumn{3}{c}{} & \multicolumn{3}{c}{} & \multicolumn{3}{c}{}\\[-7pt]
\multicolumn{3}{l}{} & \multicolumn{3}{l}{{\scriptsize{$\text{VML-MAP}$}}}  & 
\multicolumn{3}{c}{\scriptsize{$\gamma_0 = 1.25$}} & \multicolumn{3}{c}{\scriptsize{$\gamma_0 = 30.0$}}  & \multicolumn{3}{c}{\scriptsize{$\gamma_0 = 2.0$}} \\[-5pt]
\multicolumn{3}{l}{\scriptsize{$\text{FFHQ}256$}} & \multicolumn{3}{l}{}  & 
\multicolumn{3}{c}{\scriptsize{$$}} & \multicolumn{3}{c}{\scriptsize{$$}}  & \multicolumn{3}{c}{\scriptsize{$$}} \\[-5pt]
\multicolumn{3}{l}{} & \multicolumn{3}{l}{{\scriptsize{$\text{VML-MAP}_{pre}$}}}  & 
\multicolumn{3}{c}{\scriptsize{$\gamma_0 = 1.25$}} & \multicolumn{3}{c}{\scriptsize{$\gamma_0 = 1.5$}}  & \multicolumn{3}{c}{\scriptsize{$\gamma_0 = 1.5$}} \\[2pt]
\bottomrule
\end{tabular}
\end{center}
\end{table}

For DDRM, $\Pi$GDM, MAPGA, VML-MAP, and $\text{VML-MAP}_{pre}$, we use the EDM~\cite{edm} noise schedule discretization with $\sigma_{min}=0.002$, $\sigma_{max}=140$, and EDM's $\rho=7$. For VML-MAP, and $\text{VML-MAP}_{pre}$, we set our prior weight schedule $\rho(\cdot)$ to a constant, i.e., $\rho(t) = 1\ \ \forall\ \ t > 0$. Note that MAPGA requires a consistency model by default, so throughout this paper, we use the variant MAPGA(D) from~\citet{mapga}, which replaces the consistency model with a single-step denoiser approximation. Note that we use the EDM schedule also for DDRM and $\Pi$GDM, as it performed the best. For DAPS, we use the default DAPS-1K and DAPS-4K configurations mentioned in the original paper, and observed that the default hyperparameters used in the paper for box inpainting, super-resolution, and Gaussian deblurring also perform the best for our half-mask inpainting, super-resolution, and uniform deblurring tasks. 

For each experiment with DDRM and $\Pi$GDM, we search for $N$ (i.e., the number of reverse diffusion steps) over \{$20.50,100,200,500,1000$\} and report the best result. For each experiment using MAPGA, VML-MAP, and $\text{VML-MAP}_{pre}$, we search for ($N,K$) ($N$ denotes the total number of diffusion time steps, and $K$ denotes the number of gradient ascent/descent iterations per step) over \{($20,50$),($50,20$),($100,10$),($200,5$),($500$,$2$),($1000,1$)\} and report the best performance (this keeps the total budget for MAP-GA, VML-MAP and $\text{VML-MAP}_{pre}$ within 1000 optimization steps in total). In every case, we find the best configuration (with the lowest LPIPS) to be ($N,K$) = ($20,50$). We set $\sigma_\rvy=0$ for DDRM, $\Pi$GDM, and MAPGA, while for DAPS, VML-MAP and $\text{VML-MAP}_{pre}$, we set $\sigma_\rvy = 1e^{-9}$. For MAPGA, the default learning rate from the original repository was used, while for VML-MAP and $\text{VML-MAP}_{pre}$, we report the best learning rate configuration for each task and dataset as $\gamma_0 \cdot \sigma^2_\rvy$, with $\gamma_0$ shown in Table~\ref{tab:hp1}. Our implementation of DDRM, $\Pi$GDM, and MAPGA is based on the following original repositories \href{https://github.com/bahjat-kawar/ddrm}{\text{ddrm}}, \href{https://github.com/NVlabs/RED-diff}{pgdm}, \href{https://github.com/guthasaibharathchandra/MAP-GA}{mapga}, respectively.

For our previous experiments in Table~\ref{tab:exp1}, we additionally report PSNR and SSIM metrics below in Table~\ref{tab:exp1extended} for ImageNet$256$. As mentioned earlier in Section~\ref{sec:image-restoration}, non-perceptual metrics such as PSNR and SSIM are not the right metrics of consideration for image-restoration inverse problems, as they do not necessarily capture the perceptual fidelity of the reconstructed images. Furthermore, pixel-to-pixel metrics such as PSNR overlook the multimodal nature of the posterior, especially for the severe degradation operators we considered in all our experiments throughout this paper. The quantitative results (see Table~\ref{tab:exp1extended}) and the qualitative visualizations (see Figures~\ref{fig:qv1},~\ref{fig:qv2}, and~\ref{fig:qv3}) clearly show this contrast, where methods achieving high PSNR and SSIM produce perceptually worse image reconstructions. We refer to Section~\ref{sec:image-restoration} for more details on metrics.

\begin{table*}[ht!]
\caption{ImageNet$256$ Experiments from Table~\ref{tab:exp1} with additional metrics such as PSNR and SSIM. The mean and standard deviation across $4$ different runs are reported. Best values in \textbf{bold}, second best values \underline{underlined}.}
\label{tab:exp1extended}
\begin{center}
\resizebox{\textwidth}{!}{
\begin{tabular}{ l c c c c c c c c c c c c } 
\toprule
  \multicolumn{1}{c}{} & \multicolumn{12}{c}{} \\[-7pt]
  \multicolumn{1}{l}{{Method}} & \multicolumn{4}{c}{{Inpainting}} & \multicolumn{4}{c}{{$4\times${ Super-res}}}  & \multicolumn{4}{c}{{Deblurring}} \\[3pt]
  \multicolumn{1}{c}{} & \multicolumn{1}{c}{\scriptsize{PSNR$\uparrow$}} & \multicolumn{1}{c}{\scriptsize{SSIM$\uparrow$}} & \multicolumn{1}{c}{\scriptsize{LPIPS$\downarrow$}} & \multicolumn{1}{c}{\scriptsize{FID$\downarrow$}} & \multicolumn{1}{c}{\scriptsize{PSNR$\uparrow$}} & \multicolumn{1}{c}{\scriptsize{SSIM$\uparrow$}} & \multicolumn{1}{c}{\scriptsize{LPIPS$\downarrow$}} & \multicolumn{1}{c}{\scriptsize{FID$\downarrow$}} & \multicolumn{1}{c}{\scriptsize{PSNR$\uparrow$}} & \multicolumn{1}{c}{\scriptsize{SSIM$\uparrow$}} & \multicolumn{1}{c}{\scriptsize{LPIPS$\downarrow$}} & \multicolumn{1}{c}{\scriptsize{FID$\downarrow$}} \\[3pt]
\toprule 
 \multicolumn{1}{c}{} & \multicolumn{12}{c}{}\\[-7pt]
  \multicolumn{1}{l}{{\scriptsize{$\text{DDRM}$}}}  & 
 \multicolumn{1}{c}{\scriptsize{$\underline{15.01}$}\tiny{\color{lightgray}{$\pm0.055$}}} & \multicolumn{1}{c}{\scriptsize{$\underline{0.679}$}\tiny{\color{lightgray}{$\pm0.000$}}} &
 \multicolumn{1}{c}{\scriptsize{$0.393$}\tiny{\color{lightgray}{$\pm0.001$}}} & \multicolumn{1}{c}{\scriptsize{$103.5$}\tiny{\color{lightgray}{$\pm0.657$}}}  & \multicolumn{1}{c}{\scriptsize{$24.20$}\tiny{\color{lightgray}{$\pm0.065$}}} & \multicolumn{1}{c}{\scriptsize{$0.738$}\tiny{\color{lightgray}{$\pm0.000$}}} &
 \multicolumn{1}{c}{\scriptsize{$0.289$}\tiny{\color{lightgray}{$\pm0.000$}}} & \multicolumn{1}{c}{\scriptsize{$89.40$}\tiny{\color{lightgray}{$\pm0.962$}}} & \multicolumn{1}{c}{\scriptsize{$20.51$}\tiny{\color{lightgray}{$\pm0.010$}}} & \multicolumn{1}{c}{\scriptsize{$0.511$}\tiny{\color{lightgray}{$\pm0.000$}}} &
 \multicolumn{1}{c}{\scriptsize{$0.618$}\tiny{\color{lightgray}{$\pm0.000$}}} & \multicolumn{1}{c}{\scriptsize{$233.4$}\tiny{\color{lightgray}{$\pm2.033$}}}\\
 
  \multicolumn{1}{l}{{\scriptsize{$\Pi\text{GDM}$}}}  & \multicolumn{1}{c}{\scriptsize{$\bf{15.08}$}\tiny{\color{lightgray}{$\pm0.081$}}} & \multicolumn{1}{c}{\scriptsize{$\bf{0.681}$}\tiny{\color{lightgray}{$\pm0.001$}}}  &
 
 \multicolumn{1}{c}{\scriptsize{$0.373$}\tiny{\color{lightgray}{$\pm0.000$}}} & \multicolumn{1}{c}{\scriptsize{$103.7$}\tiny{\color{lightgray}{$\pm0.694$}}}  &  \multicolumn{1}{c}{\scriptsize{$\underline{24.41}$}\tiny{\color{lightgray}{$\pm0.068$}}} & \multicolumn{1}{c}{\scriptsize{$\underline{0.745}$}\tiny{\color{lightgray}{$\pm0.000$}}}  &
 
 \multicolumn{1}{c}{\scriptsize{$0.292$}\tiny{\color{lightgray}{$\pm0.001$}}} & \multicolumn{1}{c}{\scriptsize{$83.88$}\tiny{\color{lightgray}{$\pm0.226$}}} &  \multicolumn{1}{c}{\scriptsize{$\bf{20.74}$}\tiny{\color{lightgray}{$\pm0.003$}}} & \multicolumn{1}{c}{\scriptsize{$\underline{0.523}$}\tiny{\color{lightgray}{$\pm0.000$}}}  &
 
 \multicolumn{1}{c}{\scriptsize{$0.562$}\tiny{\color{lightgray}{$\pm0.000$}}} & \multicolumn{1}{c}{\scriptsize{$231.8$}\tiny{\color{lightgray}{$\pm0.631$}}}\\ 
 
  \multicolumn{1}{l}{{\scriptsize{$\text{MAPGA}$}}}  & \multicolumn{1}{c}{\scriptsize{$14.51$}\tiny{\color{lightgray}{$\pm0.044$}}} & \multicolumn{1}{c}{\scriptsize{$0.654$}\tiny{\color{lightgray}{$\pm0.000$}}}  &
 
 \multicolumn{1}{c}{\scriptsize{$\underline{0.290}$}\tiny{\color{lightgray}{$\pm0.001$}}} & \multicolumn{1}{c}{\scriptsize{$\underline{80.76}$}\tiny{\color{lightgray}{$\pm2.013$}}}  & \multicolumn{1}{c}{\scriptsize{$\bf{24.63}$}\tiny{\color{lightgray}{$\pm0.038$}}} & \multicolumn{1}{c}{\scriptsize{$\bf{0.755}$}\tiny{\color{lightgray}{$\pm0.000$}}}  &
 
 \multicolumn{1}{c}{\scriptsize{$0.273$}\tiny{\color{lightgray}{$\pm0.001$}}} & \multicolumn{1}{c}{\scriptsize{$76.12$}\tiny{\color{lightgray}{$\pm0.587$}}} & \multicolumn{1}{c}{\scriptsize{$\underline{20.71}$}\tiny{\color{lightgray}{$\pm0.031$}}} & \multicolumn{1}{c}{\scriptsize{$\bf{0.544}$}\tiny{\color{lightgray}{$\pm0.000$}}}  &
 
 \multicolumn{1}{c}{\scriptsize{$\underline{0.459}$}\tiny{\color{lightgray}{$\pm0.002$}}} & \multicolumn{1}{c}{\scriptsize{$\underline{194.9}$}\tiny{\color{lightgray}{$\pm1.548$}}}\\ 

 
 
 
  \multicolumn{1}{l}{{\scriptsize{$\text{DAPS-1K}$}}}  & \multicolumn{1}{c}{\scriptsize{$14.88$}\tiny{\color{lightgray}{$\pm0.117$}}} & \multicolumn{1}{c}{\scriptsize{$0.641$}\tiny{\color{lightgray}{$\pm0.000$}}}  &
 \multicolumn{1}{c}{\scriptsize{$0.384$}\tiny{\color{lightgray}{$\pm0.001$}}} & \multicolumn{1}{c}{\scriptsize{$98.45$}\tiny{\color{lightgray}{$\pm1.879$}}}  & \multicolumn{1}{c}{\scriptsize{$23.82$}\tiny{\color{lightgray}{$\pm0.038$}}} & \multicolumn{1}{c}{\scriptsize{$0.669$}\tiny{\color{lightgray}{$\pm0.000$}}}  &
 
 \multicolumn{1}{c}{\scriptsize{$0.254$}\tiny{\color{lightgray}{$\pm0.001$}}} & \multicolumn{1}{c}{\scriptsize{$71.60$}\tiny{\color{lightgray}{$\pm1.010$}}} & \multicolumn{1}{c}{\scriptsize{$20.22$}\tiny{\color{lightgray}{$\pm0.027$}}} & \multicolumn{1}{c}{\scriptsize{$0.466$}\tiny{\color{lightgray}{$\pm0.000$}}}  &
 
 \multicolumn{1}{c}{\scriptsize{$0.605$}\tiny{\color{lightgray}{$\pm0.001$}}} & \multicolumn{1}{c}{\scriptsize{$217.8$}\tiny{\color{lightgray}{$\pm3.042$}}}\\
  
   \multicolumn{1}{l}{{\scriptsize{$\text{DAPS-4K}$}}}  & \multicolumn{1}{c}{\scriptsize{$14.59$}\tiny{\color{lightgray}{$\pm0.115$}}} & \multicolumn{1}{c}{\scriptsize{$0.637$}\tiny{\color{lightgray}{$\pm0.001$}}}  &
  
  \multicolumn{1}{c}{\scriptsize{$0.365$}\tiny{\color{lightgray}{$\pm0.004$}}} & \multicolumn{1}{c}{\scriptsize{$93.62$}\tiny{\color{lightgray}{$\pm0.856$}}}  & \multicolumn{1}{c}{\scriptsize{$23.62$}\tiny{\color{lightgray}{$\pm0.049$}}} & \multicolumn{1}{c}{\scriptsize{$0.670$}\tiny{\color{lightgray}{$\pm0.000$}}}  &

  \multicolumn{1}{c}{\scriptsize{$0.244$}\tiny{\color{lightgray}{$\pm0.000$}}} & \multicolumn{1}{c}{\scriptsize{$69.71$}\tiny{\color{lightgray}{$\pm1.554$}}} & \multicolumn{1}{c}{\scriptsize{$20.15$}\tiny{\color{lightgray}{$\pm0.011$}}} & \multicolumn{1}{c}{\scriptsize{$0.467$}\tiny{\color{lightgray}{$\pm0.000$}}}  &
  
  \multicolumn{1}{c}{\scriptsize{$0.593$}\tiny{\color{lightgray}{$\pm0.001$}}} & \multicolumn{1}{c}{\scriptsize{$227.9$}\tiny{\color{lightgray}{$\pm2.991$}}}\\[2pt]
 \cline{2-13}
  \multicolumn{1}{c}{} & \multicolumn{12}{c}{}\\[-7pt]
 
  \multicolumn{1}{l}{\bf{\scriptsize{$\text{VML-MAP}$}}}  & \multicolumn{1}{c}{\scriptsize{$14.34$}\tiny{\color{lightgray}{$\pm0.067$}}} & \multicolumn{1}{c}{\scriptsize{$0.631$}\tiny{\color{lightgray}{$\pm0.001$}}} &
 
 \multicolumn{1}{c}{\scriptsize{$\bf{0.262}$}\tiny{\color{lightgray}{$\pm0.002$}}} & \multicolumn{1}{c}{\scriptsize{$\bf{74.21}$}\tiny{\color{lightgray}{$\pm3.209$}}} & \multicolumn{1}{c}{\scriptsize{$23.41$}\tiny{\color{lightgray}{$\pm0.049$}}} & \multicolumn{1}{c}{\scriptsize{$0.729$}\tiny{\color{lightgray}{$\pm0.000$}}} &

 \multicolumn{1}{c}{\scriptsize{$\bf{0.194}$}\tiny{\color{lightgray}{$\pm0.001$}}} & \multicolumn{1}{c}{\scriptsize{$\underline{60.08}$}\tiny{\color{lightgray}{$\pm0.687$}}} & \multicolumn{1}{c}{\scriptsize{$20.40$}\tiny{\color{lightgray}{$\pm0.028$}}} & \multicolumn{1}{c}{\scriptsize{$0.517$}\tiny{\color{lightgray}{$\pm0.001$}}} &
 
 \multicolumn{1}{c}{\scriptsize{$0.509$}\tiny{\color{lightgray}{$\pm0.003$}}} & \multicolumn{1}{c}{\scriptsize{$200.4$}\tiny{\color{lightgray}{$\pm2.812$}}}\\
 
  \multicolumn{1}{l}{\bf{\scriptsize{$\text{VML-MAP}_{pre}$}}}  & \multicolumn{1}{c}{\scriptsize{$14.34$}\tiny{\color{lightgray}{$\pm0.067$}}} & \multicolumn{1}{c}{\scriptsize{$0.631$}\tiny{\color{lightgray}{$\pm0.001$}}} &
 
 \multicolumn{1}{c}{\scriptsize{$\bf{0.262}$}\tiny{\color{lightgray}{$\pm0.002$}}} & \multicolumn{1}{c}{\scriptsize{$\bf{74.21}$}\tiny{\color{lightgray}{$\pm3.209$}}} & \multicolumn{1}{c}{\scriptsize{$23.63$}\tiny{\color{lightgray}{$\pm0.117$}}} & \multicolumn{1}{c}{\scriptsize{$0.733$}\tiny{\color{lightgray}{$\pm0.002$}}} &
 
 \multicolumn{1}{c}{\scriptsize{$\underline{0.196}$}\tiny{\color{lightgray}{$\pm0.003$}}} & \multicolumn{1}{c}{\scriptsize{$\bf{58.60}$}\tiny{\color{lightgray}{$\pm2.076$}}} & \multicolumn{1}{c}{\scriptsize{$19.70$}\tiny{\color{lightgray}{$\pm0.079$}}} & \multicolumn{1}{c}{\scriptsize{$0.511$}\tiny{\color{lightgray}{$\pm0.002$}}} &
 
 \multicolumn{1}{c}{\scriptsize{$\bf{0.367}$}\tiny{\color{lightgray}{$\pm0.002$}}} & \multicolumn{1}{c}{\scriptsize{$\bf{165.2}$}\tiny{\color{lightgray}{$\pm2.037$}}}\\[2pt]
 \bottomrule
 \end{tabular}}
\end{center}
\end{table*}

\begin{figure}[h!]
    \centering
    \begin{subfigure}[h]{0.125\textwidth}
    \centering
    \subfloat{\scriptsize{\bf{\color{gray}{\textit{Original}}}}}
    \end{subfigure}
    \begin{subfigure}[h]{0.125\textwidth}
    \centering
    \subfloat{\scriptsize{\bf{\color{gray}{\textit{Measurement}}}}}
    \end{subfigure}
    \begin{subfigure}[h]{0.125\textwidth}
    \centering
    \subfloat{\scriptsize{\bf{\color{gray}{\textit{DDRM}}}}}
    \end{subfigure}
    \begin{subfigure}[h]{0.125\textwidth}
    \centering
    \subfloat{\scriptsize{\bf{\color{gray}{\textit{$\Pi$GDM}}}}}
    \end{subfigure}
    \begin{subfigure}[h]{0.125\textwidth}
    \centering
    \subfloat{\scriptsize{\bf{\color{gray}{\textit{DAPS-1K}}}}}
    \end{subfigure}
    \begin{subfigure}[h]{0.125\textwidth}
    \centering
    \subfloat{\scriptsize{\bf{\color{gray}{\textit{DAPS-4K}}}}}
    \end{subfigure}
    \begin{subfigure}[h]{0.125\textwidth}
    \centering
    \subfloat{\scriptsize{\bf{\color{gray}{\textit{MAPGA}}}}}
    \end{subfigure}
    \begin{subfigure}[h]{0.125\textwidth}
    \centering
    \subfloat{\scriptsize{\bf{\color{gray}{\textit{VML-MAP}}}}}
    \end{subfigure}
    \\[2pt]
    \begin{subfigure}[h]{0.125\textwidth}
    \centering
    \includegraphics[width=\textwidth]{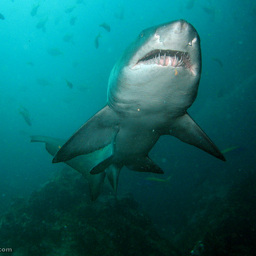}    
    \end{subfigure}
    \begin{subfigure}[h]{0.125\textwidth}
    \centering
    \includegraphics[width=\textwidth]{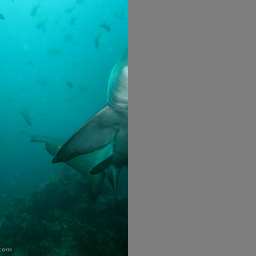}    
    \end{subfigure}
    \begin{subfigure}[h]{0.125\textwidth}
    \centering
    \includegraphics[width=\textwidth]{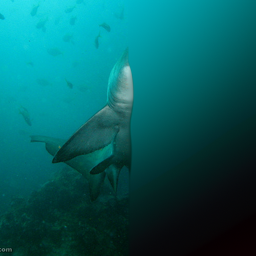}  
    \end{subfigure}
    \begin{subfigure}[h]{0.125\textwidth}
    \centering
    \includegraphics[width=\textwidth]{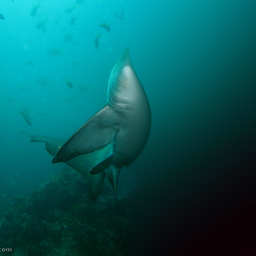}    
    \end{subfigure}
    \begin{subfigure}[h]{0.125\textwidth}
    \centering
    \includegraphics[width=\textwidth]{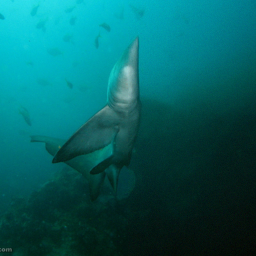}  
    \end{subfigure}
    \begin{subfigure}[h]{0.125\textwidth}
    \centering
    \includegraphics[width=\textwidth]{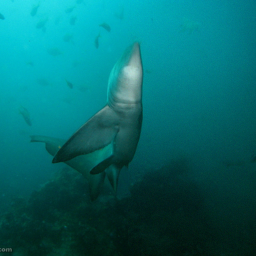}    
    \end{subfigure}
    \begin{subfigure}[h]{0.125\textwidth}
    \centering
    \includegraphics[width=\textwidth]{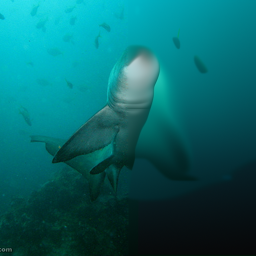}    
    \end{subfigure}
    \begin{subfigure}[h]{0.125\textwidth}
    \centering
    \includegraphics[width=\textwidth]{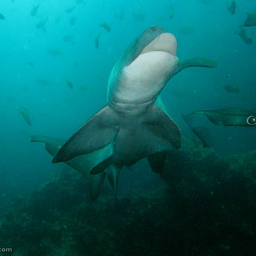}    
    \end{subfigure}
    \\ 
    \begin{subfigure}[h]{0.125\textwidth}
    \centering
    \includegraphics[width=\textwidth]{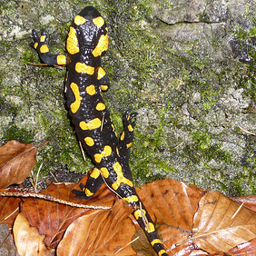}    
    \end{subfigure}
    \begin{subfigure}[h]{0.125\textwidth}
    \centering
    \includegraphics[width=\textwidth]{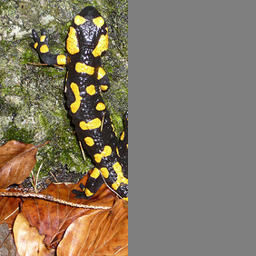}    
    \end{subfigure}
    \begin{subfigure}[h]{0.125\textwidth}
    \centering
    \includegraphics[width=\textwidth]{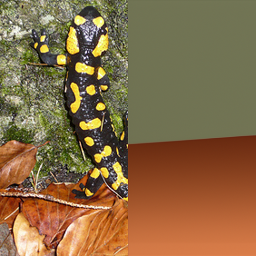}  
    \end{subfigure}
    \begin{subfigure}[h]{0.125\textwidth}
    \centering
    \includegraphics[width=\textwidth]{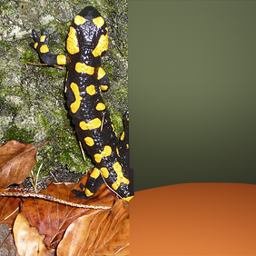}    
    \end{subfigure}
    \begin{subfigure}[h]{0.125\textwidth}
    \centering
    \includegraphics[width=\textwidth]{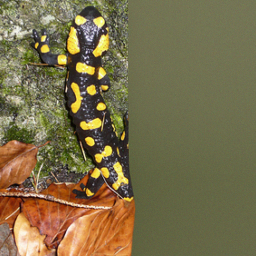}  
    \end{subfigure}
    \begin{subfigure}[h]{0.125\textwidth}
    \centering
    \includegraphics[width=\textwidth]{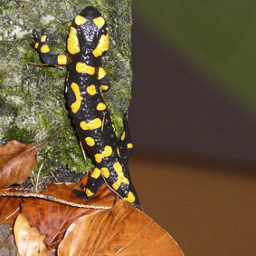}    
    \end{subfigure}
    \begin{subfigure}[h]{0.125\textwidth}
    \centering
    \includegraphics[width=\textwidth]{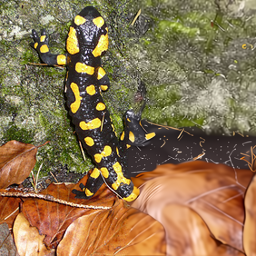}    
    \end{subfigure}
    \begin{subfigure}[h]{0.125\textwidth}
    \centering
    \includegraphics[width=\textwidth]{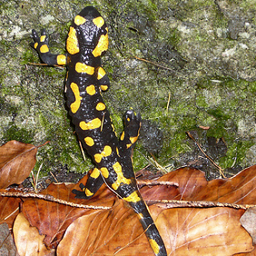}    
    \end{subfigure}
    \\ 
    \begin{subfigure}[h]{0.125\textwidth}
    \centering
    \includegraphics[width=\textwidth]{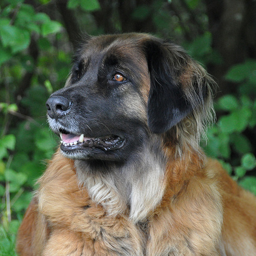}    
    \end{subfigure}
    \begin{subfigure}[h]{0.125\textwidth}
    \centering
    \includegraphics[width=\textwidth]{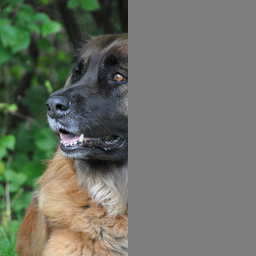}    
    \end{subfigure}
    \begin{subfigure}[h]{0.125\textwidth}
    \centering
    \includegraphics[width=\textwidth]{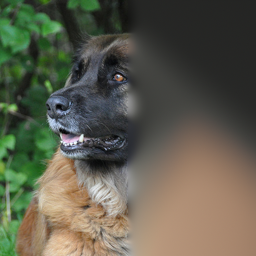}  
    \end{subfigure}
    \begin{subfigure}[h]{0.125\textwidth}
    \centering
    \includegraphics[width=\textwidth]{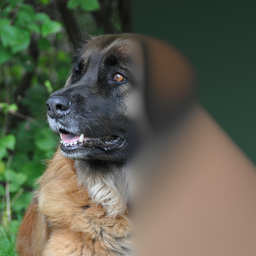}    
    \end{subfigure}
    \begin{subfigure}[h]{0.125\textwidth}
    \centering
    \includegraphics[width=\textwidth]{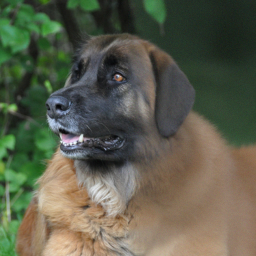}  
    \end{subfigure}
    \begin{subfigure}[h]{0.125\textwidth}
    \centering
    \includegraphics[width=\textwidth]{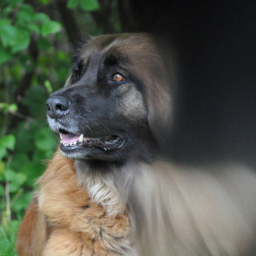}    
    \end{subfigure}
    \begin{subfigure}[h]{0.125\textwidth}
    \centering
    \includegraphics[width=\textwidth]{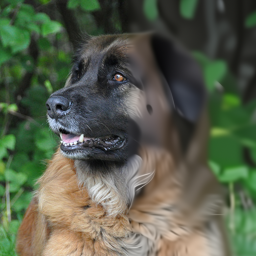}    
    \end{subfigure}
    \begin{subfigure}[h]{0.125\textwidth}
    \centering
    \includegraphics[width=\textwidth]{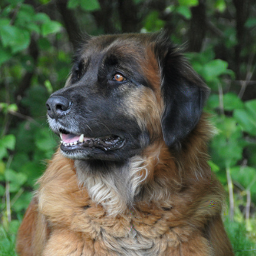}    
    \end{subfigure}
    \caption{Half-mask inpainting task from Table~\ref{tab:exp1}. Zoom in for the best view.}\label{fig:qv1}
\end{figure}
\begin{figure}[h!]
    \centering
    \begin{subfigure}[h]{0.15\textwidth}
    \centering
    \subfloat{\scriptsize{\bf{\color{gray}{\textit{Original}}}}}
    \end{subfigure}
    \begin{subfigure}[h]{0.15\textwidth}
    \centering
    \subfloat{\scriptsize{\bf{\color{gray}{\textit{Measurement}}}}}
    \end{subfigure}
    \begin{subfigure}[h]{0.15\textwidth}
    \centering
    \subfloat{\scriptsize{\bf{\color{gray}{\textit{DAPS-4K}}}}}
    \end{subfigure}
    \begin{subfigure}[h]{0.15\textwidth}
    \centering
    \subfloat{\scriptsize{\bf{\color{gray}{\textit{MAPGA}}}}}
    \end{subfigure}
    \begin{subfigure}[h]{0.15\textwidth}
    \centering
    \subfloat{\scriptsize{\bf{\color{gray}{\textit{VML-MAP}}}}}
    \end{subfigure}
    \\[2pt]
    \begin{subfigure}[h]{0.15\textwidth}
    \centering
    \includegraphics[width=\textwidth]{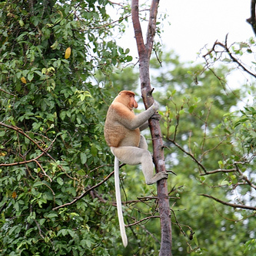}    
    \end{subfigure}
    \begin{subfigure}[h]{0.15\textwidth}
    \centering
    \includegraphics[width=\textwidth]{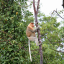}    
    \end{subfigure}
    \begin{subfigure}[h]{0.15\textwidth}
    \centering
    \includegraphics[width=\textwidth]{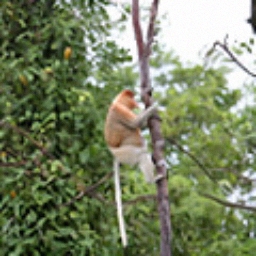}  
    \end{subfigure}
    \begin{subfigure}[h]{0.15\textwidth}
    \centering
    \includegraphics[width=\textwidth]{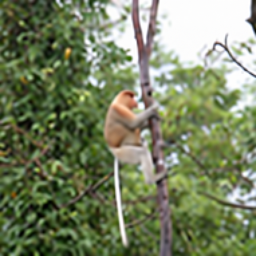}    
    \end{subfigure}
    \begin{subfigure}[h]{0.15\textwidth}
    \centering
    \includegraphics[width=\textwidth]{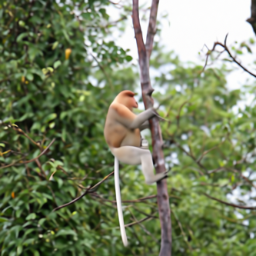}  
    \end{subfigure}
    \\ 
    \begin{subfigure}[h]{0.15\textwidth}
    \centering
    \includegraphics[width=\textwidth]{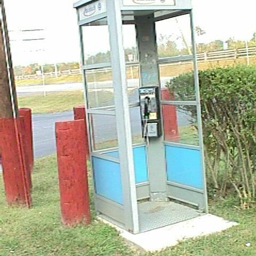}    
    \end{subfigure}
    \begin{subfigure}[h]{0.15\textwidth}
    \centering
    \includegraphics[width=\textwidth]{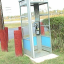}    
    \end{subfigure}
    \begin{subfigure}[h]{0.15\textwidth}
    \centering
    \includegraphics[width=\textwidth]{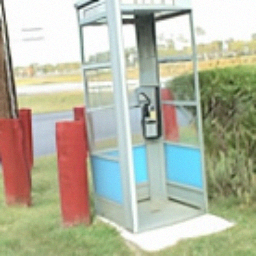}  
    \end{subfigure}
    \begin{subfigure}[h]{0.15\textwidth}
    \centering
    \includegraphics[width=\textwidth]{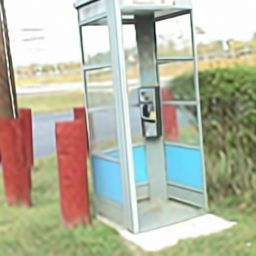}    
    \end{subfigure}
    \begin{subfigure}[h]{0.15\textwidth}
    \centering
    \includegraphics[width=\textwidth]{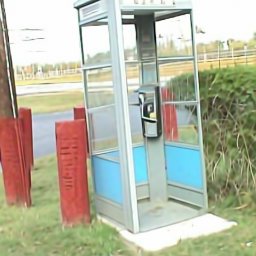}  
    \end{subfigure}
    \\ 
    \begin{subfigure}[h]{0.15\textwidth}
    \centering
    \includegraphics[width=\textwidth]{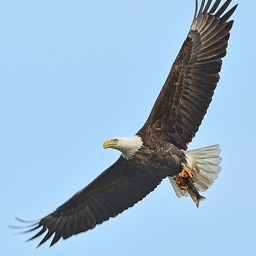}    
    \end{subfigure}
    \begin{subfigure}[h]{0.15\textwidth}
    \centering
    \includegraphics[width=\textwidth]{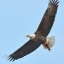}    
    \end{subfigure}
    \begin{subfigure}[h]{0.15\textwidth}
    \centering
    \includegraphics[width=\textwidth]{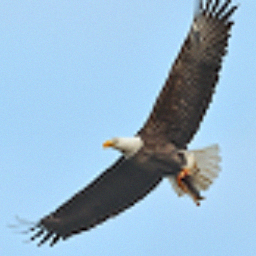}  
    \end{subfigure}
    \begin{subfigure}[h]{0.15\textwidth}
    \centering
    \includegraphics[width=\textwidth]{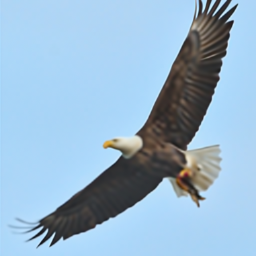}    
    \end{subfigure}
    \begin{subfigure}[h]{0.15\textwidth}
    \centering
    \includegraphics[width=\textwidth]{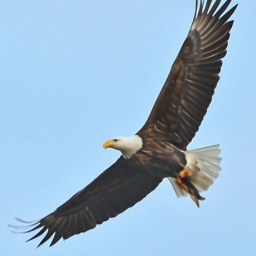}  
    \end{subfigure}
    \caption{$4\times$ super-resolution task from Table~\ref{tab:exp1}. Zoom in for the best view.}\label{fig:qv2}
\end{figure}
\begin{figure}[h!]
    \centering
    \begin{subfigure}[h]{0.15\textwidth}
    \centering
    \subfloat{\scriptsize{\bf{\color{gray}{$\textit{Original}_{\ }$}}}}
    \end{subfigure}
    \begin{subfigure}[h]{0.15\textwidth}
    \centering
    \subfloat{\scriptsize{\bf{\color{gray}{$\textit{Measurement}_{\ }$}}}}
    \end{subfigure}
    \begin{subfigure}[h]{0.15\textwidth}
    \centering
    \subfloat{\scriptsize{\bf{\color{gray}{$\textit{DAPS-4K}_{\ }$}}}}
    \end{subfigure}
    \begin{subfigure}[h]{0.15\textwidth}
    \centering
    \subfloat{\scriptsize{\bf{\color{gray}{$\textit{MAPGA}_{\ }$}}}}
    \end{subfigure}
    \begin{subfigure}[h]{0.15\textwidth}
    \centering
    \subfloat{\scriptsize{\bf{\color{gray}{$\textit{VML-MAP}_{\ }$}}}}
    \end{subfigure}
    \begin{subfigure}[h]{0.15\textwidth}
    \centering
    \subfloat{\scriptsize{\bf{\color{gray}{$\textit{VML-MAP}_{pre}$}}}}
    \end{subfigure}
    \\[2pt]
    \begin{subfigure}[h]{0.15\textwidth}
    \centering
    \includegraphics[width=\textwidth]{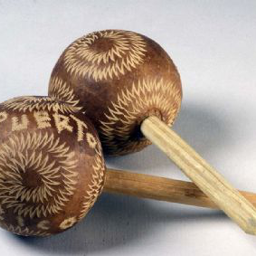}    
    \end{subfigure}
    \begin{subfigure}[h]{0.15\textwidth}
    \centering
    \includegraphics[width=\textwidth]{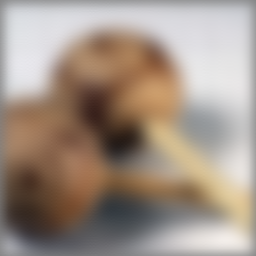}    
    \end{subfigure}
    \begin{subfigure}[h]{0.15\textwidth}
    \centering
    \includegraphics[width=\textwidth]{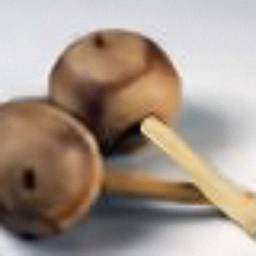}  
    \end{subfigure}
    \begin{subfigure}[h]{0.15\textwidth}
    \centering
    \includegraphics[width=\textwidth]{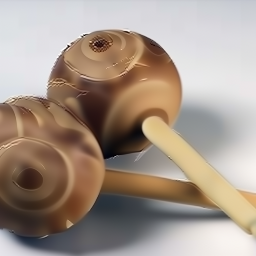}    
    \end{subfigure}
    \begin{subfigure}[h]{0.15\textwidth}
    \centering
    \includegraphics[width=\textwidth]{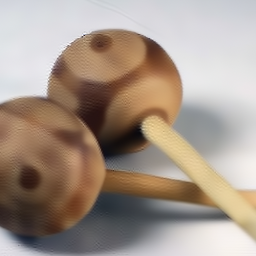}
    \end{subfigure}
    \begin{subfigure}[h]{0.15\textwidth}
    \centering
    \includegraphics[width=\textwidth]{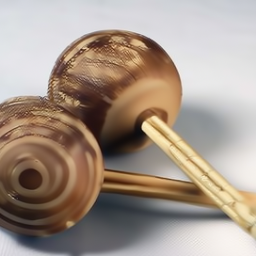}
    \end{subfigure}
    \\ 
    \begin{subfigure}[h]{0.15\textwidth}
    \centering
    \includegraphics[width=\textwidth]{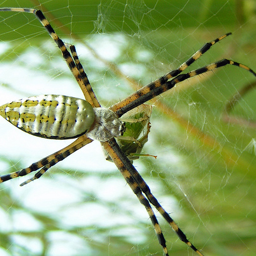}    
    \end{subfigure}
    \begin{subfigure}[h]{0.15\textwidth}
    \centering
    \includegraphics[width=\textwidth]{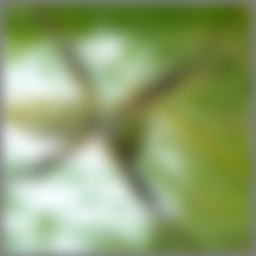}    
    \end{subfigure}
    \begin{subfigure}[h]{0.15\textwidth}
    \centering
    \includegraphics[width=\textwidth]{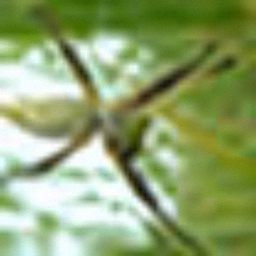}  
    \end{subfigure}
    \begin{subfigure}[h]{0.15\textwidth}
    \centering
    \includegraphics[width=\textwidth]{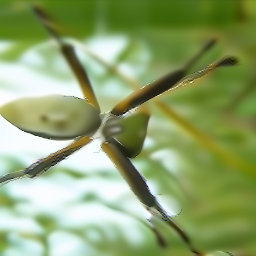}    
    \end{subfigure}
    \begin{subfigure}[h]{0.15\textwidth}
    \centering
    \includegraphics[width=\textwidth]{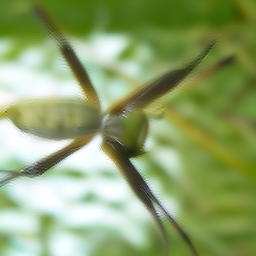}
    \end{subfigure}
    \begin{subfigure}[h]{0.15\textwidth}
    \centering
    \includegraphics[width=\textwidth]{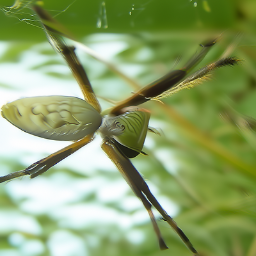}
    \end{subfigure}
    \\ 
    \begin{subfigure}[h]{0.15\textwidth}
    \centering
    \includegraphics[width=\textwidth]{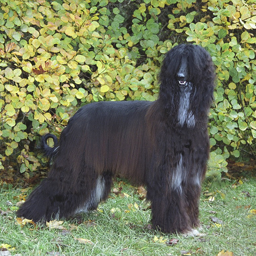}    
    \end{subfigure}
    \begin{subfigure}[h]{0.15\textwidth}
    \centering
    \includegraphics[width=\textwidth]{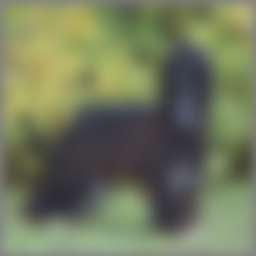}    
    \end{subfigure}
    \begin{subfigure}[h]{0.15\textwidth}
    \centering
    \includegraphics[width=\textwidth]{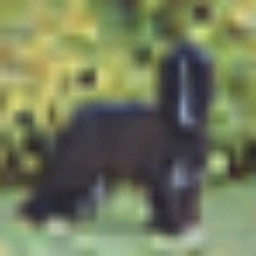}  
    \end{subfigure}
    \begin{subfigure}[h]{0.15\textwidth}
    \centering
    \includegraphics[width=\textwidth]{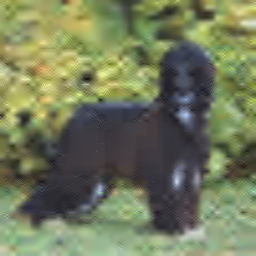}    
    \end{subfigure}
    \begin{subfigure}[h]{0.15\textwidth}
    \centering
    \includegraphics[width=\textwidth]{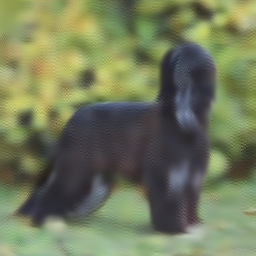}
    \end{subfigure}
    \begin{subfigure}[h]{0.15\textwidth}
    \centering
    \includegraphics[width=\textwidth]{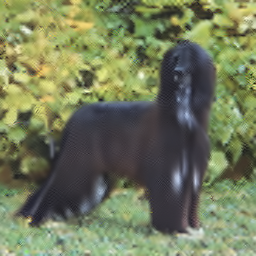}
    \end{subfigure}
    \caption{Deblurring task from Table~\ref{tab:exp1}. Zoom in for the best view.}\label{fig:qv3}
\end{figure}

\subsection{Experiments in Table~\ref{tab:exp2}}\label{ssec:implementationdetails2}

For our experiments in Table~\ref{tab:exp2}, we considered the case of noisy image restoration with the measurement noise $\sigma_\rvy = 0.05$ and the same degradation operators as in Table~\ref{tab:exp1}. For $\text{VML-MAP}^{\bm \tau}$ and $\text{VML-MAP}^{\bm \tau}_{pre}$, we set the number of reverse diffusion steps for diffusion time $t \geq \bm \tau$ given by $N = 20$, the number of gradient-descent iterations per step given by $K = 20$, and the number of reverse diffusion steps for diffusion time $t < \bm \tau$ given by $S = 100$. For the half-mask inpainting, $4\times$ super-resolution, and uniform deblurring tasks, we set the threshold $\bm \tau$ to $4\,\sigma_\rvy$, $16\, \sigma_\rvy$, and $16\, \sigma_\rvy$ respectively, and use the prior weight schedule {\small{$\rho(t) = \left( 1 + 100\, \frac{\sigma_t}{\sigma_\rvy} \right)^{-1} $}} for all the tasks. For baselines and other hyperparameters, we follow the same settings from Appendix~\ref{ssec:implementationdetails1}, with the best learning rate configuration for VML-MAP and $\text{VML-MAP}_{pre}$ reported in Table~\ref{tab:hp1}. Figures~\ref{fig:noisyexp1} and~\ref{fig:noisyexp2} depict qualitative visualizations of the reconstructed images from the experiments in Table~\ref{tab:exp2}.  

\begin{figure}[h!]
    \centering
    \begin{subfigure}[h]{0.025\textwidth}
    \centering    
    \subfloat{\bf{\color{gray}{ }}}
    \end{subfigure}
    \begin{subfigure}[h]{0.15\textwidth}
    \centering
    \subfloat{\scriptsize{\bf{{\color{gray}{\textit{Original}}}}}}    
    \end{subfigure}
    \begin{subfigure}[h]{0.15\textwidth}
    \centering
    \subfloat{\scriptsize{\bf{{\color{gray}{\textit{Measurement}}}}}}    
    \end{subfigure}
    \begin{subfigure}[h]{0.15\textwidth}
    \centering
    \subfloat{\scriptsize{\bf{{\color{gray}{$\Pi$\textit{GDM}}}}}}   
    \end{subfigure}
    \begin{subfigure}[h]{0.15\textwidth}
    \centering
    \subfloat{\scriptsize{{\bf{\color{gray}{\textit{DAPS-4K}}}}}}    
    \end{subfigure}
    \begin{subfigure}[h]{0.15\textwidth}
    \centering
    \subfloat{\scriptsize{\bf{\color{gray}{\textit{VML-MAP}}}}}
    \end{subfigure}
    \\[2pt]
    \begin{subfigure}[h]{0.025\textwidth}
    \centering
    \subfloat{{\rotatebox[origin=c]{90}{\bf{\color{gray}{\textit{Inpaint}}}}}}    
    \end{subfigure}
    \begin{subfigure}[h]{0.15\textwidth}
    \centering
    \includegraphics[width=\textwidth]{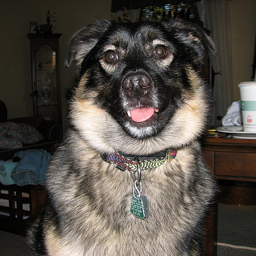}    
    \end{subfigure}
    \begin{subfigure}[h]{0.15\textwidth}
    \centering
    \includegraphics[width=\textwidth]{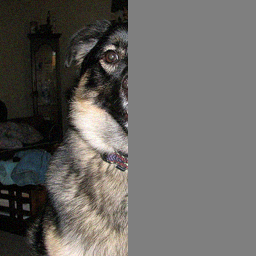}    
    \end{subfigure}
    \begin{subfigure}[h]{0.15\textwidth}
    \centering
    \includegraphics[width=\textwidth]{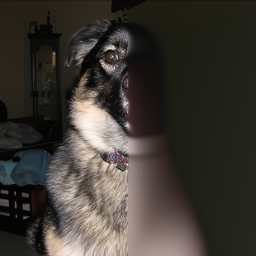}  
    \end{subfigure}
    \begin{subfigure}[h]{0.15\textwidth}
    \centering
    \includegraphics[width=\textwidth]{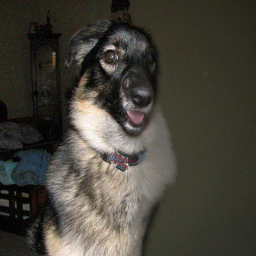}    
    \end{subfigure}
    \begin{subfigure}[h]{0.15\textwidth}
    \centering
    \includegraphics[width=\textwidth]{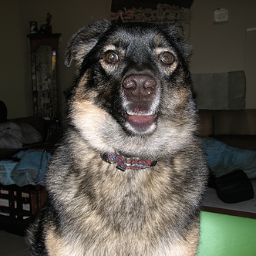}    
    \end{subfigure}
    \\
    \begin{subfigure}[h]{0.025\textwidth}
    \centering
    \subfloat{{\rotatebox[origin=c]{90}{\color{gray}{$\mathbf{4\times}$\bf{\textit{SR}}}}}}    
    \end{subfigure}
    \begin{subfigure}[h]{0.15\textwidth}
    \centering
    \includegraphics[width=\textwidth]{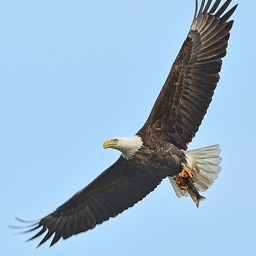}    
    \end{subfigure}
    \begin{subfigure}[h]{0.15\textwidth}
    \centering
    \includegraphics[width=\textwidth]{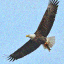}    
    \end{subfigure}
    \begin{subfigure}[h]{0.15\textwidth}
    \centering
    \includegraphics[width=\textwidth]{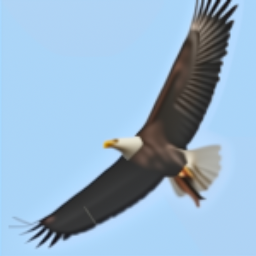}    
    \end{subfigure}
    \begin{subfigure}[h]{0.15\textwidth}
    \centering
    \includegraphics[width=\textwidth]{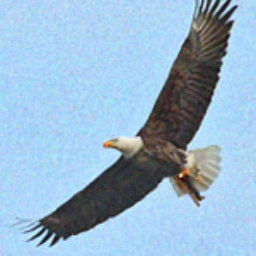}    
    \end{subfigure}
    \begin{subfigure}[h]{0.15\textwidth}
    \centering
    \includegraphics[width=\textwidth]{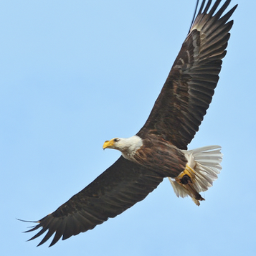}    
    \end{subfigure}
    \\
    \begin{subfigure}[h]{0.025\textwidth}
    \centering
    \subfloat{{\rotatebox[origin=c]{90}{\bf{\color{gray}{\textit{Deblur}}}}}}    
    \end{subfigure}
    \begin{subfigure}[h]{0.15\textwidth}
    \centering
    \includegraphics[width=\textwidth]{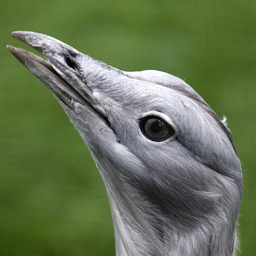}    
    \end{subfigure}
    \begin{subfigure}[h]{0.15\textwidth}
    \centering
    \includegraphics[width=\textwidth]{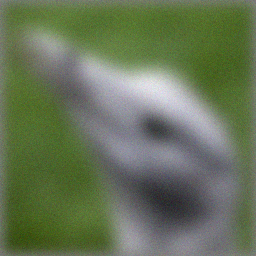}    
    \end{subfigure}
    \begin{subfigure}[h]{0.15\textwidth}
    \centering
    \includegraphics[width=\textwidth]{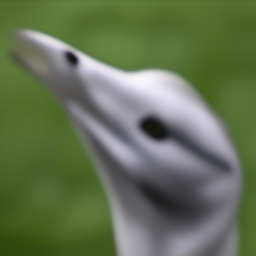} 
    \end{subfigure}
    \begin{subfigure}[h]{0.15\textwidth}
    \centering
    \includegraphics[width=\textwidth]{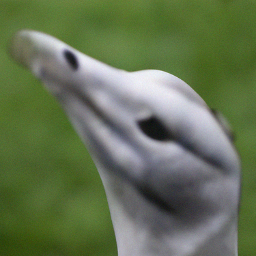} 
    \end{subfigure}
    \begin{subfigure}[h]{0.15\textwidth}
    \centering
    \includegraphics[width=\textwidth]{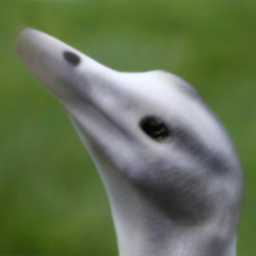} 
    \end{subfigure}
    \caption{Noisy image restoration ($\sigma_\rvy = 0.05$) experiments from Table~\ref{tab:exp2}. Zoom in for the best view.}\label{fig:noisyexp1}
\end{figure}
\begin{figure}[h!]
    \centering
    \begin{subfigure}[h]{0.025\textwidth}
    \centering    
    \subfloat{\bf{\color{gray}{ }}}
    \end{subfigure}
    \begin{subfigure}[h]{0.15\textwidth}
    \centering
    \subfloat{\tiny{\bf{{\color{gray}{$\textit{Original}_{\ }$}}}}}    
    \end{subfigure}
    \begin{subfigure}[h]{0.15\textwidth}
    \centering
    \subfloat{\tiny{\bf{{\color{gray}{$\textit{Measurement}_{\ }$}}}}}    
    \end{subfigure}
    \begin{subfigure}[h]{0.15\textwidth}
    \centering
    \subfloat{\tiny{{\bf{\color{gray}{$\Pi\textit{GDM}_{\ }$}}}}}    
    \end{subfigure}
    \begin{subfigure}[h]{0.15\textwidth}
    \centering
    \subfloat{\tiny{{\bf{\color{gray}{$\textit{DAPS-4K}_{\ }$}}}}}    
    \end{subfigure}
    \begin{subfigure}[h]{0.15\textwidth}
    \centering
    \subfloat{\tiny{\bf{\color{gray}{$\textit{VML-MAP}_{\ }$}}}}
    \end{subfigure}
    \begin{subfigure}[h]{0.15\textwidth}
    \centering
    \subfloat{\tiny{\bf{{\color{gray}{$\textit{VML-MAP}_{pre}$}}}}}   
    \end{subfigure}
    \\[2pt]
    \begin{subfigure}[h]{0.15\textwidth}
    \centering
    \includegraphics[width=\textwidth]{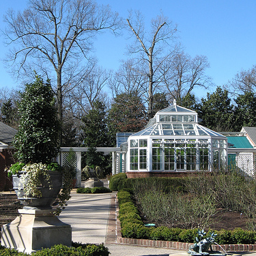}    
    \end{subfigure}
    \begin{subfigure}[h]{0.15\textwidth}
    \centering
    \includegraphics[width=\textwidth]{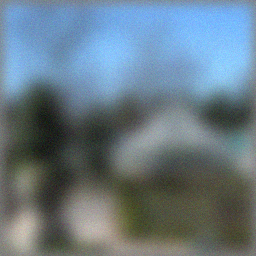}    
    \end{subfigure}
    \begin{subfigure}[h]{0.15\textwidth}
    \centering
    \includegraphics[width=\textwidth]{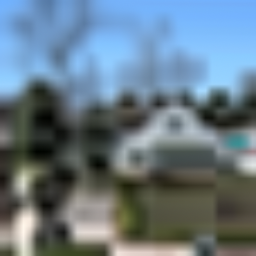}  
    \end{subfigure}
    \begin{subfigure}[h]{0.15\textwidth}
    \centering
    \includegraphics[width=\textwidth]{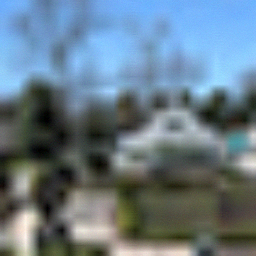}    
    \end{subfigure}
    \begin{subfigure}[h]{0.15\textwidth}
    \centering
    \includegraphics[width=\textwidth]{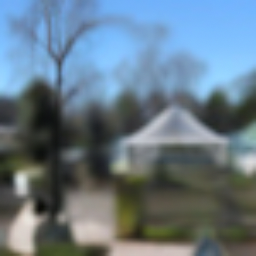}
    \end{subfigure}
    \begin{subfigure}[h]{0.15\textwidth}
    \centering
    \includegraphics[width=\textwidth]{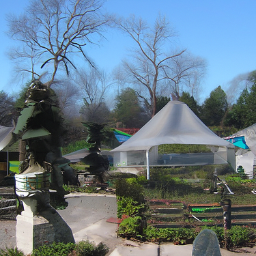}    
    \end{subfigure}
    \\
    \begin{subfigure}[h]{0.15\textwidth}
    \centering
    \includegraphics[width=\textwidth]{icml2026/figures/noisy_imagenet256new/deblur16/im36_original.png}    
    \end{subfigure}
    \begin{subfigure}[h]{0.15\textwidth}
    \centering
    \includegraphics[width=\textwidth]{icml2026/figures/noisy_imagenet256new/deblur16/im36_degraded.png}    
    \end{subfigure}
    \begin{subfigure}[h]{0.15\textwidth}
    \centering
    \includegraphics[width=\textwidth]{icml2026/figures/noisy_imagenet256new/deblur16/im36_recovered_pgdm.png} 
    \end{subfigure}
    \begin{subfigure}[h]{0.15\textwidth}
    \centering
    \includegraphics[width=\textwidth]{icml2026/figures/noisy_imagenet256new/deblur16/im36_recovered_daps4k.png} 
    \end{subfigure}
    \begin{subfigure}[h]{0.15\textwidth}
    \centering
    \includegraphics[width=\textwidth]{icml2026/figures/noisy_imagenet256new/deblur16/im36_recovered_vmlmap.png}
    \end{subfigure}
    \begin{subfigure}[h]{0.15\textwidth}
    \centering
    \includegraphics[width=\textwidth]{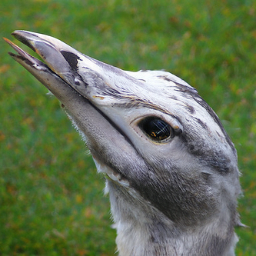}
    \end{subfigure}
    \caption{Noisy deblurring task from Table~\ref{tab:exp2}. Zoom in for the best view.}\label{fig:noisyexp2}
\end{figure}

\subsection{Additional experiments on noiseless image inpainting with diverse masks}\label{ssec:inpexp}
As mentioned in Section~\ref{ssec:noiseless-image-restoration}, DDRM, $\Pi$GDM, MAPGA, and $\text{VML-MAP}_{pre}$ require the singular value decomposition (SVD) of the linear degradation matrix $\mathrm{H}$, while DAPS and VML-MAP only require the forward operation of $\mathrm{H}$. For inpainting, $\mathrm{H}$ is a diagonal matrix with zeros for indices corresponding to masked pixels, and ones for indices corresponding to observed pixels. The SVD of $\mathrm{H}$ in this case is trivial, since $\mathrm{H}$ itself is the singular value matrix, with the left and the right singular matrices being identity. This ensures a fair comparison among all the methods, irrespective of whether a method requires the SVD of $\mathrm{H}$ or not. Therefore, in this section, we focus on the noiseless image inpainting task with different types of masks as described below. 

\begin{itemize}[label=\textbullet]
    \item \textbf{Expand mask:} Pixels outside the $128\times 128$ square center-crop are masked 
    \item \textbf{Box mask:} Pixels within the $128\times 128$ square center-crop are masked 
    \item \textbf{Super-resolution mask:} Alternative pixels are masked 
    \item \textbf{Random mask:} $70\%$ of the pixels are randomly masked 
\end{itemize}
We use $\gamma_0 = 1.0$ (i.e., a learning rate of $1.0 \, \sigma^2_\rvy$) for all tasks on both $\text{ImageNet}256$ and $\text{FFHQ}256$. For other hyperparameters, we follow the same settings mentioned in Appendix~\ref{ssec:implementationdetails1}. The results reported in Table~\ref{tab:inpexp}, and the qualitative visualizations in Figures~\ref{fig:qv4},~\ref{fig:qv5},~\ref{fig:qv6}, and~\ref{fig:qv7} reveal the superior performance of VML-MAP over existing baselines for all the inpainting tasks.

\begin{table}[t!]
\caption{Evaluation of image restoration methods on noiseless inpainting with expand mask, box mask, super-resolution mask, and random mask on $100$ validation images of ImageNet$256$, and on $100$ images of FFHQ$256$. Best values in \textbf{bold}, second best values \underline{underlined}.}
\label{tab:inpexp}
\begin{center}
\resizebox{0.75\textwidth}{!}{
\begin{tabular}{ l l l l l l c c c c c c c c }
\toprule
\multicolumn{3}{c}{} & \multicolumn{3}{l}{} & \multicolumn{8}{c}{} \\[-7pt]
\multicolumn{3}{c}{{Dataset}}  & \multicolumn{3}{l}{{Method}} & \multicolumn{2}{c}{{Expand mask}} & \multicolumn{2}{c}{{Box mask}}  & \multicolumn{2}{c}{{Sup-res mask}} & \multicolumn{2}{c}{{Random mask}} \\[3pt]
 \multicolumn{3}{l}{} & \multicolumn{3}{l}{} & \multicolumn{1}{c}{\scriptsize{LPIPS$\downarrow$}} & \multicolumn{1}{c}{\scriptsize{FID$\downarrow$}} & \multicolumn{1}{c}{\scriptsize{LPIPS$\downarrow$}} & \multicolumn{1}{c}{\scriptsize{FID$\downarrow$}} & \multicolumn{1}{c}{\scriptsize{LPIPS$\downarrow$}} & \multicolumn{1}{c}{\scriptsize{FID$\downarrow$}} & \multicolumn{1}{c}{\scriptsize{LPIPS$\downarrow$}} & \multicolumn{1}{c}{\scriptsize{FID$\downarrow$}}\\[3pt]
 \toprule
 \multicolumn{3}{c}{} & \multicolumn{3}{l}{} & \multicolumn{8}{c}{}\\[-7pt]
 \multicolumn{3}{c}{} & \multicolumn{3}{l}{{\small{$\text{DDRM}$}}}  & 
 \multicolumn{1}{c}{\scriptsize{$0.548$}} & \multicolumn{1}{c}{\scriptsize{$152.0$}}  & \multicolumn{1}{c}{\scriptsize{$ 0.211$}} & \multicolumn{1}{c}{\scriptsize{$133.1$}} & \multicolumn{1}{c}{\scriptsize{$0.089$}} & \multicolumn{1}{c}{\scriptsize{$29.24$}} & \multicolumn{1}{c}{\scriptsize{$0.052$}} & \multicolumn{1}{c}{\scriptsize{$21.81$}}\\
 \multicolumn{3}{c}{} & \multicolumn{3}{l}{{\small{$\Pi\text{GDM}$}}}  & \multicolumn{1}{c}{\scriptsize{$0.523$}} & \multicolumn{1}{c}{\scriptsize{$155.7$}}  & \multicolumn{1}{c}{\scriptsize{$0.203$}} & \multicolumn{1}{c}{\scriptsize{$123.9$}} & \multicolumn{1}{c}{\scriptsize{$\underline{0.080}$}} & \multicolumn{1}{c}{\scriptsize{$26.54$}} & \multicolumn{1}{c}{\scriptsize{$0.054$}} & \multicolumn{1}{c}{\scriptsize{$24.48$}}\\ 
 \multicolumn{3}{c}{} & \multicolumn{3}{l}{{\small{$\text{MAPGA}$}}}  & \multicolumn{1}{c}{\scriptsize{$\underline{0.466}$}} & \multicolumn{1}{c}{\scriptsize{$\underline{126.9}$}}  & \multicolumn{1}{c}{\scriptsize{$\underline{0.150}$}} & \multicolumn{1}{c}{\scriptsize{$\underline{91.14}$}} & \multicolumn{1}{c}{\scriptsize{$0.087$}} & \multicolumn{1}{c}{\scriptsize{$\underline{25.02}$}} & \multicolumn{1}{c}{\scriptsize{$\underline{0.051}$}} & \multicolumn{1}{c}{\scriptsize{$\underline{20.31}$}}\\ 
 \multicolumn{3}{c}{\small{ImageNet}} & \multicolumn{3}{l}{{\small{$\text{DAPS-1K}$}}}  & \multicolumn{1}{c}{\scriptsize{$0.550$}} & \multicolumn{1}{c}{\scriptsize{$167.2$}}  & \multicolumn{1}{c}{\scriptsize{$0.201$}} & \multicolumn{1}{c}{\scriptsize{$112.6$}} & \multicolumn{1}{c}{\scriptsize{$0.132$}} & \multicolumn{1}{c}{\scriptsize{$53.46$}} & \multicolumn{1}{c}{\scriptsize{$0.092$}} & \multicolumn{1}{c}{\scriptsize{$35.93$}}\\
\multicolumn{3}{c}{} & \multicolumn{3}{l}{{\small{$\text{DAPS-4K}$}}}  & \multicolumn{1}{c}{\scriptsize{$0.521$}} & \multicolumn{1}{c}{\scriptsize{$154.0$}}  & \multicolumn{1}{c}{\scriptsize{$0.187$}} & \multicolumn{1}{c}{\scriptsize{$102.8$}} & \multicolumn{1}{c}{\scriptsize{$0.114$}} & \multicolumn{1}{c}{\scriptsize{$45.65$}} & \multicolumn{1}{c}{\scriptsize{$0.082$}} & \multicolumn{1}{c}{\scriptsize{$31.67$}}\\[2pt]
\cline{4-14}
\multicolumn{3}{c}{} & \multicolumn{3}{l}{} & \multicolumn{8}{c}{}\\[-7pt]
\multicolumn{3}{c}{} & \multicolumn{3}{l}{\bf{\small{$\text{VML-MAP}$}}} & \multicolumn{1}{c}{\scriptsize{$\bf{0.434}$}} & \multicolumn{1}{c}{\scriptsize{$\bf{116.2}$}} & \multicolumn{1}{c}{\scriptsize{$\bf{0.138}$}} & \multicolumn{1}{c}{\scriptsize{$\bf{75.80}$}} & \multicolumn{1}{c}{\scriptsize{$\bf{0.068}$}} & \multicolumn{1}{c}{\scriptsize{$\bf{19.96}$}} & \multicolumn{1}{c}{\scriptsize{$\bf{0.044}$}} & \multicolumn{1}{c}{\scriptsize{$\bf{16.14}$}}\\[2pt]
\midrule
 \multicolumn{3}{c}{} & \multicolumn{3}{l}{} & \multicolumn{8}{c}{}\\[-7pt]
 \multicolumn{3}{c}{} & \multicolumn{3}{l}{{\small{$\text{DDRM}$}}}  & 
 \multicolumn{1}{c}{\scriptsize{$0.426$}} & \multicolumn{1}{c}{\scriptsize{$151.1$}}  & \multicolumn{1}{c}{\scriptsize{$0.087$}} & \multicolumn{1}{c}{\scriptsize{$50.60$}} & \multicolumn{1}{c}{\scriptsize{$\underline{0.031}$}} & \multicolumn{1}{c}{\scriptsize{$\bf{20.75}$}} & \multicolumn{1}{c}{\scriptsize{$\underline{0.027}$}} & \multicolumn{1}{c}{\scriptsize{$\underline{18.79}$}}\\
 \multicolumn{3}{c}{} & \multicolumn{3}{l}{{\small{$\Pi\text{GDM}$}}}  & \multicolumn{1}{c}{\scriptsize{$0.415$}} & \multicolumn{1}{c}{\scriptsize{$146.3$}}  & \multicolumn{1}{c}{\scriptsize{$0.084$}} & \multicolumn{1}{c}{\scriptsize{$46.95$}} & \multicolumn{1}{c}{\scriptsize{$0.033$}} & \multicolumn{1}{c}{\scriptsize{$23.98$}} & \multicolumn{1}{c}{\scriptsize{$\underline{0.027}$}} & \multicolumn{1}{c}{\scriptsize{$19.93$}}\\ 
 \multicolumn{3}{c}{} & \multicolumn{3}{l}{{\small{$\text{MAPGA}$}}}  & \multicolumn{1}{c}{\scriptsize{$\underline{0.393}$}} & \multicolumn{1}{c}{\scriptsize{$129.3$}}  & \multicolumn{1}{c}{\scriptsize{$\underline{0.070}$}} & \multicolumn{1}{c}{\scriptsize{$41.34$}} & \multicolumn{1}{c}{\scriptsize{$0.036$}} & \multicolumn{1}{c}{\scriptsize{$27.01$}} & \multicolumn{1}{c}{\scriptsize{$\underline{0.027}$}} & \multicolumn{1}{c}{\scriptsize{$20.88$}}\\ 
 \multicolumn{3}{c}{\small{FFHQ}} & \multicolumn{3}{l}{{\small{$\text{DAPS-1K}$}}}  & \multicolumn{1}{c}{\scriptsize{$0.423$}} & \multicolumn{1}{c}{\scriptsize{$\underline{126.7}$}}  & \multicolumn{1}{c}{\scriptsize{$0.076$}} & \multicolumn{1}{c}{\scriptsize{$\underline{33.17}$}} & \multicolumn{1}{c}{\scriptsize{$0.057$}} & \multicolumn{1}{c}{\scriptsize{$47.21$}} & \multicolumn{1}{c}{\scriptsize{$ 0.059$}} & \multicolumn{1}{c}{\scriptsize{$45.55$}}\\
 \multicolumn{3}{c}{} & \multicolumn{3}{l}{{\small{$\text{DAPS-4K}$}}}  & \multicolumn{1}{c}{\scriptsize{$0.398$}} & \multicolumn{1}{c}{\scriptsize{$117.5$}}  & \multicolumn{1}{c}{\scriptsize{$0.077$}} & \multicolumn{1}{c}{\scriptsize{$33.82$}} & \multicolumn{1}{c}{\scriptsize{$0.051$}} & \multicolumn{1}{c}{\scriptsize{$42.32$}} & \multicolumn{1}{c}{\scriptsize{$0.051$}} & \multicolumn{1}{c}{\scriptsize{$40.06$}}\\[2pt]
 \cline{4-14}
 \multicolumn{3}{c}{} & \multicolumn{3}{l}{} & \multicolumn{8}{c}{}\\[-7pt]
 \multicolumn{3}{c}{} & \multicolumn{3}{l}{\bf{\small{$\text{VML-MAP}$}}}  & \multicolumn{1}{c}{\scriptsize{$\bf{0.365}$}} & \multicolumn{1}{c}{\scriptsize{$\bf{112.9}$}} & \multicolumn{1}{c}{\scriptsize{$\bf{0.057}$}} & \multicolumn{1}{c}{\scriptsize{$\bf{28.38}$}} & \multicolumn{1}{c}{\scriptsize{$\bf{0.027}$}} & \multicolumn{1}{c}{\scriptsize{$\underline{21.40}$}} & \multicolumn{1}{c}{\scriptsize{$\bf{0.022}$}} & \multicolumn{1}{c}{\scriptsize{$\bf{16.30}$}}\\[2pt]
 \bottomrule
\end{tabular}}
\end{center}
\end{table}
\begin{figure}[h!]
    \centering
    \begin{subfigure}[h]{0.125\textwidth}
    \centering
    \subfloat{\scriptsize{\bf{\color{gray}{\textit{Original}}}}}
    \end{subfigure}
    \begin{subfigure}[h]{0.125\textwidth}
    \centering
    \subfloat{\scriptsize{\bf{\color{gray}{\textit{Measurement}}}}}
    \end{subfigure}
    \begin{subfigure}[h]{0.125\textwidth}
    \centering
    \subfloat{\scriptsize{\bf{\color{gray}{\textit{DDRM}}}}}
    \end{subfigure}
    \begin{subfigure}[h]{0.125\textwidth}
    \centering
    \subfloat{\scriptsize{\bf{\color{gray}{\textit{$\Pi$GDM}}}}}
    \end{subfigure}
    \begin{subfigure}[h]{0.125\textwidth}
    \centering
    \subfloat{\scriptsize{\bf{\color{gray}{\textit{DAPS-1K}}}}}
    \end{subfigure}
    \begin{subfigure}[h]{0.125\textwidth}
    \centering
    \subfloat{\scriptsize{\bf{\color{gray}{\textit{DAPS-4K}}}}}
    \end{subfigure}
    \begin{subfigure}[h]{0.125\textwidth}
    \centering
    \subfloat{\scriptsize{\bf{\color{gray}{\textit{MAPGA}}}}}
    \end{subfigure}
    \begin{subfigure}[h]{0.125\textwidth}
    \centering
    \subfloat{\scriptsize{\bf{\color{gray}{\textit{VML-MAP}}}}}
    \end{subfigure}
    \\[2pt]
    \begin{subfigure}[h]{0.125\textwidth}
    \centering
    \includegraphics[width=\textwidth]{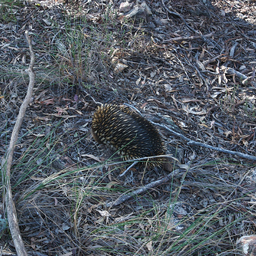}    
    \end{subfigure}
    \begin{subfigure}[h]{0.125\textwidth}
    \centering
    \includegraphics[width=\textwidth]{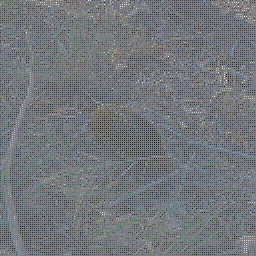}    
    \end{subfigure}
    \begin{subfigure}[h]{0.125\textwidth}
    \centering
    \includegraphics[width=\textwidth]{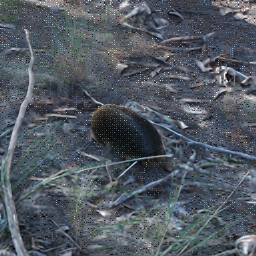}  
    \end{subfigure}
    \begin{subfigure}[h]{0.125\textwidth}
    \centering
    \includegraphics[width=\textwidth]{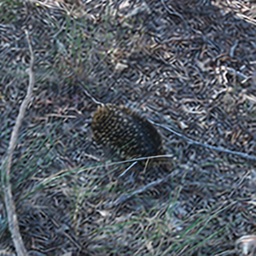}    
    \end{subfigure}
    \begin{subfigure}[h]{0.125\textwidth}
    \centering
    \includegraphics[width=\textwidth]{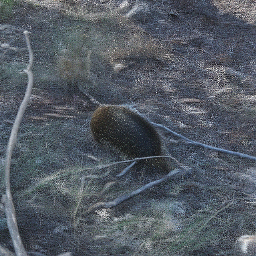}  
    \end{subfigure}
    \begin{subfigure}[h]{0.125\textwidth}
    \centering
    \includegraphics[width=\textwidth]{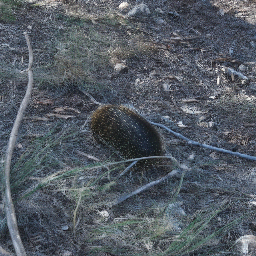}
    \end{subfigure}
    \begin{subfigure}[h]{0.125\textwidth}
    \centering
    \includegraphics[width=\textwidth]{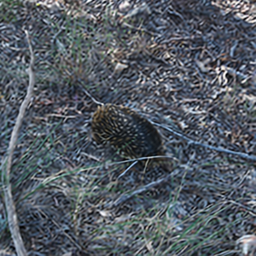}    
    \end{subfigure}
    \begin{subfigure}[h]{0.125\textwidth}
    \centering
    \includegraphics[width=\textwidth]{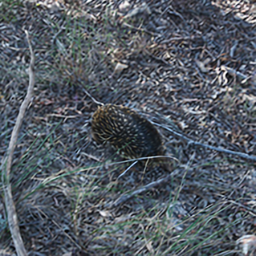}
    \end{subfigure}
    \\ 
    \begin{subfigure}[h]{0.125\textwidth}
    \centering
    \includegraphics[width=\textwidth]{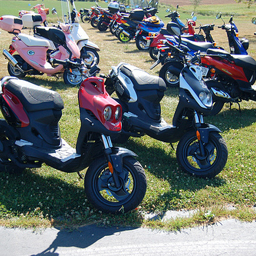}    
    \end{subfigure}
    \begin{subfigure}[h]{0.125\textwidth}
    \centering
    \includegraphics[width=\textwidth]{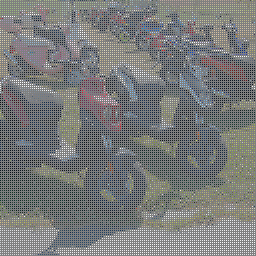}    
    \end{subfigure}
    \begin{subfigure}[h]{0.125\textwidth}
    \centering
    \includegraphics[width=\textwidth]{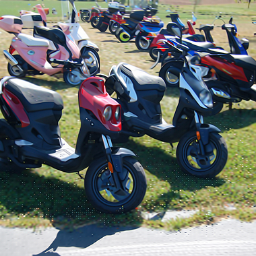}  
    \end{subfigure}
    \begin{subfigure}[h]{0.125\textwidth}
    \centering
    \includegraphics[width=\textwidth]{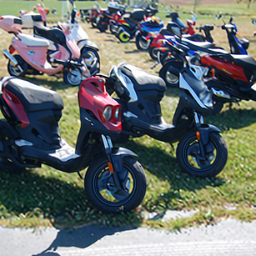}    
    \end{subfigure}
    \begin{subfigure}[h]{0.125\textwidth}
    \centering
    \includegraphics[width=\textwidth]{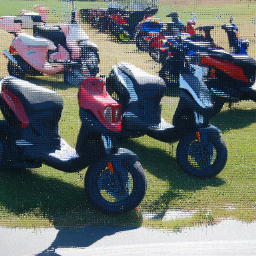}  
    \end{subfigure}
    \begin{subfigure}[h]{0.125\textwidth}
    \centering
    \includegraphics[width=\textwidth]{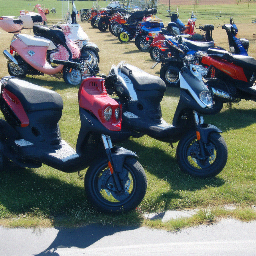}
    \end{subfigure}
    \begin{subfigure}[h]{0.125\textwidth}
    \centering
    \includegraphics[width=\textwidth]{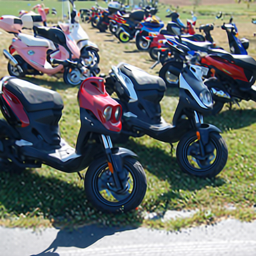}    
    \end{subfigure}
    \begin{subfigure}[h]{0.125\textwidth}
    \centering
    \includegraphics[width=\textwidth]{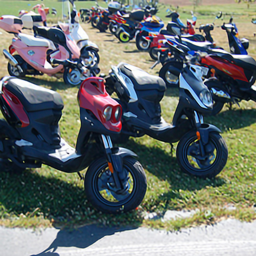}
    \end{subfigure}
    \caption{Super-res-mask inpainting task from Table~\ref{tab:inpexp}. Zoom in for the best view.}\label{fig:qv4}
\end{figure}
\begin{figure}[h!]
    \centering
    \begin{subfigure}[h]{0.125\textwidth}
    \centering
    \subfloat{\scriptsize{\bf{\color{gray}{\textit{Original}}}}}
    \end{subfigure}
    \begin{subfigure}[h]{0.125\textwidth}
    \centering
    \subfloat{\scriptsize{\bf{\color{gray}{\textit{Measurement}}}}}
    \end{subfigure}
    \begin{subfigure}[h]{0.125\textwidth}
    \centering
    \subfloat{\scriptsize{\bf{\color{gray}{\textit{DDRM}}}}}
    \end{subfigure}
    \begin{subfigure}[h]{0.125\textwidth}
    \centering
    \subfloat{\scriptsize{\bf{\color{gray}{\textit{$\Pi$GDM}}}}}
    \end{subfigure}
    \begin{subfigure}[h]{0.125\textwidth}
    \centering
    \subfloat{\scriptsize{\bf{\color{gray}{\textit{DAPS-1K}}}}}
    \end{subfigure}
    \begin{subfigure}[h]{0.125\textwidth}
    \centering
    \subfloat{\scriptsize{\bf{\color{gray}{\textit{DAPS-4K}}}}}
    \end{subfigure}
    \begin{subfigure}[h]{0.125\textwidth}
    \centering
    \subfloat{\scriptsize{\bf{\color{gray}{\textit{MAPGA}}}}}
    \end{subfigure}
    \begin{subfigure}[h]{0.125\textwidth}
    \centering
    \subfloat{\scriptsize{\bf{\color{gray}{\textit{VML-MAP}}}}}
    \end{subfigure}
    \\[2pt]
    \begin{subfigure}[h]{0.125\textwidth}
    \centering
    \includegraphics[width=\textwidth]{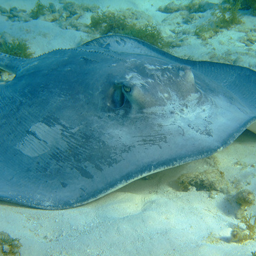}    
    \end{subfigure}
    \begin{subfigure}[h]{0.125\textwidth}
    \centering
    \includegraphics[width=\textwidth]{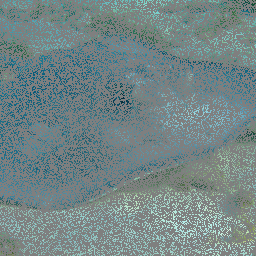}    
    \end{subfigure}
    \begin{subfigure}[h]{0.125\textwidth}
    \centering
    \includegraphics[width=\textwidth]{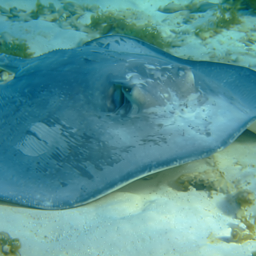}  
    \end{subfigure}
    \begin{subfigure}[h]{0.125\textwidth}
    \centering
    \includegraphics[width=\textwidth]{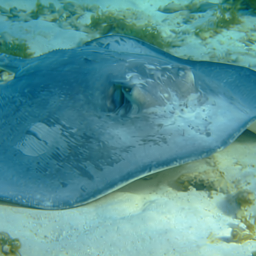}    
    \end{subfigure}
    \begin{subfigure}[h]{0.125\textwidth}
    \centering
    \includegraphics[width=\textwidth]{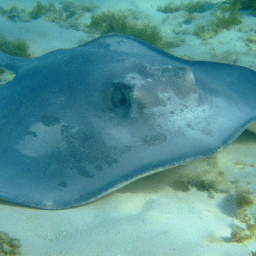}  
    \end{subfigure}
    \begin{subfigure}[h]{0.125\textwidth}
    \centering
    \includegraphics[width=\textwidth]{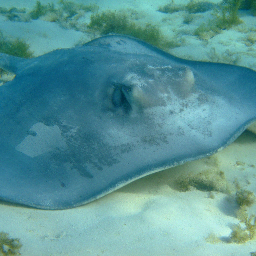}
    \end{subfigure}
    \begin{subfigure}[h]{0.125\textwidth}
    \centering
    \includegraphics[width=\textwidth]{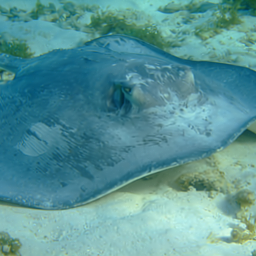}    
    \end{subfigure}
    \begin{subfigure}[h]{0.125\textwidth}
    \centering
    \includegraphics[width=\textwidth]{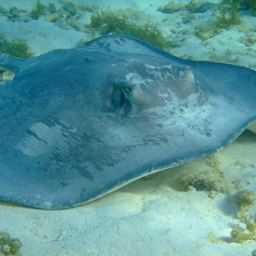}
    \end{subfigure}
    \\ 
    \begin{subfigure}[h]{0.125\textwidth}
    \centering
    \includegraphics[width=\textwidth]{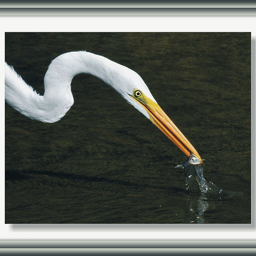}    
    \end{subfigure}
    \begin{subfigure}[h]{0.125\textwidth}
    \centering
    \includegraphics[width=\textwidth]{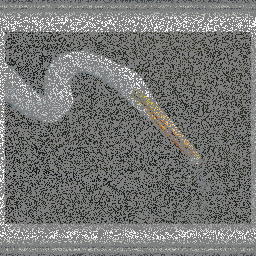}    
    \end{subfigure}
    \begin{subfigure}[h]{0.125\textwidth}
    \centering
    \includegraphics[width=\textwidth]{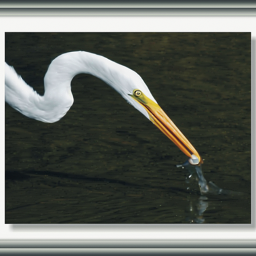}  
    \end{subfigure}
    \begin{subfigure}[h]{0.125\textwidth}
    \centering
    \includegraphics[width=\textwidth]{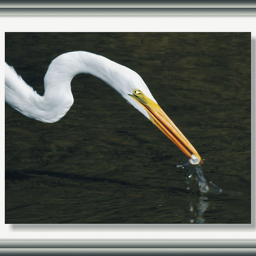}    
    \end{subfigure}
    \begin{subfigure}[h]{0.125\textwidth}
    \centering
    \includegraphics[width=\textwidth]{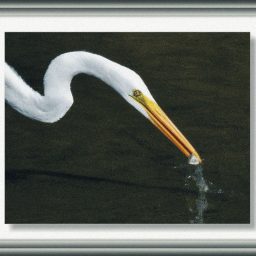}  
    \end{subfigure}
    \begin{subfigure}[h]{0.125\textwidth}
    \centering
    \includegraphics[width=\textwidth]{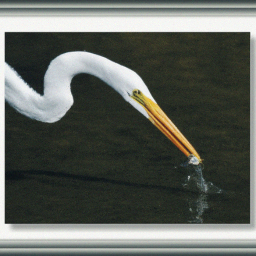}
    \end{subfigure}
    \begin{subfigure}[h]{0.125\textwidth}
    \centering
    \includegraphics[width=\textwidth]{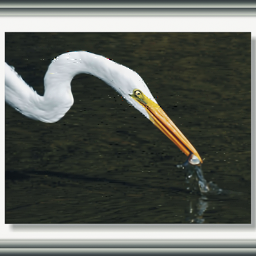}    
    \end{subfigure}
    \begin{subfigure}[h]{0.125\textwidth}
    \centering
    \includegraphics[width=\textwidth]{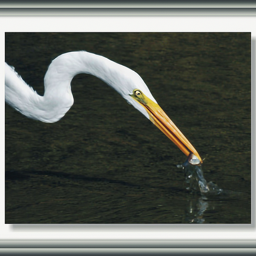}
    \end{subfigure}
    \caption{Random-mask inpainting task from Table~\ref{tab:inpexp}. Zoom in for the best view.}\label{fig:qv5}
\end{figure}
\begin{figure}[h!]
    \centering
    \begin{subfigure}[h]{0.125\textwidth}
    \centering
    \subfloat{\scriptsize{\bf{\color{gray}{\textit{Original}}}}}
    \end{subfigure}
    \begin{subfigure}[h]{0.125\textwidth}
    \centering
    \subfloat{\scriptsize{\bf{\color{gray}{\textit{Measurement}}}}}
    \end{subfigure}
    \begin{subfigure}[h]{0.125\textwidth}
    \centering
    \subfloat{\scriptsize{\bf{\color{gray}{\textit{DDRM}}}}}
    \end{subfigure}
    \begin{subfigure}[h]{0.125\textwidth}
    \centering
    \subfloat{\scriptsize{\bf{\color{gray}{\textit{$\Pi$GDM}}}}}
    \end{subfigure}
    \begin{subfigure}[h]{0.125\textwidth}
    \centering
    \subfloat{\scriptsize{\bf{\color{gray}{\textit{DAPS-1K}}}}}
    \end{subfigure}
    \begin{subfigure}[h]{0.125\textwidth}
    \centering
    \subfloat{\scriptsize{\bf{\color{gray}{\textit{DAPS-4K}}}}}
    \end{subfigure}
    \begin{subfigure}[h]{0.125\textwidth}
    \centering
    \subfloat{\scriptsize{\bf{\color{gray}{\textit{MAPGA}}}}}
    \end{subfigure}
    \begin{subfigure}[h]{0.125\textwidth}
    \centering
    \subfloat{\scriptsize{\bf{\color{gray}{\textit{VML-MAP}}}}}
    \end{subfigure}
    \\[2pt]
    \begin{subfigure}[h]{0.125\textwidth}
    \centering
    \includegraphics[width=\textwidth]{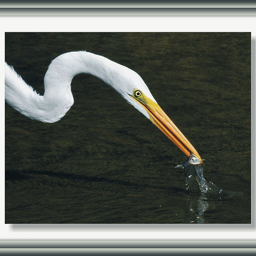}    
    \end{subfigure}
    \begin{subfigure}[h]{0.125\textwidth}
    \centering
    \includegraphics[width=\textwidth]{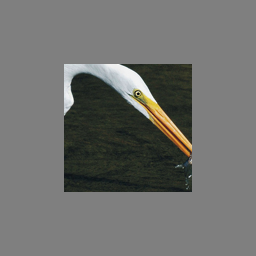}    
    \end{subfigure}
    \begin{subfigure}[h]{0.125\textwidth}
    \centering
    \includegraphics[width=\textwidth]{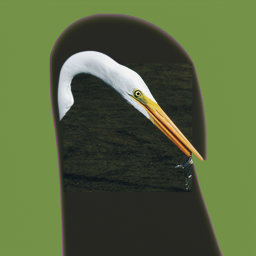}  
    \end{subfigure}
    \begin{subfigure}[h]{0.125\textwidth}
    \centering
    \includegraphics[width=\textwidth]{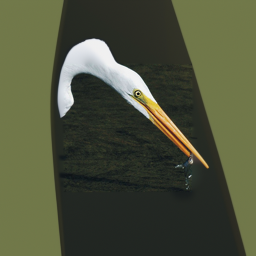}    
    \end{subfigure}
    \begin{subfigure}[h]{0.125\textwidth}
    \centering
    \includegraphics[width=\textwidth]{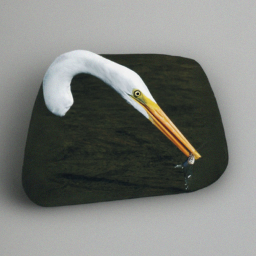}  
    \end{subfigure}
    \begin{subfigure}[h]{0.125\textwidth}
    \centering
    \includegraphics[width=\textwidth]{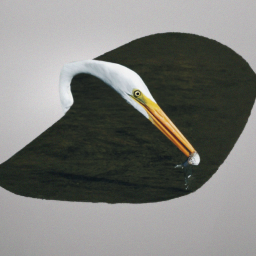}
    \end{subfigure}
    \begin{subfigure}[h]{0.125\textwidth}
    \centering
    \includegraphics[width=\textwidth]{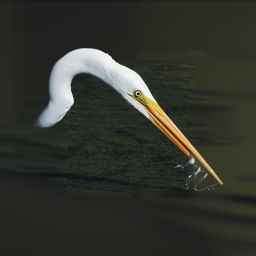}    
    \end{subfigure}
    \begin{subfigure}[h]{0.125\textwidth}
    \centering
    \includegraphics[width=\textwidth]{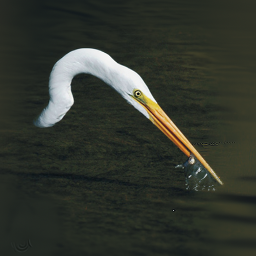}
    \end{subfigure}
    \\ 
    \begin{subfigure}[h]{0.125\textwidth}
    \centering
    \includegraphics[width=\textwidth]{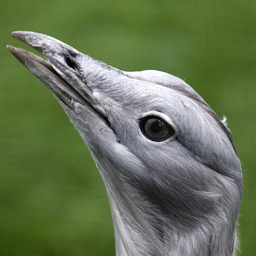}    
    \end{subfigure}
    \begin{subfigure}[h]{0.125\textwidth}
    \centering
    \includegraphics[width=\textwidth]{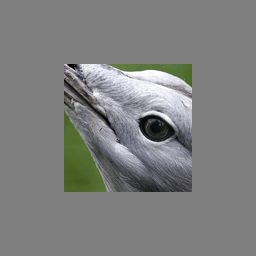}    
    \end{subfigure}
    \begin{subfigure}[h]{0.125\textwidth}
    \centering
    \includegraphics[width=\textwidth]{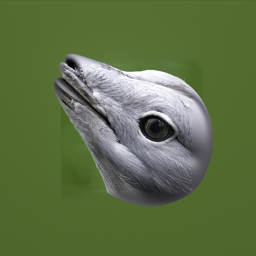}  
    \end{subfigure}
    \begin{subfigure}[h]{0.125\textwidth}
    \centering
    \includegraphics[width=\textwidth]{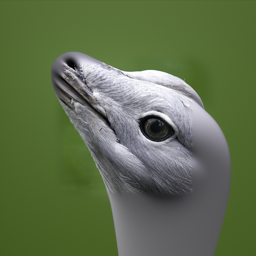}    
    \end{subfigure}
    \begin{subfigure}[h]{0.125\textwidth}
    \centering
    \includegraphics[width=\textwidth]{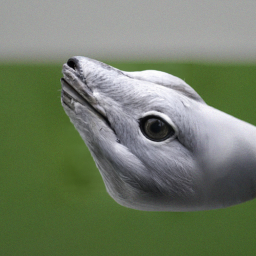}  
    \end{subfigure}
    \begin{subfigure}[h]{0.125\textwidth}
    \centering
    \includegraphics[width=\textwidth]{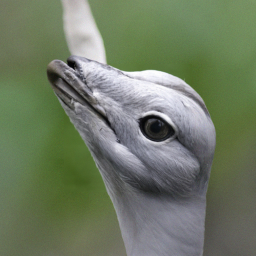}
    \end{subfigure}
    \begin{subfigure}[h]{0.125\textwidth}
    \centering
    \includegraphics[width=\textwidth]{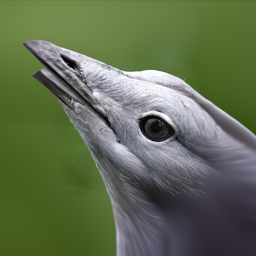}    
    \end{subfigure}
    \begin{subfigure}[h]{0.125\textwidth}
    \centering
    \includegraphics[width=\textwidth]{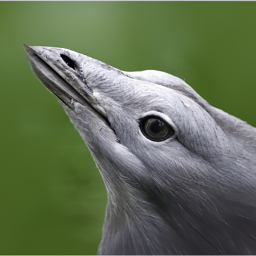}
    \end{subfigure}
    \\ 
    \begin{subfigure}[h]{0.125\textwidth}
    \centering
    \includegraphics[width=\textwidth]{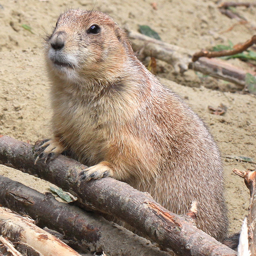}    
    \end{subfigure}
    \begin{subfigure}[h]{0.125\textwidth}
    \centering
    \includegraphics[width=\textwidth]{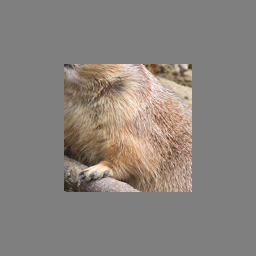}    
    \end{subfigure}
    \begin{subfigure}[h]{0.125\textwidth}
    \centering
    \includegraphics[width=\textwidth]{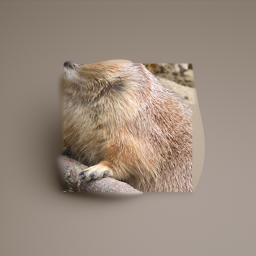}  
    \end{subfigure}
    \begin{subfigure}[h]{0.125\textwidth}
    \centering
    \includegraphics[width=\textwidth]{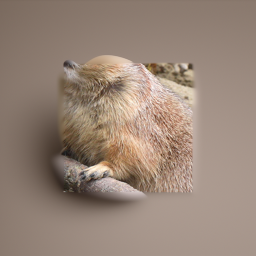}    
    \end{subfigure}
    \begin{subfigure}[h]{0.125\textwidth}
    \centering
    \includegraphics[width=\textwidth]{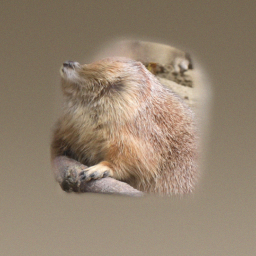}  
    \end{subfigure}
    \begin{subfigure}[h]{0.125\textwidth}
    \centering
    \includegraphics[width=\textwidth]{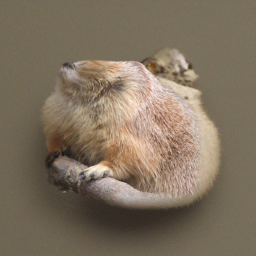}
    \end{subfigure}
    \begin{subfigure}[h]{0.125\textwidth}
    \centering
    \includegraphics[width=\textwidth]{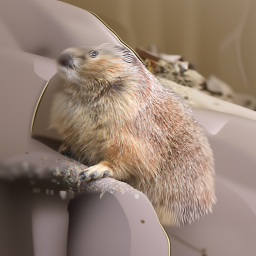}    
    \end{subfigure}
    \begin{subfigure}[h]{0.125\textwidth}
    \centering
    \includegraphics[width=\textwidth]{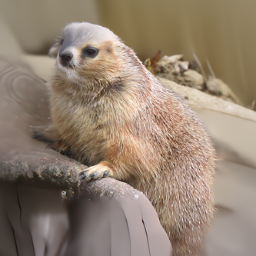}
    \end{subfigure}
    \caption{Expand-mask inpainting task from Table~\ref{tab:inpexp}. Zoom in for the best view.}\label{fig:qv6}
\end{figure}
\begin{figure}[h!]
    \centering
    \begin{subfigure}[h]{0.125\textwidth}
    \centering
    \subfloat{\scriptsize{\bf{\color{gray}{\textit{Original}}}}}
    \end{subfigure}
    \begin{subfigure}[h]{0.125\textwidth}
    \centering
    \subfloat{\scriptsize{\bf{\color{gray}{\textit{Measurement}}}}}
    \end{subfigure}
    \begin{subfigure}[h]{0.125\textwidth}
    \centering
    \subfloat{\scriptsize{\bf{\color{gray}{\textit{DDRM}}}}}
    \end{subfigure}
    \begin{subfigure}[h]{0.125\textwidth}
    \centering
    \subfloat{\scriptsize{\bf{\color{gray}{\textit{$\Pi$GDM}}}}}
    \end{subfigure}
    \begin{subfigure}[h]{0.125\textwidth}
    \centering
    \subfloat{\scriptsize{\bf{\color{gray}{\textit{DAPS-1K}}}}}
    \end{subfigure}
    \begin{subfigure}[h]{0.125\textwidth}
    \centering
    \subfloat{\scriptsize{\bf{\color{gray}{\textit{DAPS-4K}}}}}
    \end{subfigure}
    \begin{subfigure}[h]{0.125\textwidth}
    \centering
    \subfloat{\scriptsize{\bf{\color{gray}{\textit{MAPGA}}}}}
    \end{subfigure}
    \begin{subfigure}[h]{0.125\textwidth}
    \centering
    \subfloat{\scriptsize{\bf{\color{gray}{\textit{VML-MAP}}}}}
    \end{subfigure}
    \\[2pt]
    \begin{subfigure}[h]{0.125\textwidth}
    \centering
    \includegraphics[width=\textwidth]{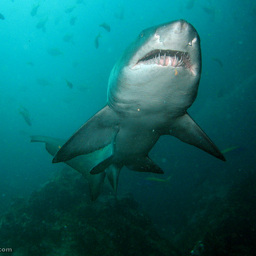}    
    \end{subfigure}
    \begin{subfigure}[h]{0.125\textwidth}
    \centering
    \includegraphics[width=\textwidth]{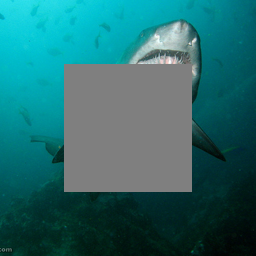}    
    \end{subfigure}
    \begin{subfigure}[h]{0.125\textwidth}
    \centering
    \includegraphics[width=\textwidth]{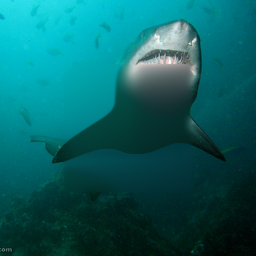}  
    \end{subfigure}
    \begin{subfigure}[h]{0.125\textwidth}
    \centering
    \includegraphics[width=\textwidth]{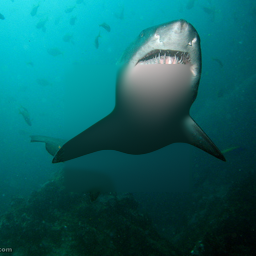}    
    \end{subfigure}
    \begin{subfigure}[h]{0.125\textwidth}
    \centering
    \includegraphics[width=\textwidth]{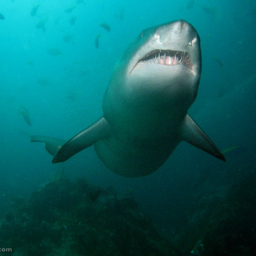}  
    \end{subfigure}
    \begin{subfigure}[h]{0.125\textwidth}
    \centering
    \includegraphics[width=\textwidth]{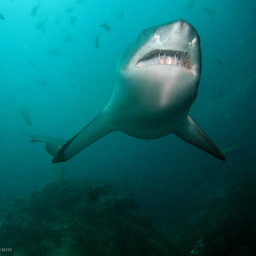}
    \end{subfigure}
    \begin{subfigure}[h]{0.125\textwidth}
    \centering
    \includegraphics[width=\textwidth]{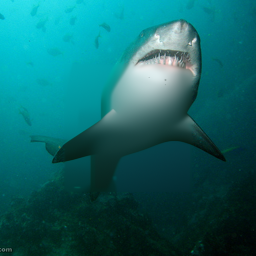}    
    \end{subfigure}
    \begin{subfigure}[h]{0.125\textwidth}
    \centering
    \includegraphics[width=\textwidth]{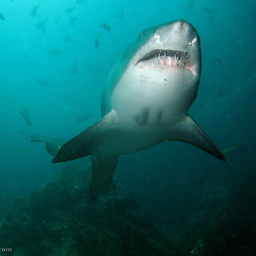}
    \end{subfigure}
    \\ 
    \begin{subfigure}[h]{0.125\textwidth}
    \centering
    \includegraphics[width=\textwidth]{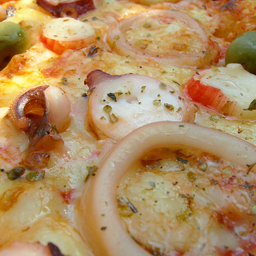}    
    \end{subfigure}
    \begin{subfigure}[h]{0.125\textwidth}
    \centering
    \includegraphics[width=\textwidth]{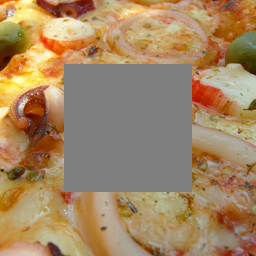}    
    \end{subfigure}
    \begin{subfigure}[h]{0.125\textwidth}
    \centering
    \includegraphics[width=\textwidth]{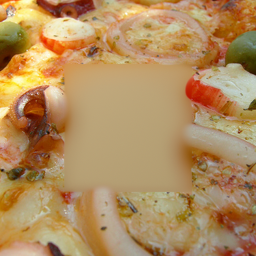}  
    \end{subfigure}
    \begin{subfigure}[h]{0.125\textwidth}
    \centering
    \includegraphics[width=\textwidth]{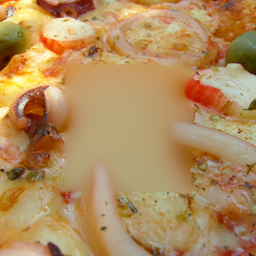}    
    \end{subfigure}
    \begin{subfigure}[h]{0.125\textwidth}
    \centering
    \includegraphics[width=\textwidth]{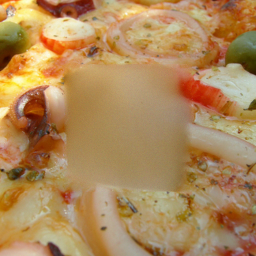}  
    \end{subfigure}
    \begin{subfigure}[h]{0.125\textwidth}
    \centering
    \includegraphics[width=\textwidth]{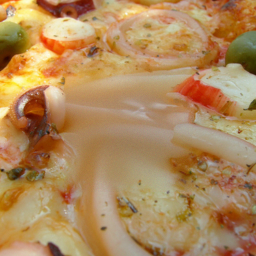}
    \end{subfigure}
    \begin{subfigure}[h]{0.125\textwidth}
    \centering
    \includegraphics[width=\textwidth]{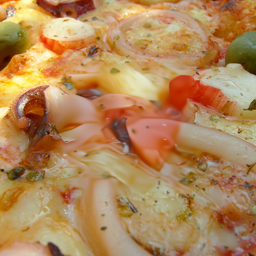}    
    \end{subfigure}
    \begin{subfigure}[h]{0.125\textwidth}
    \centering
    \includegraphics[width=\textwidth]{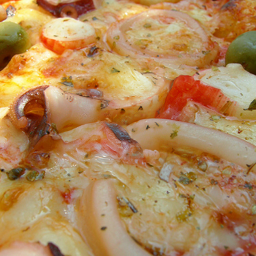}
    \end{subfigure}
    \\ 
    \begin{subfigure}[h]{0.125\textwidth}
    \centering
    \includegraphics[width=\textwidth]{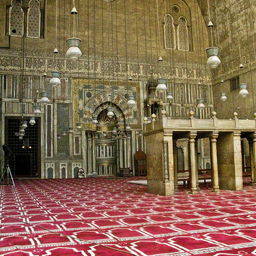}    
    \end{subfigure}
    \begin{subfigure}[h]{0.125\textwidth}
    \centering
    \includegraphics[width=\textwidth]{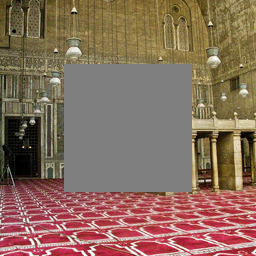}    
    \end{subfigure}
    \begin{subfigure}[h]{0.125\textwidth}
    \centering
    \includegraphics[width=\textwidth]{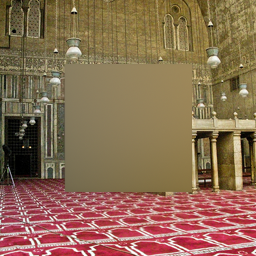}  
    \end{subfigure}
    \begin{subfigure}[h]{0.125\textwidth}
    \centering
    \includegraphics[width=\textwidth]{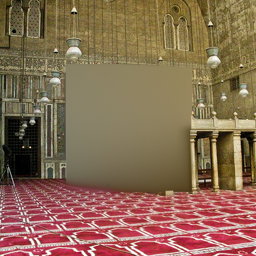}    
    \end{subfigure}
    \begin{subfigure}[h]{0.125\textwidth}
    \centering
    \includegraphics[width=\textwidth]{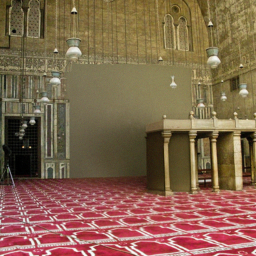}  
    \end{subfigure}
    \begin{subfigure}[h]{0.125\textwidth}
    \centering
    \includegraphics[width=\textwidth]{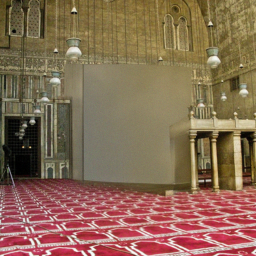}
    \end{subfigure}
    \begin{subfigure}[h]{0.125\textwidth}
    \centering
    \includegraphics[width=\textwidth]{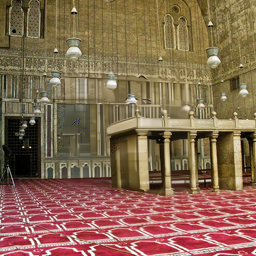}    
    \end{subfigure}
    \begin{subfigure}[h]{0.125\textwidth}
    \centering
    \includegraphics[width=\textwidth]{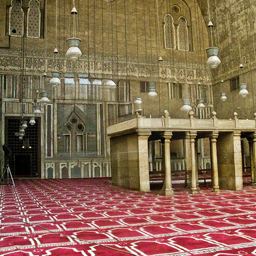}
    \end{subfigure}
    \caption{Box-mask inpainting task from Table~\ref{tab:inpexp}. Zoom in for the best view.}\label{fig:qv7}
\end{figure}

\subsection{Additional experiments on ImageNet$64$}\label{ssec:ir1}

In these experiments, we consider the case of noiseless image restoration, which includes the half-mask inpainting, $4\times$ super-resolution, and the uniform deblurring operators (similar to the experiments in Table~\ref{tab:exp1}). We evaluate existing baselines, VML-MAP, and $\text{VML-MAP}_{pre}$ on $1000$ images (each from a different class) from the ImageNet~\cite{russakovsky2015imagenet} validation set with a resolution of $64 \times 64$. We use the corresponding pre-trained \emph{class-conditional} diffusion model from~\citet{ddpm}. Note that the experiments on ImageNet$256$ in Table~\ref{tab:exp1} earlier use the \emph{unconditional} ImageNet$256$ pre-trained model, and the methods were evaluated on $100$ images of the ImageNet$256$ validation set.

Quantitative evaluation results from Table~\ref{tab:ir1} again indicate the effectiveness of VML-MAP in practice over existing baselines, especially for inpainting and super-resolution. Figure~\ref{fig:ir1} presents a qualitative comparison of the same. However, when the linear degradation matrix $\mathrm{H}$ is ill-conditioned, VML-MAP can be seen to struggle with the optimization, as seems to be the case with deblurring. Here, our proposed preconditioned variant $\text{VML-MAP}_{pre}$ alleviates this problem, as can be seen from the quantitative results in Table~\ref{tab:ir1}, and from the qualitative visualizations from Figure~\ref{fig:ir2}. 

Note again that DDRM, $\Pi$GDM, MAPGA, $\text{VML-MAP}_{pre}$ require access to the Singular Value Decomposition (SVD) of the linear degradation matrix $\mathrm{H}$ to find the pseudoinverse of terms involved therein, but VML-MAP only requires the forward operation of $\mathrm{H}$. We use the same experimental setup and hyperparameters mentioned in Appendix~\ref{ssec:implementationdetails1}, with the best learning rate configuration for VML-MAP and $\text{VML-MAP}_{pre}$ reported in Table~\ref{tab:hp1}.

\begin{table}[h!]
\caption{Noiseless image restoration experiments: Half-mask inpainting, $4\times$ Super-resolution, and Uniform deblurring tasks on $1000$ images of ImageNet$64$ validation set are repeated across 4 different seeds. The mean and standard deviation across these runs are reported. Best values in \textbf{bold}, second best values \underline{underlined}.}
\label{tab:ir1}
\begin{center}
\resizebox{0.8\textwidth}{!}{
\begin{tabular}{ l c c c c c c }
 \toprule
 \multicolumn{1}{l}{} & \multicolumn{2}{c}{} & \multicolumn{2}{c}{} & \multicolumn{2}{c}{}\\[-10pt]
 \multicolumn{1}{l}{{Method}} & \multicolumn{2}{c}{{Inpainting}} & \multicolumn{2}{c}{$4\times$ {Super-res}}  & \multicolumn{2}{c}{{Deblurring}} \\
 \multicolumn{1}{l}{} & {\scriptsize{LPIPS$\downarrow$}} & {\scriptsize{FID$\downarrow$}} & {\scriptsize{LPIPS$\downarrow$}} & {\scriptsize{FID$\downarrow$}} & {\scriptsize{LPIPS$\downarrow$}} & {\scriptsize{FID$\downarrow$}} \\[2pt]
 \toprule 
 \multicolumn{1}{l}{} & \multicolumn{2}{c}{} & \multicolumn{2}{c}{} & \multicolumn{2}{c}{}\\[-10pt]
 \multicolumn{1}{l}{{\scriptsize{$\text{DDRM}$}}} & \scriptsize{$0.263$}\tiny{\color{lightgray}{$\pm0.000$}} & \scriptsize{$57.15$}\tiny{\color{lightgray}{$\pm0.241$}} & \scriptsize{$0.234$}\tiny{\color{lightgray}{$\pm0.000$}} & \scriptsize{$77.82$}\tiny{\color{lightgray}{$\pm0.526$}} & \scriptsize{$0.467$}\tiny{\color{lightgray}{$\pm0.001$}} & \scriptsize{$197.7$}\tiny{\color{lightgray}{$\pm0.600$}} \\
 \multicolumn{1}{l}{{\scriptsize{$\Pi\text{GDM}$}}}  & \scriptsize{$0.242$}\tiny{\color{lightgray}{$\pm0.000$}} & \scriptsize{$54.69$}\tiny{\color{lightgray}{$\pm0.334$}} & \scriptsize{$0.241$}\tiny{\color{lightgray}{$\pm0.000$}} & \scriptsize{$88.63$}\tiny{\color{lightgray}{$\pm0.472$}} & \scriptsize{$0.439$}\tiny{\color{lightgray}{$\pm0.001$}} & \scriptsize{$164.7$}\tiny{\color{lightgray}{$\pm1.098$}} \\
 \multicolumn{1}{l}{{\scriptsize{$\text{MAPGA}$}}} & \scriptsize{$\underline{0.172}$}\tiny{\color{lightgray}{$\pm0.000$}} & \scriptsize{$\underline{46.43}$}\tiny{\color{lightgray}{$\pm0.150$}} & \scriptsize{$0.204$}\tiny{\color{lightgray}{$\pm0.000$}} & \scriptsize{$84.72$}\tiny{\color{lightgray}{$\pm0.779$}} & \scriptsize{$\underline{0.322}$}\tiny{\color{lightgray}{$\pm0.001$}} & \scriptsize{$113.9$}\tiny{\color{lightgray}{$\pm0.753$}} \\
 \midrule 
 \multicolumn{1}{l}{} & \multicolumn{2}{c}{} & \multicolumn{2}{c}{} & \multicolumn{2}{c}{}\\[-10pt]
 \multicolumn{1}{l}{{\scriptsize{$\textbf{VML-MAP}$}}}  &\scriptsize{$\bf{0.146}$}\tiny{\color{lightgray}{$\pm0.000$}} & \scriptsize{$\bf{38.84}$}\tiny{\color{lightgray}{$\pm0.233$}}  & \scriptsize{$\underline{0.136}$}\tiny{\color{lightgray}{$\pm0.001$}} & \scriptsize{$\underline{62.19}$}\tiny{\color{lightgray}{$\pm0.298$}} & \scriptsize{$0.356$}\tiny{\color{lightgray}{$\pm0.001$}} & \scriptsize{$\underline{106.6}$}\tiny{\color{lightgray}{$\pm0.682$}} \\
 \multicolumn{1}{l}{{\scriptsize{$\textbf{VML-MAP}_{pre}$}}} & \scriptsize{$\bf{0.146}$}\tiny{\color{lightgray}{$\pm0.000$}} & \scriptsize{$\bf{38.84}$}\tiny{\color{lightgray}{$\pm0.233$}}  & \scriptsize{$\bf{0.129}$}\tiny{\color{lightgray}{$\pm0.000$}} & \scriptsize{$\bf{60.01}$}\tiny{\color{lightgray}{$\pm0.400$}} & \scriptsize{$\bf{0.266}$}\tiny{\color{lightgray}{$\pm0.001$}} & \scriptsize{$\bf{77.46}$}\tiny{\color{lightgray}{$\pm0.424$}} \\
 \bottomrule
\end{tabular}}
\end{center}
\end{table}

\begin{figure}[h!]
    \centering
    \begin{subfigure}[h]{0.025\textwidth}
    \centering    
    \subfloat{\bf{\color{gray}{ }}}
    \end{subfigure}
    \begin{subfigure}[h]{0.14\textwidth}
    \centering
    \subfloat{\scriptsize{\bf{{\color{gray}{\textit{Original}}}}}}    
    \end{subfigure}
    \begin{subfigure}[h]{0.14\textwidth}
    \centering
    \subfloat{\scriptsize{\bf{{\color{gray}{\textit{Measurement}}}}}}    
    \end{subfigure}
    \begin{subfigure}[h]{0.14\textwidth}
    \centering
    \subfloat{\scriptsize{\bf{{\color{gray}{\textit{DDRM}}}}}}   
    \end{subfigure}
    \begin{subfigure}[h]{0.14\textwidth}
    \centering
    \subfloat{\scriptsize{{\bf{\color{gray}{\textit{$\Pi$GDM}}}}}}    
    \end{subfigure}
    \begin{subfigure}[h]{0.14\textwidth}
    \centering
    \subfloat{\scriptsize{\bf{\color{gray}{\textit{MAPGA}}}}}
    \end{subfigure}
    \begin{subfigure}[h]{0.14\textwidth}
    \centering
    \subfloat{\scriptsize{\bf{\color{gray}{\textit{VML-MAP}}}}}
    \end{subfigure}
    \\[2pt]
    \begin{subfigure}[h]{0.025\textwidth}
    \centering
    \subfloat{{\rotatebox[origin=c]{90}{{\color{gray}{\bf{\textit{Inpaint}}}}}}}    
    \end{subfigure}
    \begin{subfigure}[h]{0.14\textwidth}
    \centering
    \includegraphics[width=\textwidth]{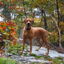}    
    \end{subfigure}
    \begin{subfigure}[h]{0.14\textwidth}
    \centering
    \includegraphics[width=\textwidth]{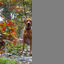}
    \end{subfigure}
    \begin{subfigure}[h]{0.14\textwidth}
    \centering
    \includegraphics[width=\textwidth]{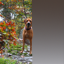}    
    \end{subfigure}
    \begin{subfigure}[h]{0.14\textwidth}
    \centering
    \includegraphics[width=\textwidth]{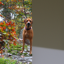}    
    \end{subfigure}
    \begin{subfigure}[h]{0.14\textwidth}
    \centering
    \includegraphics[width=\textwidth]{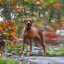}    
    \end{subfigure}
    \begin{subfigure}[h]{0.14\textwidth}
    \centering
    \includegraphics[width=\textwidth]{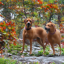}    
    \end{subfigure}
    \\
    \begin{subfigure}[h]{0.025\textwidth}
    \centering
    \subfloat{{\rotatebox[origin=c]{90}{\color{gray}{$\mathbf{4\times}$\bf{\textit{SR}}}}}}    
    \end{subfigure}
    \begin{subfigure}[h]{0.14\textwidth}
    \centering
    \includegraphics[width=\textwidth]{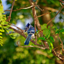}    
    \end{subfigure}
    \begin{subfigure}[h]{0.14\textwidth}
    \centering
    \includegraphics[width=\textwidth]{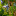}    
    \end{subfigure}
    \begin{subfigure}[h]{0.14\textwidth}
    \centering
    \includegraphics[width=\textwidth]{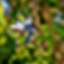}
    \end{subfigure}
    \begin{subfigure}[h]{0.14\textwidth}
    \centering
    \includegraphics[width=\textwidth]{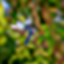}
    \end{subfigure}
    \begin{subfigure}[h]{0.14\textwidth}
    \centering
    \includegraphics[width=\textwidth]{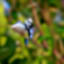}
    \end{subfigure}
    \begin{subfigure}[h]{0.14\textwidth}
    \centering
    \includegraphics[width=\textwidth]{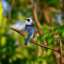}    
    \end{subfigure}
    \\
    \begin{subfigure}[h]{0.025\textwidth}
    \centering
    \subfloat{{\rotatebox[origin=c]{90}{\bf{\color{gray}{\textit{Deblur}}}}}}   
    \end{subfigure}
    \begin{subfigure}[h]{0.14\textwidth}
    \centering
    \includegraphics[width=\textwidth]{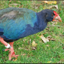}    
    \end{subfigure}
    \begin{subfigure}[h]{0.14\textwidth}
    \centering
    \includegraphics[width=\textwidth]{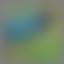}    
    \end{subfigure}
    \begin{subfigure}[h]{0.14\textwidth}
    \centering
    \includegraphics[width=\textwidth]{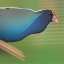}
    \end{subfigure}
    \begin{subfigure}[h]{0.14\textwidth}
    \centering
    \includegraphics[width=\textwidth]{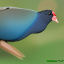}
    \end{subfigure}
    \begin{subfigure}[h]{0.14\textwidth}
    \centering
    \includegraphics[width=\textwidth]{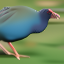}
    \end{subfigure}
    \begin{subfigure}[h]{0.14\textwidth}
    \centering
    \includegraphics[width=\textwidth]{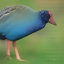}
    \end{subfigure}
    \caption{Half-mask inpainting, $4\times$ super-resolution, and deblurring tasks from Table~\ref{tab:ir1}. Zoom in for the best view.}\label{fig:ir1}
\end{figure}

\begin{figure}[h!]
    \centering
    \begin{subfigure}[h]{0.16\textwidth}
    \centering
    \subfloat{\scriptsize{\bf{{\color{gray}{$\textit{Original}_{\ }$}}}}}    
    \end{subfigure}
    \begin{subfigure}[h]{0.16\textwidth}
    \centering
    \subfloat{\scriptsize{\bf{{\color{gray}{$\textit{Measurement}_{\ }$}}}}}    
    \end{subfigure}
    \begin{subfigure}[h]{0.16\textwidth}
    \centering
    \subfloat{\scriptsize{{\bf{\color{gray}{$\textit{$\Pi$GDM}_{\ }$}}}}}    
    \end{subfigure}
    \begin{subfigure}[h]{0.16\textwidth}
    \centering
    \subfloat{\scriptsize{\bf{\color{gray}{$\textit{MAPGA}_{\ }$}}}}
    \end{subfigure}
    \begin{subfigure}[h]{0.16\textwidth}
    \centering
    \subfloat{\scriptsize{\bf{\color{gray}{$\textit{VML-MAP}_{\ }$}}}}
    \end{subfigure}
    \begin{subfigure}[h]{0.16\textwidth}
    \centering
    \subfloat{\scriptsize{\bf{{\color{gray}{$\textit{VML-MAP}_{pre}$}}}}}   
    \end{subfigure}
    \\[2pt]
    \begin{subfigure}[h]{0.16\textwidth}
    \centering
    \includegraphics[width=\textwidth]{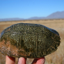}    
    \end{subfigure}
    \begin{subfigure}[h]{0.16\textwidth}
    \centering
    \includegraphics[width=\textwidth]{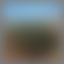}    
    \end{subfigure}
    \begin{subfigure}[h]{0.16\textwidth}
    \centering
    \includegraphics[width=\textwidth]{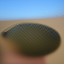}    
    \end{subfigure}
    \begin{subfigure}[h]{0.16\textwidth}
    \centering
    \includegraphics[width=\textwidth]{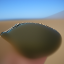}    
    \end{subfigure}
    \begin{subfigure}[h]{0.16\textwidth}
    \centering
    \includegraphics[width=\textwidth]{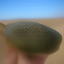}    
    \end{subfigure}
    \begin{subfigure}[h]{0.16\textwidth}
    \centering
    \includegraphics[width=\textwidth]{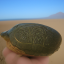}    
    \end{subfigure}
    \\
    \begin{subfigure}[h]{0.16\textwidth}
    \centering
    \includegraphics[width=\textwidth]{icml2026/figures/imagenet64c/deblur16/im136_original.png}    
    \end{subfigure}
    \begin{subfigure}[h]{0.16\textwidth}
    \centering
    \includegraphics[width=\textwidth]{icml2026/figures/imagenet64c/deblur16/im136_degraded.png}    
    \end{subfigure}
    \begin{subfigure}[h]{0.16\textwidth}
    \centering
    \includegraphics[width=\textwidth]{icml2026/figures/imagenet64c/deblur16/im136_recovered_pgdm.png}    
    \end{subfigure}
    \begin{subfigure}[h]{0.16\textwidth}
    \centering
    \includegraphics[width=\textwidth]{icml2026/figures/imagenet64c/deblur16/im136_recovered_mapga.png}    
    \end{subfigure}
    \begin{subfigure}[h]{0.16\textwidth}
    \centering
    \includegraphics[width=\textwidth]{icml2026/figures/imagenet64c/deblur16/im136_recovered_vmlmap.png}    
    \end{subfigure}
    \begin{subfigure}[h]{0.16\textwidth}
    \centering
    \includegraphics[width=\textwidth]{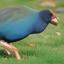}    
    \end{subfigure}
    \caption{Deblurring task from Table~\ref{tab:ir1}. Zoom in for the best view.}\label{fig:ir2}
\end{figure}

\subsection{Validation check experiments from Figure~\ref{fig:toyexample}}\label{ssec:toysetup}

Each row of Figure~\ref{fig:toyexample} corresponds to a different example, where the input data follows a different 2D GMM distribution. For the first example (i.e., the top row of Figure~\ref{fig:toyexample}), the input data distribution is given by
\begin{align*}
p(\rvx_0) = 0.1&\,\mathcal{N}([1.0, 0.75]^{\top}, 0.15\,\mathbf{I}) + 0.3\, \mathcal{N}([0.75, -0.25]^\top, 0.15\,\mathbf{I}) + 0.2\,\mathcal{N}([-0.5, 0.25]^\top, 0.15\,\mathbf{I}) \\ 
& + 0.2\,\mathcal{N}([-0.9, -1.0]^\top, 0.15\,\mathbf{I}) + 0.2\,\mathcal{N}([0.,0.]^\top, 0.15\,\mathbf{I})    
\end{align*}
For the second example (i.e., the bottom row of Figure~\ref{fig:toyexample}), the input data distribution is given by
\begin{align*}
p(\rvx_0) =  0.1&\,\mathcal{N}([1.0, 0.75]^{\top}, 0.15\,\mathbf{I}) + 0.05\, \mathcal{N}([0.75, -0.25]^\top, 0.15\,\mathbf{I}) + 0.1\,\mathcal{N}([-0.5, 0.25]^\top, 0.15\,\mathbf{I}) \\ 
& + 0.15\,\mathcal{N}([-0.9, -1.0]^\top, 0.15\,\mathbf{I}) + 0.1\,\mathcal{N}([0.,0.]^\top, 0.5\,\mathbf{I}) + 0.1\,\mathcal{N}([1.0, -1.25]^\top, 0.2\,\mathbf{I}) \\
& + 0.1\, \mathcal{N}([1.5, 1.5]^\top, 0.25\,\mathbf{I}) + 0.15\, \mathcal{N}([-1.2, 0.6]^\top, 0.25\,\mathbf{I}) + 0.15\, \mathcal{N}([1.0, 1.0]^\top, 0.2\,\mathbf{I})
\end{align*}
Note that for any given diffusion time $t$, the marginal distribution $p(\rvx_t)$ can be analytically computed in a closed form for the above cases, given that the perturbation kernel $p(\rvx_t|\rvx_0) = \mathcal{N}(0,\sigma^2_t\mathbf{I})$ for some monotonically increasing diffusion noise schedule $\sigma(\cdot)$. Therefore, the score function of the intermediate marginal, i.e., $\nabla_{\rvx_t}\log p(\rvx_t)$ and corresponding true denoiser, i.e.,  $\mathrm{D}(\rvx_t,t) = \rvx_t + \sigma^2_t\, \nabla_{\rvx_t}\log p(\rvx_t)$ (see Appendix~\ref{ssec:tweediesformula}) can be analytically computed. Note that in the validation check, we used the analytically computed true denoiser to apply VML-MAP, so that we need not be concerned with the model approximation errors involved in using a learned denoiser.

For VML-MAP, we set $N = K = 20$, $\sigma(t) = t$, the learning rate $\gamma=0.5$, and the EDM~\cite{edm} noise schedule discretization with $\sigma_{min}=2e^{-3}$, $\sigma_{max}=40.0$, and EDM's polynomial co-efficient $\rho=7$. We set the prior weight schedule $\rho(t) = 1.0 + \rvw\, \frac{\sigma_t}{\sigma_\rvy}$, with $\rvw=2e^3$ for noiseless inpainting, and $\rvw=1.0$ for noisy inpainting ($\sigma_\rvy = 0.5$).

\subsection{Qualitative reconstructions across different seeds for the half-mask inpainting task in Table~\ref{tab:exp1}}\label{ssec:multipleseeds}

\begin{figure}[htbp]
    \centering
    \begin{subfigure}[h]{0.10\textwidth}
    \centering
    \subfloat{\scriptsize{\bf{\color{gray}{\textit{Original}}}}}
    \end{subfigure}
    \begin{subfigure}[h]{0.10\textwidth}
    \centering
    \subfloat{\scriptsize{\bf{\color{gray}{\textit{Measurement}}}}}
    \end{subfigure}
    \begin{subfigure}[h]{0.10\textwidth}
    \centering
    \subfloat{\scriptsize{\bf{\color{gray}{\textit{Seed-1}}}}}
    \end{subfigure}
    \begin{subfigure}[h]{0.10\textwidth}
    \centering
    \subfloat{\scriptsize{\bf{\color{gray}{\textit{Seed-2}}}}}
    \end{subfigure}
    \begin{subfigure}[h]{0.10\textwidth}
    \centering
    \subfloat{\scriptsize{\bf{\color{gray}{\textit{Seed-3}}}}}
    \end{subfigure}
    \begin{subfigure}[h]{0.10\textwidth}
    \centering
    \subfloat{\scriptsize{\bf{\color{gray}{\textit{Seed-4}}}}}
    \end{subfigure}
    \\[3pt]
    \begin{subfigure}[h]{0.10\textwidth}
    \centering
    \includegraphics[width=\textwidth]{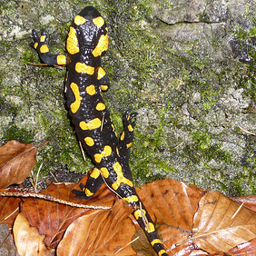}    
    \end{subfigure}
    \begin{subfigure}[h]{0.10\textwidth}
    \centering
    \includegraphics[width=\textwidth]{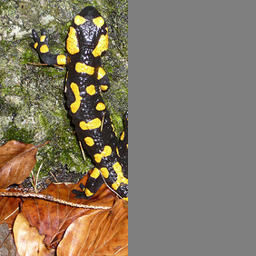}    
    \end{subfigure}
    \begin{subfigure}[h]{0.10\textwidth}
    \centering
    \includegraphics[width=\textwidth]{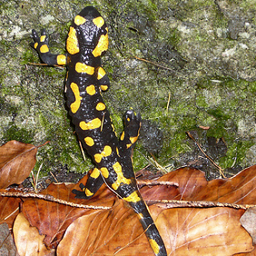}  
    \end{subfigure}
    \begin{subfigure}[h]{0.10\textwidth}
    \centering
    \includegraphics[width=\textwidth]{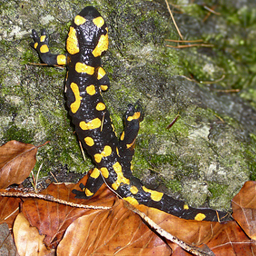}    
    \end{subfigure}
    \begin{subfigure}[h]{0.10\textwidth}
    \centering
    \includegraphics[width=\textwidth]{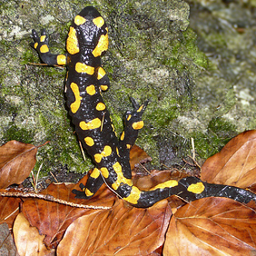}
    \end{subfigure}
    \begin{subfigure}[h]{0.10\textwidth}
    \centering
    \includegraphics[width=\textwidth]{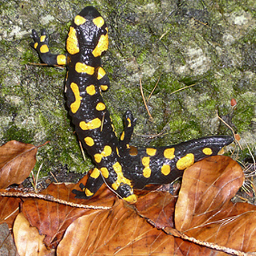}
    \end{subfigure}
    \\ 
    \begin{subfigure}[h]{0.10\textwidth}
    \centering
    \includegraphics[width=\textwidth]{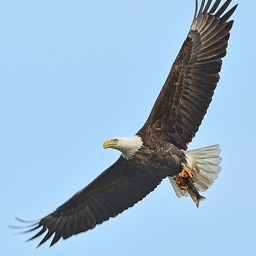}    
    \end{subfigure}
    \begin{subfigure}[h]{0.10\textwidth}
    \centering
    \includegraphics[width=\textwidth]{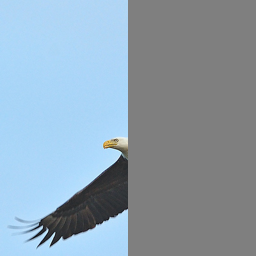}    
    \end{subfigure}
    \begin{subfigure}[h]{0.10\textwidth}
    \centering
    \includegraphics[width=\textwidth]{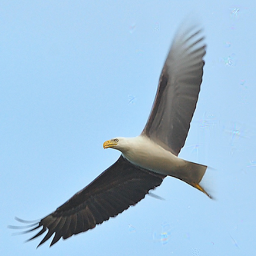}  
    \end{subfigure}
    \begin{subfigure}[h]{0.10\textwidth}
    \centering
    \includegraphics[width=\textwidth]{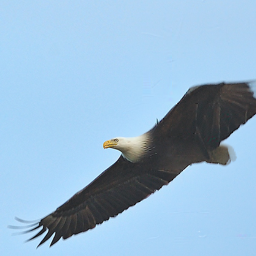}    
    \end{subfigure}
    \begin{subfigure}[h]{0.10\textwidth}
    \centering
    \includegraphics[width=\textwidth]{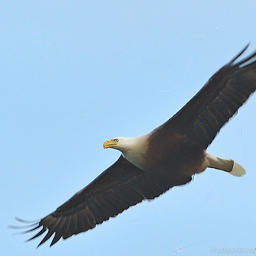}
    \end{subfigure}
    \begin{subfigure}[h]{0.10\textwidth}
    \centering
    \includegraphics[width=\textwidth]{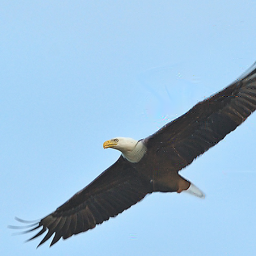}
    \end{subfigure}
    \\ 
    \begin{subfigure}[h]{0.10\textwidth}
    \centering
    \includegraphics[width=\textwidth]{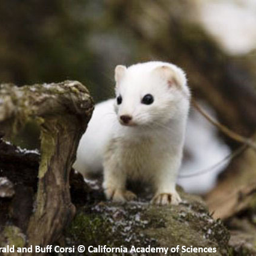}    
    \end{subfigure}
    \begin{subfigure}[h]{0.10\textwidth}
    \centering
    \includegraphics[width=\textwidth]{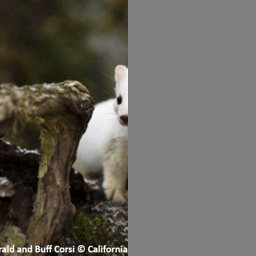}    
    \end{subfigure}
    \begin{subfigure}[h]{0.10\textwidth}
    \centering
    \includegraphics[width=\textwidth]{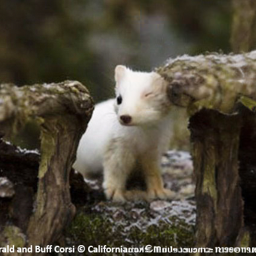}  
    \end{subfigure}
    \begin{subfigure}[h]{0.10\textwidth}
    \centering
    \includegraphics[width=\textwidth]{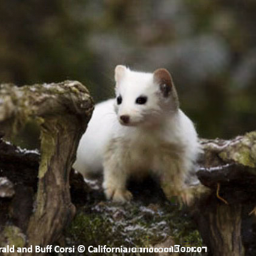}    
    \end{subfigure}
    \begin{subfigure}[h]{0.10\textwidth}
    \centering
    \includegraphics[width=\textwidth]{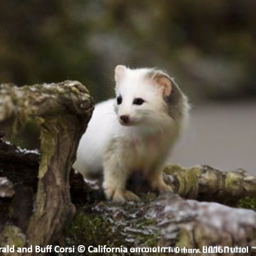}
    \end{subfigure}
    \begin{subfigure}[h]{0.10\textwidth}
    \centering
    \includegraphics[width=\textwidth]{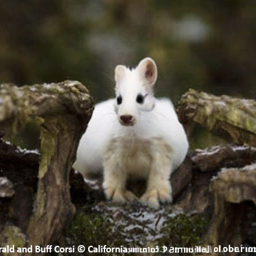}
    \end{subfigure}
    \\ 
    \begin{subfigure}[h]{0.10\textwidth}
    \centering
    \includegraphics[width=\textwidth]{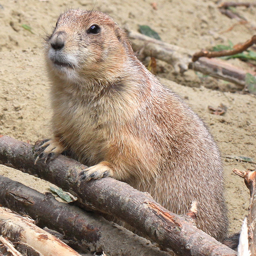}    
    \end{subfigure}
    \begin{subfigure}[h]{0.10\textwidth}
    \centering
    \includegraphics[width=\textwidth]{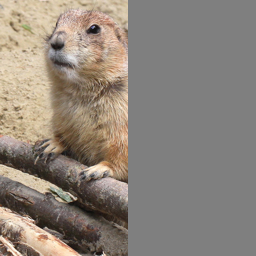}    
    \end{subfigure}
    \begin{subfigure}[h]{0.10\textwidth}
    \centering
    \includegraphics[width=\textwidth]{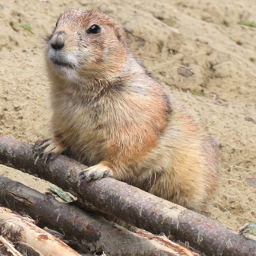}  
    \end{subfigure}
    \begin{subfigure}[h]{0.10\textwidth}
    \centering
    \includegraphics[width=\textwidth]{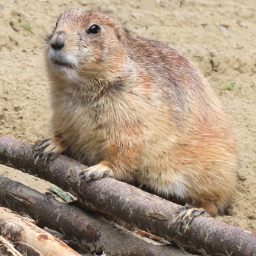}    
    \end{subfigure}
    \begin{subfigure}[h]{0.10\textwidth}
    \centering
    \includegraphics[width=\textwidth]{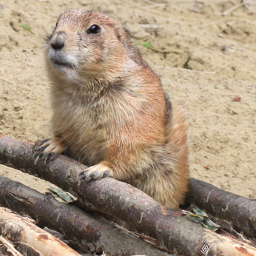}
    \end{subfigure}
    \begin{subfigure}[h]{0.10\textwidth}
    \centering
    \includegraphics[width=\textwidth]{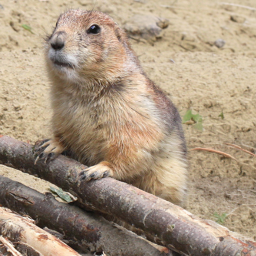}
    \end{subfigure}
    \\ 
    \begin{subfigure}[h]{0.10\textwidth}
    \centering
    \includegraphics[width=\textwidth]{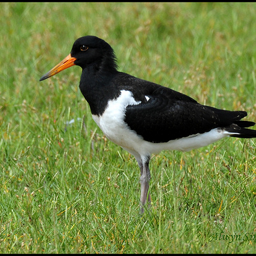}    
    \end{subfigure}
    \begin{subfigure}[h]{0.10\textwidth}
    \centering
    \includegraphics[width=\textwidth]{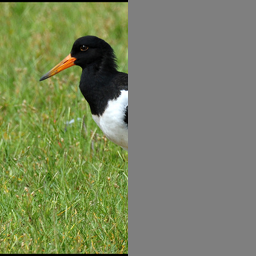}    
    \end{subfigure}
    \begin{subfigure}[h]{0.10\textwidth}
    \centering
    \includegraphics[width=\textwidth]{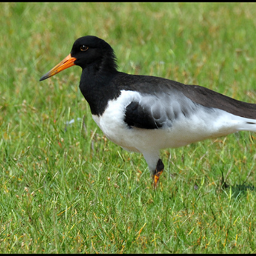}  
    \end{subfigure}
    \begin{subfigure}[h]{0.10\textwidth}
    \centering
    \includegraphics[width=\textwidth]{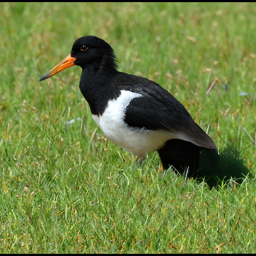}    
    \end{subfigure}
    \begin{subfigure}[h]{0.10\textwidth}
    \centering
    \includegraphics[width=\textwidth]{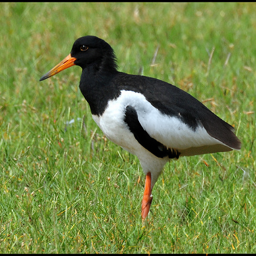}
    \end{subfigure}
    \begin{subfigure}[h]{0.10\textwidth}
    \centering
    \includegraphics[width=\textwidth]{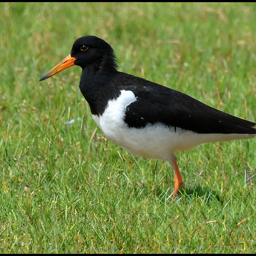}
    \end{subfigure}
    \\ 
    \begin{subfigure}[h]{0.10\textwidth}
    \centering
    \includegraphics[width=\textwidth]{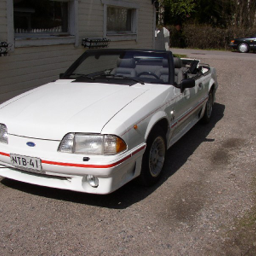}    
    \end{subfigure}
    \begin{subfigure}[h]{0.10\textwidth}
    \centering
    \includegraphics[width=\textwidth]{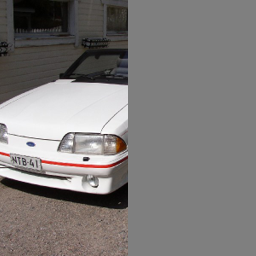}    
    \end{subfigure}
    \begin{subfigure}[h]{0.10\textwidth}
    \centering
    \includegraphics[width=\textwidth]{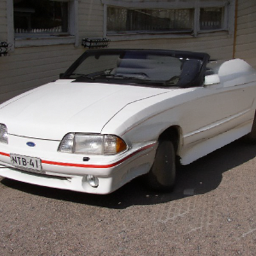}  
    \end{subfigure}
    \begin{subfigure}[h]{0.10\textwidth}
    \centering
    \includegraphics[width=\textwidth]{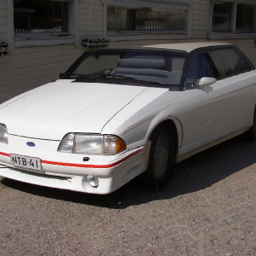}    
    \end{subfigure}
    \begin{subfigure}[h]{0.10\textwidth}
    \centering
    \includegraphics[width=\textwidth]{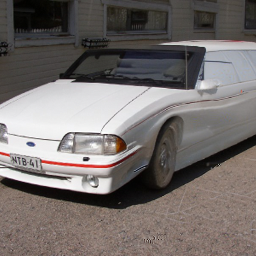}
    \end{subfigure}
    \begin{subfigure}[h]{0.10\textwidth}
    \centering
    \includegraphics[width=\textwidth]{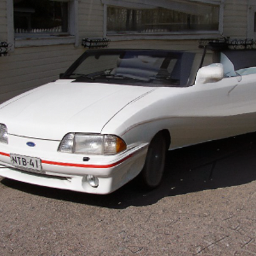}
    \end{subfigure}
    \caption{Restored images of VML-MAP across different seeds for the half-mask inpainting task in Table~\ref{tab:exp1}.}\label{fig:diverse_seeds}
\end{figure}
\newpage
\section{Extension to Latent Diffusion Models (LDM)}\label{sec:extensiontoldm}
\subsection{Approximating the VML for LDMs}
Here, we provide an extension of the VML objective to Latent Diffusion Models (LDM) for solving inverse problems. In LDMs, we treat both the encoder $\mathcal{E}$ and the decoder $\mathcal{D}$ as deterministic mappings from a clean image $\rvx_0$ to a clean latent variable $\rvz_0$ and vice-versa, respectively. To solve an inverse problem with a pre-trained LDM and a given measurement $\rvy$, we first aim to solve for the MAP estimate $\rvz^{*}_{0} = \arg \max_{\rvz_0}\log p(\rvz_0|\rvy)$ with the VML objective extended to LDMs and later use the decoder to predict a clean image in a deterministic manner i.e., $\rvx^{*}_0 = \mathcal{D}(\rvz^{*}_{0})$. First, we define the VML objective in LDMs, which we term $\mathrm{VML}_{\mathrm{LDM}}$, as the KL divergence between $p(\rvz_0|\rvz_t)$ and $p(\rvz_0|\rvy)$, and further approximate it as below.

\begin{prop}\label{prop:6}
The variational mode-seeking-loss for an LDM ($\text{VML}^{\text{LDM}}$) at diffusion time $t$, for a degradation operator $\mathcal{A}$, measurement $\rvy$, and measurement noise variance $\sigma^2_{\rvy}$ is given by
{\small{\begin{align*}
\mathrm{VML}^{\mathrm{LDM}}(\rvz_t,t) &= \KL(p(\rvz_0|\rvz_t) || p(\rvz_0|\rvy)) \approx - \log p(\rvz_t) - \frac{\|\mathrm{D}(\rvz_t,t)-\rvz_t\|^2}{2\sigma^2_t}  - \frac{1}{2\sigma^2_t} \mathrm{Tr}\left\{ \Cov[\rvz_0|\rvz_t] \right\} \\
+& \frac{\| \rvy - \mathcal{A}(\mathcal{D}(\mathrm{D}(\rvz_t,t))) \|^2}{2\sigma^2_\rvy} + \frac{1}{2\sigma^2_\rvy} \left( \mathrm{Tr}\left\{ \frac{\partial \mathcal{A}(\mathcal{D}(\mathrm{D}(\rvz_t,t)))}{\partial \mathrm{D}(\rvz_t,t)} \Cov[\rvz_0|\rvz_t] \frac{\partial \mathcal{A}(\mathcal{D}(\mathrm{D}(\rvz_t,t)))^{\top}}{\partial \mathrm{D}(\rvz_t,t)} \right\} \right) + \mathrm{C}
\end{align*}}}
where $\mathrm{C}$ is a constant, independent of $\rvx_t$. $\mathrm{Tr}$ denotes the matrix trace, $\Cov$ denotes the covariance matrix, $\mathrm{D}(\cdot,\cdot)$ denotes the denoiser, and $\mathcal{D}(\cdot)$ denotes the LDM decoder.
\end{prop}
\begin{proof}
{\scriptsize{
\begin{align*}
\KL(p(\rvz_0|\rvz_t) || p(\rvz_0|\rvy)) &= \int_{\rvz_0} p(\rvz_0|\rvz_t) \log \frac{p(\rvz_0|\rvz_t)}{p(\rvz_0|\rvy)} \mathrm{d}\rvz_0 \\
\KL(p(\rvz_0|\rvz_t) || p(\rvz_0|\rvy)) &= \int_{\rvz_0} p(\rvz_0|\rvz_t) \log \frac{p(\rvz_t|\rvz_0)\cancel{p(\rvz_0)}p(\rvy)}{p(\rvz_t)p(\rvy|\rvz_0)\cancel{p(\rvz_0)} } \mathrm{d}\rvz_0 \\
\KL(p(\rvz_0|\rvz_t) || p(\rvz_0|\rvy)) &= \log p(\rvy) - \log p(\rvz_t) + \int_{\rvz_0} p(\rvz_0|\rvz_t) \log \frac{p(\rvz_t|\rvz_0)}{p(\rvy|\rvz_0)} \mathrm{d}\rvz_0 \\
\KL(p(\rvz_0|\rvz_t) || p(\rvz_0|\rvy)) &= \log p(\rvy) - \log p(\rvz_t) + \left( \int_{\rvz_0} p(\rvz_0|\rvz_t) \log p(\rvz_t|\rvz_0) \mathrm{d}\rvz_0 \right)\\
&- \left(  \int_{\rvz_0} p(\rvz_0|\rvz_t) \log p(\rvy|\rvz_0) \mathrm{d}\rvz_0 \right) \\[8pt]
\big\{\text{Note that }p(\rvz_t|\rvz_0) &= \mathcal{N}(\rvz_0, \sigma^2_t \mathbf{I})\text{ and we approximate }p(\rvy|\rvz_0) \approx \mathcal{N}(\mathcal{A}(\mathcal{D}(\rvz_0)), \sigma^2_{\rvy}\mathbf{I}) \big\} \\[8pt]
\KL(p(\rvz_0|\rvz_t) || p(\rvz_0|\rvy)) &\approx - \log p(\rvz_t) - \frac{1}{2} \left( \int_{\rvz_0} p(\rvz_0|\rvz_t) \frac{\|\rvz_t-\rvz_0\|^2}{\sigma^2_t} \mathrm{d}\rvz_0 \right)\\
&+ \frac{1}{2} \left(  \int_{\rvz_0} p(\rvz_0|\rvz_t) \frac{\| \rvy - \mathcal{A}(\mathcal{D}(\rvz_0)) \|^2}{\sigma^2_\rvy} \mathrm{d}\rvz_0 \right) + \mathrm{C}\\
\KL(p(\rvz_0|\rvz_t) || p(\rvz_0|\rvy)) &\approx - \log p(\rvz_t) - \frac{1}{2\sigma^2_t} \left( \|\rvz_t\|^2 - 2 \rvz_t^{\top} \mathrm{D}(\rvz_t,t) + \int_{\rvz_0} 
\|\rvz_0\|^2 p(\rvz_0|\rvz_t) \mathrm{d}\rvz_0 \right) \\ 
+ \frac{1}{2\sigma^2_\rvy}& \left(  -2 \rvy^{\top} \int_{\rvz_0} \mathcal{A}(\mathcal{D}(\rvz_0)) p(\rvz_0|\rvz_t) \mathrm{d}\rvz_0   + \int_{\rvz_0} \|\mathcal{A}(\mathcal{D}(\rvz_0))\|^2 p(\rvz_0|\rvz_t) \mathrm{d}\rvz_0  \right) + \mathrm{C}\\[8pt]
\text{We make a linear}&\text{ approximation of }\mathcal{A}(\mathcal{D}(\rvz_0))\text{ around }\hat{\rvz}_t = \mathrm{D}(\rvz_t,t) = \int_{\rvz_0}\rvz_0 p(\rvz_0|\rvz_t) \text{  as follows  }\\
&\mathcal{A}(\mathcal{D}(\rvz_0)) \approx \mathcal{A}(\mathcal{D}(\hat{\rvz}_t)) + \frac{\partial \mathcal{A}(\mathcal{D}(\hat{\rvz}_t))}{\partial \hat{\rvz}_t} (\rvz_0 - \hat{\rvz}_t)\\[8pt]
\KL(p(\rvz_0|\rvz_t) || p(\rvz_0|\rvy)) &\approx - \log p(\rvz_t) - \frac{1}{2\sigma^2_t} \left( \|\rvz_t\|^2 - 2 \rvz_t^{\top} \mathrm{D}(\rvz_t,t) + \mathrm{Tr}\left\{ \Cov[\rvz_0|\rvz_t] \right\} + \|\mathrm{D}(\rvz_t,t)\|^2  \right) \\ 
+ \frac{1}{2\sigma^2_\rvy} \Biggl( -2 \rvy^{\top} \Biggl\{ \mathcal{A}(\mathcal{D}(\hat{\rvz}_t)) &+ {\cancel{\frac{\partial \mathcal{A}(\mathcal{D}(\hat{\rvz}_t))}{\partial \hat{\rvz}_t} \int_{\rvz_0} (\rvz_0 - \hat{\rvz}_t) p(\rvz_0|\rvz_t) \mathrm{d}\rvz_0}} \Biggr\} + \int_{\rvz_0} \|\mathcal{A}(\mathcal{D}(\rvz_0))\|^2 p(\rvz_0|\rvz_t) \mathrm{d}\rvz_0  \Biggr) + \mathrm{C}\\
\KL(p(\rvz_0|\rvz_t) || p(\rvz_0|\rvy)) & \approx - \log p(\rvz_t) - \frac{\|\mathrm{D}(\rvz_t,t)-\rvz_t\|^2}{2\sigma^2_t}  - \frac{1}{2\sigma^2_t} \mathrm{Tr}\left\{ \Cov[\rvz_0|\rvz_t] \right\} \\ 
+ \frac{1}{2\sigma^2_\rvy} \Biggl(  -2 \rvy^{\top} \mathcal{A}(\mathcal{D}(&\hat{\rvz}_t))  + \|\mathcal{A}(\mathcal{D}(\hat{\rvz}_t))\|^2 + \cancel{2 \mathcal{A}(\mathcal{D}(\hat{\rvz}_t))^{\top} \frac{\partial \mathcal{A}(\mathcal{D}(\hat{\rvz}_t))}{\partial \hat{\rvz}_t} \int_{\rvz_0} (\rvz_0 - \hat{\rvz}_t)  p(\rvz_0|\rvz_t) \mathrm{d}\rvz_0}  \Biggr)\\
+ \frac{1}{2\sigma^2_\rvy} & \left( \int_{\rvz_0} \left\| \frac{\partial \mathcal{A}(\mathcal{D}(\hat{\rvz}_t))}{\partial \hat{\rvz}_t} (\rvz_0 - \hat{\rvz}_t) \right\|^2 p(\rvz_0|\rvz_t) \mathrm{d}\rvz_0 \right) + \mathrm{C} \\
\KL(p(\rvz_0|\rvz_t) || p(\rvz_0|\rvy)) & \approx - \log p(\rvz_t) - \frac{\|\mathrm{D}(\rvz_t,t)-\rvz_t\|^2}{2\sigma^2_t}  - \frac{1}{2\sigma^2_t} \mathrm{Tr}\left\{ \Cov[\rvz_0|\rvz_t] \right\} + \frac{\| \rvy - \mathcal{A}(\mathcal{D}(\hat{\rvz}_t)) \|^2}{2\sigma^2_\rvy}\\
+ \frac{1}{2\sigma^2_\rvy} \Biggl( \mathrm{Tr}\Biggl\{ & \frac{\partial \mathcal{A}(\mathcal{D}(\hat{\rvz}_t))}{\partial \hat{\rvz}_t} \left( \int_{\rvz_0} (\rvz_0 - \hat{\rvz}_t)(\rvz_0 - \hat{\rvz}_t)^{\top} p(\rvz_0|\rvz_t) \mathrm{d}\rvz_0 \right) \frac{\partial \mathcal{A}(\mathcal{D}(\hat{\rvz}_t))^{\top}}{\partial \hat{\rvz}_t} \Biggr\} \Biggr) + \mathrm{C} \\
\KL(p(\rvz_0|\rvz_t) || p(\rvz_0|\rvy)) & \approx - \log p(\rvz_t) - \frac{\|\mathrm{D}(\rvz_t,t)-\rvz_t\|^2}{2\sigma^2_t}  - \frac{1}{2\sigma^2_t} \mathrm{Tr}\left\{ \Cov[\rvz_0|\rvz_t] \right\} + \frac{\| \rvy - \mathcal{A}(\mathcal{D}(\hat{\rvz}_t)) \|^2}{2\sigma^2_\rvy} \\
+ \frac{1}{2\sigma^2_\rvy} & \left( \mathrm{Tr}\left\{ \frac{\partial \mathcal{A}(\mathcal{D}(\hat{\rvz}_t))}{\partial \hat{\rvz}_t} \Cov[\rvz_0|\rvz_t] \frac{\partial \mathcal{A}(\mathcal{D}(\hat{\rvz}_t))^{\top}}{\partial \hat{\rvz}_t} \right\} \right) + \mathrm{C} \\
\KL(p(\rvz_0|\rvz_t) || p(\rvz_0|\rvy)) & \approx - \log p(\rvz_t) - \frac{\|\mathrm{D}(\rvz_t,t)-\rvz_t\|^2}{2\sigma^2_t}  - \frac{1}{2\sigma^2_t} \mathrm{Tr}\left\{ \Cov[\rvz_0|\rvz_t] \right\} + \frac{\| \rvy - \mathcal{A}(\mathcal{D}(\mathrm{D}(\rvz_t,t))) \|^2}{2\sigma^2_\rvy}\\
+ \frac{1}{2\sigma^2_\rvy} &\left( \mathrm{Tr}\left\{ \frac{\partial \mathcal{A}(\mathcal{D}(\mathrm{D}(\rvz_t,t)))}{\partial \mathrm{D}(\rvz_t,t)} \Cov[\rvz_0|\rvz_t] \frac{\partial \mathcal{A}(\mathcal{D}(\mathrm{D}(\rvz_t,t)))^{\top}}{\partial \mathrm{D}(\rvz_t,t)} \right\} \right) + \mathrm{C} \\
\end{align*}}}
\end{proof}
\subsection{Simplified-VML and Latent VML-MAP for LDMs}\label{ssec:svml_ldm}
Similar to the case of pixel diffusion models, the higher-order terms involving $\Cov[\rvz_0|\rvz_t]$ in $\text{VML}^{\text{LDM}}(\cdot,t)$ converge to a constant as $t \to 0$, under mild assumptions on $\mathcal{A}$ and $\mathcal{D}$. Ignoring these higher-order terms, we define the simplified-$\text{VML}^{\text{LDM}}$ objective and its gradient as follows. 

\textbf{Simplified-$\text{VML}^{\text{LDM}}$ and its gradient.} For a linear degradation matrix $\mathrm{H}$, the simplified-$\text{VML}^{\text{LDM}}$ denoted as $\mathrm{VML}^{\mathrm{LDM}}_\mathrm{S}(\cdot,\cdot)$ and its gradient is given by
{\small{
\begin{align*}
\mathrm{VML}^{\mathrm{LDM}}_{\mathrm{S}}(\rvz_t,t) = -\log p(\rvz_t) - \frac{\| \mathrm{D}(\rvz_t,t) - \rvz_t\|^2 }{2\sigma^2_t} + \frac{\| \rvy - \mathrm{H}\mathcal{D}(\mathrm{D}(\rvz_t,t)) \|^2}{2\sigma^2_\rvy}
\end{align*}}}
{\small{
\begin{align*}
\nabla_{\rvz_t} \mathrm{VML}^{\mathrm{LDM}}_{\mathrm{S}}(\rvz_t,t) = \underbrace{ - \frac{\partial \mathrm{D}^{\top}(\rvz_t,t)}{\partial \rvz_t} \frac{\partial \mathcal{D}^{\top}(\mathrm{D}(\rvz_t,t))}{\partial \mathrm{D}(\rvz_t,t)} \frac{\mathrm{H}^{\top}(\rvy-\mathrm{H} \mathcal{D}(\mathrm{D}(\rvz_t,t)))}{\sigma^2_\rvy} }_{\textit{measurement consistency gradient}} \underbrace{ - \frac{\partial \mathrm{D}^{\top}(\rvz_t,t)}{\partial \rvz_t} \frac{(\mathrm{D}(\rvz_t,t)-\rvz_t)}{\sigma^2_t} }_{\textit{prior gradient}}
\end{align*}}}
\begin{proof}
{\small{
\begin{align*}
\nabla_{\rvz_t} \mathrm{VML}^{\mathrm{LDM}}_\mathrm{S}(\rvz_t,t) &= \left\{ -\nabla_{\rvz_t} \log p(\rvz_t) \right\} - \left\{ \nabla_{\rvz_t} \frac{\|\mathrm{D}(\rvz_t,t)-\rvz_t\|^2}{2\sigma^2_t} \right\} + \left\{ \nabla_{\rvz_t} \frac{\|\rvy-\mathrm{H}\mathcal{D}(\mathrm{D}(\rvz_t,t))\|^2}{2\sigma^2_\rvy} \right\} \\
\nabla_{\rvz_t} \mathrm{VML}^{\mathrm{LDM}}_\mathrm{S}(\rvz_t,t)  &= \left\{ - \frac{\mathrm{D}(\rvz_t,t)-\rvz_t}{\sigma^2_t} \right\} - \left\{ \left( \frac{\partial \mathrm{D}^{\top}(\rvz_t,t)}{\partial \rvz_t} \frac{(\mathrm{D}(\rvz_t,t)-\rvz_t)}{\sigma^2_t} \right) - \frac{\mathrm{D}(\rvz_t,t)-\rvz_t}{\sigma^2_t} \right\} \\
&+ \left\{ - \frac{\partial \mathrm{D}^{\top}(\rvz_t,t)}{\partial \rvz_t} \frac{\partial \mathcal{D}^{\top}(\mathrm{D}(\rvz_t,t))}{\partial \mathrm{D}(\rvz_t,t)} \frac{\mathrm{H}^{\top}(\rvy-\mathrm{H}\mathcal{D}(\mathrm{D}(\rvx_t,t)))}{\sigma^2_\rvy} \right\} \\
\nabla_{\rvz_t} \mathrm{VML}^{\mathrm{LDM}}_{\mathrm{S}}(\rvz_t,t) &= - \frac{\partial \mathrm{D}^{\top}(\rvz_t,t)}{\partial \rvz_t} \frac{\partial \mathcal{D}^{\top}(\mathrm{D}(\rvz_t,t))}{\partial \mathrm{D}(\rvz_t,t)} \frac{\mathrm{H}^{\top}(\rvy-\mathrm{H} \mathcal{D}(\mathrm{D}(\rvz_t,t)))}{\sigma^2_\rvy}  - \frac{\partial \mathrm{D}^{\top}(\rvz_t,t)}{\partial \rvz_t} \frac{(\mathrm{D}(\rvz_t,t)-\rvz_t)}{\sigma^2_t}\\
\end{align*}}}
\end{proof}

We present LatentVML-MAP (Algorithm~\ref{algo:latentvmlmap}) as an extension of VML-MAP (Algorithm~\ref{algo:vmlmap}) to LDMs. In principle, LatentVML-MAP minimizes $\mathrm{VML}^{\mathrm{LDM}}_{\mathrm{S}}(\cdot,\cdot)$ at each reverse diffusion step to find $\rvz^{*}_0 = \arg \max_{\rvz_0} \log p(\rvz_0|\rvy)$ and finally, uses the decoder $\mathcal{D}$ to return $\rvx_0 = \mathcal{D}(\rvz^{*}_0)$.

The inputs to Algorithm~\ref{algo:latentvmlmap} consists of the latent diffusion denoiser $\mathrm{D}_{\theta}(\cdot,\cdot)$, the decoder $\mathcal{D}(\cdot)$, the linear degradation matrix $\mathrm{H}$, the measurement $\rvy$ with noise variance $\sigma^2_{\rvy}$, the diffusion noise schedule $\sigma(\cdot)$, the total number of reverse diffusion steps $N$ with the discretized time step schedule specified by $t_{i\in\{0,\dots N\}}$, where $t_0 = 0$, the gradient descent iterations per step given by $K$, and the learning rate $\gamma$. We use the notations $\sigma(t)$ and $\sigma_t$ interchangeably.

\begin{algorithm}[H]
\caption{\textbf{LatentVML-MAP}}
\label{algo:latentvmlmap}
\begin{algorithmic}[1]
\STATE {\bfseries Input:} $\mathrm{D}_{\theta}(\cdot,\cdot), \mathcal{D}(\cdot), \mathrm{H}, \rvy, \sigma_\rvy, \sigma(\cdot), t_{i\in \{0,\dots N \}}, K$, $\gamma$
\STATE {\bfseries Output:} $\mathbf{x}_{t_0}$ 
\STATE {\bfseries Initialize} $\rvz_{t_N} \sim \mathcal{N}(\mathbf{0},\sigma^2_{t_N}\mathbf{I})$
\FOR{$i = N$ \textbf{down to} $1$}
    \FOR{$j = 1$ \textbf{to} $K$}
        \STATE 
        $\rvz_{t_i} \gets \rvz_{t_i} - \gamma \, \nabla_{\rvz_{t_i}}\mathrm{VML}^{\mathrm{LDM}}_{\mathrm{S}}(\rvz_t,t)$\hspace{.5cm} \COMMENTNOBRACES{see Appendix~\ref{ssec:svml_ldm}}    
    \ENDFOR
    \STATE $\rvz_{t_{i-1}} \sim \mathcal{N}\left(\mathrm{D_{\theta}}(\rvz_{t_i},t_i), \sigma^2_{t_{i-1}}\mathbf{I}\right)$
\ENDFOR
\STATE $\rvx_{t_0} = \mathcal{D}(\rvz_{t_0})$
\STATE {\bfseries Return} $\mathbf{x}_{t_0}$
\end{algorithmic}
\end{algorithm}

\subsection{Experiments on CelebA}\label{ssec:ldmexperiments}

We conduct experiments on 100 test images from the CelebA~\citep{celeba} dataset on Half-mask inpainting, Box-mask inpainting, $4\times$ Super-resolution, and Deblurring tasks using the pre-trained latent diffusion model and the autoencoder from~\citet{ldm}. For LatentVML-MAP, we fix $N=20$, $K=50$, and use the EDM scheduler with $\sigma_{min}=0.002$, $\sigma_{max}=80$, and $\rho=7$. For Resample~\citep{resample}, we follow the default settings with 500 time steps, and for LatentDAPS~\citep{daps}, instead of the default setting of 100 neural function evaluations, we select 1000 NFEs to allow for a fair comparison, keeping other parameters fixed.  We set $\sigma_y = 1e\text{-}9$ for all the methods. For LatentVML-MAP, we report the best learning rate configuration for each task as $\gamma_0 \cdot \sigma^2_y$, with $\gamma_0$ as follows: Half-mask inpainting ($\gamma_0 = 0.1$), Box-mask inpainting ($\gamma_0 = 0.1$), $4\times$ Super-resolution ($\gamma_0 = 0.9$), Deblurring ($\gamma_0 = 0.075$).

\begin{table}[h!]
\caption{Evaluation of LDM-based image restoration methods on Half-mask inpainting, Box-mask inpainting, $4\times$ super-resolution, Deblurring on $100$ images of CelebA$256$. Excluding Resample, we denote the best values in \textbf{bold}, second best values \underline{underlined}.}
\label{tab:ldmexp1}
\begin{center}
\resizebox{0.90\textwidth}{!}{
\begin{tabular}{ l l l  l l l  c c c c c c c c }
\toprule
\multicolumn{3}{c}{} & \multicolumn{3}{c}{} & \multicolumn{8}{c}{} \\[-10pt]
\multicolumn{3}{c}{{Dataset}}  & \multicolumn{3}{l}{{Method}} & \multicolumn{2}{c}{{Half-Inpaint}} & \multicolumn{2}{c}{{Box-Inpaint}}  & \multicolumn{2}{c}{{$4\times $Sup-res}} & \multicolumn{2}{c}{{Deblurring}} \\[1pt]
 \multicolumn{3}{c}{} & \multicolumn{3}{c}{} & \multicolumn{1}{c}{\scriptsize{LPIPS$\downarrow$}} & \multicolumn{1}{c}{\scriptsize{FID$\downarrow$}} & \multicolumn{1}{c}{\scriptsize{LPIPS$\downarrow$}} & \multicolumn{1}{c}{\scriptsize{FID$\downarrow$}} & \multicolumn{1}{c}{\scriptsize{LPIPS$\downarrow$}} & \multicolumn{1}{c}{\scriptsize{FID$\downarrow$}} & \multicolumn{1}{c}{\scriptsize{LPIPS$\downarrow$}} & \multicolumn{1}{c}{\scriptsize{FID$\downarrow$}}\\[1pt]
 \toprule
 \multicolumn{3}{c}{} & \multicolumn{3}{c}{} & \multicolumn{8}{c}{}\\[-10pt]
 \multicolumn{3}{c}{} & \multicolumn{3}{l}{{\small{$\text{\color{lightgray}{Resample}}$}}}  & \multicolumn{1}{c}{\scriptsize{\color{lightgray}{$0.235$}}} & \multicolumn{1}{c}{\scriptsize{\color{lightgray}{$55.80$}}}  & \multicolumn{1}{c}{\scriptsize{\color{lightgray}{$0.092$}}} & \multicolumn{1}{c}{\scriptsize{\color{lightgray}{$53.12$}}} & \multicolumn{1}{c}{\scriptsize{\color{lightgray}{$0.084$}}} & \multicolumn{1}{c}{\scriptsize{\color{lightgray}{$41.85$}}} & \multicolumn{1}{c}{\scriptsize{\color{lightgray}{$0.202$}}} & \multicolumn{1}{c}{\scriptsize{\color{lightgray}{$63.41$}}}\\
 \multicolumn{3}{c}{} & \multicolumn{3}{l}{{\small{$\text{LatentDAPS-1K}$}}}  & 
 \multicolumn{1}{c}{\scriptsize{$\underline{0.240}$}} & \multicolumn{1}{c}{\scriptsize{$\underline{54.47}$}}  & \multicolumn{1}{c}{\scriptsize{$\bf{0.070}$}} & \multicolumn{1}{c}{\scriptsize{$\underline{37.69}$}} & \multicolumn{1}{c}{\scriptsize{$0.114$}} & \multicolumn{1}{c}{\scriptsize{$\underline{38.72}$}} & \multicolumn{1}{c}{\scriptsize{$0.328$}} & \multicolumn{1}{c}{\scriptsize{$104.2$}}\\
 \multicolumn{3}{c}{\small{CelebA}} & \multicolumn{3}{l}{{\small{$\text{Resample w/o PO}$}}}  & \multicolumn{1}{c}{\scriptsize{$0.259$}} & \multicolumn{1}{c}{\scriptsize{$64.32$}}  & \multicolumn{1}{c}{\scriptsize{$0.106$}} & \multicolumn{1}{c}{\scriptsize{$68.30$}} & \multicolumn{1}{c}{\scriptsize{$\underline{0.112}$}} & \multicolumn{1}{c}{\scriptsize{$50.95$}} & \multicolumn{1}{c}{\scriptsize{$\bf{0.210}$}} & \multicolumn{1}{c}{\scriptsize{$\bf{65.36}$}}\\[2pt] 
 \cline{4-14}
 \multicolumn{3}{c}{} & \multicolumn{3}{c}{} & \multicolumn{8}{c}{}\\[-7pt]
 \multicolumn{3}{l}{} & \multicolumn{3}{l}{\bf{\small{$\text{LatentVML-MAP}$}}}  & \multicolumn{1}{c}{\scriptsize{$\bf{0.208}$}} & \multicolumn{1}{c}{\scriptsize{$\bf{47.04}$}} & \multicolumn{1}{c}{\scriptsize{$\underline{0.074}$}} & \multicolumn{1}{c}{\scriptsize{$\bf{34.73}$}} & \multicolumn{1}{c}{\scriptsize{$\bf{0.099}$}} & \multicolumn{1}{c}{\scriptsize{$\bf{36.82}$}} & \multicolumn{1}{c}{\scriptsize{\underline{$0.212$}}} & \multicolumn{1}{c}{\scriptsize{\underline{$71.43$}}}\\
 \bottomrule
\end{tabular}}
\end{center}
\end{table}

For LatentVML-MAP, we observed that even for linear inverse problems, the optimization becomes more challenging due to the non-linearity of the LDM decoder. As a result, we noticed that the final reconstructed images are blurry and inconsistent with the measurement $\rvy$. We also observed this pattern with LatentDAPS, but not with Resample, as it uses pixel-space optimization. Note that Resample also requires the LDM encoder to project the pixel-space-optimized result back into the latent space, unlike LatentDAPS and LatentVML-MAP. For a fair comparison, in our experiments,
\begin{itemize}[left=1em, label=\textbullet]
    \item we also report the performance of Resample with its pixel-space optimization replaced with latent-space optimization (see~\citet{resample}), denoted as \textbf{Resample w/o PO}
    \item we project the final reconstructed images of all methods onto the measurement subspace to ensure that all the reconstructed images are consistent with the measurements.
\end{itemize}

For inpainting, the measurement subspace projection implies that we paste the observed pixels back into the reconstructed images. For super-resolution and deblurring, we paste back the observations in the spectral space of the linear operator by using the SVD accessible operators from~\citet{ddrm}. Results from Table~\ref{tab:ldmexp1} validate the effectiveness of LatentVML-MAP in practice. We believe that the performance bottleneck primarily exists due to the challenging optimization and that improved optimization techniques can further enhance the performance of LatentVML-MAP.
 
\begin{figure*}[htbp]
    \centering
    \begin{subfigure}[h]{0.03\textwidth}
    \centering
    \subfloat{\scriptsize{ }}
    \end{subfigure}
    \begin{subfigure}[h]{0.155\textwidth}
    \centering
    \subfloat{\scriptsize{\bf{\color{gray}{\textit{Original}}}}}    
    \end{subfigure}
    \begin{subfigure}[h]{0.155\textwidth}
    \centering
    \subfloat{\scriptsize{\bf{\color{gray}{\textit{Measurement}}}}}    
    \end{subfigure}
    \begin{subfigure}[h]{0.155\textwidth}
    \centering
    \subfloat{\scriptsize{\bf{\color{gray}{\textit{LatentDAPS-1K}}}}}    
    \end{subfigure}
    \begin{subfigure}[h]{0.155\textwidth}
    \centering
    \subfloat{\scriptsize{\bf{\color{gray}{\textit{Resample w/o PO}}}}}    
    \end{subfigure}
    \begin{subfigure}[h]{0.155\textwidth}
    \centering
    \subfloat{\scriptsize{\bf{\color{gray}{\textit{LatentVML-MAP}}}}}    
    \end{subfigure}
    \\[3pt]
    \begin{subfigure}[h]{0.03\textwidth}
    \centering
    \subfloat{{\rotatebox[origin=c]{90}{\color{gray}{\bf{\textit{Half Inpaint}}}}}}    
    \end{subfigure}
    \begin{subfigure}[h]{0.155\textwidth}
    \centering
    \includegraphics[width=\textwidth]{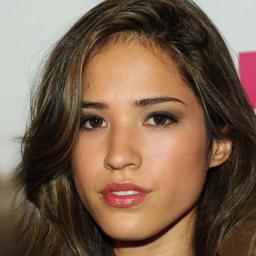}    
    \end{subfigure}
    \begin{subfigure}[h]{0.155\textwidth}
    \centering
    \includegraphics[width=\textwidth]{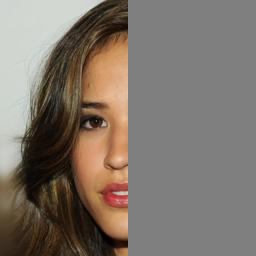}    
    \end{subfigure}
    \begin{subfigure}[h]{0.155\textwidth}
    \centering
    \includegraphics[width=\textwidth]{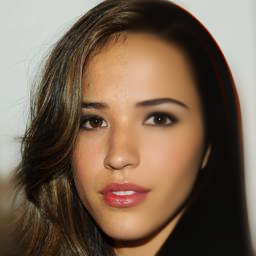}  
    \end{subfigure}
    \begin{subfigure}[h]{0.155\textwidth}
    \centering
    \includegraphics[width=\textwidth]{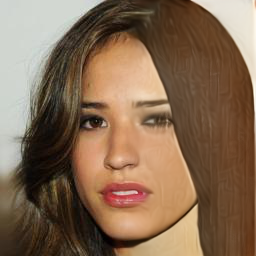}    
    \end{subfigure}
    \begin{subfigure}[h]{0.155\textwidth}
    \centering
    \includegraphics[width=\textwidth]{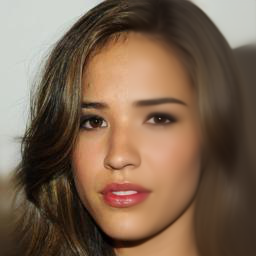}  
    \end{subfigure}
    \\ 
    \begin{subfigure}[h]{0.03\textwidth}
    \centering
    \subfloat{{\rotatebox[origin=c]{90}{\color{gray}{\bf{\textit{Half Inpaint}}}}}}    
    \end{subfigure}
    \begin{subfigure}[h]{0.155\textwidth}
    \centering
    \includegraphics[width=\textwidth]{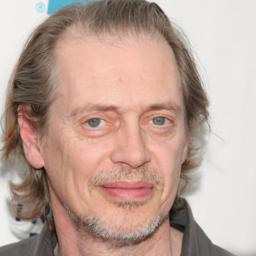}    
    \end{subfigure}
    \begin{subfigure}[h]{0.155\textwidth}
    \centering
    \includegraphics[width=\textwidth]{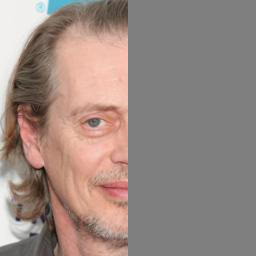}    
    \end{subfigure}
    \begin{subfigure}[h]{0.155\textwidth}
    \centering
    \includegraphics[width=\textwidth]{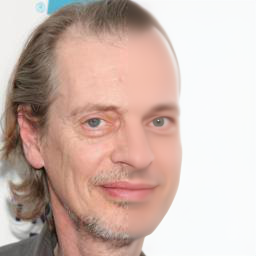}  
    \end{subfigure}
    \begin{subfigure}[h]{0.155\textwidth}
    \centering
    \includegraphics[width=\textwidth]{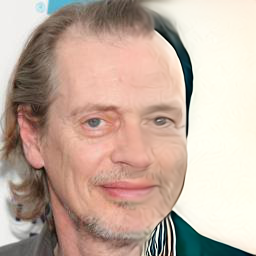}    
    \end{subfigure}
    \begin{subfigure}[h]{0.155\textwidth}
    \centering
    \includegraphics[width=\textwidth]{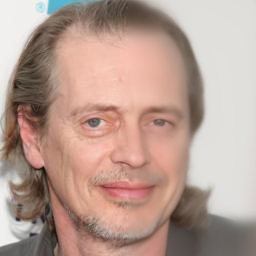}  
    \end{subfigure}
    \\ 
    \begin{subfigure}[h]{0.03\textwidth}
    \centering
    \subfloat{{\rotatebox[origin=c]{90}{\color{gray}{\bf{\textit{Box Inpaint}}}}}}    
    \end{subfigure}
    \begin{subfigure}[h]{0.155\textwidth}
    \centering
    \includegraphics[width=\textwidth]{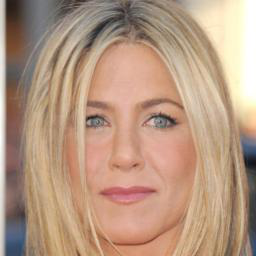}    
    \end{subfigure}
    \begin{subfigure}[h]{0.155\textwidth}
    \centering
    \includegraphics[width=\textwidth]{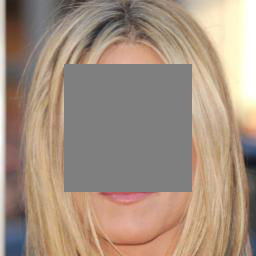}    
    \end{subfigure}
    \begin{subfigure}[h]{0.155\textwidth}
    \centering
    \includegraphics[width=\textwidth]{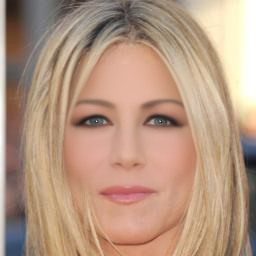}  
    \end{subfigure}
    \begin{subfigure}[h]{0.155\textwidth}
    \centering
    \includegraphics[width=\textwidth]{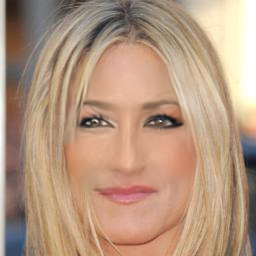}    
    \end{subfigure}
    \begin{subfigure}[h]{0.155\textwidth}
    \centering
    \includegraphics[width=\textwidth]{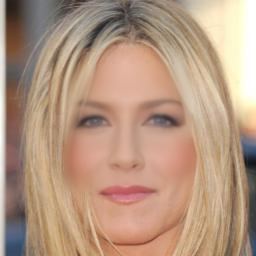}  
    \end{subfigure}
    \\ 
    \begin{subfigure}[h]{0.03\textwidth}
    \centering
    \subfloat{{\rotatebox[origin=c]{90}{\color{gray}{\bf{\textit{Box Inpaint}}}}}}    
    \end{subfigure}
    \begin{subfigure}[h]{0.155\textwidth}
    \centering
    \includegraphics[width=\textwidth]{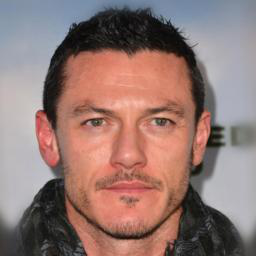}    
    \end{subfigure}
    \begin{subfigure}[h]{0.155\textwidth}
    \centering
    \includegraphics[width=\textwidth]{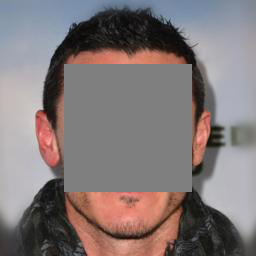}    
    \end{subfigure}
    \begin{subfigure}[h]{0.155\textwidth}
    \centering
    \includegraphics[width=\textwidth]{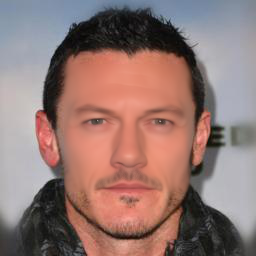}  
    \end{subfigure}
    \begin{subfigure}[h]{0.155\textwidth}
    \centering
    \includegraphics[width=\textwidth]{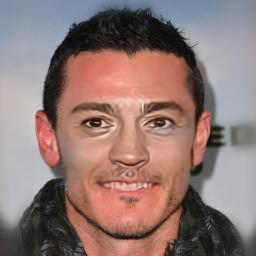}    
    \end{subfigure}
    \begin{subfigure}[h]{0.155\textwidth}
    \centering
    \includegraphics[width=\textwidth]{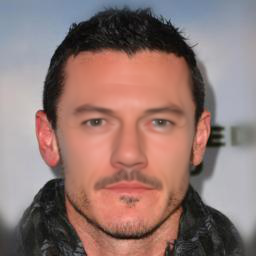}  
    \end{subfigure}
    \\ 
    \begin{subfigure}[h]{0.03\textwidth}
    \centering
    \subfloat{{\rotatebox[origin=c]{90}{\color{gray}{\bf{$\mathbf{4\times}$ \textit{Super-res}}}}}}
    \end{subfigure}
    \begin{subfigure}[h]{0.155\textwidth}
    \centering
    \includegraphics[width=\textwidth]{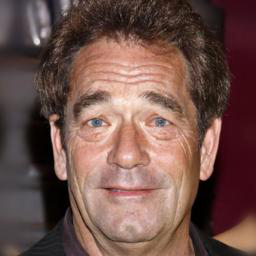}    
    \end{subfigure}
    \begin{subfigure}[h]{0.155\textwidth}
    \centering
    \includegraphics[width=\textwidth]{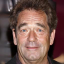}    
    \end{subfigure}
    \begin{subfigure}[h]{0.155\textwidth}
    \centering
    \includegraphics[width=\textwidth]{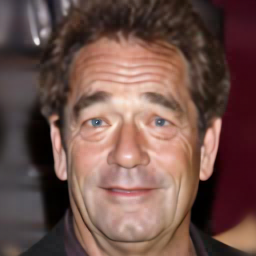}  
    \end{subfigure}
    \begin{subfigure}[h]{0.155\textwidth}
    \centering
    \includegraphics[width=\textwidth]{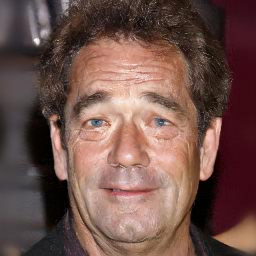}    
    \end{subfigure}
    \begin{subfigure}[h]{0.155\textwidth}
    \centering
    \includegraphics[width=\textwidth]{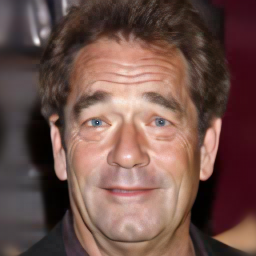}  
    \end{subfigure}
    \\ 
    \begin{subfigure}[h]{0.03\textwidth}
    \centering
    \subfloat{{\rotatebox[origin=c]{90}{\color{gray}{\bf{$\mathbf{4\times}$ \textit{Super-res}}}}}}
    \end{subfigure}
    \begin{subfigure}[h]{0.155\textwidth}
    \centering
    \includegraphics[width=\textwidth]{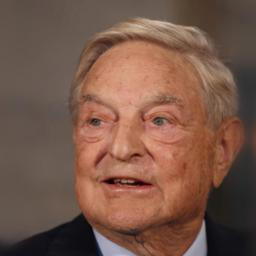}    
    \end{subfigure}
    \begin{subfigure}[h]{0.155\textwidth}
    \centering
    \includegraphics[width=\textwidth]{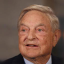}    
    \end{subfigure}
    \begin{subfigure}[h]{0.155\textwidth}
    \centering
    \includegraphics[width=\textwidth]{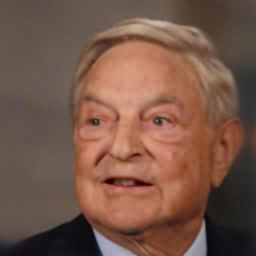}  
    \end{subfigure}
    \begin{subfigure}[h]{0.155\textwidth}
    \centering
    \includegraphics[width=\textwidth]{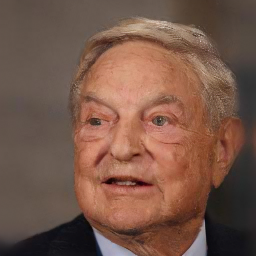}    
    \end{subfigure}
    \begin{subfigure}[h]{0.155\textwidth}
    \centering
    \includegraphics[width=\textwidth]{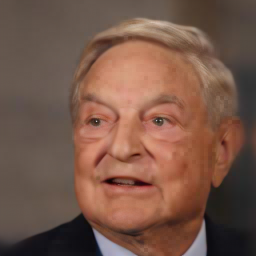}  
    \end{subfigure}
    \\ 
    \begin{subfigure}[h]{0.03\textwidth}
    \centering
    \subfloat{{\rotatebox[origin=c]{90}{\color{gray}{\bf{\textit{Deblurring}}}}}}
    \end{subfigure}
    \begin{subfigure}[h]{0.155\textwidth}
    \centering
    \includegraphics[width=\textwidth]{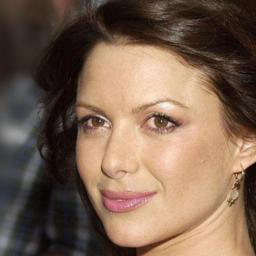}    
    \end{subfigure}
    \begin{subfigure}[h]{0.155\textwidth}
    \centering
    \includegraphics[width=\textwidth]{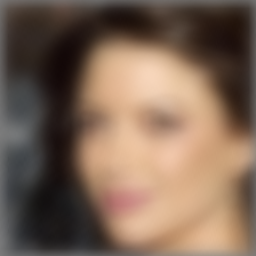}    
    \end{subfigure}
    \begin{subfigure}[h]{0.155\textwidth}
    \centering
    \includegraphics[width=\textwidth]{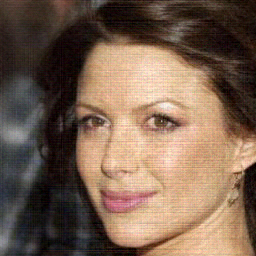}  
    \end{subfigure}
    \begin{subfigure}[h]{0.155\textwidth}
    \centering
    \includegraphics[width=\textwidth]{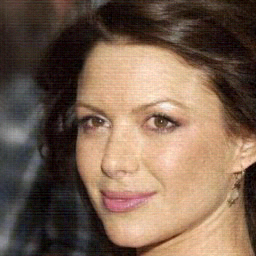}    
    \end{subfigure}
    \begin{subfigure}[h]{0.155\textwidth}
    \centering
    \includegraphics[width=\textwidth]{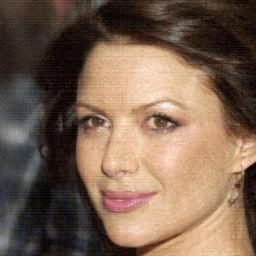}  
    \end{subfigure}
    \\ 
    \begin{subfigure}[h]{0.03\textwidth}
    \centering
    \subfloat{{\rotatebox[origin=c]{90}{\color{gray}{\bf{\textit{Deblurring}}}}}}
    \end{subfigure}
    \begin{subfigure}[h]{0.155\textwidth}
    \centering
    \includegraphics[width=\textwidth]{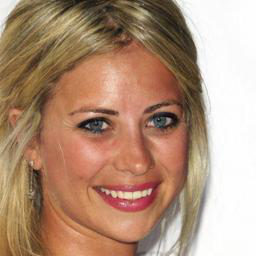}    
    \end{subfigure}
    \begin{subfigure}[h]{0.155\textwidth}
    \centering
    \includegraphics[width=\textwidth]{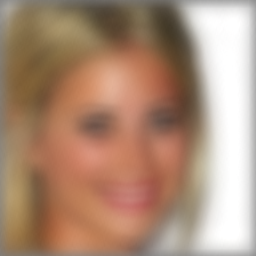}    
    \end{subfigure}
    \begin{subfigure}[h]{0.155\textwidth}
    \centering
    \includegraphics[width=\textwidth]{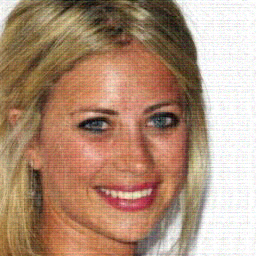}  
    \end{subfigure}
    \begin{subfigure}[h]{0.155\textwidth}
    \centering
    \includegraphics[width=\textwidth]{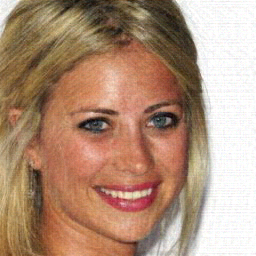}    
    \end{subfigure}
    \begin{subfigure}[h]{0.155\textwidth}
    \centering
    \includegraphics[width=\textwidth]{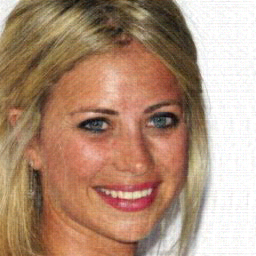}  
    \end{subfigure}
    \caption{Original image, measurement, and restored images with LatentDAPS-1K, Resample w/o PO, and LatentVML-MAP.}\label{fig:ldmqv}
\end{figure*}

\end{document}